\documentclass[11pt]{article} %

\usepackage{etoolbox}
\newtoggle{neurips}
\toggletrue{neurips}
\togglefalse{neurips}
\newcommand{\neurips}[1]{\iftoggle{neurips}{#1}{}}
\newcommand{\arxiv}[1]{\iftoggle{neurips}{}{#1}}

\neurips{\usepackage[nonatbib]{neurips_2020}}

\usepackage[colorlinks=true,linkcolor=blue!70!black,citecolor=blue!70!black,urlcolor=black,breaklinks=true]{hyperref}
\usepackage{url}            %
\usepackage{booktabs,makecell}       %

\usepackage{amsfonts,bm,upgreek}       %
\usepackage{nicefrac}       %
\usepackage{microtype}      %
\usepackage{etoolbox}
\usepackage{algorithm}
\usepackage{verbatim}
\usepackage[noend]{algpseudocode}

\usepackage{setspace}

\neurips{%

\usepackage[dvipsnames]{xcolor}
\usepackage{microtype}

\usepackage{algorithm}
\usepackage[square,numbers]{natbib}
\bibliographystyle{plainnat}

\usepackage{amsthm,thmtools}
\usepackage{mathtools}
\usepackage{amsmath}
\usepackage{amsfonts}
\usepackage{amssymb}
\let\vec\undefined
\usepackage{MnSymbol} %
\usepackage{xpatch}

\usepackage{amsfonts}
\usepackage{graphicx}
\usepackage{pdfpages}
\usepackage{amsthm}
\usepackage{thm-restate}
\usepackage{euscript}
\usepackage{url}
\usepackage{comment}
\usepackage{mathrsfs}

\usepackage[capitalize,nameinlink]{cleveref}
\newcommand{\savehyperref}[2]{\texorpdfstring{\hyperref[#1]{#2}}{#2}}
\newcommand{\pref}[1]{\Cref{#1}}
\newcommand{\pfref}[1]{Proof of \pref{#1}}

\crefname{equation}{Eq.}{Eqs.}
\Crefname{equation}{Eq.}{Eqs.}

\renewcommand{\eqref}[1]{\savehyperref{#1}{(\ref*{#1})}}

\theoremstyle{definition} %

\newtheorem{bodycorollary}{Corollary}
\newtheorem{bodylemma}{Lemma}
\newtheorem{bodyproposition}{Proposition}

\newtheorem{bodytheorem}{Theorem}

\newtheorem{exercise}{Exercise}
\newtheorem{claim}{Claim}[section]
\newtheorem{lemma}{Lemma}[section]
\newtheorem{conjecture}{Conjecture}[section]
\newtheorem{corollary}{Corollary}[section]
\newtheorem{proposition}{Proposition}[section]
\newtheorem{fact}{Fact}[section]
\newtheorem{assumption}{Assumption}
\newtheorem{problem}{Problem}
\newtheorem{question}{Question}
\newtheorem{model}{Model}
\theoremstyle{plain}
\newtheorem{remark}{Remark}
\newtheorem{example}{Example}
\newtheorem{theorem}[section]{Theorem}
\newtheorem{definition}{Definition}

\xpatchcmd{\proof}{\itshape}{\normalfont\proofnameformat}{}{}
\newcommand{\proofnameformat}{\bfseries}

}
\arxiv{%
\usepackage[letterpaper, left=1in, right=1in, top=1in, bottom=1in]{geometry}

\usepackage[dvipsnames]{xcolor}
\usepackage{microtype}

\usepackage{algorithm}

\usepackage{natbib}
\bibliographystyle{plainnat}
\bibpunct{(}{)}{;}{a}{,}{,}

\usepackage{amsthm}
\usepackage{mathtools}
\usepackage{amsmath}
\usepackage{bbm}
\usepackage{amsfonts}
\usepackage{amssymb}
\let\vec\undefined
\usepackage{graphicx}
\usepackage{pdfpages}
\usepackage{amsthm}
\usepackage{thm-restate}
\usepackage{euscript}
\usepackage{url}
\usepackage{comment}
\usepackage{mathrsfs}

\usepackage{xpatch}

\usepackage[capitalize,nameinlink]{cleveref}
\newcommand{\savehyperref}[2]{\texorpdfstring{\hyperref[#1]{#2}}{#2}}
\newcommand{\pref}[1]{\Cref{#1}}
\newcommand{\pfref}[1]{Proof of \pref{#1}}

\crefname{equation}{Eq.}{Eqs.}
\Crefname{equation}{Eq.}{Eqs.}

\renewcommand{\eqref}[1]{\savehyperref{#1}{(\ref*{#1})}}

\theoremstyle{definition} %

\newtheorem{claim}{Claim}[section]
\newtheorem{lemma}{Lemma}[section]

\newtheorem{corollary}{Corollary}[section]
\newtheorem{proposition}{Proposition}[section]
\newtheorem{fact}{Fact}[section]
\newtheorem{assumption}{Assumption}

\newtheorem{remark}{Remark}
\theoremstyle{plain}

\newtheorem{theorem}{Theorem}[section]
\newtheorem{definition}{Definition}

\xpatchcmd{\proof}{\itshape}{\normalfont\proofnameformat}{}{}
\newcommand{\proofnameformat}{\bfseries}

}
\makeatletter
\renewenvironment{proof}[1][\proofname] {\par\pushQED{\qed}\normalfont\topsep6\p@\@plus6\p@\relax\trivlist\item[\hskip\labelsep\bfseries#1\@addpunct{.}]\ignorespaces}{\popQED\endtrivlist\@endpefalse}
\makeatother

\DeclarePairedDelimiter{\abs}{\lvert}{\rvert} %
\DeclarePairedDelimiter{\brk}{[}{]}
\DeclarePairedDelimiter{\crl}{\{}{\}}
\DeclarePairedDelimiter{\prn}{(}{)}
\DeclarePairedDelimiter{\nrm}{\|}{\|}
\DeclarePairedDelimiter{\tri}{\langle}{\rangle}

\DeclarePairedDelimiter{\ceil}{\lceil}{\rceil}

\let\Pr\undefined
\let\P\undefined
\DeclareMathOperator{\En}{\mathbb{E}}

\DeclareMathOperator{\P}{P}
\DeclareMathOperator{\Pr}{Pr}

\DeclareMathOperator*{\argmin}{arg\,min} %

\newcommand{\wt}[1]{\widetilde{#1}}
\newcommand{\wh}[1]{\widehat{#1}}
\newcommand{\wb}[1]{\widebar{#1}}
\newcommand{\mb}[1]{\boldsymbol{#1}}

\def\ddefloop#1{\ifx\ddefloop#1\else\ddef{#1}\expandafter\ddefloop\fi}
\def\ddef#1{\expandafter\def\csname bb#1\endcsname{\ensuremath{\mathbb{#1}}}}
\ddefloop ABCDEFGHIJKLMNOPQRSTUVWXYZ\ddefloop

\def\ddefloop#1{\ifx\ddefloop#1\else\ddef{#1}\expandafter\ddefloop\fi}
\def\ddef#1{\expandafter\def\csname fr#1\endcsname{\ensuremath{\mathfrak{#1}}}}
\ddefloop ABCDEFGHIJKLMNOPQRSTUVWXYZ\ddefloop

\def\ddefloop#1{\ifx\ddefloop#1\else\ddef{#1}\expandafter\ddefloop\fi}
\def\ddef#1{\expandafter\def\csname eul#1\endcsname{\ensuremath{\EuScript{#1}}}}
\ddefloop ABCDEFGHIJKLMNOPQRSTUVWXYZ\ddefloop

\def\ddefloop#1{\ifx\ddefloop#1\else\ddef{#1}\expandafter\ddefloop\fi}
\def\ddef#1{\expandafter\def\csname scr#1\endcsname{\ensuremath{\mathscr{#1}}}}
\ddefloop ABCDEFGHIJKLMNOPQRSTUVWXYZ\ddefloop

\def\ddefloop#1{\ifx\ddefloop#1\else\ddef{#1}\expandafter\ddefloop\fi}
\def\ddef#1{\expandafter\def\csname b#1\endcsname{\ensuremath{\mathbf{#1}}}}
\ddefloop ABCDEFGHIJKLMNOPQRSTUVWXYZ\ddefloop
\def\ddef#1{\expandafter\def\csname c#1\endcsname{\ensuremath{\mathcal{#1}}}}
\ddefloop ABCDEFGHIJKLMNOPQRSTUVWXYZ\ddefloop
\def\ddef#1{\expandafter\def\csname h#1\endcsname{\ensuremath{\widehat{#1}}}}
\ddefloop ABCDEFGHIJKLMNOPQRSTUVWXYZ\ddefloop
\def\ddef#1{\expandafter\def\csname hc#1\endcsname{\ensuremath{\widehat{\mathcal{#1}}}}}
\ddefloop ABCDEFGHIJKLMNOPQRSTUVWXYZ\ddefloop
\def\ddef#1{\expandafter\def\csname t#1\endcsname{\ensuremath{\widetilde{#1}}}}
\ddefloop ABCDEFGHIJKLMNOPQRSTUVWXYZ\ddefloop
\def\ddef#1{\expandafter\def\csname tc#1\endcsname{\ensuremath{\widetilde{\mathcal{#1}}}}}
\ddefloop ABCDEFGHIJKLMNOPQRSTUVWXYZ\ddefloop

\newcommand{\Holder}{H{\"o}lder}

\newcommand{\ls}{\ell}
\newcommand{\ind}{\mathbbm{1}}    %

\newcommand{\eps}{\epsilon}
\newcommand{\veps}{\varepsilon}

\newcommand{\ldef}{\vcentcolon=}
\newcommand{\rdef}{=\vcentcolon}

\newcommand{\trn}{\intercal} %

\newcommand{\Imark}{\textnormal{(\textbf{I})}\xspace}
\newcommand{\IImark}{\textnormal{(\textbf{II})}\xspace}
\newcommand{\IIImark}{\textnormal{(\textbf{III})}\xspace}

\Crefname{equation}{Eq.}{Eqs.}
\Crefname{assumption}{Assumption}{Assumptions}
\Crefname{condition}{Condition}{Conditions}

\DeclareMathSymbol{\shortminus}{\mathbin}{AMSa}{"39}

\newcommand{\algcomment}[1]{\textcolor{blue!70!black}{\footnotesize{\texttt{\textbf{//\hspace{2pt}#1}}}}}

      \newcommand{\algparen}[1]{\algcomment{#1.}}

\makeatletter
\newcommand{\neutralize}[1]{\expandafter\let\csname c@#1\endcsname\count@}
\makeatother

\newenvironment{thmmod}[2]
  {\renewcommand{\thetheorem}{\ref*{#1}#2}%
   \neutralize{theorem}\phantomsection
   \begin{theorem}}
  {\end{theorem}}

\Crefname{claim}{Claim}{Claims}

\newcommand{\Fclass}{\scrF}

\newcommand{\Hclass}{\scrH}

\newcommand{\Mclass}{\scrM}

\newcommand{\wtilde}[1]{\widetilde{#1}}

\newcommand{\reals}{\mathbb{R}}
\newcommand{\what}[1]{\widehat{#1}}
\newcommand{\E}{\mathbb{E}}
\newcommand{\rmd}{\mathrm{d}}

\newcommand{\commentout}[1]{}
\DeclareMathOperator{\supp}{\mathrm{supp}}
\renewcommand{\P}{\mathbb{P}}

\newcommand{\tv}{D_{\mathrm{TV}}}
\newcommand{\kl}{D_{\mathrm{KL}}}
\newcommand{\chisquared}{\chi^{2}}

\newcount\Comments  %
\definecolor{darkgreen}{rgb}{0,0.5,0}
\definecolor{darkred}{rgb}{0.7,0,0}
\definecolor{teal}{rgb}{0.3,0.8,0.8}
\definecolor{orange}{rgb}{1.0,0.5,0.0}
\definecolor{purple}{rgb}{0.8,0.0,0.8}

\newcommand{\Ahat}{\widehat{A}}
\newcommand{\Bhat}{\widehat{B}}

\newcommand{\hstar}{h_{\star}}

\newcommand{\richlqr}{\textsf{RichLQR}\xspace}
\newcommand{\richid}{\textsf{RichID}\xspace}
\newcommand{\richidce}{\textsf{RichID-CE}\xspace}

\newcommand{\dare}{\textsf{DARE}\xspace}

\renewcommand{\trn}{\top}
\newcommand{\psdleq}{\preceq}
\newcommand{\psdgeq}{\succeq}

\newcommand{\psdgt}{\succ}
\newcommand{\approxleq}{\lesssim}

\newcommand{\eigmin}{\lambda_{\mathrm{min}}}

\newcommand{\sigmamin}{\sigma_{\mathrm{min}}}

\newcommand{\Ind}{\mathbb{I}}

\newcommand{\bu}{\mb{u}}
\newcommand{\bv}{\mb{v}}
\newcommand{\bz}{\mb{z}}
\newcommand{\bw}{\mb{w}}

\newcommand{\fhat}{\hat{f}}

\renewcommand{\ind}[1]{^{(#1)}}

\newcommand{\BigOh}{\cO}

\newcommand{\fstar}{f_{\star}}
\newcommand{\gstar}{g_{\star}}

\newcommand{\pihat}{\wh{\pi}}
\newcommand{\pistar}{\pi^{\star}}
\newcommand{\piinf}{\pi_{\infty}}

\newcommand{\alg}{\textsf{A}}

\newcommand{\indic}{\mathbb{I}}

\renewcommand{\Pr}{\bbP}

\newcommand{\poly}{\mathrm{poly}}

\newcommand{\mainalg}{\textup{\textsf{SquareCB}}\xspace}

\newcommand{\bigoh}{\cO}
\newcommand{\bigohs}{\cO_{\star}}
\newcommand{\bigohch}{\check{\cO}}

\newcommand{\bigoht}{\wt{\cO}}
\newcommand{\bigom}{\Omega}

\newcommand{\bigoms}{\Omega_{\star}}

\newcommand{\iid}{\textrm{i.i.d.}\xspace}

\newcommand{\matx}{\mathbf{x}}
\newcommand{\matu}{\mathbf{u}}
\newcommand{\maty}{\mathbf{y}}
\newcommand{\matz}{\mathbf{z}}
\newcommand{\matv}{\mathbf{v}}

\newcommand{\matw}{\mathbf{w}}
\newcommand{\matc}{\mathbf{c}}
\newcommand{\mateps}{\mb{\veps}}

\newcommand{\dimx}{d_{\matx}}

\newcommand{\dimu}{d_{\matu}}
\newcommand{\dimy}{d_{\maty}}
\newcommand{\dimk}{\kappa}

\newcommand{\conv}{\ast}

\newcommand{\nn}{\nonumber}
\renewcommand{\bu}{\mathbf{u}}
\renewcommand{\bv}{\mathbf{v}}
\renewcommand{\bw}{\mathbf{w}}
\newcommand{\bx}{\mathbf{x}}
\newcommand{\by}{\mathbf{y}}

\newcommand{\bxi}{\bm{\xi}}
\newcommand{\Var}{\textsc{Var}}
\newcommand{\op}{\mathrm{op}}

\newcommand{\bnu}{\bm{\upnu}}

\newcommand{\Psistar}{\Psi_{\star}}

\newcommand{\Pinf}{P_{\infty}}
\newcommand{\Siginf}{\Sigma_{\infty}}
\newcommand{\Rx}{Q}
\newcommand{\Ru}{R}
\newcommand{\Kinf}{K_{\infty}}
\newcommand{\be}{\mathbf{e}}
\newcommand{\bclip}{\bar b}

\newcommand{\gammaa}{\gamma_{A}}
\newcommand{\gammaab}{\bar{\gamma}_{A}}
\newcommand{\alphaa}{\alpha_{A}}
\newcommand{\gammainf}{\gamma_{\infty}}
\newcommand{\alphainf}{\alpha_{\infty}}
\newcommand{\Sigmaa}{\Sigma_{A}}

\newcommand{\Aclinf}{A_{\mathrm{cl},\infty}}

\newcommand{\Vf}{\mathbf{V}}
\newcommand{\Qf}{\mathbf{Q}}
\newcommand{\Vbar}{\wb{\mathbf{V}}}
\newcommand{\Qbar}{\wb{\mathbf{Q}}}

\newcommand{\Qhat}{\wh{Q}}

\newcommand{\Phat}{\wh{P}}
\newcommand{\Khat}{\wh{K}}

\newcommand{\bdelta}{\bm{\updelta}}

\newcommand{\calM}{\mathcal{M}}
\newcommand{\calN}{\mathcal{N}}
\newcommand{\sys}{\mathrm{sys}}

		\newcommand{\epsidh}{\veps_{\mathrm{id},h}}
		\newcommand{\nid}{n_{\mathrm{id}}}
		\newcommand{\cbaridone}{\bar{c}_{\mathrm{id},1}}
		\newcommand{\hhatid}{\hhat_{\mathrm{id}}}
		
		\newcommand{\eventidh}{\cE_{\mathrm{id},h}}
		\newcommand{\Sigkid}[1][k]{\Sigma_{#1,\mathrm{id}}}
		\newcommand{\Sigstid}{\Sigma_{\infty,\mathrm{id}}}

		\newcommand{\byk}{\by_k}

		\newcommand{\cconcid}{c_{\mathrm{conc},\mathrm{id}}}

		\newcommand{\epsidpca}{\veps_{\mathrm{id},\mathrm{pca}}}
		\newcommand{\eventidpca}{\cE_{\mathrm{id},\mathrm{pca}}}
		\newcommand{\Sid}{S_{\mathrm{id}}}

		\newcommand{\bxk}{\bx_k}

		\newcommand{\bytil}{\tilde{\by}}
			\newcommand{\betil}{\tilde{\be}}
			\newcommand{\matdeltil}{\tilde{\matdel}}
			\newcommand{\butil}{\tilde{\bu}}
			\newcommand{\Sigutil}{\Sigma_{\tilde{u}}}
			
		\newcommand{\Vhatkapid}{\Vhat_{\kappa,\mathrm{id}}}
		
		\newcommand{\Vid}{V_{\mathrm{id}}}

		\newcommand{\sigidmax}{\sigma_{\max,\mathrm{id}}}
		\newcommand{\sigidmin}{\sigma_{\min,\mathrm{id}}}
		\newcommand{\fhatid}{\fhat_{\mathrm{id}}}
		\newcommand{\fstid}{f_{\star,\mathrm{id}}}
		\newcommand{\bwtil}{\widetilde{\bw}}
		\newcommand{\Sigwhat}{\widehat{\Sigma}_{w}}
		\newcommand{\eventidls}{\cE_{\mathrm{id},\mathrm{ls}}}
		\newcommand{\Aid}{A_{\mathrm{id}}}
		\newcommand{\Bid}{B_{\mathrm{id}}}
		\newcommand{\Sigwid}{\Sigma_{w,\mathrm{id}}}
		\newcommand{\errfid}{\mathrm{err}_{f,\mathrm{id}}}

		\newcommand{\R}{\reals}
		
		\newcommand{\trace}{\mathrm{tr}}
		\newcommand{\Qst}{Q^{\star}}
		\newcommand{\Mtil}{\tilde{M}}
		\newcommand{\fro}{\mathrm{F}}
		\newcommand{\Term}{\mathrm{Term}}

		\newcommand{\supi}{^{(i)}}

                \newcommand{\Sigxnot}{\Sigma_{0}}
		
		\newcommand{\I}{\mathbb{I}}

		\newcommand{\htil}{\tilde{h}}

		\newcommand{\matetil}{\tilde{\mate}}
		\newcommand{\Sigw}{\Sigma_{w}}
		\newcommand{\cont}{\cC}
		\newcommand{\mate}{\mathbf{e}}
		\newcommand{\Mst}{M_{\star}}
		\newcommand{\Mhat}{\widehat{M}}
		\newcommand{\matdel}{\bm{\delta}}
		\newcommand{\matDel}{\bm{\Delta}}
		\newcommand{\Sigu}{\Sigma_{u}}
		\newcommand{\Lamhat}{\what{\Lambda}}
		\newcommand{\Vhat}{\widehat{V}}
		\newcommand{\gst}{g_{\star}}
		\newcommand{\ghat}{\hat{g}}
\newcommand{\init}{\mathrm{init}}
\newcommand{\Qid}{Q_{\mathrm{id}}}
\newcommand{\noise}{\mathrm{noise}}

\newcommand{\idinf}{{\infty,\mathrm{id}}}
\newcommand{\dec}{\mathrm{dec}}
\newcommand{\ol}{\mathrm{ol}}

\newcommand{\kappast}{\kappa_{\star}}

\newcommand{\Sigst}{\Sigma_{\star}}
\newcommand{\hst}{h_{\star}}

\newcommand{\hhat}{\hat{h}}

\newcommand{\Sigwidhat}{\widehat{\Sigma}_{w,\mathrm{id}}}

\newcommand{\Qhatid}{\widehat{Q}_{\mathrm{id}}}
\newcommand{\Qtilid}{\wtilde{Q}_{\mathrm{id}}}
\newcommand{\Ahatid}{\widehat{A}_{\mathrm{id}}}
\newcommand{\Bhatid}{\widehat{B}_{\mathrm{id}}}
\newcommand{\Sigwhatid}{\widehat{\Sigma}_{w,\mathrm{id}}}

\newcommand{\contk}{\cont_k}
\newcommand{\contkap}{\cont_{\kappa}}
\newcommand{\Hid}{\scrH_{\mathrm{id}}}

\newcommand{\kapnot}{\kappa_0}

\newcommand{\Sigkaponeid}{\Sigkid[\kapone]}

\newcommand{\bykapone}{\by_{\kapone}}
\newcommand{\bxkapone}{\bx_{\kapone}}
\newcommand{\kapone}{\kappa_1}

\newcommand{\hstid}{h_{\star,\mathrm{id}}}
\newcommand{\cv}{\mathrm{cov}}

\newcommand{\Vinf}{\Vf_{\infty}}
\newcommand{\Qinf}{\Qf_{\infty}}
\newcommand{\dlyap}{\dlyap}
\newcommand{\gammainfb}{\bar{\gamma}_{\infty}}
\newcommand{\Tnot}{T_0}
\newcommand{\Csim}{C_{\mathrm{sim}}}
\newcommand{\vepsk}{\veps_{K}}
\newcommand{\vepsf}{\veps_{f}}
\newcommand{\Bf}{\bclip}
\newcommand{\cx}{c_{\matx}}
\newcommand{\cost}{J_{T}}
\newcommand{\sigmanu}{\sigma_{\bnu}}
\newcommand{\signu}{\sigma_{\bnu}}

\newcommand{\Vhatid}{\Vhat_{\mathrm{id}}}

\newcommand{\nref}{n_{\mathrm{op}}}
\newcommand{\nop}{n_{\mathrm{op}}}

\newcommand{\ninit}{n_{\mathrm{init}}}
\newcommand{\vepsinit}{\veps_{\mathrm{init}}}

\newcommand{\kappastar}{\kappa_{\star}}
\newcommand{\Ceil}[1]{\left \lceil{#1}\right \rceil }

\newcommand{\bc}{\mathbf{c}}

\newcommand{\vepssys}{\veps_{\mathrm{sys}}}

\newcommand{\vepsid}{\veps_{\mathrm{id}}}

\newcommand{\lambdam}{\lambda_{\cM}}

\newcommand{\bigthetas}{\Theta_{\star}}

\newcommand{\gammast}{\gammastar}
\newcommand{\alphast}{\alphastar}
\newcommand{\Psist}{\Psistar}

\newcommand{\sysid}{\textsc{SysID}}
\newcommand{\computepol}{\textsc{ComputePolicy}}
\newcommand{\getcoarsedecoder}{\textsc{GetCoarseDecoder}}
\newcommand{\onpo}{\mathrm{op}}
\newcommand{\id}{\mathrm{id}}
\newcommand{\gammab}{\bar{\gamma}_{\infty}}

\renewcommand{\log}{\ln}

\newcommand{\mathand}{\quad\text{and}\quad}

\newcommand{\dist}{\mathsf{dist}}
\newcommand{\Qtil}{\widetilde{Q}}
\newcommand{\bvhat}{\hat{\bv}}
\newcommand{\bVhat}{\what{\bV}}
\newcommand{\qst}{q_{\star}}
\newcommand{\Sige}{\Sigma_{e}}

\newcommand{\darece}{\mathsf{DARE}\xspace}

\newcommand{\Mbar}{\wb{M}}
\newcommand{\cMbar}{\wb{\cM}}

\newcommand{\vec}{\mathrm{vec}}

\let\widebar\undefined

\usepackage{amsmath}
\makeatletter
\let\save@mathaccent\mathaccent
\newcommand*\if@single[3]{%
  \setbox0\hbox{${\mathaccent"0362{#1}}^H$}%
  \setbox2\hbox{${\mathaccent"0362{\kern0pt#1}}^H$}%
  \ifdim\ht0=\ht2 #3\else #2\fi
  }
\newcommand*\rel@kern[1]{\kern#1\dimexpr\macc@kerna}
\newcommand*\widebar[1]{\@ifnextchar^{{\wide@bar{#1}{0}}}{\wide@bar{#1}{1}}}
\newcommand*\wide@bar[2]{\if@single{#1}{\wide@bar@{#1}{#2}{1}}{\wide@bar@{#1}{#2}{2}}}
\newcommand*\wide@bar@[3]{%
  \begingroup
  \def\mathaccent##1##2{%
    \let\mathaccent\save@mathaccent
    \if#32 \let\macc@nucleus\first@char \fi
    \setbox\z@\hbox{$\macc@style{\macc@nucleus}_{}$}%
    \setbox\tw@\hbox{$\macc@style{\macc@nucleus}{}_{}$}%
    \dimen@\wd\tw@
    \advance\dimen@-\wd\z@
    \divide\dimen@ 3
    \@tempdima\wd\tw@
    \advance\@tempdima-\scriptspace
    \divide\@tempdima 10
    \advance\dimen@-\@tempdima
    \ifdim\dimen@>\z@ \dimen@0pt\fi
    \rel@kern{0.6}\kern-\dimen@
    \if#31
      \overline{\rel@kern{-0.6}\kern\dimen@\macc@nucleus\rel@kern{0.4}\kern\dimen@}%
      \advance\dimen@0.4\dimexpr\macc@kerna
      \let\final@kern#2%
      \ifdim\dimen@<\z@ \let\final@kern1\fi
      \if\final@kern1 \kern-\dimen@\fi
    \else
      \overline{\rel@kern{-0.6}\kern\dimen@#1}%
    \fi
  }%
  \macc@depth\@ne
  \let\math@bgroup\@empty \let\math@egroup\macc@set@skewchar
  \mathsurround\z@ \frozen@everymath{\mathgroup\macc@group\relax}%
  \macc@set@skewchar\relax
  \let\mathaccentV\macc@nested@a
  \if#31
    \macc@nested@a\relax111{#1}%
  \else
    \def\gobble@till@marker##1\endmarker{}%
    \futurelet\first@char\gobble@till@marker#1\endmarker
    \ifcat\noexpand\first@char A\else
      \def\first@char{}%
    \fi
    \macc@nested@a\relax111{\first@char}%
  \fi
  \endgroup
}
\makeatother

\usepackage{inconsolata}
\usepackage[scaled=.90]{helvet}
\usepackage{xspace}

\usepackage{enumitem}
\usepackage{verbatim}

\title{Learning the Linear Quadratic Regulator \\from Nonlinear Observations}

\neurips{
\author{%
}
}

\arxiv{
\author{Zakaria Mhammedi\\
        {\small\texttt{zak.mhammedi@anu.edu.au}}
          \and
	Dylan J.\ Foster\\
        {\small\texttt{dylanf@mit.edu}}
	  \and
Max Simchowitz\\
                {\small\texttt{msimchow@berkeley.edu}}
       \and
        Dipendra Misra\\
        {\small\texttt{dimisra@microsoft.com}}
        \and        
        Wen Sun\\
        {\small\texttt{sun.wen@microsoft.com}}
                \and        
        Akshay Krishnamurthy\\
        {\small\texttt{akshaykr@microsoft.com}}
        \and
                Alexander Rakhlin\\
        {\small\texttt{rakhlin@mit.edu}}
        \and
                John Langford\\
        {\small\texttt{jcl@microsoft.com}}
}
\date{}
}

\begin{document}

\maketitle

\begin{abstract}

We introduce a new problem setting for continuous control called the LQR with Rich Observations, or \richlqr. In our setting, the environment is summarized by a low-dimensional continuous latent state with linear dynamics and quadratic costs, but the
agent operates on high-dimensional, nonlinear observations such as images from a
camera. To enable sample-efficient learning, we assume that the learner has access to a class of \emph{decoder
  functions} (e.g., neural networks) that is flexible enough to
capture the mapping from observations to latent states. We introduce a
new algorithm, \richid, which learns a near-optimal policy for the
\richlqr with sample complexity scaling only with the dimension of the
latent state space and the capacity of the decoder function
class. \richid is oracle-efficient and accesses the decoder class only
through calls to a least-squares regression oracle. Our results constitute the first provable sample complexity guarantee
for continuous control with an \emph{unknown} nonlinearity in the
system model and general function approximation.

 \end{abstract}

\section{Introduction}
\label{sec:intro}

In reinforcement learning and control, an agent must learn to
minimize its overall cost in a dynamic environment that responds to
its actions. In recent years, the field has developed a comprehensive
understanding of the non-asymptotic sample complexity of \emph{linear control}, where the dynamics of the
environment are determined by a noisy linear system of equations.
While studying linear models has led to a number of new theoretical
insights, most practical
control tasks are nonlinear. In this paper, we develop efficient
algorithms with provable sample complexity guarantees for nonlinear control with rich, flexible
function approximation.

For some control applications, the dynamics themselves are truly nonlinear, but
another case---which is particularly relevant to real-world systems---is where there are (unknown-before-learning) latent
linear dynamics which are identifiable through a nonlinear observation
process. For example, cameras watching a robot may capture enough information
to control its actuators, but the optimal control law is unlikely to
be a simple linear function of the pixels. More broadly, with the
decrease in costs of sensing hardware, it is now common to instrument
complex control tasks with high-throughput measurement
apparatus such as cameras, lidar, contact sensors, or other
alternatives.  These measurements often constitute \emph{rich
  observations} %
which capture relevant information
about the system state.  However, deriving a control policy from these
complex, high-dimensional sources remains a significant challenge in
both theory and practice.

\paragraph{The \richlqr setting.} We propose a learning-theoretic framework
for rich observation
continuous control in which the environment is summarized by a low
dimensional continuous latent state (such as joint angles), while the
agent operates on high-dimensional observations (such as images from a
camera). While this setup is more general, we focus our technical
developments on perhaps the simplest instantiation: the \emph{rich
  observation linear quadratic regulator} (\richlqr). The \richlqr
posits that latent states evolve according to noisy linear equations
and that each observation can be associated with a latent state by an
unknown nonlinear mapping.

We assume that every possible high-dimensional observation of the
system corresponds to a particular latent system state, a property we
term \emph{decodability}. This assumption is natural in applications
where the observations contain significantly more information than
needed to control the system. 
However, decoding the latent state may require a
highly nonlinear mapping, in which case linear control on the raw
observations will perform poorly.
Our aim is to learn such
a mapping from data and use it for optimal control in the latent space. 

\subsection{LQR with Rich Observations}
\label{sec:problem_setting}
\newcommand{\fst}{f_{\star}}
\newcommand{\Sigmaw}{\Sigma_{w}}
\newcommand{\iidsim}{\overset{\mathrm{i.i.d.}}{\sim}}
\newcommand{\alphastar}{\alpha_{\star}}
\newcommand{\gammastar}{\gamma_{\star}}

\richlqr is a continuous control problem described by the following dynamics:
\begin{equation}
  \label{eq:dynamics}
  \begin{aligned}
    \matx_{t+1} &= A\matx_{t} + B\matu_t + \matw_t,\quad
    \maty_{t} \sim{} q(\cdot\mid{}\matx_t).
  \end{aligned}
\end{equation}
Starting from $\matx_0$, the system state $\matx_t\in\bbR^{\dimx}$ evolves as a linear combination of the previous state, a control input $\matu_t\in\bbR^{\dimu}$ selected by the learner, and zero-mean \iid~process noise $\matw_t\in\bbR^{\dimx}$. The learner does not directly observe the state, and instead sees an \emph{observation} $\maty_t\in\bbR^{\dimy}$ drawn from the observation distribution $q(\cdot \mid \matx_t)$.\footnote{Our results do not depend on $\dimy$, and in fact do not even require that $\maty$ belongs to a vector space.} Here $\dimy\gg\dimx$---for example, $\matx_t$ might represent the state of a robot's joints, while $\maty_t$ might represent an image of the robot in a scene. Given a policy $\pi_t(\maty_0,\ldots,\maty_t)$ that selects control inputs $\matu_t$ based on past and current observations, we measure performance as
\begin{align}
J_{T}(\pi) := \E_{\pi}\left[\frac{1}{T}\sum_{t=1}^T\matx_t^\top Q \matx_t + \matu_t^\top R \matu_t  \right],\label{eq:cost_functional}
\end{align}
where $Q,R \psdgt 0$ are quadratic state and control cost matrices and $\En_{\pi}$ denotes the expectation when the system's dynamics \eqref{eq:dynamics} evolve under $\matu_t=\pi_t(\maty_0,\ldots,\maty_t)$.

In our model, the dynamics matrices $(A,B)$ and the observation distribution $q(\cdot{}\mid\matx)$ are unknown to the learner. We assume that the control cost matrix $R \succ 0 $ is known, but the state cost matrix $Q \succ 0 $ is unknown (so as not to
tie the cost matrices to the system representation). We also assume the instantaneous costs $\matc_t\ldef{}\matx_t^{\trn}Q\matx_t+\matu_t^{\trn}R\matu_t$, are revealed on each trajectory at time $t$ (this facilitates learning $Q$, but not $A,B$). The learner's goal is to PAC-learn an $\veps$-optimal policy: given access to $n$ trajectories from the dynamics \eqref{eq:dynamics}, produce a policy $\pihat$ such that $    J_T(\pihat) - J_T(\piinf) \leq{} \veps,$
where $\piinf$ is the optimal infinite-horizon policy.
If the dynamics matrices $(A,B)$ were known and the state $\matx_t$ were directly observed, the \richlqr would reduce to the classical LQR problem \citep{kalman1960contributions}, and we could compute an optimal policy for \eqref{eq:cost_functional} using dynamic programming. In particular, the optimal policy has the form $\pi_{\infty}(\matx_t)  = \Kinf \matx_t$, where $\Kinf$ is the optimal infinite-horizon state-feedback matrix given by the Discrete Algebraic Riccati Equation (\pref{eq:Kinf} in the sequel). To facilitate the use of optimal control tools in our nonlinear observation model, we make the following assumption, which asserts the state $\matx_t$ can be uniquely recovered from the observation $\maty_t$.
\begin{assumption}[Perfect decodability]
  \label{ass:perfect}
There exists a \emph{decoder function} $f_\star : \reals^{\dimy} \rightarrow \reals^{\dimx}$ such that $f_\star(y) = x$ for all $y \in \supp q(\cdot \mid x)$.\footnote{We remark that $\fstar$ is typically referred to as an \emph{encoder} rather than a decoder in the autoencoding literature.}
\end{assumption}
While a perfect decoder $\fst$ is guaranteed to exist under \pref{ass:perfect} (and thus the optimal LQR policy can be executed from observations), $\fst$ is \emph{not known} to the learner in advance. Instead, we assume that the learner has access to a class of functions $\Fclass$ (e.g., neural networks) that is rich enough to express the perfect decoder. Our statistical rates depend on the capacity of this class.
\begin{assumption}[Realizability]
  \label{ass:realizable}
  The learner's decoder class $\Fclass$ contains the true decoder $\fst$. %
\end{assumption}
While these assumptions---especially decodability---may seem strong at first glance, we show that without strong assumptions on the observation distribution, the problem quickly becomes statistically intractable. Consider the following variant of the model \eqref{eq:dynamics}:
\begin{equation}
  \label{eq:noisy}
  \begin{aligned}
    \matx_{t+1} &= A\matx_{t} + B\matu_t + \matw_t, \quad\quad \maty_{t} = \fstar^{-1}(\matx_t) + \mateps_t,
  \end{aligned}
\end{equation}
where $\mateps_t$ is an independent \emph{output noise} variable with $\En\brk*{\mateps_t}=0$. In the absence of the noise $\mateps_t$, the system \eqref{eq:noisy} is a special case of \eqref{eq:dynamics} for which $\fstar$ is the true decoder, but in general the noise breaks perfect decodability. 
Unfortunately, our first theorem shows that output noise can lead to exponential sample complexity for learning nonlinear decoders, even under very benign conditions. %
\newtheorem*{thm:lb_informal}{Theorem (informal)}
\begin{thm:lb_informal}
  Consider the dynamics \eqref{eq:noisy} with $\dimx = \dimy = \dimu =T = 1$ and unit Gaussian noise. %
  For every $\veps>0$ there exists an $\bigoh(\veps^{-1})$-Lipschitz decoder $\fstar$ and realizable function class $\Fclass$ with $\abs*{\Fclass}=2$ such that any algorithm requires $\Omega(2^{\prn*{\frac{1}{\veps}}^{2/3}})$ trajectories to learn an $\veps$-optimal decoder.
\end{thm:lb_informal}
A full statement and proof for this lower bound is deferred to \pref{app:lower_bound} for space. 
\paragraph{Our Algorithm: \richid.} Our main contribution is a new algorithmic principle, \emph{Rich Iterative Decoding}, or \richid, which solves the \richlqr problem with sample complexity scaling polynomially in the latent dimension $\dimx$ and $\log\abs*{\Fclass}$. We analyze an algorithm based on this principle called \richidce (``\richid with Certainty Equivalence''), which solves the \richlqr by learning an off-policy estimator for the decoder, using the off-policy decoder to approximately recover the dynamics $(A,B)$, and then using these estimates to iteratively learn a sequence of on-policy decoders along the trajectory of a near-optimal policy. Our main theorem is as follows.
\begin{theorem}[Main theorem]
  \label{thm:main}
  Under appropriate regularity conditions on the system parameters and noise process (\savehyperref{ass:perfect}{Assumptions \ref*{ass:perfect}}-\ref{ass:m_matrix}), \richidce learns an $\veps$-optimal policy for horizon $T$ using $C\cdot{}\frac{(\dimx+\dimu)^{16}T^{4}\log\abs*{\Fclass}}{\veps^{6}}$ trajectories, where $C$ is a problem-dependent constant.\footnote{See \pref{thm:main_full} in \pref{sec:main_proof} for the full theorem statement.}
\end{theorem}
\pref{thm:main} shows that it is possible to learn the \richlqr with complexity polynomial in the the latent dimension and decoder class capacity $\log\abs*{\Fclass}$, and independent of the observation space. To our knowledge, this is the first polynomial-in-dimension sample complexity guarantee for continuous control with an unknown system nonlinearity and general function classes. 
\iftoggle{neurips}{The main challenge we overcome in attaining \pref{thm:main} is trajectory mismatch: a learned decoder $\fhat$ which accurately approximates the true decoder $\fst$ well on one trajectory may significantly deviate from $\fst$ on another. Our algorithm addresses this issue using a carefully designed \emph{iterative decoding} procedure to learn a sequence of decoders on-policy.}{}
We present our main theorem for finite classes $\Fclass$ for simplicity, but this quantity arises only through standard generalization bounds for least squares, and can trivially be replaced by learning-theoretic complexity measures such as Rademacher complexity (in fact, local Rademacher complexity). For example, if $\Fclass$ has pseudodimension $d$, one can replace $\log\abs*{\Fclass}$ with $d$ in \pref{thm:main}.

\Cref{thm:main} requires relatively strong assumptions on the dynamical system---in particular, we require that the system matrix $A$ is stable, and that the process noise is Gaussian. Nonetheless, we believe that our results represent an important first step toward developing provable and practical sample-efficient algorithms for continuous control beyond the linear setting, and we are excited to see technical improvements addressing these issues in future research.

\subsection{Our approach}
Our algorithm is broken into three phases. In the first phase, we excite the system with Gaussian inputs, then solve a carefully designed regression problem which recovers a decoder $\fhat$ whose performance is near-optimal under the steady state distribution. The choice of \emph{what} regression problem to solve is rather subtle, and we show (\pref{sec:bayes_pred}) that many naive approaches (e.g., predicting observations from inputs) fail. Our first key contribution is to show that an approach based on predicting \emph{inputs from observations} succeeds under appropriate assumptions.

The second phase of our algorithm estimates the dynamics matrices $(A,B)$ and certain other system parameters using our learned decoder's prediction $\fhat(\maty_t)$ as a plug-in estimate for the system state $\matx_t$. We then use these estimates to synthesize a near-optimal linear controller $\Khat$. The analysis here is rather straightforward, albeit somewhat technical due to the misspecification error caused by the inexact state estimates.

\paragraph{Key challenge: Trajectory mismatch.}
The third phase of our algorithm solves a major issue we call \emph{trajectory mismatch}. Suppose for simplicity that $\Khat=\Kinf$, i.e. we exactly recover the optimal controller in the second phase (in reality, we must account for approximation error). A tempting approach is to select $\matu_t=\Kinf\fhat(\matx_t)$, where $\fhat$ is the decoder learned in the first phase. Unfortunately, this decoder is only guaranteed to be accurate on the steady state distribution induced by the Gaussian inputs we use for the first phase. There is no guarantee that this decoder will be accurate on the state distribution induced by the policy above. Indeed, this is an instance of a common technical issue in statistical learning: In general, given a function $\hat{f}$ such that $\mathbb{E}_{P}\|\hat{f}(x) - f^{\star}(x)\|^2 \leq \varepsilon$ for a distribution $P$, we have no guarantee that $\mathbb{E}_{Q}\|\hat{f}(x) - f^{\star}(x)\|^2 \leq \varepsilon$ for a different distribution $Q$ unless we put strong structural assumptions on either $P/Q$ or the function class $\mathcal{F}$. Since we do not make such assumptions, we solve this problem by learning a new decoder. This is where our work departs from recent efforts such as \cite{dean2020certainty}, who---by working with nonparametric classes which incur exponential sample complexity---learn a decoder which \emph{uniformly} approximates $\fstar$; such an approach does not succeed in the general setting we consider here.

At this point, the challenge we face is how to learn a new decoder $\fhat$ that approximates $\fst$ on trajectories induced by playing $\pihat(\maty_t)=\Kinf\fhat(\maty_t)$. In particular, the foundational performance difference lemma \citep{kakade2003sample} implies that it suffices to ensure that
\begin{align}
\sum_{t=1}^T\En_{\pihat}\brk[\big]{\|\fhat(\maty_t) - \fstar(\maty_t)\|^2}  \le \veps. \label{eq:fhat_Informal_desired}
\end{align}
This presents a clear chicken-and-egg problem: how do we ensure that  $\fhat$ enjoys \eqref{eq:fhat_Informal_desired} on its \emph{own} induced policy $\pihat$?

\paragraph{Our solution: Iterative decoding.} We address this issue by  iteratively learning a sequence of time-dependent decoders $\crl{\fhat_{t}}_{t=1}^{T}$. For each iteration $1\leq{}t \le T$ we predict $\matx_t$ with $\fhat_t(\maty_t)$, where $\fhat_t$ is a decoder learned at the previous iteration, and follow the induced policy $\pihat_t(\maty_t) = \Kinf \fhat_t(\maty_t)$. We then estimate $\fhat_{t+1}$ by learning to predict $\matx_{t+1}$ under the trajectory induced by playing $\matu_1=\pihat_1(\maty_1),\ldots,\matu_t=\pihat_t(\maty_t)$. They key idea here is that the induced distribution for $\matx_t$ does not depend on $\fhat_t$, only on $\fhat_1,\ldots,\fhat_{t-1}$, thereby solving the chicken-and-egg problem.

A major technical challenge is ensuring that this iterative decoding procedure does not lead to errors which compound exponentially in the horizon $T$; this is a serious issue which can easily arise if the misspecification error for the regression problem we solve to learn $\fhat_t$ depends on the quality of the previous decoders $\fhat_1,\ldots,\fhat_{t-1}$. To solve this issue, we work with another carefully designed regression problem. They key idea is to roll in with the policies $\pihat_1,\ldots,\pihat_{t-1}$, but roll out with purely Gaussian inputs for steps $\tau = t,t+1,\dots$. This allows us to set up a regression problem which is \emph{well-specified} and enjoys the advantages of Gaussianity, while remaining valid under the trajectory induced by $\crl*{\pihat_t}$. The analysis for this phase is quite technical due to the inexact estimates from the first two phases, and showing that the indirect regression problems we solve eventually lead to a good predictor for the state requires substantial effort.

\subsection{Technical Preliminaries}
\label{sec:preliminaries}
\neurips{In the interest of brevity, we present an abridged discussion of technical preliminaries; all omitted formal definitions, further assumptions, and additional notation are deferred to \Cref{app:prelim}.}
The main assumptions used by \richid are as follows.
\begin{assumption}[Gaussian initial state and process noise]\label{ass:gaussian}
The initial state satisfies $\matx_0\sim{}\cN(0,\Sigma_0)$, and process noise is i.i.d. $\bw_t \sim \cN(0,\Sigmaw)$. Here, $\Sigma_0,\Sigmaw$ are \emph{unknown} to the learner, with $\Sigmaw \succ 0$.
\end{assumption}
\begin{assumption}[Controllability]\label{ass:controllability} For each $k \ge 1$, define $\cC_{k}\ldef[A^{k-1}B \mid \dots \mid B ] \in \reals^{\dimx \times k \dimu}$. We assume that $(A,B)$ is \emph{controllable}, meaning that $\cC_{\kappast}$ has full column rank for some $\kappast\in \bbN$. 
\end{assumption}
Note that \pref{ass:controllability} imposes the constraint $\dimu \kappast \ge \dimx$, which we use to simplify various expressions.
\begin{assumption}[Growth Condition]\label{asm:f_growth} There exists $L\geq{}1$ such that $\|f(y)\| \le L\max\{1,\|\fst(y)\|\}$ for all $y \in \cY$ and $f \in \Fclass$.
\end{assumption}
\begin{assumption}[Stability]\label{ass:stability} $A$ is stable; that is, $\rho(A) < 1$, where $\rho(\cdot)$ denotes the spectral radius.
\end{assumption}
Our algorithms and analysis make heavy use of the Gaussian process noise assumption, which we use to calculate closed-form expressions for certain conditional expectations that arise under the dynamics model \eqref{eq:dynamics}. We view relaxing this assumption as an important direction for future work. Controllability is somewhat more standard  \citep{mania2019certainty}, and the growth condition ensures predictions do not behave too erratically. Stability ensures the state remains bounded without an initial stabilizing controller. While assuming access to an initial stabilizing controller is fairly standard in the recent literature on linear control, this issue is more subtle in our nonlinear observation setting. These assumptions can be relaxed somewhat; see \Cref{ssec:relax_asm}. \neurips{We make the stability assumption
quantitative via the notion of ``strong stability'' (\pref{app:prelim})\citep{cohen2018online}. Finally, we assume access to upper bounds on various system parameters.}

\paragraph{Policies, interaction model, and sample complexity.}
Formally, a policy $\pi$ for the setup \eqref{eq:dynamics} is a
sequence of mappings $(\pi_t)_{t=0}^{T}$, where $\pi_t$ maps the
observations $\maty_0,\ldots,\maty_t$ to an output control signal
$\matu_t$. In each round of interaction, the learner proposes a policy
$\pi$ and observes a trajectory
$\matu_0,\maty_0,\ldots,\matu_{T},\maty_{T}$ where
$\matu_t=\pi_t(\maty_0,\ldots,\maty_t)$.  We measure the sample complexity to learn an $\veps$-optimal policy
for $J_T$ in terms of the number of trajectories observed in this model. However, to simplify the description of our algorithm, we
allow the learner to execute trajectories of length $2T +
\bigoh_{\star}(\ln\ln(n))$ during the learning process, even though the
objective is $J_T$. To avoid trivial issues caused by unidentifiability of the initial
state $\matx_0$, we define $J_T$ to measure cost on times
$1,\dots,T$. On the other hand, our rollouts begin at time $0$: the initial state is $\matx_0$, and the first control input executed
is $\matu_0$.

\paragraph{Cost functions. }
We assume that the control cost matrix $R \succ 0 $ is known but,
to
avoid tying costs to the unknown latent representation $\matx$, we
assume that the state cost matrix $Q \succ 0 $ is unknown. Instead, we
assume that the learner has access to an additional \emph{cost oracle}
which on each trajectory at time $t$ reveals
$\matc_t\ldef{}\matx_t^{\trn}Q\matx_t+\matu_t^{\trn}R\matu_t$. For
simplicity, we place the following mild regularity conditions on the
cost matrices.%
 \begin{assumption}
    The cost matrices $\Rx$ and $\Ru$ satisfy $\eigmin(\Rx),\eigmin(\Ru)\geq{}1$. 
  \end{assumption}
  This assumption can be made to hold without loss of generality
  whenever $Q,R\psdgt{}0$ via rescaling.

\paragraph{The \dare and infinite-horizon optimal control.}
Controllability (and more generally stabilizability) implies that there is a unique positive definite solution $\Pinf\psdgt{}0$ to the \emph{discrete algebraic
  Riccati equation} (\dare),
  \begin{align}
  \tag{DARE}
    \label{eq:dare}
    P = A^{\trn}PA + \Rx - A^{\trn}PB(\Ru+B^{\trn}PB)^{-1}B^{\trn}PA,
  \end{align}
  which characterizes the optimal cost function for the LQR problem in the infinite-horizon setting. Our analysis uses $\Pinf$, and our algorithms use the optimal infinite-horizon state feedback controller 
  \begin{align}
  \Kinf\ldef{}-(\Ru+B^{\trn}\Pinf{}B)^{-1}B^{\trn}\Pinf{}A. \label{eq:Kinf}
  \end{align}
  When the state $\matx_t$ is directly observed, the optimal
  infinite-horizon controller is the time-invariant feedback policy $u
  = \Kinf x$. Thus, the optimal infinite-horizon policy for \richlqr,
  given the exact decoder, is $\piinf(y) = \Kinf \fst(y)$. We use this
  controller as our benchmark. Our analysis also uses the infinite-horizon covariance matrix
  \begin{align*}
   \Siginf \coloneqq \Ru+B^{\trn}\Pinf{}B.
\end{align*}
Our algorithm relies on \emph{certainty equivalence}, in which we estimate $\Kinf$ by solving the \dare{} with plug-in estimates $\Ahat,\Bhat$ of $(A,B)$ to obtain a matrix $\Phat$, and take $\Khat := -(\Ru+\Bhat^{\trn}\Phat \Bhat)^{-1}\Bhat^{\trn}\Phat\Ahat$. 
  \begin{definition}[\dare operator]\label{defn:darece} We define the  $\darece$ operator as the operator which takes in matrices $(A_0,B_0,R_0,Q_0)$ with $R_0,Q_0 \succeq 0$, and returns $(P,K)$ such that
  \begin{align*}
  P &= A_0^{\trn}PA_0 + \Rx - A_0^{\trn}PB_0(\Ru+B_0^{\trn}PB_0)^{-1}B_0^{\trn}PA_0,\\
  K&=  -(\Ru+B_0^{\trn}PB_0)^{-1}B_0^{\trn}PA_0.
  \end{align*}
  \end{definition}

\neurips{%
\paragraph{Strong stability.}
We quantify stability via \emph{strong stability}
\citep{cohen2018online}. Intuitively, a matrix $X$ is \emph{strongly
  stable} if its powers $X^n$ decay geometrically in a quantitative sense.
  \begin{definition}[Strong stability]
\label{def:ss}
A matrix $X\in \bbR^{\dimx\times\dimx}$ is said to be $(\alpha,\gamma)$-strongly stable if there exists $S\in \bbR^{\dimx\times\dimx}$ such that $\|S\|_{\op}\|S^{-1}\|_{\op} \le \alpha$ and $\|S^{-1}XS\|_{\op} \le \gamma<1$. 
\end{definition}
We frequently make use of the fact that if $X$ is $(\alpha,\gamma)$-strongly stable, then 
\begin{align*}
\|X^n\|_{\op} = \|S(S^{-1}XS)^nS^{-1}\|_{\op} \le \|S^{-1}\|_{\op}\|S\|_{\op}\|S^{-1}XS\|_{\op}^n \le \alpha \gamma^n.
\end{align*}
We let $(\alphaa,\gammaa)$ and $(\alphainf,\gammainf)$ be the strong
parameters for $A$ and $\Aclinf\coloneqq A + B K_\infty$ respectively. Under \Cref{ass:controllability} and \Cref{ass:stability}, we are guaranteed
that $\gammaa,\gammainf<1$ (see \pref{prop:closed_loop_ss} and
\pref{prop:open_loop_ss} for quantitative bounds). Finally, we recall from
\pref{asm:para_upper_bounds} that we assume the learner knows
upper bounds $(\alphastar,\gammastar)$ such that
$\alphaa\vee\alphainf\leq\alphastar$ and $\gammaa\vee\gammainf\leq\gammastar$.

}

\paragraph{Strong stability.}
We quantify stability of various matrices that arise in our analysis via \emph{strong stability}
\citep{cohen2018online}. Intuitively, a matrix $X$ is \emph{strongly
  stable} if its powers $X^n$ decay geometrically in a quantitative sense.
  \begin{definition}[Strong stability]
\label{def:ss}
A matrix $X\in \bbR^{\dimx\times\dimx}$ is said to be $(\alpha,\gamma)$-strongly stable if there exists $S\in \bbR^{\dimx\times\dimx}$ such that $\|S\|_{\op}\|S^{-1}\|_{\op} \le \alpha$ and $\|S^{-1}XS\|_{\op} \le \gamma<1$. 
\end{definition}
We make frequent
use of the fact that if $X$ is $(\alpha,\gamma)$-strongly stable, then 
\begin{align*}
\|X^n\|_{\op} = \|S(S^{-1}XS)^nS^{-1}\|_{\op} \le \|S^{-1}\|_{\op}\|S\|_{\op}\|S^{-1}XS\|_{\op}^n \le \alpha \gamma^n.
\end{align*}
We let $(\alphaa,\gammaa)$ and $(\alphainf,\gammainf)$ be the strong stability
parameters for $A$ and $\Aclinf\coloneqq A + B K_\infty$, respectively. Under \Cref{ass:controllability} and \Cref{ass:stability}, we are guaranteed
that $\gammaa,\gammainf<1$ (see \pref{prop:closed_loop_ss} and
\pref{prop:open_loop_ss} for quantitative bounds).

Finally, we assume access to upper bounds on various system parameters.
\begin{assumption}\label{asm:para_upper_bounds} We assume that the learner has access to parameter upper bounds $\Psistar \ge 1$, $\alphastar \ge 1$, $\gammastar\in (0,1)$, and $\kappa \in \bbN$ such that (\textbf{I}) $\kappa \ge \kappast$, (\textbf{II}) $A$ and $(A + B\Kinf)$  are both $(\alphastar,\gammastar)$-\text{strongly stable}, and (\textbf{III}) $\Psistar$ is an upper bound on the operator norms of $A$, $B$, $Q$, $R$, $\Sigw$, $\Sigw^{-1}$, $\Sigma_0$, $\Kinf$, and $\Pinf$.\footnote{Here, $\Pinf$ solves the \dare \arxiv{\eqref{eq:dare}}\neurips{(\eqref{eq:dare} in \Cref{ssec:prelim_unabridge})}, and $\Kinf$ is the optimal infinite horizon controller.}
\end{assumption}

\neurips{\subsection{Additional Notation}}
\paragraph{Asymptotic notation.}
Lastly, we adopt standard non-asymptotic big-oh notation. For functions
	$f,g:\cX\to\bbR_{+}$, we write $f=\bigoh(g)$ if there exists some universal constant
	$C>0$, which doesn not depend on problem parameters, such that $f(x)\leq{}Cg(x)$ for all $x\in\cX$. Our proofs also use the shorthand $f(x) \lesssim g(x)$ to denote $f = \BigOh\prn{g}$. We use $\bigoht(\cdot)$ so suppress logarithmic dependence on system parameters, time horizon, and dimension. We use $\bigohs(\cdot)$ to suppress polynomial factors in $\alphastar,\gammastar^{-1},(1-\gammastar)^{-1},\Psistar$, $L$, and $\sigmamin^{-1}(\cC_{\kappa})$, and all logarithmic factors except for $\log\abs*{\Fclass}$ and $\log(1/\delta)$. We write $f=\bigoms(g)$ if $f(x)\geq{}Cg(x)$ for all $x\in\cX$, where $C$ is a \emph{sufficiently large} constant whose value is polynomial in the same parameters. Lastly, we write $f=\bigohch(g)$ if $f(x)\leq{}cg(x)$ for all $x\in\cX$, where
$c=\poly(\gammastar(1-\gammastar),\alphastar^{-1},\Psistar^{-1},L^{-1},\sigmamin(\cC_{\kappa}))$ is a \emph{sufficiently small} constant.

\paragraph{General notation.} 
	For a vector $x\in\bbR^{d}$, we let $\nrm*{x}$ denote the euclidean
	norm and $\nrm*{x}_{\infty}$ denote the element-wise $\ls_{\infty}$
	norm. We let $\nrm*{x}_{A}=\sqrt{x^{\trn}Ax}$ for $A\psdgeq{}0$. For a matrix $A$, we let $\nrm*{A}_{\op}$ denote the
        operator norm. If $A$ is symmetric, we let $\eigmin(A)$ denote the
	minimum eigenvalue. For a potentially asymmetric matrix $A\in\bbR^{d\times{}d}$, 
        we let $\rho(A)\coloneqq \max \{|\lambda_1(A)|, \dots
        ,|\lambda_{d}(A)|\}$ denote the spectral radius. For a
        symmetric matrix $M\in\bbR^{d}$, $(M)_{+}$ denotes the result
        of thresholding all negative eigenvalues to zero, and we let
        $\lambda_1(M),\ldots,\lambda_d(M)$ denote the eigenvalues of
        $M$, sorted in decreasing order. Similarly, for a matrix
        $A\in\bbR^{d_1\times{}d_2}$, we let
        $\sigma_1(A),\ldots,\sigma_{d_1\wedge{}d_2}(A)$ denote the
        singular values of $A$, sorted in decreasing order, and use the shorthand $\sigma_{\min}(A)=\sigma_{d_1\wedge{}d_2}(A)$. We let $\vec(A)\in\bbR^{d_1d_2}$ be the vectorization of $A$. For matrices $A$ and $B$, we use $[A\mid{}B]$ or $[A;B]$ to denote their horizontal concatenation.

 \arxiv{
\subsection{Related Work}

\neurips{\paragraph{Related work.} }Our model and approach are related to the literature on Embedding to
Control (E2C), and related techniques
\citep{watter2015embed,banijamali2018robust,hafner2019learning,levine2019prediction,shu2020predictive,dean2019robust}
(see also \cite{levine2016end}). At a high level, these
approaches learn a decoder that maps images down to a latent
space, then performs simple control techniques such as iterative LQR (iLQR) in the
latent space (\citet{watter2015embed} is a canonical
example). These approaches are based on heuristics, and
do not offer provable sample complexity guarantees to learn the
decoder in our setting.

Our work is also related to recent results on rich observation
reinforcement learning with discrete
actions~\citep{jiang2017contextual}. We view our model as the
control-theoretic analog of the block MDP model studied
by~\citet{DuKJAD019,Homer}, in which a latent state space associated
with a discrete Markov Decision Process is decodable from rich
observations. However, our setting is considerably different, in part
because of the continuous nature of the \richlqr, and so the results
and techniques are incomparable. In particular, discretization
approaches immediately face a curse-of-dimensionality phenomenon and
do not yield tractable algorithms. Interestingly, even ignoring
the issue of continuous actions, our setting does not appear to have
low Bellman rank in the sense of \citet{jiang2017contextual}.

A recent line of
work~\citep{oymak2019stochastic,sattar2020non,foster2020learning} gives
non-asymptotic system identification guarantees for a simple class of
``generalized linear'' dynamical systems. These results address a non-linear
dynamic system, but are incomparable to our own as the non-linearity
is known and the state is directly observed.
Our results also are related to the LQG problem,
which is a special case of \eqref{eq:noisy} with linear observations; recent work provides non-asymptotic guarantees
\citep{mania2019certainty,simchowitz2020improper,lale2020logarithmic}. These
results show that linear classes do not
encounter the sample complexity barrier exhibited by \pref{thm:lower_bound}.%

Finally, we mention two concurrent works which consider similar
settings. First, \cite{frandsen2020extracting} give guarantees for a
simpler setting in which we observe a linear combination of the latent
state and a nonlinear nuisance parameter, and where there is no noise. Second, \citet{dean2020certainty} (see also \citet{dean2019robust})
give sample complexity guarantees for a variant of the our setting in
which there is no system noise, and where $\maty_t=\gstar(C\maty_t)$,
where $C\in\bbR^{p\times{}\dimx}$ and $\gstar:\bbR^{p\to{}q}$ is a smooth function. They
provide a nonparametric approach which scales exponentially in the
dimension $p$. 
Compared to this result, the main advantage of our approach is that
it allows for general function approximation; that is, we allow
for arbitrary function classes $\Fclass$, and our results depend only
on the capacity of the class under consideration. In terms of assumptions, the addition of the $C$ matrix allows for maps that
(weakly) violate the perfect decodability assumption; we suspect
that our results can be generalized in this fashion. Likewise, we
believe that our assumption concerning the stability of $A$ can be removed in the absence
of system noise (indeed, system noise is one of the primary technical
challenges overcome by our approach).

 }

\section{An Algorithm for LQR with Rich Observations}
\label{sec:algorithm}
\begin{algorithm}[h]
  \setstretch{1.1}
  \begin{algorithmic}[1]\onehalfspacing
    \State\textbf{Inputs:}
    \Statex{}~~~~$\veps$ (suboptimality), $T$ (horizon), $\Fclass$
    (decoder class), $\dimx, \dimu$ ({latent dimensions}), \Statex{}~~~~$\Psistar, \kappa,\alphastar, \gammastar$ ({system parameter
       upper bounds}), $R$ ({control cost}). %
       \State\textbf{Parameters:} \algcomment{see \pref{sec:main_proof} for values.}
        \Statex{}~~~~$\nid$, $\nref$ \algparen{sample size for
          Phase/Phase II and Phase III, respectively}
    \Statex{}~~~~$\kappa_0$ \algparen{burn-in time index}
         \Statex{}~~~~$r_{\id}, r_{\op}$ \algparen{radius for sets
           $\Hid$ and $\Hclass_{\op}$}
         \Statex{}~~~~$\sigma^{2}$ \algparen{exploration variance}
         \Statex{}~~~~$\bclip$ \algparen{clipping parameter for
           decoders}
         \smallskip

    	    \State{}\textbf{Phase I} \algcomment{learn a coarse decoder (see \pref{sec:phase1})}
    	        \State  Set $\fhatid \leftarrow
                \getcoarsedecoder(\nid, \kappa_0, \kappa, r_{\id})$. \algcomment{\pref{alg:phase1}}
    \State{}\textbf{Phase II} \algcomment{learn system's dynamics and cost (see \pref{sec:phase2})}
    \State  Set $(\what A_{\id}, \what B_{\id}, \what \Sigma_{w,\id},
    \what Q_{\id})  \leftarrow \sysid(\fhatid, \nid, \kappa_0, \kappa)$. \algcomment{\pref{alg:phase2}}
    \State{}\textbf{Phase III} \algcomment{compute optimal policy (see \pref{sec:phase3}\neurips{ and \pref{sec:phase3_proofs}})}
     \State Set $\what \pi \leftarrow \computepol(\what A_\id,\what
     B_\id, \what \Sigma_{w,\id},\what Q_{\id},R, \nref, \kappa,
     \sigma^2, T, \bclip, r_{\op})$. \algcomment{\pref{alg:phase3}}
     \State{}\textbf{Return:}  $\what \pi$.
  \end{algorithmic}
  \caption{\richidce}
  \label{alg:main}
\end{algorithm}

        \newcommand{\fref}{\fhat}
        \newcommand{\freftil}{\tilde{f}}
    
        \newcommand{\vepstarget}{\veps}
	\newcommand{\nsys}{n_{\mathrm{sys}}}

We now present out main algorithm, \richidce (\pref{alg:main}), which
attains a polynomial sample complexity guarantee for the \richlqr. 

\paragraph{Algorithm overview.}
\pref{alg:main} consists of three phases. In Phase I (\pref{alg:phase1}), we roll in with
Gaussian control inputs and learn a good decoder under this roll-in
distribution by solving a certain regression problem involving our
decoder class $\Fclass$. In Phase II (\pref{alg:phase2}), we leverage this decoder to
learn a \emph{model} $(\Ahat,\Bhat)$ for the system dynamics (up to a
similarity transform). Due to linearity of the dynamics, this model is
valid on any trajectory. Moreover, we can synthesize a controller $\Khat$ so that the feedback controller $\matu_t = \Khat \matx_t$ is optimal for $(\Ahat,\Bhat)$, and thus near-optimal for $(A,B)$.

	To actually implement this feedback controller, we still need
        a good decoder for the state. Unfortunately, our decoder from
        Phase I may be inaccurate along the optimal (or near-optimal)
        trajectory. Thus, in Phase III (\pref{alg:phase3}) we inductively solve a sequence of regression problems---one for each
	time $t=0,\ldots,T$---to learn a sequence of state decoders
        $(\fhat_t)$, such that for each $t$, $\fhat_t \approx \fst$ under the
        roll-in distribution induced by playing
        $\Khat\fhat_{s}(\maty_{s})$ for $s<t$. We do this by rolling
        in with this near-optimal policy until $t$, but rolling out with
        purely Gaussian inputs. The former ensures that the decoder is
        accurate along the desired trajectory. The latter ensures that
        the regression at time $t$ is essentially ``independent'' of
        approximation errors incurred by steps $0,\dots,t-1$, avoiding
        an accumulation of errors which would otherwise compound exponentially in the horizon $T$.

	In what follows, we walk through each phase in detail and explain the
	motivation, the technical assumptions required, and the key
	performance guarantees.

\arxiv{%
\subsection{Predicting Inputs from Outputs: The Bayes Regression Function}
\label{sec:bayes_pred}
At the core of our algorithm is a simple but indispensible identity for the Bayes predictor that arises when we aim to predict control inputs $\matu$ from observations $\maty$ in the \richlqr model. As a motivating example, let a time $\tau\geq{}1$ be fixed, suppose we take Gaussian inputs $\matu_{1:\tau} := (\matu_1^\top,\dots,\matu_{\tau}^\top)^\top \sim \calN(0,I_{\tau\dimu})$, and consider the resulting state $\matx_{\tau+1}$. Suppose that our goal is to estimate $\fst$ with expected $L_2$ error under the marginal distribution of $\matx_{\tau+1}$. That is, we wish to ensure
\begin{align}
\E\left[\|\fhat(\matx_{\tau+1}) - \fst(\matx_{\tau+1})\|^2\right] \le \text{(something small)}, \label{eq:something_small}
\end{align}
where $\En\brk*{\cdot}$ denotes the expectation under the Gaussian inputs above.
\paragraph{Attempt 1.} The natural strategy to attain \eqref{eq:something_small} is to regress $\maty_{\tau+1}$ to $\matu_{1:\tau}$. For example, note that linearity of the dynamics ensures that there exists a matrix $M_{\star} \in \R^{\dimx \times \tau\dimu}$ such that $\E[\matx_{\tau+1} \mid \matu_{1:\tau}] = M_{\star}\matu_{1:\tau}$. Thus, one could attempt the regression
\begin{align*}
\min_{f \in \Fclass, M \in \R^{\dimx \times \tau\dimu}} \E\left[\| M \matu_{1:\tau} - f(\maty_{\tau+1})\|^2\right].
\end{align*}
Unfortunately, there are too many degrees of freedom in this minimization problem: if $0 \in \Fclass$, then the above is minimized with $f=0$ and $M = 0$. 

\paragraph{Attempt 2.}A second attempt might be to hope that all  $f \in \Fclass$ are invertible, and try to solve a regression problem based on reconstructing the observations:
\begin{align*}
\min_{f^{-1}: f \in \Fclass, M \in \R^{\dimx \times \tau\dimu}} \E\left[\| f^{-1}(M \matu_{1:\tau}) - \maty_{\tau+1}\|^2\right],
\end{align*}
Unfortunately, since $\matx_{\tau+1} = M_{\star}\matu_{1:\tau} + \text{(noise)}$, passing through the nonlinearity $f^{-1}$ obviates any clear guarantees. In particular, this setup does not satisfy the usual first-order condition for regression with a well-specified model. A secondary issue is that even in the absence of system noise, this approach would likely incur dependence on the \emph{observation} dimension $\dimy$. %

\paragraph{Our Approach.} Our approach is to flip the input and target and regress $\matu_{1:\tau}$ to $\maty_{\tau+1}$. Specifically, we consider the regression:
\begin{align}
\min_{g = M f : f \in \Fclass,  M\in \R^{\tau\dimu \times \dimx }}\E\left[\|g(\maty_{\tau+1}) - \matu_{1:\tau}\|^2\right]. \label{eq:reg_g_y_u}
\end{align}
Let us motivate this approach and shed some light on the properties of the solution to this problem. Leveraging the perfect decodability assumption, one can show that $\E[\matu_{1:\tau} \mid \maty_{\tau+1}] = \E[\matu_{1:\tau} \mid \matx_{\tau+1} = \fst(\maty_{\tau+1})]
$.  Moreover, since $\matx_{\tau+1}$ and $\matu_{1:\tau}$ are \emph{jointly Gaussian} (due to linearity of the dynamics and Gaussianity of the process noise), a simple calculation reveals that there exists a matrix $\widetilde{M}$ such that $ \E[\matu_{1:\tau} \mid \matx_{\tau+1} = x] = \widetilde{M}x$. Hence, 
\begin{align*}
\E[\matu_{1:\tau} \mid \maty_{\tau+1}]  = \widetilde{M} \fst(\maty_{\tau+1}).
\end{align*}
In particular, this implies that the unconstrained minimizer (i.e., over all measurable functions $g$) in \eqref{eq:reg_g_y_u} lies in the set $\{M f: f \in \Fclass\}$. Hence, since conditional expectations minimize square loss, we find:
\begin{quote} \emph{Up to a set of measure zero, any minimizer of $\eqref{eq:reg_g_y_u}$ must have the form $g = \widetilde{M} \fst$. In other words, the population risk minimizer recovers $\fst$ up to a linear} transformation. 
\end{quote}
Note that this crucially relies on Gaussianity, because while $\E[\matx_{\tau+1} \mid \matu_{1:\tau}]$ is linear in $\matu_{1:\tau}$ for any mean-zero process noise, the same is no longer true when considering $\E[ \matu_{1:\tau} \mid \matx_{\tau+1}]$. But with this strong assumption, we find that \eqref{eq:reg_g_y_u} allows us to recover $\fst$ up to a global linear transformation. Of course, there are numerous remaining subtleties including: 
\begin{itemize}
	\item Inverting $\widetilde{M}$ to recover $\fst$. 
	\item Identifying $\widetilde{M}$, especially since the learner does not know the system dynamics or noise covariance at first.
	\item Passing from population risk to empirical risk from finite samples.
\end{itemize}
How we address the above issues varies in different phases of the \richidce, and the remainder of this section supplies these details. But the fundamental principle---that we can solve empirical versions of \eqref{eq:reg_g_y_u} to recover linear transformations of $\fst$---remains the core workhorse of \richidce. 

\begin{remark}[Oracle Efficiency]\label{rem:oracle_efficiency} Consider the empirical version of \eqref{eq:reg_g_y_u} in which we gather $n$ trajectories and solve
\begin{align*}
\min_{g = Mf, f \in \Fclass} \sum_{i=1}^n \|g(\maty_{\tau+1}^{(i)}) - \matu_{1:\tau}^{(i)}\|_2^2,
\end{align*}
where the superscript $i$ denotes the $i$-th trajectory.  Solving problems of this form is computationally efficient whenever we have
	a \emph{regression oracle} for the induced class $\crl*{g=Mf\mid{}f\in\Fclass, M\in\bbR^{\dimu\tau\times{}\dimx}}$. For many
	function classes of interest, such as linear functions and neural
	networks, solving regression over this class is no harder than
	regression over the original decoder class $\Fclass$. We believe this is a reasonable and practical assumption.
\end{remark}

 }
\subsection{Phase I: Learning a Coarse Decoder}
\label{sec:phase1}
	In Phase I (\pref{alg:phase1}), we gather $2\nid$ trajectories by selecting independent
	standard Gaussian inputs $\bu_t \sim \cN(0,I_{\dimu})$ for each $0\leq{}t \leq
	\kappa_1\ldef{}\kappa_0+\kappa$, where we recall that $\kappa$
        is an upper-bound on the controllability index $\kappast$, and
        where $\kapnot$ is a ``burn-in'' time used to ensure mixing to a
        near-stationary distribution, defined as follows:\footnote{This is useful for learning $(A,B)$ in \eqref{eq:recover_AB}, ensuring $\fhatid$ is accurate at both times $\kappa_1$ and $\kappa_1 + 1$.}
	\begin{align}
	\kapnot \ldef \Ceil{(1-\gammastar)^{-1} \ln
          \left({84\Psistar^5 \alphastar^4 \dimx
          (1-\gammastar)^{-2}\ln(10^3\cdot\nid)} \right)}.%
          \label{eq:kappanot_def}
	\end{align}

        	\arxiv{
\begin{algorithm}[htp]
  \setstretch{1.1}
  \begin{algorithmic}[1]\onehalfspacing
    \State\textbf{Inputs:}
    \Statex{}~~~~$\nid$ \algparen{sample size}
    \Statex{}~~~~$\kappa_0$ \algparen{``burn-in'' time index}
    \Statex{}~~~~$\kappa$ \algparen{upper bound on the controllability index $\kappa_\star$}
         \Statex{}~~~~$r_{\id}$ \algparen{upper bound on the matrix $M$ in the definition of $\scrH_\id$}
    \smallskip

    \State{}Set $\Hid := \left\{M  f(\cdot) \mid f \in \Fclass,\  M \in
    \R^{\kappa\dimu \times \dimx}, ~\|M\|_{\op} \le r_{\id}
    \right\}.$
             \State Set $\kappa_1 = \kappa_0 + \kappa$.
          \State{}Gather $2\nid$ trajectories by sampling control
          inputs $\matu_0,\ldots,\matu_{\kappa_1-1}\sim{}\cN(0,I_{\dimu})$.
    \State{}\textbf{Phase I:} \algcomment{Learn coarse
      decoder (see \pref{sec:phase1}).}
    \State{}Set $\hhatid = \argmin_{h \in \Hid}\sum_{i=1}^{\nid} \|h(\bykapone\supi) - \bv\supi\|_2^2$, where $\bv \coloneqq (\bu_{\kappa_0}^\top, \dots, \bu_{\kappa_1-1})^\top$.
    \State{}Set $\Vhatkapid$ to be an orthonormal basis for
                                               top $\dimx$-eigenvectors of
                                               $\frac{1}{\nid}\sum_{i=\nid+1}^{2\nid}
              \hhatid(\bykapone\supi) \hhatid(\bykapone\supi)^\top$
    \State{}Set $\fhatid(\cdot) \coloneqq \Vhatkapid^\top \hhatid(\cdot)$. \algcomment{coarse decoder.}
\State{}\textbf{Return:} coarse decoder $\fhatid$.
  \end{algorithmic}
  \caption{\getcoarsedecoder: Phase I of \richidce{} (\pref{sec:phase1}).}
  \label{alg:phase1}
\end{algorithm}

 }
    Let $(\matu_0\ind{i},\maty_0\ind{i},\matc_{0}\ind{i}), \ldots,
	(\matu_{\kapone}\ind{i},\maty_{\kapone}\ind{i},\matc_{\kapone}\ind{i}),\maty_{\kapone+1}\supi$ denote the $i$th trajectory
	gathered in this fashion.
        Following the template described in \pref{sec:bayes_pred}, we
        show that for the state distribution induced the control
        inputs above, the true decoder $\fstar$ can
	be recovered up to a linear transformation by solving a regression
	problem whose goal is to predict a sequence of control inputs
        from the observations at time $\kappa_1$.
        Define $\bv := (\bu_{\kapnot}^\top, \dots,
        \bu_{\kapone-1}^\top)^\top$. Our key lemma (\pref{lem:sysid_bayes_opt}) shows that 
	\begin{align}
\forall y \in \reals^{\dimy}, \quad 	      h_{\star}(y)  \coloneqq  \E[\bv \mid \by_{\kapone} = y] = \cC_{\kappa}^\top\Sigma_{\kapone}^{-1}\fst(y),
	          \label{eq:phase1_bayes}
	\end{align}
	\neurips{$\cC_{\kappa}\ldef[A^{\kappa-1}B \mid \dots  \mid B ]$; and $\Sigma_{\kapone} \ldef{} A^{\kappa_1} \Sigma_0 (A^{\kappa_1})^\top  +  \sum_{t=0}^{\kappa_1}
	A^{t-1}(\Sigw + BB^\top ) (A^{t-1})^\top$.}
	where we recall that $\cC_{\kappa}=[A^{\kappa-1}B \mid
        \dots  \mid B ]$ and define
	\begin{align*}
	\Sigma_{\kapone} &\ldef{} A^{\kappa_1} \Sigma_0 (A^{\kappa_1})^\top  +  \sum_{t=0}^{\kappa_1}
	A^{t-1}(\Sigw + BB^\top ) (A^{t-1})^\top.
	\end{align*}
	\iftoggle{neurips}{This lemma}{This lemma follows from the discussion in \Cref{sec:bayes_pred}, and} relies on perfect decodability and the fact
	that $\bv$ and $\bx_{\kappa_1}$ are jointly Gaussian. In particular, by verifying 
	$\nrm*{\cC_{\kappa}^\top\Sigma_{\kappa_1}^{-1}}_{\op}\leq\sqrt{\Psistar}$,
	the expression \eqref{eq:phase1_bayes} ensures that $h_{\star}$ belongs to the class
	$
	\Hid := \left\{ M  f(\cdot)\mid  f \in \Fclass,~~ M \in
	  \R^{\kappa\dimu \times \dimx}, ~~\|M\|_{\op} \le \sqrt{\Psistar}
	  \right\}.
          $ (i.e., we can take $r_{\id}=\sqrt{\Psist}$).
          The main step of Phase I solves the well-specified
          regression problem:
	\begin{align}\label{eq:h_id_solve}
		\hhatid \in \argmin_{h \in \Hid}\sum_{i=1}^{\nid} \|h(\bykapone\supi) -  \bv\supi\|_2^2.
	\end{align}
	\neurips{Phase I is computationally efficient whenever we have
	a \emph{regression oracle} for the induced function class $\Hid$. For many
	function classes of interest, such as linear functions and neural
	networks, solving regression over this class is no harder than
	regression over the original decoder class $\Fclass$, so we believe
	this is a reasonably practical assumption.}
	\arxiv{Phase I is computationally efficient whenever we have an appropriate regression oracle, as explained in \Cref{rem:oracle_efficiency}.}
	\arxiv{\par}
	For $\nid$ sufficiently large, a standard analysis for least
        squares shows that the regressor
	$\hhatid$ has low prediction error relative to
	$h_{\star}$ in \eqref{eq:phase1_bayes}. However, this representation is overparameterized and takes values in $\bbR^{\kappa\dimx{}}$ even though the true
	state lies in only $\dimx$ dimensions. For the second part of Phase I,
	we perform principle component analysis to reduce the dimension to
	$\dimx$.\footnote{This step is not strictly required, but
	  leads to tighter statistical analysis and more intuitive presentation.} Specifically, we compute a dimension-reduced decoder via
	  \begin{align}
	  \fhatid(y) := \Vhatid^\top \cdot \hhatid(y) \in \R^{\dimx},  \label{eq:fhatid}
	  \end{align}
	  where $\Vhatid \in
	\R^{\kappa\dimu \times \dimx}$ is an arbitrary orthonormal basis for
	                                               the top $\dimx$ eigenvectors of the empirical second moment matrix
	                                               $\sum_{i=\nid+1}^{2\nid}
	                                               \hhatid(\bykapone\supi) \hhatid(\bykapone\supi)^\top/\nid$. This approach exploits that the output of the Bayes regressor $\hstar$---being a linear function of the $\dimx$-dimensional system state---lies in a $\dimx$-dimensional subspace.
	\arxiv{\par}
        Having reviewed the two components of Phase I, we can now
        state the main guarantee for this phase. In light of
	\eqref{eq:phase1_bayes}, the result essentially follows from standard
	tools for least-squares regression with a
        well-specified model, plus an analysis for PCA with errors in variables.
	\begin{theorem}[Guarantee for Phase I]\label{thm:phase_one}
	If $\nid = \Omega_{\star}(\dimx\dimu \kappa (\log|\Fclass| + \dimu \dimx
	\kappa))$, then with probability at least $1-3\delta$, there exists an
	invertible matrix $\Sid \in \R^{\dimx\times\dimx}$ such that
	\begin{align*}
	\E\|\fhatid(\bykapone) - \Sid \fst(\bykapone)\|_2^2 \le \bigohs\prn*{\frac{ \dimu \kappa (\log|\Fclass| + \dimu \dimx \kappa) \ln^3(\nid/\delta)}{\nid}},
	\end{align*}
	and for which
	$ \sigma_{\min}(\Sid)\geq \sigidmin \coloneqq \sigma_{\min}(\contkap)(1-\gammastar)(4\Psistar^2\alphastar^2)^{-1}$
	and $\|\Sid\|_{\op} \le \sigidmax := \sqrt{\Psistar}$.
	\end{theorem}

\subsection{Phase II: System Identification}
\label{sec:phase2}
	In Phase II, we use the decoder from Phase I to learn the system
	dynamics, state cost, and process noise covariance up to the basis induced by the transformation $\Sid$. Our targets are:
		\begin{align}
		\Aid := \Sid A\Sid^{-1}, \quad \Bid:= \Sid B, \quad \Sigwid := \Sid \Sigw \Sid^\top, \quad \Qid := \Sid^{-\trn} Q \Sid^{-1}. \label{eq:ID_params}
		\end{align}
	        The key technique we use is to pretend that the decoder's output
	        $\fhatid(\maty_{\kapone})$ is the
	        true state $\matx_{\kapone}$, then perform regressions which mimic the dynamical equations \eqref{eq:dynamics}: 
	\begin{align}
	  &(\Ahatid,\Bhatid) \in \argmin_{(A,B)} \sum_{i=2\nid + 1}^{3\nid} \|\fhatid(\by_{\kapone+1}\supi) - A\fhatid(\bykapone\supi) - B\bu_{\kapone}\supi \|^2,\quad\text{and} \label{eq:recover_AB}\\
	  &\Sigwhatid = \frac{1}{\nid}\sum_{i=2\nid + 1}^{3\nid}
	    (\fhatid(\by_{\kapone+1}\supi) - \Ahatid\fhatid(\bykapone\supi) -
	    \Bhatid\bu_{\kapone}\supi )^{\otimes 2}, \quad \text{where $v^{\otimes 2} \ldef
	    vv^\top$.} \label{eq:recover_Sigmaw}
	\end{align}
	Similarly, we recover the state cost $Q$ by fitting a quadratic function to
	observed costs
		\begin{align}
	          \Qtilid \in  \argmin_{Q}\sum_{i=2\nid + 1}^{3\nid} \left(\bc_{\kapone}^{(i)} - (\bu^{(i)}_{\kapone})^\top R \bu^{(i)}_{\kapone} -  \fhatid(\bykapone^{(i)})^\top Q\fhatid(\bykapone^{(i)})\right)^2,\label{eq:Qhat_id}
		\end{align}
and then setting $\Qhatid = \left(\frac{1}{2}\Qtilid  +
	          \frac{1}{2}\Qtilid^\top\right)_{+}$ as the final estimator, where $(\cdot)_+$ truncates non-positive eignvalues to zero. This is the only place where the algorithm uses the cost
	        oracle.

                Since \pref{thm:phase_one} ensures that
	        $\fhatid(\maty_{\kapone})$ is not far from $\Sid\matx_{\kapone}$, the regression problems \eqref{eq:recover_AB}--\eqref{eq:Qhat_id} are all nearly-well-specified, and we have the following guarantee.
	        \newcommand{\vepsphaseii}{\veps_{\mathrm{id}}}
		\begin{restatable}[Guarantee for Phase II]{theorem}{theoremphasetwo}\label{thm:phase_two}
	          If $\nid = \Omega_{\star}\prn*{\dimx^2\dimu \kappa (\log|\Fclass| +
	            \dimu \dimx \kappa)\max\{1,\sigma_{\min}(\contkap)^{-4}\}
	          }$, then with probability at least $1 - 11\delta$ over
	          Phases I and II, 
			\begin{align}
	                  \|[\Ahatid;\Bhatid] - [\Aid;\Bid]\|_{\op}\vee \|\Qhatid - \Qid\|_{\op} 
	                  \vee\|\Sigwhatid - \Sigwid\|_{\op}
	                  \le \vepsphaseii, \label{eq:epsid}
			\end{align}
	                where $\vepsphaseii\le \bigohs\prn*{\nid^{-1/2}\ln^2(\nid/\delta)\sqrt{\dimx\dimu \kappa (\log|\Fclass| + \dimu \dimx \kappa)} }$. %
                      \end{restatable}
                      \arxiv{

\begin{algorithm}[H]
  \setstretch{1.1}
  \begin{algorithmic}[1]\onehalfspacing
  \State\textbf{Require:}
  \State{}~~~~Cost oracle to access the cost $\bc_t$ at time $t\geq 1$.
    \State\textbf{Inputs:}
     \Statex{}~~~~$\fhatid$ \algparen{coarse decoder}
    \Statex{}~~~~$\nid$ \algparen{sample size}
    \Statex{}~~~~$\kappa_0$ \algparen{burn-in time index}
    \Statex{}~~~~$\kappa$ \algparen{upper bound on the controllability index $\kappa_\star$}
          \smallskip

           \State Set $\kappa_1 = \kappa_0 + \kappa$.
          \State{}Gather $\nid$ trajectories by sampling control
          inputs $\matu_0,\ldots,\matu_{\kappa_1-1}\sim{}\cN(0,I_{\dimu})$.
    \State{}\textbf{Phase II:} \algcomment{Recover system
      dynamics and cost (see \pref{sec:phase2}).}
      \State{}Set $(\Ahatid,\Bhatid) \in \argmin_{(A,B)} \sum_{i=2\nid + 1}^{3\nid} \|\fhatid(\by_{\kapone+1}\supi) - A\fhatid(\bykapone\supi) - B\bu_{\kapone}\supi \|^2$.
  \State{}Set $\Sigwhatid = \frac{1}{\nid}\sum_{i=2\nid+1}^{3\nid}
	    (\fhatid(\by_{\kapone+1}\supi) - \Ahatid\fhatid(\bykapone\supi) -
	    \Bhatid\bu_{\kapone}\supi )^{\otimes 2}$, where $v^{\otimes 2} \ldef
	    vv^\top$.
 \State{}Set $\Qtilid = \min_{Q}\sum_{i=2\nid+1}^{3\nid} \left(\bc_{\kapone}^{(i)} - (\bu^{(i)}_{\kapone})^\top R \bu^{(i)}_{\kapone} -  \fhatid(\bykapone^{(i)})^\top Q\fhatid(\bykapone^{(i)})\right)^2$.
  \State{}Set $\Qhatid = \left(\frac{1}{2}\Qtilid  +\frac{1}{2}\Qtilid^\top\right)_{+}$, where $(\cdot)_+$ truncates all negative eigenvalues to zero.
\State{}\textbf{Return:} system and cost matrices $(\Ahatid,\Bhatid, \Sigwhatid ,\Qhatid)$.
  \end{algorithmic}
  \caption{\sysid: Phase II of \richidce{} (\pref{sec:phase2}).}
  \label{alg:phase2}
\end{algorithm}

 }
	        To simplify presentation, we assume going forward that
	        $\Sid=I_{\dimx}$, which is without loss of generality (at the cost of increasing parameters
	        such as $\Psistar$ and $\alpha_\star$ by a factor of
                $\nrm*{\Sid}_{\op}\vee\nrm{\Sid^{-1}}_{\op}$),\footnote{The controller $\Sid K_{\infty}$ attains the same performance on $(\Aid,\Bid)$ as $K_{\infty}$ on $(A,B)$} and drop the ``$\mathrm{id}$''
                subscript on the estimators $\Ahatid$, $\Bhatid$, and so forth to reflect this.\footnote{We
                  make this reasoning precise in the proof of
                  \pref{thm:main}.}

\subsection{Phase III: Decoding Observations Along the Optimal
          Path}

      \label{sec:phase3}
      \newcommand{\stepref}[1]{\hyperlink{step#1}{\small\texttt{\textbf{Step #1}}}}
      \newcommand{\stepnum}[1]{\hypertarget{step#1}{{\color{blue!70!black}\small\texttt{\textbf{Step #1}}}}}

        Given the estimates $(\what A,\what B,\what Q)$ from
        \pref{thm:phase_two}, we can use certainty equivalence to synthesize an optimal controller matrix $\Khat$ for
        the estimated dynamics. As long as $\vepsphaseii$ in \eqref{eq:epsid} is
        sufficiently small, the policy $\matu_t = \Khat \matx_t$  is
        stabilizing and near optimal.
                 \vspace{5pt}

        To (approximately) implement this policy from rich observations, it remains to accurately estimate the latent state. The decoder learned in
         Phase I does not suffice; it only ensures low error on trajectories generated with random Gaussian inputs, and not on the trajectory induced by the near-optimal
         policy. Indeed, while it is tempting to imagine that the initial
         decoder $\fhat$ might generalize across different
         trajectories, this is not the case in unless we place strong
         structural assumptions on $\Fclass$. 
         \vspace{5pt}
         
         Instead, we iteratively learn a sequence of
        decoders $\fref_t$---one per timestep $t=1,\ldots,T$. Assuming $\Khat \approx \Kinf$ is near optimal, the suboptimality $\cost(\pi) - \cost(\piinf)$ of the policy $\pi(\by_{0:t})\coloneqq \Khat\fref_{t}(\maty_{0:t})$ is controlled by the sum $\sum_{t=1}^{T}\En_{\pi}\nrm[\big]{\fref_t(\maty_{0:t})-\fstar(\maty_t)}^2_2$.\footnote{Note that regret does not take into account step $0$.} Thus, to ensure low regret, we ensure that, for all $t\geq 1$, the  decoder $\fref_{t}$ has low prediction
        error \emph{on the distribution induced by running $\pi$} with
        previous decoders $(\fref_\tau)_{1\leq \tau <t}$ and $\Khat$. This motivates the
        following iterative decoding procedure, executed for each time step $t=1,\ldots,T$:
        \neurips{\begin{enumerate}[topsep=0pt] }
        \arxiv{\begin{enumerate}[leftmargin=37pt] }
        \item[\stepnum{1}.] Collect $2n_{\onpo}$ trajectories by executing the randomized control input $\bu_\tau=\what K \hat f_\tau(\by_{0:\tau}) + \bnu_\tau$, for $0\leq{}\tau\leq t$, and $\bu_\tau=\bnu_{\tau}$, for $t< \tau < t+\kappa$, where $\bnu_\tau \sim\cN(0,\sigma^{2}I_{\dimu})$; here, $n_{\onpo} \in \mathbb{N}$ and $\sigma^2\leq 1$ are algorithm parameters to be specified later. \label{step:one}
        \item[\stepnum{2}.] Obtain a \emph{residual decoder}
          $\hat{h}_{t}$ satisfying \eqref{eq:gcheck_bayes} by solving
          regressions \eqref{eq:gtk_simple} and \eqref{eq:gt_simple}
          using a regression oracle.
        \item[\stepnum{3}.] Form a state decoder $\fref_{t+1}$ from $\hat{h}_{t}$
          and $\fref_{t}$ using the update equation \eqref{eq:decoder_simple}.
        \end{enumerate}
    
        Forming the decoder $\fhat_1$ requires additional regression
        steps (described in \pref{app:phaseiii}) which account for the
        uncertainty in the initial state $\matx_0$. At each subsequent
        time $t$, the most important part of the procedure above is
        \stepnum{2}, which aims
to produce a regressor $\hat{h}_t$ such that
\begin{equation}
  \hat{h}_t(\maty_{t+1})-A\hat{h}_t(\maty_t)  \approx{}B\bu_t+\matw_t  =\matx_{t+1}-A\matx_t.
  \label{eq:gcheck_bayes}
\end{equation}
As we shall see, enforcing accuracy on the \emph{increments} $\matx_{t+1}-A\matx_t $ allows us to set up regression problems which do not depend on, and thus \emph{do not} propagate forward, the errors in $\fhat_t$. In contrast, a naive regression---say, $\argmin_f \E\left[\|f(\by_{t+1}) - ( A +  B \what K) \fhat_t(\by_t) -  B\bnu_t\|^2_2\right]$---could compound decoding errors exponentially in $t$.

Luckily, the increments in \eqref{eq:gcheck_bayes} are sufficient for
recovery of the state by unfolding a recursion; this comprises
\stepnum{3}. Let $\bclip>0$ be an algorithm parameter. Given a regressor $\hat{h}_t$ satisfying \eqref{eq:gcheck_bayes} and the current decoder $\hat{f}_t$, we form next state decoder $\hat f_{t+1}$ via
\begin{align}
  \fref_{t+1}(\cdot) \ldef \freftil_{t+1}(\cdot)\I\{\|\freftil_{t+1}(\cdot)\|_2\le \bclip\},~~ \text{and} ~~~  \freftil_{t+1}(\by_{0:t+1}) \ldef \left(\hat{h}_t(\by_{t+1})  -
  \what A  \cdot\hat{h}_t(\by_{t})\right) + \what A\cdot \fref_t(\by_{0:t}),\label{eq:decoder_simple}
\end{align}
where we set $\freftil_0\equiv\fref_0\equiv 0$. By clipping
$\freftil_t$, we ensure states remain bounded, which simplifies the
analysis. Crucially, by building our decoders $(\hat f_\tau)$ this
way, we ensure that the decoding error grows at most linearly in
$t$---as opposed to exponentially---as long as the system is stable
(i.e. $\rho(A)< 1$), as assumed.

It remains to describe how to obtain a regressor $\hat{h}_t$
satisfying \eqref{eq:gcheck_bayes}. To this end, we use the added Gaussian noise $\bnu_t$ to set up the regression. 

\paragraph{Warm-up: Invertible $B$.}  As a warm-up, suppose that $B$ is invertible. Then, for the matrix $M_1 :=  B^\top ( BB^\top + \sigma^{-2}\Sigw)^{-1}$, one can compute
\begin{align}
  \En\brk*{\bnu_{t}\mid{}\maty_{0:t+1}} \stackrel{(*)}{=} \En\brk*{\bnu_{t}\mid{}\bw_t +B\bnu_t} & \stackrel{(**)}{=}M_1 (\bw_t + B\bnu_t)  = M_1(\matx_{t+1}  - A\matx_t - B\Khat\fhat_t(\by_{0:t})). \label{eq:decoder_simple_bayes}
\end{align}
Following the discussion in \Cref{sec:bayes_pred}, the
  identity $(*)$ uses the fact that conditioning on $\maty_{0:t+1}$ is
  equivalent to conditioning on $\matx_{0:t+1}$, due to perfect
  decodability. However, unlike in \pref{sec:bayes_pred},
  $\matx_{0:t+1}$ are \emph{not} jointly Gaussian, because errors in
  the decoder may yield non-Gaussian control input. Instead, we use
  that the conditional distribution of $\bnu_{t} \mid \matx_{0:t+1}$
  is equivalent to $\bnu_t \mid
  \matx_{t},\matx_{t+1}$, since $\matx_{t+1} = A\matx_t + B \matu_t $. The equality $(**)$ uses a general formula for Gaussian conditional expectations, also described in \Cref{sec:bayes_pred}.

Since conditional expectations minimize the square loss, learning a
residual regressor $\hat h_t$ which approximately minimizes \begin{align}
h\mapsto\E\left[\|\bnu_t -  M_1( h(\maty_{t+1})  - Ah(\maty_t) - B\Khat\fhat_t(\by_{0:t}))\|^2_2\right] \label{eq:reg}
\end{align} 
produces a decoder $\hhat_{t+1}$ approximately satisfying \eqref{eq:gcheck_bayes}:
\begin{align}
                              M_1 (\hhat_t(\by_{t+1}) - A\hhat_t(\by_{t+1})) &\approx M_1 (\bx_{t+1} - A\bx_{t+1}), \label{eq:approx}
                                                                               \intertext{since}
                                                                               M_1 (\hhat_t(\by_{t+1}) - A\hhat_t(\by_{t+1}) -
  B\Khat\fhat_t(\by_{0:t})) &\approx M_1 (\bx_{t+1} - A\bx_{t+1} -
                              B\Khat\fhat_t(\by_{0:t})).\nn
\end{align}
For invertible $B$, the matrix $M_1$ is invertible, and so from
\eqref{eq:approx}, our state decoder $\hat h_{t+1}$ indeed satisfies
\eqref{eq:gcheck_bayes}: $\hat{h}_t(\maty_{t+1})-A\hat{h}_t(\maty_t)
\approx{}\matx_{t+1}-A\matx_t$. We emphasize that regressing to purely
Gaussian inputs $\bnu_t$ is instrumental in ensuring the conditional
expectation equality in \eqref{eq:decoder_simple_bayes} holds. The noise variance $\sigma^2$ trades off between the conditioning of the regression, and the excess suboptimality caused by
noise injection; we choose it so that the final suboptimality is $\bigohs(\eps)$.

\paragraph{Extension to general controllable systems.} For non-invertible $B$, we  aggregate more regressions. For $k \in [\dimk]$, let \neurips{$M_k \coloneqq \cC_{k}^\top ( \cC_k  \cC^\top_k + \sigma^{-2} \sum_{i=0}^{k}  A^{i-1} \Sigma_w  (A^{i-1})^\top)^{-1}$,}\arxiv{\[M_k \coloneqq \cC_{k}^\top \prn*{ \cC_k  \cC^\top_k + \sigma^{-2} \sum_{i=0}^{k}  A^{i-1} \Sigma_w  (A^{i-1})^\top}^{-1},\]} where we recall $\cC_{k}$ from \Cref{ass:controllability}. Generalizing \eqref{eq:decoder_simple_bayes}, we show (\pref{lem:firstregression} in \pref{sec:phase3_proofs}) that the outputs $(\maty_\tau)$ and the Gaussian perturbation vector $\bnu_{t:t+k-1} \coloneqq (\bnu_t^\top, \dots, \bnu_{t+k-1}^\top)^\top$ generated according to \stepref{1} above satisfy, for all $k\in[\dimk]$,
\begin{align}
  \En\brk*{\bnu_{t:t+k-1}\mid{}\maty_{0:t},\maty_{t+k}} %
  &= M_k(\matx_{t+k}-A^{k}\matx_t-A^{k-1}B\fref_{t}(\maty_{0:t})) =: \phi^\star_{t,k}(\maty_{0:t+k}).\label{eq:bayes1}   %
\end{align}
Defining concatentations $\upphi^\star_{t} := (\phi^\star_{t,1},\dots,\phi^\star_{t,\kappa})$ and $\calM \coloneqq [M_1^\top,(M_2A)^{\top}, \dots, (M_\dimk
A^{\dimk-1})^\top]^\top$ and stacking the conditional expectations gives:
\begin{align}
  \En\brk*{\upphi^\star_{t}(\maty_{0:t+\kappa})\mid{}\maty_{0:t+1}} & =\cM(B\bnu_t+\matw_t)= \cM(f_\star(\by_{t+1})-Af_\star(\by_t)-B \what K\fref_{t}(\maty_{0:t})). \label{eq:bayes2}
\end{align}
Hence, with infinite samples (and knowledge of $B$), we are able
to recover the residual quantity
$\cM(f_\star(\by_{t+1})-Af_\star(\by_t))$. Again, the Gaussian inputs
enable the conditional expectations \eqref{eq:bayes1} and
\eqref{eq:bayes2}. The crucial insight for the {stacked} regression is
that by rolling in and switching to pure Gaussian noise only \emph{after} time $t$, we maintain gaussianity, while still yielding decoders that are valid on-trajectory \emph{up to} time $t$. To ensure that we accurately recover the increment $f_\star(\by_{t+1})-Af_\star(\by_t)$, we require the 
overdetermined matrix $\cM$ to be invertible. To facilitate this, let
$\cM_{\sigma^{2}}$ denote the value of $\cM$ as a function of
$\sigma^{2}$, and let
\begin{equation}
  \cMbar=\lim_{\sigma\to{}0}\cM_{\sigma^{2}}/\sigma^{2} \label{eq:calM}
\end{equation}
be the (normalized) limiting matrix as noise tends to zero, which is an
intrinsic problem parameter.
\begin{assumption}
  \label{ass:m_matrix}
  The limiting matrix $\cMbar$ satisfies $\lambda_{\cM}\coloneqq  \eigmin^{1/2}(\cMbar^{\trn}\cMbar)>0$.
      \end{assumption}
\neurips{
This assumption is central to the analysis, and we believe it is reasonable: we are guaranteed that it holds if the system
is controllable and either $A$ or $B$ has full column
rank. See \pref{ssec:cM_asm} for further discussion.
}
\arxiv{
This assumption is central to the analysis, and we believe it is reasonable: it holds whenever either $A$ or $B$ is full rank and the system is
controllable. However, we are interested to understand if there are other more transparent conditions under which our recovery guarantees hold.
}

To approximate the conditional expectations \eqref{eq:bayes1},
\eqref{eq:bayes2} from finite samples, we define another expanded
function class $$\scrH_{\onpo}\ldef{}\crl{M f(\cdot)\mid{}f\in\Fclass,  M \in \reals^{\dimx \times \dimx} ,\|M\|_{\op} \leq
  \Psi_\star^{3}},$$ and use $(\Mhat_k)$ and $\what \cM$ to denote
plugin estimates of $(M_k)$ and $\cM$, respectively, constructed from
$\Ahat$ and $\Bhat$. Here, the subscript  ``op'' subscript on $\Hclass_{\onpo}$ abbreviates ``on-policy''.

Next, given a state decoder $\hat f_t$ for time $t$ and $k \in[\dimk]$, we define $$\what\phi_{t,k}(h, \by_{0:t}, \by_{t+k}) \coloneqq \what M_{k}  \left(h(\by_{t+k})-\what A^{k}  h(\by_t) - \what A^{k-1}\what B \what K \fref_t(\by_{0:t})\right)\quad \text{for $h \in  \scrH_{\onpo}$.} $$ With this and the $2 n_{\onpo}$ trajectories $\{(\by_\tau^{(i)}, \bnu_{\tau}^{(i)})\}_{1\leq i \leq 2n_{\onpo}}$ gathered in \stepref{1} above, we obtain $\hat{h}_t$ by solving the following two-step regression:
  \begin{gather}
\hat{h}_{t,k} \in \argmin_{h \in  \scrH_{\onpo}}
\sum_{i=1}^{n_{\onpo}} \left\|\what\phi_{t,k}(h, \by^{(i)}_{0:t},
  \by^{(i)}_{t+k}) -
  \bnu^{(i)}_{t:t+k-1}\right\|_2^2,\;\;\;\forall{}k\in\brk{\kappa},\label{eq:gtk_simple}\intertext{followed
by}
\hat{h}_{t} \in \argmin_{h \in  \scrH_{\onpo}} \sum_{i=n_{\onpo}+1}^{2n_{\onpo}} \left\|\what \calM \cdot \left(h(\by^{(i)}_{t+1}) - \what A \cdot h(\by^{(i)}_{t}) - \what B \what K \cdot \fref_t (\by_{0:t}^{(i)})  \right) - \what\upphi_t(\by_{0:t+\dimk}^{(i)})  \right\|_2^2, \label{eq:gt_simple}
\intertext{where} \what\upphi_t(\by_{0:t+\dimk})\coloneqq [\what
\phi_{t,1}(\hat{h}_{t,1},\by_{0:t}, \by_{t+1})^\top, \dots,
\what\phi_{t,\dimk}(\hat{h}_{t,\dimk},\by_{0:t},
\by_{t+\dimk})^\top]^\top \in \reals^{(1+\dimk)\dimk \dimu/2}.\label{eq:phi_simple}
\end{gather}
We see that the first regression approximates \eqref{eq:bayes1}, while the
second approximates \eqref{eq:bayes2}.
	 We can now
	state the guarantee for Phase III.

\begin{theorem}
  \label{thm:phase_three}
  Suppose $\vepsid^{2} \leq{}
\bigohch((\log\abs*{\Fclass}+\dimx^2)\nref^{-1})$.
  If we set $\bclip^2=\Theta_{\star}((\dimx+\dimu)\log(\nref))$, $r_{\op}=\Psist^{3}$, and $\sigma^2=\bigohch(\lambdam)$, we are guaranteed that for any $\delta\in(0,1/e]$, with probability at least $1-\bigoh(\kappa{}T\delta)$,
	\begin{align}
	\E_{\what \pi} \left[ \max_{ 1\leq  t \leq T} \|
          \fref_t(\by_{0:t}) - \fstar(\by_{t}) \|_2^2   \right]   
          \leq
  \bigohs\prn*{
\frac{\lambdam^{-2}}{\sigma^{4}}\cdot{}T^3\kappa^2(\dimx+\dimu)^4\cdot\frac{(\dimx^2+\log\abs{\Fclass}) \log^{5}(\nref/\delta)}{\nref}
            }.
      \end{align}
\end{theorem}
To obtain \pref{thm:main}, we combine \pref{thm:phase_two} and
\pref{thm:phase_three}, then appeal to
\pref{thm:performance_difference} (\pref{sec:linear_control}), which
bounds the policy suboptimailty in terms of regression errors. Finally, we set $\sigma\propto\veps$ so that the suboptimality due to adding the Gaussian noise
$(\bnu_t)_{0\leq t\leq T}$ is low. See \pref{sec:main_proof} for details.
 \arxiv{
\subsection{Learning the Initial State}
\label{app:phaseiii}
\neurips{In this section, we give an overview for}\arxiv{We now overview} how Phase III of \richid{} (\pref{alg:phase3}) learns a predictor for the initial state $\bx_0$; this is an edge case \arxiv{not discussed above}\neurips{which is not discussed in the main body due to space limitation}, and comprises \pref{line:start} through \pref{line:ned} in \pref{alg:phase3}. \neurips{This discussion supplements \pref{sec:phase3} of the main body, and together these sections give constitute our high-level overview of Phase III.}

If we ignore the clipping in \eqref{eq:decoder_simple}, the state decoders $(\hat f_\tau)_{t\geq 2}$ are defined through the recursion: 
\begin{align}
\hat f_{t+1}(\by_{0:t+1}) = \hat{h}_t(\by_{t+1})  -
	\what A  \hat{h}_t(\by_{t}) + \what A \fref_{t}(\by_{t}), \quad \text{for $t\geq 1$},\nn
	\end{align}
        which means that all the decoding error for any $t$ will depend on the error of the decoder $\hat f_1$ for the state $\matx_1$. To ensure that $\hat f_1$ is accurate, we need to somehow learn to decode the inital state $\bx_0$, which we recall is assumed to be distributed as $\cN(0,\Sigma_0)$. The challenge here is that the covariance matrix $\Sigma_0$ is unknown, and we need to estimate it in order to ``back out'' the initial state through the approach in \pref{sec:bayes_pred}. This is achieved by \pref{line:start} through \pref{line:ned} of \pref{alg:phase3}, which we explain in detail below. Briefly, the idea is that since $\matx_1=A\matx_0+B\matu_0+\matw_0$, to accurately predict $\matx_1$ it suffices to have good predictors for $\matw_0$ and $A\matx_0$. We can learn a predictor for $\matw_0$ in the same fashion as for all the other timesteps, and most of the work in \pref{line:start} through \pref{line:ned} is to learn a regression function $\hat{f}_{A,0}$ that accurately predicts $A\matx_0$.

\arxiv{
  \begin{algorithm}[htp]
  \setstretch{1.1}
    \begin{algorithmic}[1]\onehalfspacing
  	\State\textbf{Inputs:} $(\what A,\what B, \what \Sigma_w,\what Q,R)$ \algparen{estimates for the system parameters and cost matrices}
        \State\textbf{Parameters:}
  	\Statex{}~~~~~~~~~~~~$n_\onpo$ \algparen{proportional to the sample size}
  	\Statex{}~~~~~~~~~~~~$\kappa$ \algparen{upper-bound on the controllability index $\kappa_\star$}
  	\Statex{}~~~~~~~~~~~~$\sigma^2$ \algparen{exploration parameter}
  		\Statex{}~~~~~~~~~~~~$\bclip$ \algparen{clipping parameter for the decoders}
  		\Statex{}~~~~~~~~~~~~$r_{\op}$ \algparen{parameter to define
                  the function class}
                \smallskip
  		\State{}Set $\scrH_\onpo= \left\{M  f(\cdot) \mid f \in \Fclass,\  M \in
  		\R^{\dimx \times \dimx}, ~\|M\|_{\op} \le r_{\op}
  		\right\}.$
  		\State{}Set $n_\init = n_\op$.
  		\State{}\textbf{Phase III:} \algcomment{Learn on-policy decoders (see \pref{sec:phase3} and \pref{app:phaseiii}).}
  		\State Set $(\what P,\what K):= \darece(\what A,\what B,\what Q, R)$ (\pref{defn:darece}).
  		\For{$k=1, \dots, \kappa$}
  		\State Set $\what \cC_k = [\what A^{k-1} \what B \mid \dots \mid \what B]$. 
  		\State Set $\what M_k \coloneqq \what\cC_{k}^\top (\what \cC_k  \what\cC^\top_k + \sigma^{-2} \sum_{i=0}^{k}  \what A^{i-1} \what \Sigma_w  (\what A^{i-1})^\top)^{-1}$.
  		\EndFor
  		\State Set $\what \cM = [\what M_1,\dots, (\what M_\dimk  \what A^{\dimk -1})^\top  ]^\top$.
  		\State Define $\hat f_0(y_0) = 0$ for all $y_0\in\cY$.\label{line:fhat0}
  		\For{$t=0,\ldots,T-1$}
  		\State{}Collect $2n_{\onpo}$ trajectories by executing the randomized control input $\bu_\tau=\what K \hat f_\tau(\by_{0:\tau}) + \bnu_\tau$,\\ \hskip\algorithmicindent \hskip\algorithmicindent for $0\leq{}\tau\leq t$, and $\bu_\tau=\bnu_{\tau}$, for $t< \tau < t+\kappa$, where $\bnu_\tau \sim\cN(0,\sigma^{2}I_{\dimu}).$
  		\For{$k=1, \dots, \kappa$}
  		\State Set $\hat{h}_{t,k} \in \argmin_{h \in  \scrH_\onpo} \sum_{i=1}^{n_{\onpo}} \left\|\what\phi_{t,k}(h, \by^{(i)}_{0:t}, \by^{(i)}_{t+k}) - \bnu^{(i)}_{t:t+k-1}\right\|_2^2$, 
  		\Statex{}\hskip\algorithmicindent \hskip\algorithmicindent \hskip\algorithmicindent where $\what\phi_{t,k}(h, \by_{0:t}, \by_{t+k}) \coloneqq \what M_{k}  \left(h(\by_{t+k})-\what A^{k}  h(\by_t) - \what A^{k-1}\what B \what K \fref_t(\by_{0:t})\right)$.
  		\State Set $\hat{h}_{t} \in \argmin_{h \in  \scrH_\onpo} \sum_{i=n_{\onpo}+1}^{2n_{\onpo}} \left\|\what \calM  \left(h(\by^{(i)}_{t+1}) - \what A  h(\by^{(i)}_{t}) - \what B \what K  \fref_t (\by_{0:t}^{(i)})  \right) - \what\upphi_t(\by_{0:t+\dimk}^{(i)})  \right\|_2^2,$\label{line:hhat_t}
  		\Statex{}\hskip\algorithmicindent \hskip\algorithmicindent \hskip\algorithmicindent where $ \what\upphi_t(\by_{0:t+\dimk})\coloneqq [\what
  		\phi_{t,1}(\hat{h}_{t,1},\by_{0:t}, \by_{t+1})^\top, \dots,
  		\what\phi_{t,\dimk}(\hat{h}_{t,\dimk},\by_{0:t},
  		\by_{t+\dimk})^\top]^\top$.
  		\EndFor
  		\If{$t=0$}~~\algcomment{Initial state learning phase (\pref{app:phaseiii}).}
  		\State Collect $2 n_{\init}$ trajectories by executing the control input $\bu_\tau= \bnu_\tau$, for $0\leq \tau < \kappa$, \label{line:start} \\ \hskip\algorithmicindent \hskip\algorithmicindent \hskip\algorithmicindent where $\bnu_\tau \sim\cN(0,\sigma^{2}I_{\dimu})$. 
  		\State Set $\hat{h}_{\ol,1} \in \argmin_{h \in \scrH_{\onpo}} \sum_{i=1}^{n_\init} \left\| h(\by^{(i)}_1)  - \left(\hat{h}_{0}(\by^{(i)}_1) - \what A \hat{h}_{0}(\by^{(i)}_0) -  \what B \bnu_0^{(i)}\right)  \right\|_2^2$. \label{line:start+1} 
  		\State Set $\what \Sigma_\cv \coloneqq  \frac{1}{n_\init}\sum_{i=n_\init +1}^{2 n_\init} \hat{h}_{\ol,1}(\by_1^{(i)})  \hat{h}_{\ol,1}(\by_1^{(i)})^\top. $\label{line:start+2} 
  		\State Set $\tilde{h}_{\ol,0} \in \argmin_{h \in \scrH_{\onpo}} \sum_{i=n_\init +1}^{2n_\init} \left\|h(\by_0^{(i)})  -    \hat{h}_{\ol,1}(\by^{(i)}_1) \right\|_2^2$.\label{line:start+3} 
  		\State Set $\hat{f}_{A,0}(\by_0)= \what \Sigma_w \what \Sigma_\cv^{-1} \tilde{h}_{\ol,0} (\by_0)$.\label{line:start+4} 
  		\State{}Set $\freftil_{1}(\by_{0:1}) = \hat{h}_0(\by_{1})  -
  		\what A  \hat{h}_0(\by_{0}) +  \fref_{A,0}(\by_{0}).$ \label{line:ned}
  		\Else
  		\State{}Set $\freftil_{t+1}(\by_{0:t+1}) = \hat{h}_t(\by_{t+1})  -
  		\what A  \hat{h}_t(\by_{t}) + \what A \fref_{t}(\by_{t}).$
  		\EndIf
  		\State{}Set $\hat f_{t+1}(\by_{0:t+1}) = \tilde f_{t+1}(\by_{0:t+1}) \Ind\{\|\tilde f_{t+1}(\by_{0:{t+1}})\|_2 \leq \bclip \}$.\label{line:clip}
  		\State{}Set controller 
  		$\pihat_{t+1}(\maty_{0:t+1})=\Khat\fref_{t+1}(\maty_{0:t+1})+\bnu_{t+1}$, with $\bnu_{t+1}\sim\cN(0,\sigma^{2}I_{\dimu})$.
  		\EndFor
  		\State{}\textbf{Return:} Controller $\pihat=(\pihat_{t})_{t=1}^{T}$.
  	\end{algorithmic}
  \caption{\computepol: Phase III of \richidce}
  \label{alg:phase3}
\end{algorithm} 

 }
        
To begin, in \pref{line:start} we execute Gaussian control inputs $\bnu_{\tau}$ for $0\leq{}\tau<\kappa$. We then proceed as follows.
        
\paragraph{\pref{line:start+1}.} As we show in \pref{thm:monster} (\pref{sec:phase3_proofs}), $\hat{h}_{t}(\by_{t+1}) - \what A \hat{h}_{t}(\by_t) -  \what B ( \what K \hat f_t(\by_{0:t}) + \bnu_t)$ approximates the system's noise $\bw_t$, for $t\geq 0$. In particular, since $\hat f_0 \equiv 0$ by definition (\pref{line:fhat0}), $\hat{h}_{0}(\by_{1}) - \what A \hat{h}_{0}(\by_0) -  \what B \bnu_0$ approximates the noise $\bw_0$. Since we have $\matx_1=A\matx_0+B\matu_0+\matw_0$, it remains to get a good estimator for $A\matx_0$. To this end, we observe that the predictor $\hat h_{\ol,1 }$ in \pref{line:start+1} is (up to a generalization bound) equal to
\[
  \argmin_{h\in\Hclass_{\onpo}}\E_{\what \pi} \left[ \| h(\by_1) - \bw_0\|^2\right],
\]
which we show---under the realizability assumption---is given by
\begin{align}
	\E \left[ \bw_0 \mid \by_1  \right] = 	\E\left[ \bw_0 \mid \bx_1  \right] = \Sigma_w (\sigma^2 BB^\top +  \Sigma_w + A\Sigma_0 A^\top)^{-1} (A \bx_0 + B \bnu_0 + \bw_0), \label{eq:closed}
	\end{align}
where the last equality---like the rest of our Bayes characterizations---follows by \pref{fact:gaussian_expectation}.

\paragraph{\pref{line:start+2}.} Now, given that $\hat h_{\ol,1}(\by_1)\approx \E[\bw_0\mid \by_1]$, one can recognize that the matrix $\what \Sigma$ in \pref{line:start+2} is an estimator for the matrix 
\begin{align}
	\Sigma_w (\sigma^2 BB^\top +  \Sigma_w + A\Sigma_0 A^\top)^{-1}\Sigma_w.\label{eq:mat}
\end{align}
In particular, even though we cannot recover the covariance matrix $\Sigma_0$, the estimator $\what \Sigma$ gives a means to predict $A \bx_0$,  leading to an accurate decoder $\hat f_1$.
\paragraph{\pref{line:start+3}.} Since $\hat h_{\ol,1}(\by_1)$ accurately predicts $\E[\bw_0 \mid \by_1]$ (whose closed form expression we recall is given by the RHS of \eqref{eq:closed}), the predictor $\tilde h_{\ol,0}$ in \pref{line:start+3} can be seen to approximate
\begin{gather}
  \argmin_{h\in\Hclass_{\onpo}}\E_{\what \pi}\left[\| h(\by_0)-\Sigma_w (\sigma^2 BB^\top +  \Sigma_w + A\Sigma_0 A^\top)^{-1} (A \bx_0 + B \bnu_0 + \bw_0) \|^2 \right],\nn
	\shortintertext{which (under realizability) is simply}
	\E[\Sigma_w (\sigma^2 BB^\top +  \Sigma_w + A\Sigma_0 A^\top)^{-1} (A \bx_0 + B \bnu_0 + \bw_0)  \mid \bx_0]= \Sigma_w (\sigma^2 BB^\top +  \Sigma_w + A\Sigma_0 A^\top)^{-1} A \bx_0. \label{eq:condexp} 
	\end{gather}
\paragraph{\pref{line:start+4,line:ned}.} In light of \eqref{eq:condexp} and the fact that $\what \Sigma$ is an estimator of the matrix in \eqref{eq:mat}, we are guaranteed that $\what \Sigma_w \what \Sigma^{-1} \hat h_{\ol,0}(\by_0)$ accurately predicts $A \bx_0$, which motivates the updates in \pref{line:start+4,line:ned}.

 }

\arxiv{
\section{Extensions \label{ssec:relax_asm}}

\paragraph{Relaxing the stability assumption.} We believe that our
algorithm can be extended to so-called \emph{marginally stable}
systems, where $\rho(A)$ can be as large as $1$ (rather than strictly
less than $1$). In such systems, there exist system-dependent
constants $c_1,c_2 > 0$ for which $\|A^n\|_{\op} \le c_1 n^{c_2}$ for
all $n$. In general, these constants may be large, and in the worst
case $c_2$ may be as large as $\dimx$ (or, more generally, the largest
Jordan block of $A$); see, e.g., \citet{simchowitz2018learning} for
discussion. Nevertheless, if $c_1,c_2$ are treated as problem
dependent constants, we can attain polynomial sample complexity. The
majority of \pref{alg:main} can remain as-is, but the analysis will
replace the geometric decay of $A$ with the polynomial growth bound
above. This will increase our sample complexity by a $\poly(c_1
T^{c_2})$ factor, where $T$ is the time horizon.

The only difficulty is that we can no longer directly identify the
matrices $A$ and $B$ in Phase II. This is because our current analysis
uses the \emph{mixing} property of $A$, which entails that if $\rho(A)
< 1$, then for $t$ sufficiently large, under purely Gaussian
inputs $\matx_{t}$ and $\matx_{t+1}$ have similar
distributions. This ensures that predictors learned at time $t$ are similar to
those at time $t+1$. However, this is no longer true if $\rho(A) =
1$. To remedy this, we observe that it is still possible to recover
the controllability matrix $[B; AB; A^2 B;\dots; A^{k-1}B]$ from the
regression problem in Phase I up to a change of basis (see,
e.g., \citet{simchowitz2019learning} for guarantees for learning such
a matrix in the marginally stable setting). We can then recover the matrices $A$ and
$B$ from the controllability matrix up to orthogonal transformation
using the Ho-Kalman procedure (see \citet{oymak2019non} or \citet{sarkar2019finite} for refined guarantees). 
\paragraph{Relaxing the controllability assumption.}
If the system is not controllable, then we may not be able to recover
the state exactly. Instead, we can recover the state up to the
limiting-column space of the matrices $(\contk)$, which is always
attained for $k \leq \dimx$. We can then use this to run a weaker
controller (e.g., an observer-feedback controller) based on observations of the projection of the state onto this subspace. 
\paragraph{Other extensions.}
The assumption on the growth rate for $\Fclass$ can be replaced with
the bound $\|f(y)\| \le L\max\{1,\|\fst(y)\|^p\}\;\forall{}f\in\Fclass$ for any $p \ge 1$, at the expense of degrading the final sample complexity.
 }

\section{Discussion}
\label{sec:discussion}

We introduced \richid, a new algorithm for sample-efficient
continuous control with rich observations.
We hope that our work will serve as a starting
point for further research into sample-efficient continuous control
with nonlinear observations, and we are excited to develop the
techniques we have presented further, both in theory and practice. To this end, we list a
few interesting directions and open questions for future work.

\begin{itemize}[topsep=0pt]
\item While our results constitute the first polynomial
  sample complexity guarantee for the \richlqr, the sample complexity
  can certainly be improved. An important problem is to
  characterize the fundamental limits of learning in the \richlqr and
  design algorithms to achieve these limits, which may require new techniques.
  Of more practical importance, however, is to remove various
  technical assumptions used by \richid. We believe the most
  important assumptions to remove are (I) the assumption that the open-loop system is stable
  (\pref{ass:stability}), which is rarely satisfied in practice; and (II)
  the assumption that process noise is Gaussian, which is currently
  used in a rather strong sense to characterize the Bayes optimal solutions to the regression problems
  solved in \richid.
\item \richidce is a model-based reinforcement learning algorithm. We
  are excited at the prospect of expanding the family of algorithms
  for \richlqr to include provable model-free and direct policy
  search-based algorithms. It may also be interesting to develop algorithms with
  guarantees for more challenging variants of the \richlqr, including
  regret rather than PAC-RL, and learning from a single trajectory
  rather than multiple episodes.
\item Can we extend our guarantees to more rich classes of latent
  dynamical systems? For example, in practice, rather than assuming the latent system is linear,
  it is common to assume that it is \emph{locally linear}, and apply
  techniques such as iterative LQR \citep{watter2015embed}. 
\end{itemize}

\neurips{%

\neurips{\paragraph{Related work.} }Our model and approach are related to the literature on Embedding to
Control (E2C), and related techniques
\citep{watter2015embed,banijamali2018robust,hafner2019learning,levine2019prediction,shu2020predictive,dean2019robust}
(see also \cite{levine2016end}). At a high level, these
approaches learn a decoder that maps images down to a latent
space, then performs simple control techniques such as iterative LQR (iLQR) in the
latent space (\citet{watter2015embed} is a canonical
example). These approaches are based on heuristics, and
do not offer provable sample complexity guarantees to learn the
decoder in our setting.

Our work is also related to recent results on rich observation
reinforcement learning with discrete
actions~\citep{jiang2017contextual}. We view our model as the
control-theoretic analog of the block MDP model studied
by~\citet{DuKJAD019,Homer}, in which a latent state space associated
with a discrete Markov Decision Process is decodable from rich
observations. However, our setting is considerably different, in part
because of the continuous nature of the \richlqr, and so the results
and techniques are incomparable. In particular, discretization
approaches immediately face a curse-of-dimensionality phenomenon and
do not yield tractable algorithms. Interestingly, even ignoring
the issue of continuous actions, our setting does not appear to have
low Bellman rank in the sense of \citet{jiang2017contextual}.

A recent line of
work~\citep{oymak2019stochastic,sattar2020non,foster2020learning} gives
non-asymptotic system identification guarantees for a simple class of
``generalized linear'' dynamical systems. These results address a non-linear
dynamic system, but are incomparable to our own as the non-linearity
is known and the state is directly observed.
Our results also are related to the LQG problem,
which is a special case of \eqref{eq:noisy} with linear observations; recent work provides non-asymptotic guarantees
\citep{mania2019certainty,simchowitz2020improper,lale2020logarithmic}. These
results show that linear classes do not
encounter the sample complexity barrier exhibited by \pref{thm:lower_bound}.%

Finally, we mention two concurrent works which consider similar
settings. First, \cite{frandsen2020extracting} give guarantees for a
simpler setting in which we observe a linear combination of the latent
state and a nonlinear nuisance parameter, and where there is no noise. Second, \citet{dean2020certainty} (see also \citet{dean2019robust})
give sample complexity guarantees for a variant of the our setting in
which there is no system noise, and where $\maty_t=\gstar(C\maty_t)$,
where $C\in\bbR^{p\times{}\dimx}$ and $\gstar:\bbR^{p\to{}q}$ is a smooth function. They
provide a nonparametric approach which scales exponentially in the
dimension $p$. 
Compared to this result, the main advantage of our approach is that
it allows for general function approximation; that is, we allow
for arbitrary function classes $\Fclass$, and our results depend only
on the capacity of the class under consideration. In terms of assumptions, the addition of the $C$ matrix allows for maps that
(weakly) violate the perfect decodability assumption; we suspect
that our results can be generalized in this fashion. Likewise, we
believe that our assumption concerning the stability of $A$ can be removed in the absence
of system noise (indeed, system noise is one of the primary technical
challenges overcome by our approach).

}

\vspace{-5pt}
\paragraph{Appendix.} All of our proofs, as well
as detailed versions of the theorems in the main body, are presented in
the appendix (following the
convention \pref{thm:main}$\to$\pref{thm:main_full}), with apologies
for the exceptional length. On a first pass the reader may wish to focus on \pref{sec:phase12_proofs}
and \pref{sec:phase3_proofs}, which constitute the core proof of
\pref{thm:main}.%

\subsection*{Acknowledgements}
This work was done while ZM was an intern at Microsoft Research. DF
acknowledges the support of NSF Tripods grant \#1740751. MS was
supported by an Open Philanthropy AI Fellowship. AR
acknowledges the support of ONR awards \#N00014-20-1-2336 and \#N00014-20-1-2394

\neurips{\clearpage}

\neurips{
\section*{Broader Impact}
There is potential for research into the RichLQR setting, or more
generally perception-based control, to have significant societal
impact. Perception-based control systems are already being deployed in
applications such as autonomous driving and aerial vehicles where
algorithmic errors can have catastrophic consequences. Unfortunately,
there has been little research into the theoretical foundations of
such systems, and so the methods being deployed do not enjoy the
formal guarantees that we should demand for high-stakes applications.
Thus, we are hopeful that with a principled understanding of the
foundations of perception-based control, which we pursue here, we will
develop the tools and techniques to make these systems safe, robust,
and reliable.
 }

\bibliography{refs} \vfill
\newpage
\appendix
\numberwithin{equation}{section}
\tableofcontents
\addtocontents{toc}{\protect\setcounter{tocdepth}{3}}
\clearpage
\neurips{\section{Organization, Notation, and Preliminaries}}
\arxiv{\section*{Organization and Notation}}
\label{app:prelim}

\neurips{\subsection{Appendix Organization}}
This appendix is organized as follows.
\begin{itemize}
\neurips{
\item This section (\pref{app:prelim})---beyond this overview---contains additional notation and technical preliminaries omitted from
the main body for space.
\item \pref{app:algo} contains omitted details for the main algorithm
  (\richidce), including pseudocode with parameter values instantiated
  precisely (\pref{app:pseudocode}), an overview
  of the initial state learning phase (\pref{app:phaseiii}), and
  extensions and additional discussion of assumptions
  (\pref{ssec:relax_asm}, \pref{ssec:cM_asm}).
  }
\item \pref{app:lower_bound} contains a formal statement and proof for the
  lower bound (\pref{thm:lower_bound}).
\item \pref{app:learning_theory} and \pref{sec:linear_control} contain
  basic technical tools for learning theory and linear control theory,
  respectively, which are invoked within the proof of the main
  theorem.
\item All subsequent sections are devoted to proving our main theorem (\pref{thm:main}):
  \begin{itemize}\item \pref{sec:phase12_proofs} contains statements and proofs for
    all results concerning Phases I and II of \richidce, including
    \pref{thm:phase_one} and \pref{thm:phase_two}.
  \item \pref{sec:phase3_proofs} contains statements and proofs for
    Phase III, including \pref{thm:phase_three}.
  \item \pref{sec:main_proof} states the full version of the \pref{thm:main} (\pref{thm:main_full}), and shows how to deduce the proof from
    the results of \pref{sec:phase12_proofs} and
    \pref{sec:phase3_proofs}.
  \end{itemize}
\end{itemize}
We remind the reader that each numbered theorem from the main body has
an corresponding ``full'' version in the appendix, which we denote
using the ``a'' suffix (e.g., the full version of \pref{thm:main}
is \pref{thm:main_full}).

\neurips{
\subsection{Additional Preliminaries \label{ssec:prelim_unabridge}}

\paragraph{Policies, interaction model, and sample complexity.}
Formally, a policy $\pi$ for the setup \eqref{eq:dynamics} is a
sequence of mappings $(\pi_t)_{t=0}^{T}$, where $\pi_t$ maps the
observations $\maty_0,\ldots,\maty_t$ to an output control signal
$\matu_t$. In each round of interaction, the learner proposes a policy
$\pi$ and observes a trajectory
$\matu_0,\maty_0,\ldots,\matu_{T},\maty_{T}$ where
$\matu_t=\pi_t(\maty_0,\ldots,\maty_t)$.  We measure the sample complexity to learn an $\veps$-optimal policy
for $J_T$ in terms of the number of trajectories observed in this model. However, to simplify the description of our algorithm, we
allow the learner to execute trajectories of length $2T +
\bigoh_{\star}(\ln\ln(n))$ during the learning process, even though the
objective is $J_T$. To circumvent unidentifiability of the initial
state $\matx_0$, we define $J_T$ to measure cost on times
$1,\dots,T$. On the other hand, our rollouts begin at time $0$ (i.e.,
the initial state is $\matx_0$, and the first control input executed
is $\matu_0$).

\paragraph{Cost functions. }
We assume that the control cost matrix $R \succ 0 $ is known, but to
avoid tying costs to the unknown latent representation $\matx$, we
assume that the state cost matrix $Q \succ 0 $ is unknown. Instead, we
assume that the learner has access to an additional \emph{cost oracle}
which on each trajectory at time $t$ reveals
$\matc_t\ldef{}\matx_t^{\trn}Q\matx_t+\matu_t^{\trn}R\matu_t$. For
simplicity, we place the following mild regularity conditions on the
cost matrices.%
 \begin{assumption}
    The cost matrices $\Rx$ and $\Ru$ satisfy $\eigmin(\Rx),\eigmin(\Ru)\geq{}1$. 
  \end{assumption}
  This assumption can be made to hold without loss of generality
  whenever $Q,R\psdgt{}0$ via rescaling.

\paragraph{The \dare and infinite-horizon optimal control.}
Controllability (and more generally stabilizability) implies that there is a unique positive definite solution $\Pinf\psdgt{}0$ to the \emph{discrete algebraic
  Riccati equation} (\dare),
  \begin{align}
  \tag{DARE}
    \label{eq:dare}
    P = A^{\trn}PA + \Rx - A^{\trn}PB(\Ru+B^{\trn}PB)^{-1}B^{\trn}PA,
  \end{align}
  which characterizes the optimal cost function for the LQR problem in the infinite-horizon setting. Our analysis uses $\Pinf$, and our algorithms use the optimal infinite-horizon state feedback controller 
  \begin{align*}
  \Kinf\ldef{}-(\Ru+B^{\trn}\Pinf{}B)^{-1}B^{\trn}\Pinf{}A.
  \end{align*}
  When the state $\matx_t$ is directly observed, the optimal
  infinite-horizon controller is the time-invariant feedback policy $u
  = \Kinf x$. Thus, the optimal infinite-horizon policy for \richlqr,
  given the exact decoder, is $\piinf(y) = \Kinf \fst(y)$. We use this
  controller as our benchmark.

Our analysis also uses the infinite-horizon covariance matrix
  \begin{align*}
   \Siginf \coloneqq \Ru+B^{\trn}\Pinf{}B.
\end{align*}

\paragraph{Certainty equivalence.}
  \emph{Certainty equivalence} refers to the strategy of estimating $\Kinf$ by solving the \dare{} with plug-in estimates $\Ahat,\Bhat$ of $(A,B)$ to obtain a matrix $\Phat$, and taking $\Khat := -(\Ru+\Bhat^{\trn}\Phat \Bhat)^{-1}\Bhat^{\trn}\Phat\Ahat$. We formalize this as follows.
  \begin{definition}[\dare with certainty equivalence]\label{defn:darece} We define the  $\darece$ operator as the operator which takes in matrix $(A_0,B_0,R_0,Q_0) \in \R^{\dimx^2} \times \R^{\dimx \dimu} \times \R^{\dimx^2} \times \R^{\dimu^2}$, with $R_0,Q_0 \succeq 0$, and returns $(P,K):= \darece(A_0,B_0,R_0,Q_0) \in  \R^{\dimx^2} \times \R^{\dimu \dimx}$ such that
  \begin{align*}
  P &= A_0^{\trn}PA_0 + \Rx - A_0^{\trn}PB_0(\Ru+B_0^{\trn}PB_0)^{-1}B_0^{\trn}PA_0,\\
  K&\gets -(\Ru+B_0^{\trn}PB_0)^{-1}B_0^{\trn}PA_0.
  \end{align*}
  \end{definition}

\neurips{%
\paragraph{Strong stability.}
We quantify stability via \emph{strong stability}
\citep{cohen2018online}. Intuitively, a matrix $X$ is \emph{strongly
  stable} if its powers $X^n$ decay geometrically in a quantitative sense.
  \begin{definition}[Strong stability]
\label{def:ss}
A matrix $X\in \bbR^{\dimx\times\dimx}$ is said to be $(\alpha,\gamma)$-strongly stable if there exists $S\in \bbR^{\dimx\times\dimx}$ such that $\|S\|_{\op}\|S^{-1}\|_{\op} \le \alpha$ and $\|S^{-1}XS\|_{\op} \le \gamma<1$. 
\end{definition}
We frequently make use of the fact that if $X$ is $(\alpha,\gamma)$-strongly stable, then 
\begin{align*}
\|X^n\|_{\op} = \|S(S^{-1}XS)^nS^{-1}\|_{\op} \le \|S^{-1}\|_{\op}\|S\|_{\op}\|S^{-1}XS\|_{\op}^n \le \alpha \gamma^n.
\end{align*}
We let $(\alphaa,\gammaa)$ and $(\alphainf,\gammainf)$ be the strong
parameters for $A$ and $\Aclinf\coloneqq A + B K_\infty$ respectively. Under \Cref{ass:controllability} and \Cref{ass:stability}, we are guaranteed
that $\gammaa,\gammainf<1$ (see \pref{prop:closed_loop_ss} and
\pref{prop:open_loop_ss} for quantitative bounds). Finally, we recall from
\pref{asm:para_upper_bounds} that we assume the learner knows
upper bounds $(\alphastar,\gammastar)$ such that
$\alphaa\vee\alphainf\leq\alphastar$ and $\gammaa\vee\gammainf\leq\gammastar$.

}

\neurips{\subsection{Additional Notation}}
\paragraph{Asymptotic notation.}
	For functions
	$f,g:\cX\to\bbR_{+}$, we write $f=\bigoh(g)$ if there exists some universal constant
	$C>0$, which doesn not depend on problem parameters, such that $f(x)\leq{}Cg(x)$ for all $x\in\cX$. Our proofs also use the shorthand $f(x) \lesssim g(x)$ to denote $f = \BigOh\prn{g}$. We use $\bigoht(\cdot)$ so suppress logarithmic dependence on system
    parameters, time horizon, and dimension. As indicated in the body,
    we use $\bigohs(\cdot)$ to suppress polynomial factors in system
    parameters
    (e.g. $\kappa,\alphastar,\gammastar^{-1},(1-\gammastar)^{-1},\Psistar$,
    $L$, and $\sigmamin(\cC_{\kappa})$). Lastly, we write $f=\bigohch(g)$ if $f(x)\leq{}cg(x)$ for all $x\in\cX$, where
$c=\poly(\gammastar(1-\gammastar),\alphastar^{-1},\Psistar^{-1},\kappa^{-1},L^{-1},\sigmamin(\cC_{\kappa}))$ is a sufficiently \emph{small} constant.

\paragraph{General notation.} 
	For a vector $x\in\bbR^{d}$, we let $\nrm*{x}$ denote the euclidean
	norm and $\nrm*{x}_{\infty}$ denote the element-wise $\ls_{\infty}$
	norm. For a matrix $A$, we let $\nrm*{A}_{\op}$ denote the
        operator norm. If $A$ is symmetric, we let $\eigmin(A)$ denote the
	minimum eigenvalue. For a potentially asymmetric matrix $A\in\bbR^{d\times{}d}$, 
        we let $\rho(A)\coloneqq \max \{|\lambda_1(A)|, \dots
        ,|\lambda_{d}(A)|\}$ denote the spectral radius. For a
        symmetric matrix $M\in\bbR^{d}$, $(M)_{+}$ denotes the result
        of thresholding all negative eigenvalues to zero, and we let
        $\lambda_1(M),\ldots,\lambda_d(M)$ denote the eigenvalues of
        $M$, sorted in decreasing order. Similarly, for a matrix
        $A\in\bbR^{d_1\times{}d_2}$, we let
        $\sigma_1(A),\ldots,\sigma_{d_1\wedge{}d_2}(A)$ denote the
        singular values of $A$, sorted in decreasing order.

 }

\begin{table}[h!]
\centering
\begin{tabular}{| l | l |}
\hline
\multicolumn{1}{|c|}{\textbf{Notation}} & \multicolumn{1}{|c|}{\textbf{Definition}}  \\
  \hline
  \multicolumn{2}{|c|}{\textbf{Basic Definitions}}\\
  \hline
$(A,B)$& system matrices\\
$q(\cdot | x)$ & emissions model\\
$\fst$ & true decoder\\
$\Fclass$ & function class\\
$\matw_t$ & process noise, which follows $\matw_t\sim \cN(0,\Sigma_w)$\\
  $\matx_t$ & system state, which has initial state $\matx_0 \sim
              \cN(0,\Sigma_{0})$\\
  $\matu_t$ & control input\\
  $\maty_t$ & observations\\
  $\matc_t$ & observed cost, with $\matc_t=\matx_t^{\trn}Q\matx_t + \matu_t^{\trn}R\matu_t$\\
  $R,Q$ & control cost, state cost \\
  $T$ & horizon\\
$J_T$ & cost functional\\
$\piinf$ & infinite horizon optimal policy \\
$\cont_k$ & controllability matrix $\prn*{[A^{k-1}B \mid \dots \mid
            B]}$\\
  $\Kinf$, $\Pinf$, $\Siginf$ & infinite-horizon optimal controller, Lyapunov
                     matrix, covariance matrix \eqref{eq:dare}\\
    \hline
  \multicolumn{2}{|c|}{\textbf{System Parameter Bounds}}\\
  \hline
$\Psistar$ & upper bound on system parameter norms (\Cref{asm:para_upper_bounds})\\
$(\alphastar,\gammastar)$ & upper bound on strong stability parameter (\Cref{asm:para_upper_bounds})\\
$\kappa$ & \makecell[l]{controllability index upper bound (\pref{asm:para_upper_bounds})}\\
  $L$ & growth condition on $\Fclass$ (\Cref{asm:f_growth})\\
  $(\alphaa,\gammaa)$ & strong stability parameters for $A$
                        (\pref{prop:open_loop_ss})\\
  $(\alphainf,\gammainf)$ & strong stability parameters for $A+B\Kinf$
                            (\pref{prop:closed_loop_ss})\\
      \hline
\end{tabular}
\caption{Summary of notation.}
\label{tab:notation}
\end{table}

\newpage
\neurips{
\section{Full Algorithm Description}
\label{app:algo}
\subsection{Pseudocode}
\label{app:pseudocode}
\begin{algorithm}[H]
  \setstretch{1.1}
  \begin{algorithmic}[1]\onehalfspacing
    \State\textbf{Inputs:}
    \Statex{} $\veps$ (sub-optimality), $T$ (horizon), $\Fclass$ (decoder class), $(\dimx, \dimu)$ ({latent dimensions}), $(\Psistar, \kappa,\alphastar, \gammastar)$ ({system parameter upper-bounds}), $\sigma^{2}$ (exploration parameter), $\lambda_{\cM}$ (as in \pref{ass:m_matrix}).
 \State\textbf{Initialize:}
   \Statex{}~~~~$n_{\id}=\bigoms\prn*{
   	\lambdam^{-2}\kappa^4(\dimx+\dimu)^{16}T^3\log^{14}(\veps^{-1}/\delta)\cdot\frac{\log\abs{\Fclass}}{\veps^{6}}
   }$.
   \Statex{}~~~~$n_{\onpo}=\bigoms\prn*{
   	\lambdam^{-2}\kappa^3(\dimx+\dimu)^{14}T^3\log^{13}(\veps^{-1}/\delta)\cdot\frac{\log\abs{\Fclass}}{\veps^{6}}}$.
   \Statex{}~~~~$\kappa_0  = \Ceil{(1-\gammastar)^{-1} \ln
   	\left({84\Psistar^5 \alphastar^4 \dimx
    		(1-\gammastar)^{-2}\ln(10^3\cdot\nid)} \right)}$. \algcomment{burn-in time}
    	
    \Statex{}~~~~$r_\star=  \sigma_{\min}(\contkap)^{-3}(1-\gammastar)^{-3}(4^3|\Psistar|^{21/2}\alphastar^6)$. \algcomment{upper-bound on $\prn*{\Psistar\|S_\id\|_{\op} \|S^{-1}_\id\|_{\op}}^{3}$}
 \Statex{}~~~~$\bclip = \Theta_{\star}((\dimx+\dimu)\log(\veps^{-1}/\delta))$. \algcomment{clipping parameter for the decoders}
 \Statex{}~~~~$\sigma^{2}=\veps^{2}/\bclip^{2}$. \algcomment{exploration parameter}
    	    \State{}\textbf{Phases I} \algcomment{learn a coarse decoder (see \pref{sec:phase1})}
    	        \State  Set $\fhatid \leftarrow \getcoarsedecoder(\nid, \kappa_0, \kappa,\sqrt{\Psi_\star})$.\algcomment{\pref{alg:phase1}}
    \State{}\textbf{Phases II} \algcomment{learn system's dynamics and cost (see \pref{sec:phase2})}
    \State  Set $(\what A_{\id}, \what B_{\id}, \what \Sigma_{w,\id}, \what Q_{\id})  \leftarrow \sysid(\fhatid, \nid, \kappa_0, \kappa, \sqrt{\Psi_\star})$.\algcomment{\pref{alg:phase2}}
    \State{}\textbf{Phase III} \algcomment{compute optimal policy (see \pref{sec:phase3} and \pref{sec:phase3_proofs})}
     \State Set $\what \pi \leftarrow \computepol(\what A_\id,\what B_\id, \what \Sigma_{w,\id},\what Q_{\id},R, n_\onpo, \kappa, \sigma^2, T,\bclip,r_\star)$. \algcomment{ \pref{alg:phase3}}
     \State{}\textbf{return:}  $\what \pi$.
  \end{algorithmic}
  \caption{\richidce (Full version with explicit parameter values)}
  \label{alg:main2}
\end{algorithm} 

\begin{algorithm}[htp]
  \setstretch{1.1}
  \begin{algorithmic}[1]\onehalfspacing
    \State\textbf{Inputs:}
    \Statex{}~~~~$\nid$ \algparen{sample size}
    \Statex{}~~~~$\kappa_0$ \algparen{``burn-in'' time index}
    \Statex{}~~~~$\kappa$ \algparen{upper bound on the controllability index $\kappa_\star$}
         \Statex{}~~~~$r_{\id}$ \algparen{upper bound on the matrix $M$ in the definition of $\scrH_\id$}
    \smallskip

    \State{}Set $\Hid := \left\{M  f(\cdot) \mid f \in \Fclass,\  M \in
    \R^{\kappa\dimu \times \dimx}, ~\|M\|_{\op} \le r_{\id}
    \right\}.$
             \State Set $\kappa_1 = \kappa_0 + \kappa$.
          \State{}Gather $2\nid$ trajectories by sampling control
          inputs $\matu_0,\ldots,\matu_{\kappa_1-1}\sim{}\cN(0,I_{\dimu})$.
    \State{}\textbf{Phase I:} \algcomment{Learn coarse
      decoder (see \pref{sec:phase1}).}
    \State{}Set $\hhatid = \argmin_{h \in \Hid}\sum_{i=1}^{\nid} \|h(\bykapone\supi) - \bv\supi\|_2^2$, where $\bv \coloneqq (\bu_{\kappa_0}^\top, \dots, \bu_{\kappa_1-1})^\top$.
    \State{}Set $\Vhatkapid$ to be an orthonormal basis for
                                               top $\dimx$-eigenvectors of
                                               $\frac{1}{\nid}\sum_{i=\nid+1}^{2\nid}
              \hhatid(\bykapone\supi) \hhatid(\bykapone\supi)^\top$
    \State{}Set $\fhatid(\cdot) \coloneqq \Vhatkapid^\top \hhatid(\cdot)$. \algcomment{coarse decoder.}
\State{}\textbf{Return:} coarse decoder $\fhatid$.
  \end{algorithmic}
  \caption{\getcoarsedecoder: Phase I of \richidce{} (\pref{sec:phase1}).}
  \label{alg:phase1}
\end{algorithm}

\begin{algorithm}[H]
  \setstretch{1.1}
  \begin{algorithmic}[1]\onehalfspacing
  \State\textbf{Require:}
  \State{}~~~~Cost oracle to access the cost $\bc_t$ at time $t\geq 1$.
    \State\textbf{Inputs:}
     \Statex{}~~~~$\fhatid$ \algparen{coarse decoder}
    \Statex{}~~~~$\nid$ \algparen{sample size}
    \Statex{}~~~~$\kappa_0$ \algparen{burn-in time index}
    \Statex{}~~~~$\kappa$ \algparen{upper bound on the controllability index $\kappa_\star$}
          \smallskip

           \State Set $\kappa_1 = \kappa_0 + \kappa$.
          \State{}Gather $\nid$ trajectories by sampling control
          inputs $\matu_0,\ldots,\matu_{\kappa_1-1}\sim{}\cN(0,I_{\dimu})$.
    \State{}\textbf{Phase II:} \algcomment{Recover system
      dynamics and cost (see \pref{sec:phase2}).}
      \State{}Set $(\Ahatid,\Bhatid) \in \argmin_{(A,B)} \sum_{i=2\nid + 1}^{3\nid} \|\fhatid(\by_{\kapone+1}\supi) - A\fhatid(\bykapone\supi) - B\bu_{\kapone}\supi \|^2$.
  \State{}Set $\Sigwhatid = \frac{1}{\nid}\sum_{i=2\nid+1}^{3\nid}
	    (\fhatid(\by_{\kapone+1}\supi) - \Ahatid\fhatid(\bykapone\supi) -
	    \Bhatid\bu_{\kapone}\supi )^{\otimes 2}$, where $v^{\otimes 2} \ldef
	    vv^\top$.
 \State{}Set $\Qtilid = \min_{Q}\sum_{i=2\nid+1}^{3\nid} \left(\bc_{\kapone}^{(i)} - (\bu^{(i)}_{\kapone})^\top R \bu^{(i)}_{\kapone} -  \fhatid(\bykapone^{(i)})^\top Q\fhatid(\bykapone^{(i)})\right)^2$.
  \State{}Set $\Qhatid = \left(\frac{1}{2}\Qtilid  +\frac{1}{2}\Qtilid^\top\right)_{+}$, where $(\cdot)_+$ truncates all negative eigenvalues to zero.
\State{}\textbf{Return:} system and cost matrices $(\Ahatid,\Bhatid, \Sigwhatid ,\Qhatid)$.
  \end{algorithmic}
  \caption{\sysid: Phase II of \richidce{} (\pref{sec:phase2}).}
  \label{alg:phase2}
\end{algorithm}

  \begin{algorithm}[htp]
  \setstretch{1.1}
    \begin{algorithmic}[1]\onehalfspacing
  	\State\textbf{Inputs:} $(\what A,\what B, \what \Sigma_w,\what Q,R)$ \algparen{estimates for the system parameters and cost matrices}
        \State\textbf{Parameters:}
  	\Statex{}~~~~~~~~~~~~$n_\onpo$ \algparen{proportional to the sample size}
  	\Statex{}~~~~~~~~~~~~$\kappa$ \algparen{upper-bound on the controllability index $\kappa_\star$}
  	\Statex{}~~~~~~~~~~~~$\sigma^2$ \algparen{exploration parameter}
  		\Statex{}~~~~~~~~~~~~$\bclip$ \algparen{clipping parameter for the decoders}
  		\Statex{}~~~~~~~~~~~~$r_{\op}$ \algparen{parameter to define
                  the function class}
                \smallskip
  		\State{}Set $\scrH_\onpo= \left\{M  f(\cdot) \mid f \in \Fclass,\  M \in
  		\R^{\dimx \times \dimx}, ~\|M\|_{\op} \le r_{\op}
  		\right\}.$
  		\State{}Set $n_\init = n_\op$.
  		\State{}\textbf{Phase III:} \algcomment{Learn on-policy decoders (see \pref{sec:phase3} and \pref{app:phaseiii}).}
  		\State Set $(\what P,\what K):= \darece(\what A,\what B,\what Q, R)$ (\pref{defn:darece}).
  		\For{$k=1, \dots, \kappa$}
  		\State Set $\what \cC_k = [\what A^{k-1} \what B \mid \dots \mid \what B]$. 
  		\State Set $\what M_k \coloneqq \what\cC_{k}^\top (\what \cC_k  \what\cC^\top_k + \sigma^{-2} \sum_{i=0}^{k}  \what A^{i-1} \what \Sigma_w  (\what A^{i-1})^\top)^{-1}$.
  		\EndFor
  		\State Set $\what \cM = [\what M_1,\dots, (\what M_\dimk  \what A^{\dimk -1})^\top  ]^\top$.
  		\State Define $\hat f_0(y_0) = 0$ for all $y_0\in\cY$.\label{line:fhat0}
  		\For{$t=0,\ldots,T-1$}
  		\State{}Collect $2n_{\onpo}$ trajectories by executing the randomized control input $\bu_\tau=\what K \hat f_\tau(\by_{0:\tau}) + \bnu_\tau$,\\ \hskip\algorithmicindent \hskip\algorithmicindent for $0\leq{}\tau\leq t$, and $\bu_\tau=\bnu_{\tau}$, for $t< \tau < t+\kappa$, where $\bnu_\tau \sim\cN(0,\sigma^{2}I_{\dimu}).$
  		\For{$k=1, \dots, \kappa$}
  		\State Set $\hat{h}_{t,k} \in \argmin_{h \in  \scrH_\onpo} \sum_{i=1}^{n_{\onpo}} \left\|\what\phi_{t,k}(h, \by^{(i)}_{0:t}, \by^{(i)}_{t+k}) - \bnu^{(i)}_{t:t+k-1}\right\|_2^2$, 
  		\Statex{}\hskip\algorithmicindent \hskip\algorithmicindent \hskip\algorithmicindent where $\what\phi_{t,k}(h, \by_{0:t}, \by_{t+k}) \coloneqq \what M_{k}  \left(h(\by_{t+k})-\what A^{k}  h(\by_t) - \what A^{k-1}\what B \what K \fref_t(\by_{0:t})\right)$.
  		\State Set $\hat{h}_{t} \in \argmin_{h \in  \scrH_\onpo} \sum_{i=n_{\onpo}+1}^{2n_{\onpo}} \left\|\what \calM  \left(h(\by^{(i)}_{t+1}) - \what A  h(\by^{(i)}_{t}) - \what B \what K  \fref_t (\by_{0:t}^{(i)})  \right) - \what\upphi_t(\by_{0:t+\dimk}^{(i)})  \right\|_2^2,$\label{line:hhat_t}
  		\Statex{}\hskip\algorithmicindent \hskip\algorithmicindent \hskip\algorithmicindent where $ \what\upphi_t(\by_{0:t+\dimk})\coloneqq [\what
  		\phi_{t,1}(\hat{h}_{t,1},\by_{0:t}, \by_{t+1})^\top, \dots,
  		\what\phi_{t,\dimk}(\hat{h}_{t,\dimk},\by_{0:t},
  		\by_{t+\dimk})^\top]^\top$.
  		\EndFor
  		\If{$t=0$}~~\algcomment{Initial state learning phase (\pref{app:phaseiii}).}
  		\State Collect $2 n_{\init}$ trajectories by executing the control input $\bu_\tau= \bnu_\tau$, for $0\leq \tau < \kappa$, \label{line:start} \\ \hskip\algorithmicindent \hskip\algorithmicindent \hskip\algorithmicindent where $\bnu_\tau \sim\cN(0,\sigma^{2}I_{\dimu})$. 
  		\State Set $\hat{h}_{\ol,1} \in \argmin_{h \in \scrH_{\onpo}} \sum_{i=1}^{n_\init} \left\| h(\by^{(i)}_1)  - \left(\hat{h}_{0}(\by^{(i)}_1) - \what A \hat{h}_{0}(\by^{(i)}_0) -  \what B \bnu_0^{(i)}\right)  \right\|_2^2$. \label{line:start+1} 
  		\State Set $\what \Sigma_\cv \coloneqq  \frac{1}{n_\init}\sum_{i=n_\init +1}^{2 n_\init} \hat{h}_{\ol,1}(\by_1^{(i)})  \hat{h}_{\ol,1}(\by_1^{(i)})^\top. $\label{line:start+2} 
  		\State Set $\tilde{h}_{\ol,0} \in \argmin_{h \in \scrH_{\onpo}} \sum_{i=n_\init +1}^{2n_\init} \left\|h(\by_0^{(i)})  -    \hat{h}_{\ol,1}(\by^{(i)}_1) \right\|_2^2$.\label{line:start+3} 
  		\State Set $\hat{f}_{A,0}(\by_0)= \what \Sigma_w \what \Sigma_\cv^{-1} \tilde{h}_{\ol,0} (\by_0)$.\label{line:start+4} 
  		\State{}Set $\freftil_{1}(\by_{0:1}) = \hat{h}_0(\by_{1})  -
  		\what A  \hat{h}_0(\by_{0}) +  \fref_{A,0}(\by_{0}).$ \label{line:ned}
  		\Else
  		\State{}Set $\freftil_{t+1}(\by_{0:t+1}) = \hat{h}_t(\by_{t+1})  -
  		\what A  \hat{h}_t(\by_{t}) + \what A \fref_{t}(\by_{t}).$
  		\EndIf
  		\State{}Set $\hat f_{t+1}(\by_{0:t+1}) = \tilde f_{t+1}(\by_{0:t+1}) \Ind\{\|\tilde f_{t+1}(\by_{0:{t+1}})\|_2 \leq \bclip \}$.\label{line:clip}
  		\State{}Set controller 
  		$\pihat_{t+1}(\maty_{0:t+1})=\Khat\fref_{t+1}(\maty_{0:t+1})+\bnu_{t+1}$, with $\bnu_{t+1}\sim\cN(0,\sigma^{2}I_{\dimu})$.
  		\EndFor
  		\State{}\textbf{Return:} Controller $\pihat=(\pihat_{t})_{t=1}^{T}$.
  	\end{algorithmic}
  \caption{\computepol: Phase III of \richidce}
  \label{alg:phase3}
\end{algorithm} 

 }
\neurips{
\subsection{Overview for Learning the Initial State}
\label{app:phaseiii}
\neurips{In this section, we give an overview for}\arxiv{We now overview} how Phase III of \richid{} (\pref{alg:phase3}) learns a predictor for the initial state $\bx_0$; this is an edge case \arxiv{not discussed above}\neurips{which is not discussed in the main body due to space limitation}, and comprises \pref{line:start} through \pref{line:ned} in \pref{alg:phase3}. \neurips{This discussion supplements \pref{sec:phase3} of the main body, and together these sections give constitute our high-level overview of Phase III.}

If we ignore the clipping in \eqref{eq:decoder_simple}, the state decoders $(\hat f_\tau)_{t\geq 2}$ are defined through the recursion: 
\begin{align}
\hat f_{t+1}(\by_{0:t+1}) = \hat{h}_t(\by_{t+1})  -
	\what A  \hat{h}_t(\by_{t}) + \what A \fref_{t}(\by_{t}), \quad \text{for $t\geq 1$},\nn
	\end{align}
        which means that all the decoding error for any $t$ will depend on the error of the decoder $\hat f_1$ for the state $\matx_1$. To ensure that $\hat f_1$ is accurate, we need to somehow learn to decode the inital state $\bx_0$, which we recall is assumed to be distributed as $\cN(0,\Sigma_0)$. The challenge here is that the covariance matrix $\Sigma_0$ is unknown, and we need to estimate it in order to ``back out'' the initial state through the approach in \pref{sec:bayes_pred}. This is achieved by \pref{line:start} through \pref{line:ned} of \pref{alg:phase3}, which we explain in detail below. Briefly, the idea is that since $\matx_1=A\matx_0+B\matu_0+\matw_0$, to accurately predict $\matx_1$ it suffices to have good predictors for $\matw_0$ and $A\matx_0$. We can learn a predictor for $\matw_0$ in the same fashion as for all the other timesteps, and most of the work in \pref{line:start} through \pref{line:ned} is to learn a regression function $\hat{f}_{A,0}$ that accurately predicts $A\matx_0$.

\arxiv{
  \begin{algorithm}[htp]
  \setstretch{1.1}
    \begin{algorithmic}[1]\onehalfspacing
  	\State\textbf{Inputs:} $(\what A,\what B, \what \Sigma_w,\what Q,R)$ \algparen{estimates for the system parameters and cost matrices}
        \State\textbf{Parameters:}
  	\Statex{}~~~~~~~~~~~~$n_\onpo$ \algparen{proportional to the sample size}
  	\Statex{}~~~~~~~~~~~~$\kappa$ \algparen{upper-bound on the controllability index $\kappa_\star$}
  	\Statex{}~~~~~~~~~~~~$\sigma^2$ \algparen{exploration parameter}
  		\Statex{}~~~~~~~~~~~~$\bclip$ \algparen{clipping parameter for the decoders}
  		\Statex{}~~~~~~~~~~~~$r_{\op}$ \algparen{parameter to define
                  the function class}
                \smallskip
  		\State{}Set $\scrH_\onpo= \left\{M  f(\cdot) \mid f \in \Fclass,\  M \in
  		\R^{\dimx \times \dimx}, ~\|M\|_{\op} \le r_{\op}
  		\right\}.$
  		\State{}Set $n_\init = n_\op$.
  		\State{}\textbf{Phase III:} \algcomment{Learn on-policy decoders (see \pref{sec:phase3} and \pref{app:phaseiii}).}
  		\State Set $(\what P,\what K):= \darece(\what A,\what B,\what Q, R)$ (\pref{defn:darece}).
  		\For{$k=1, \dots, \kappa$}
  		\State Set $\what \cC_k = [\what A^{k-1} \what B \mid \dots \mid \what B]$. 
  		\State Set $\what M_k \coloneqq \what\cC_{k}^\top (\what \cC_k  \what\cC^\top_k + \sigma^{-2} \sum_{i=0}^{k}  \what A^{i-1} \what \Sigma_w  (\what A^{i-1})^\top)^{-1}$.
  		\EndFor
  		\State Set $\what \cM = [\what M_1,\dots, (\what M_\dimk  \what A^{\dimk -1})^\top  ]^\top$.
  		\State Define $\hat f_0(y_0) = 0$ for all $y_0\in\cY$.\label{line:fhat0}
  		\For{$t=0,\ldots,T-1$}
  		\State{}Collect $2n_{\onpo}$ trajectories by executing the randomized control input $\bu_\tau=\what K \hat f_\tau(\by_{0:\tau}) + \bnu_\tau$,\\ \hskip\algorithmicindent \hskip\algorithmicindent for $0\leq{}\tau\leq t$, and $\bu_\tau=\bnu_{\tau}$, for $t< \tau < t+\kappa$, where $\bnu_\tau \sim\cN(0,\sigma^{2}I_{\dimu}).$
  		\For{$k=1, \dots, \kappa$}
  		\State Set $\hat{h}_{t,k} \in \argmin_{h \in  \scrH_\onpo} \sum_{i=1}^{n_{\onpo}} \left\|\what\phi_{t,k}(h, \by^{(i)}_{0:t}, \by^{(i)}_{t+k}) - \bnu^{(i)}_{t:t+k-1}\right\|_2^2$, 
  		\Statex{}\hskip\algorithmicindent \hskip\algorithmicindent \hskip\algorithmicindent where $\what\phi_{t,k}(h, \by_{0:t}, \by_{t+k}) \coloneqq \what M_{k}  \left(h(\by_{t+k})-\what A^{k}  h(\by_t) - \what A^{k-1}\what B \what K \fref_t(\by_{0:t})\right)$.
  		\State Set $\hat{h}_{t} \in \argmin_{h \in  \scrH_\onpo} \sum_{i=n_{\onpo}+1}^{2n_{\onpo}} \left\|\what \calM  \left(h(\by^{(i)}_{t+1}) - \what A  h(\by^{(i)}_{t}) - \what B \what K  \fref_t (\by_{0:t}^{(i)})  \right) - \what\upphi_t(\by_{0:t+\dimk}^{(i)})  \right\|_2^2,$\label{line:hhat_t}
  		\Statex{}\hskip\algorithmicindent \hskip\algorithmicindent \hskip\algorithmicindent where $ \what\upphi_t(\by_{0:t+\dimk})\coloneqq [\what
  		\phi_{t,1}(\hat{h}_{t,1},\by_{0:t}, \by_{t+1})^\top, \dots,
  		\what\phi_{t,\dimk}(\hat{h}_{t,\dimk},\by_{0:t},
  		\by_{t+\dimk})^\top]^\top$.
  		\EndFor
  		\If{$t=0$}~~\algcomment{Initial state learning phase (\pref{app:phaseiii}).}
  		\State Collect $2 n_{\init}$ trajectories by executing the control input $\bu_\tau= \bnu_\tau$, for $0\leq \tau < \kappa$, \label{line:start} \\ \hskip\algorithmicindent \hskip\algorithmicindent \hskip\algorithmicindent where $\bnu_\tau \sim\cN(0,\sigma^{2}I_{\dimu})$. 
  		\State Set $\hat{h}_{\ol,1} \in \argmin_{h \in \scrH_{\onpo}} \sum_{i=1}^{n_\init} \left\| h(\by^{(i)}_1)  - \left(\hat{h}_{0}(\by^{(i)}_1) - \what A \hat{h}_{0}(\by^{(i)}_0) -  \what B \bnu_0^{(i)}\right)  \right\|_2^2$. \label{line:start+1} 
  		\State Set $\what \Sigma_\cv \coloneqq  \frac{1}{n_\init}\sum_{i=n_\init +1}^{2 n_\init} \hat{h}_{\ol,1}(\by_1^{(i)})  \hat{h}_{\ol,1}(\by_1^{(i)})^\top. $\label{line:start+2} 
  		\State Set $\tilde{h}_{\ol,0} \in \argmin_{h \in \scrH_{\onpo}} \sum_{i=n_\init +1}^{2n_\init} \left\|h(\by_0^{(i)})  -    \hat{h}_{\ol,1}(\by^{(i)}_1) \right\|_2^2$.\label{line:start+3} 
  		\State Set $\hat{f}_{A,0}(\by_0)= \what \Sigma_w \what \Sigma_\cv^{-1} \tilde{h}_{\ol,0} (\by_0)$.\label{line:start+4} 
  		\State{}Set $\freftil_{1}(\by_{0:1}) = \hat{h}_0(\by_{1})  -
  		\what A  \hat{h}_0(\by_{0}) +  \fref_{A,0}(\by_{0}).$ \label{line:ned}
  		\Else
  		\State{}Set $\freftil_{t+1}(\by_{0:t+1}) = \hat{h}_t(\by_{t+1})  -
  		\what A  \hat{h}_t(\by_{t}) + \what A \fref_{t}(\by_{t}).$
  		\EndIf
  		\State{}Set $\hat f_{t+1}(\by_{0:t+1}) = \tilde f_{t+1}(\by_{0:t+1}) \Ind\{\|\tilde f_{t+1}(\by_{0:{t+1}})\|_2 \leq \bclip \}$.\label{line:clip}
  		\State{}Set controller 
  		$\pihat_{t+1}(\maty_{0:t+1})=\Khat\fref_{t+1}(\maty_{0:t+1})+\bnu_{t+1}$, with $\bnu_{t+1}\sim\cN(0,\sigma^{2}I_{\dimu})$.
  		\EndFor
  		\State{}\textbf{Return:} Controller $\pihat=(\pihat_{t})_{t=1}^{T}$.
  	\end{algorithmic}
  \caption{\computepol: Phase III of \richidce}
  \label{alg:phase3}
\end{algorithm} 

 }
        
To begin, in \pref{line:start} we execute Gaussian control inputs $\bnu_{\tau}$ for $0\leq{}\tau<\kappa$. We then proceed as follows.
        
\paragraph{\pref{line:start+1}.} As we show in \pref{thm:monster} (\pref{sec:phase3_proofs}), $\hat{h}_{t}(\by_{t+1}) - \what A \hat{h}_{t}(\by_t) -  \what B ( \what K \hat f_t(\by_{0:t}) + \bnu_t)$ approximates the system's noise $\bw_t$, for $t\geq 0$. In particular, since $\hat f_0 \equiv 0$ by definition (\pref{line:fhat0}), $\hat{h}_{0}(\by_{1}) - \what A \hat{h}_{0}(\by_0) -  \what B \bnu_0$ approximates the noise $\bw_0$. Since we have $\matx_1=A\matx_0+B\matu_0+\matw_0$, it remains to get a good estimator for $A\matx_0$. To this end, we observe that the predictor $\hat h_{\ol,1 }$ in \pref{line:start+1} is (up to a generalization bound) equal to
\[
  \argmin_{h\in\Hclass_{\onpo}}\E_{\what \pi} \left[ \| h(\by_1) - \bw_0\|^2\right],
\]
which we show---under the realizability assumption---is given by
\begin{align}
	\E \left[ \bw_0 \mid \by_1  \right] = 	\E\left[ \bw_0 \mid \bx_1  \right] = \Sigma_w (\sigma^2 BB^\top +  \Sigma_w + A\Sigma_0 A^\top)^{-1} (A \bx_0 + B \bnu_0 + \bw_0), \label{eq:closed}
	\end{align}
where the last equality---like the rest of our Bayes characterizations---follows by \pref{fact:gaussian_expectation}.

\paragraph{\pref{line:start+2}.} Now, given that $\hat h_{\ol,1}(\by_1)\approx \E[\bw_0\mid \by_1]$, one can recognize that the matrix $\what \Sigma$ in \pref{line:start+2} is an estimator for the matrix 
\begin{align}
	\Sigma_w (\sigma^2 BB^\top +  \Sigma_w + A\Sigma_0 A^\top)^{-1}\Sigma_w.\label{eq:mat}
\end{align}
In particular, even though we cannot recover the covariance matrix $\Sigma_0$, the estimator $\what \Sigma$ gives a means to predict $A \bx_0$,  leading to an accurate decoder $\hat f_1$.
\paragraph{\pref{line:start+3}.} Since $\hat h_{\ol,1}(\by_1)$ accurately predicts $\E[\bw_0 \mid \by_1]$ (whose closed form expression we recall is given by the RHS of \eqref{eq:closed}), the predictor $\tilde h_{\ol,0}$ in \pref{line:start+3} can be seen to approximate
\begin{gather}
  \argmin_{h\in\Hclass_{\onpo}}\E_{\what \pi}\left[\| h(\by_0)-\Sigma_w (\sigma^2 BB^\top +  \Sigma_w + A\Sigma_0 A^\top)^{-1} (A \bx_0 + B \bnu_0 + \bw_0) \|^2 \right],\nn
	\shortintertext{which (under realizability) is simply}
	\E[\Sigma_w (\sigma^2 BB^\top +  \Sigma_w + A\Sigma_0 A^\top)^{-1} (A \bx_0 + B \bnu_0 + \bw_0)  \mid \bx_0]= \Sigma_w (\sigma^2 BB^\top +  \Sigma_w + A\Sigma_0 A^\top)^{-1} A \bx_0. \label{eq:condexp} 
	\end{gather}
\paragraph{\pref{line:start+4,line:ned}.} In light of \eqref{eq:condexp} and the fact that $\what \Sigma$ is an estimator of the matrix in \eqref{eq:mat}, we are guaranteed that $\what \Sigma_w \what \Sigma^{-1} \hat h_{\ol,0}(\by_0)$ accurately predicts $A \bx_0$, which motivates the updates in \pref{line:start+4,line:ned}.

 \subsection{Extensions \label{ssec:relax_asm}}

\paragraph{Relaxing the stability assumption.} We believe that our
algorithm can be extended to so-called \emph{marginally stable}
systems, where $\rho(A)$ can be as large as $1$ (rather than strictly
less than $1$). In such systems, there exist system-dependent
constants $c_1,c_2 > 0$ for which $\|A^n\|_{\op} \le c_1 n^{c_2}$ for
all $n$. In general, these constants may be large, and in the worst
case $c_2$ may be as large as $\dimx$ (or, more generally, the largest
Jordan block of $A$); see, e.g., \citet{simchowitz2018learning} for
discussion. Nevertheless, if $c_1,c_2$ are treated as problem
dependent constants, we can attain polynomial sample complexity. The
majority of \pref{alg:main} can remain as-is, but the analysis will
replace the geometric decay of $A$ with the polynomial growth bound
above. This will increase our sample complexity by a $\poly(c_1
T^{c_2})$ factor, where $T$ is the time horizon.

The only difficulty is that we can no longer directly identify the
matrices $A$ and $B$ in Phase II. This is because our current analysis
uses the \emph{mixing} property of $A$, which entails that if $\rho(A)
< 1$, then for $t$ sufficiently large, under purely Gaussian
inputs $\matx_{t}$ and $\matx_{t+1}$ have similar
distributions. This ensures that predictors learned at time $t$ are similar to
those at time $t+1$. However, this is no longer true if $\rho(A) =
1$. To remedy this, we observe that it is still possible to recover
the controllability matrix $[B; AB; A^2 B;\dots; A^{k-1}B]$ from the
regression problem in Phase I up to a change of basis (see,
e.g., \citet{simchowitz2019learning} for guarantees for learning such
a matrix in the marginally stable setting). We can then recover the matrices $A$ and
$B$ from the controllability matrix up to orthogonal transformation
using the Ho-Kalman procedure (see \citet{oymak2019non} or \citet{sarkar2019finite} for refined guarantees). 
\paragraph{Relaxing the controllability assumption.}
If the system is not controllable, then we may not be able to recover
the state exactly. Instead, we can recover the state up to the
limiting-column space of the matrices $(\contk)$, which is always
attained for $k \leq \dimx$. We can then use this to run a weaker
controller (e.g., an observer-feedback controller) based on observations of the projection of the state onto this subspace. 
\paragraph{Other extensions.}
The assumption on the growth rate for $\Fclass$ can be replaced with
the bound $\|f(y)\| \le L\max\{1,\|\fst(y)\|^p\}\;\forall{}f\in\Fclass$ for any $p \ge 1$, at the expense of degrading the final sample complexity.
 \subsection{Invertibility of the $\cM_0$-matrix}
 \label{ssec:cM_asm}
As discussed in the main body, the matrix $\cM^{\trn}\cM$ is full rank whenever either $A$ or $B$ is full rank and the system is
controllable. Indeed, if $\mathrm{rank}(B) =\dimx$, then $\Mbar_{1}$, the first block of $\cMbar$, can be checked to be invertible. If $(A,B)$ are controllable, then $\Mbar_{\dimk}$ is full rank. Thus, if in addition, if $A$ is invertible, then the last block of $\cMbar$, $\Mbar_{\dimk} A^{\dimk-1}$, is full rank, ensuring our desired assumption holds. We are interested to understand if there are other more transparent conditions under which our recovery guarantees hold. \pref{ass:m_matrix} has the following immediate implication:
\begin{lemma}
	\label{lem:eiglem}
	Suppose \pref{ass:m_matrix} holds and let $\cM_{\sigma^2}$ denote the value of the matrix $\cM$ in \eqref{eq:bayes2} for noise
	parameter $\sigma^{2}$. Then, for all
	$\sigma^{2}$ sufficiently small, $\eigmin^{1/2}(\cM_{\sigma^2}^{\trn}\cM_{\sigma^2})\geq{}\lambda_{\cM}\cdot\sigma^{2} /2>0$.
\end{lemma}
\begin{proof}
	By continuity of the matrix inverse \cite[Theorem 2.2]{phien2012some}, we have \begin{align}\|\cM_{\sigma^2}^\top \cM_{\sigma^2} /\sigma^4 -\cMbar^\top \cMbar \|_{\textrm{F}} \stackrel{\sigma\to 0 }\rightarrow 0. \label{eq:froblim}\end{align}
	On the other hand, by \citep[Corollary~6.3.8]{HJ2012}, we have \begin{align*}
	|\eigmin (\cM_{\sigma^2}^\top \cM_{\sigma^2}/\sigma^4) -  \eigmin(\cMbar^\top \cMbar)| \leq  \|\cM_{\sigma^2}^\top \cM_{\sigma^2} /\sigma^4 -\cMbar^\top \cMbar \|_{\textrm{F}}.
	\end{align*}
	Combining this with \eqref{eq:froblim} implies that $\eigmin^{1/2} (\cM_{\sigma^2}^\top \cM_{\sigma^2})/\sigma^2\stackrel{\sigma\to0}{\rightarrow} \eigmin^{1/2} (\cMbar^\top \cMbar) =\lambda_{\cM}$, and so for all sufficiently small $\sigma$, we have $\eigmin^{1/2} (\cM_{\sigma^2}^\top \cM_{\sigma^2}) \geq \lambda_{\cM} \cdot \sigma^2/2$.
\end{proof}
 }

\clearpage

\section{Lower Bound for \richlqr Without Perfect Decodability}
\label{app:lower_bound}
\newcommand{\mstar}{m_{\star}}
\newcommand{\dmid}{\,\|\,}

\subsection{Formal Statement of Lower Bound}

In this section of the appendix we formally state and prove our sample
complexity lower bound for \richlqr without perfect decodability. The
protocol for the lower bound is as follows: The learning algorithm
\alg{} accesses the system \eqref{eq:noisy} through $n$ trajectories
on which it can play any (possibly adaptively chosen) sequence of
control inputs $\matu_{0:T}$ and observe $\maty_{0:T}$. At the end of this process, the algorithm outputs a decoder $\fhat_{\alg}$, and the prediction performance of the decoder (at time $t=1$) is measured under an arbitrary roll-in policy (chosen a-priori).
\begin{theorem}[Lower bound for \richlqr without perfect decodability.]
  \label{thm:lower_bound}
  Let $\matw_t,\mateps_t\sim\cN(0,1)$, let $n\geq{}n_0$, where $n_0$ is an
  absolute constant, and suppose we require that inputs are bounded so that $\abs*{\matu_t}\leq{}64\log^{1/2}n$. For every such $n$, there exists a function class $\Fclass$ with $\abs*{\Fclass}=2$ and system with $\dimx=\dimu=\dimy=1$ and $T=1$ such that for learning algorithm $\alg$ using only $n$ trajectories, and any roll-in policy $\pi$, we have
\[
  \En_{\alg}\En_{\pi}\brk*{\prn[\Big]{\hat{f}_{\alg}(\maty_1)-\fstar(\maty_1)}^{2}}\geq{}
  \bigom(1)\cdot{}\frac{1}{\log^{3/2}n}.
\]
Moreover, each $f\in\Fclass$ is $\bigoh(\log^{1/2}n)$-Lipschitz and invertible, with
$f'(y)\geq{}1$ for all $y\in\bbR$.
\end{theorem}
\pref{thm:lower_bound} shows that to learn a $\veps$-suboptimal decoder
under output noise for a particular function class $\Fclass$ with
$\abs*{\Fclass}=2$, any algorithm requires an exponential number of
samples. We note however that since the Lipschitz parameter for the
functions in the construction grows with $n$ (as $\log^{1/2}n$), the
construction does not rule out a sample complexity guarantee that is
polynomial in $1/n$ but exponential in the Lipschitz
parameter. Nonetheless, the algorithms we develop in this paper under
the perfect decodability assumption enjoy polynomial dependence on
both $1/n$ and the Lipschitz parameter, which the lower bound shows is
impossible under unit output noise. We remark that the constraint that
$\abs{\matu_t}\leq{}64\log^{1/2}n$ can be weakened to
$\abs{\matu_t}\leq{}C\log^{1/2}n$ for any $C\geq{}64$ at the cost of weakening the final
lower bound to $\frac{1}{C\log^{3/2}n}$. Finally, we remark that the lower bound only rules out learning a
$\veps$-optimal decoder, not an $\veps$-optimal policy; such a lower
bound may require a more sophisticated construction.

Beyond \pref{thm:lower_bound}, an additional challenge for solving
\richlqr without perfect decodability is that the optimal controller
is no longer reactive: since the problem is partially observable, the
optimal controller will in general depend on the entire history, which
makes it difficult to characterize its performance and analyze the
suboptimality of data-driven algorithms. We believe that developing more tractable models for \richlqr under weaker decodability assumptions is an important direction for future research.

\subsection{Additional Preliminaries}
For an $L_2$-integrable function $f:\bbR\to\bbR$, we define the Fourier transform $\wh{f}$ via
\[
\wh{f}(\omega) = \int{}e^{-i2\pi\omega{}x}f(x)dx.
\]
For functions $f,g:\bbR\to\bbR$, we let $f\conv{}g$ denote their
convolution, which is given by 
\[
(f\conv{}g)(x)=\int{}f(x-y)g(y)dy.
\]
For a pair of distributions $P\ll{}Q$ with densities $p$ and $q$, we
define
\[
\kl(P\dmid{}Q) = \int{}p(x)\log(p(x)/q(x))dx
\]
and
\[
\chisquared(P\dmid{}Q) = \int\frac{\prn{p(x)-q(x)}^{2}}{q(x)}dx.
\]

\subsection{Proof of \pref{thm:lower_bound}}

Throughout this proof we use $C$ to denote an absolute numerical
constant whose value may change from line to line.

We begin the proof by instantiating the LQR parameters. We set $T=1$, $\dimu=\dimx=\dimy=1$, $\matw_t\sim\cN(0,1)$ and $\mateps_t\sim{}\cN(0,1)$. We select $a=\frac{1}{2}$ (this
choice is arbitrary) and $b=1$. We assume
that $\matx_0$ is always initialized to the same value, and this value
is known to the learner. The precise value will be specified shortly,
but it will be chosen such that $\maty_0$ reveals no information about
the underlying instance. With the parameters above, the observation $\maty_1$ follows the following data-generating process:
\begin{equation}
  \label{eq:lb_dgp0}
  \begin{aligned}
    &\maty_1 = \fstar^{-1}(\matx_1) + \mateps_1\\
    &\matx_1 = \matu_0 + \tfrac{1}{2}\matx_0 +  \matw_0.
  \end{aligned}
\end{equation}
Since $\matx_0$ is known to the learner, we
reparameterize the control inputs for the sake of notational compactness via
$\matu_0\ldef{}\matu_0-\frac{1}{2}\matx_0$, so the data-generating
process simplifies to
\begin{equation}
  \label{eq:lb_dgp}
  \begin{aligned}
    &\maty_1 = \fstar^{-1}(\matx_1) + \mateps_1\\
    &\matx_1 = \matu_0 +  \matw_0.
  \end{aligned}
\end{equation}
The basic observation underlying our lower bound is that the
data-generating process \eqref{eq:lb_dgp} is an instance of the
classical \emph{error-in-variable regression} problem in the
\emph{Berkson} error model \citep{meister2009deconvolution,meister2010nonparametric,schennach2013regressions,schennach2016recent}. To emphasize the similarity to the
setting, we rebind the variables as $Y=\maty_1$,
$\veps=\veps_1$, $Z=\matx_1$, $X=\matu_0$, $W=\matw_0$, and
$\mstar=\fstar^{-1}$, so that \pref{eq:lb_dgp} becomes
\begin{equation}
  \label{eq:berkson}
  \begin{aligned}
    &Y = \mstar(Z) + \veps\\
    &Z = X +  W.
  \end{aligned}
\end{equation}
We can interpret $X$ (the control $\matu_0$) as a true covariate known
to the learner, and $Z$ (the state $\matx_1$) as an unobserved noisy version of
this covariate obtained by adding the noise $W$. The noisy covariate
is passed through the regression function $\mstar$, then the noise
$\veps$ is
added, leading to the target variable $Y$ (the observation $\maty_1$).

Ultra-slow $1/\log{}n$-type rates appear in many variants of the error-in-variable
regression problem
\citep{fan1993nonparametric,meister2009deconvolution,meister2010nonparametric}, as well as
the closely related nonparametric deconvolution problem
\citep{fan1991optimal}. Our lower bound is based on Theorem 2 of
\citet{meister2010nonparametric}, but with two important changes that
add additional complications to the analysis. First, we ensure that
the regression functions in our construction are invertible, so that
the perfect decodability assumption holds in absence of noise, and
second, our lower bound holds even for actively
chosen covariates, since
these correspond to control inputs chosen by the learner in the
\richlqr problem.

Rather than constructing a decoder class $\Fclass$ directly, it will be more convenient to construct a class of encoders $\Mclass$
(so that $\mstar\in\Mclass$), then take $\Fclass=\crl*{m^{-1}\mid{}m\in\Mclass}$
to be the induced decoder class.

Let $0<\alpha\leq{}1$, $\beta\geq{}1$, and $\gamma>0$ be parameters of the
construction. We define the following functions:
\begin{align}
  \begin{aligned}
    &r(z) = \gamma{}z,\\
    &\phi(z) = e^{-\frac{z^2}{2\beta^{2}}},\\
    &\psi(z) = \cos(4\pi\beta{}z),\\
    &h(z) = \alpha\phi(z)\psi(z).
  \end{aligned}
\label{eq:lb_defs}
\end{align}
We consider two alternate regression functions: $m_0(z) \ldef{} r(z) + h(z)$ and
$m_1(z) \ldef{} r(z) - h(z)$, and take $\Mclass=\crl*{m_0,m_1}$. We
define $f_i = m_i^{-1}$.
\begin{lemma}
  \label{lem:m_derivatives}
  For $m\in\crl*{m_0,m_1}$, we have
  \[
m'(z) \in\brk*{\gamma-14\alpha\beta,\gamma+14\alpha\beta}.
    \]
\end{lemma}
In light of this lemma, we will leave $\beta\geq{}1$ free for the time being, but
choose
\begin{equation}
\alpha=\frac{1}{28\beta^{2}}, \quad\text{and} \quad\gamma=\frac{1}{\beta},\label{eq:alpha_gamma}
\end{equation}
which ensures that
\begin{equation}
  \label{eq:derivative_bounds
  }
0< \frac{1}{2\beta} \leq{} m'(z) \leq{} \frac{3}{2\beta}.
\end{equation}
In particular, this implies that $m$ is $\frac{3}{2}$-Lipschitz and
invertible (since $\beta\geq{}1$). %

We now specify the starting state as $\matx_0=\frac{1}{8\beta}$. This ensures that
$\psi(\matx_0)=\cos(\pi/2)=0$, so that $m_0(\matx_1)=m_1(\matx_0)$, and
consequently the observation $\maty_0$ is statistically
independent of the underlying instance.

Let $\matx_0\ind{i},\matu_0\ind{i}$, $\matx_1\ind{i}$,
$\maty_1\ind{i}$, and so forth denote the
realizations of the sytem variables in the $i$th trajectory played by
the learner, and let
$S=(\maty_0\ind{1},\matu_0\ind{1},\maty_1\ind{1}, \matu_1\ind{1}),\ldots,(\maty_0\ind{n}\matu_0\ind{n},\maty_1\ind{n},\matu_1\ind{n})$
denote the observables collected throughout the entire learning
process. For
$i\in\crl*{0,1}$, we let $\bbP_{S;i}$ denote the law of $S$ when $m_i$ is
the true encoding function, and let $\En_{i}$ denote the expectation
under $\bbP_{S;i}$. We also let
$\bbP_{\maty_1|\matu_0;i}(\cdot{}\mid{}u)$ denote the
law of $\maty_1$ given $\matu_0$ when $\mstar=m_i$ and $p_i(y\mid{}u)$ be the
corresponding density (we suppress dependence on $\matx_0$, which
takes on the constant value $\frac{1}{8\beta}$ in both instances). Lastly, we let $\En_{\pi;
  i}$ denote the
expectation over $(\maty_0,\matu_0,\maty_1,\matu_1)$ when we roll in
with $\pi$ and $m_i$ is the underlying encoder.

Let $\fhat_{\alg}(\cdot)$ be the decoder returned by $\alg$, which we assume
to be $\sigma(S)$-measurable. We first observe that since the roll-in
policy has $\abs*{\matu_t}\leq{}\beta$ with probability $1$,
\pref{lem:loss_lower_bound} implies that 
\[
  \max_{i\in\crl*{0,1}}\En_{i}\En_{\pi;i}\brk*{\prn{\fhat_{\alg}(\maty_1)-f_i(\maty_1)}^2}
  \geq{}
      c\cdot{}\max_{i\in\crl*{0,1}}\En_{i}\brk*{\int_{-1}^{1}\prn{\fhat_{\alg}(y)-f_i(y)}^2dy},
    \]
    meaning that going forward we can dispense with the roll-in policy and lower bound the simpler quantity on the right-hand side above. Now, let $P_i$ denote the density corresponding to the
    law $\bbP_{S;i}$. We can further lower bound the worst-case risk
    of $\alg$ as
\begin{align*}
&  \max_{i\in\crl*{0,1}}\En_{i}\brk*{\int_{-1}^{1}(\hat{f}_{\alg}(y)-f_i(y))^{2}dy}
  \\
  &\geq{}\frac{1}{2}\brk*{\En_{0}\brk*{\int_{-1}^{1}(\hat{f}_{\alg}(y)-f_0(y))^{2}dy}
  + \En_{1}\brk*{\int_{-1}^{1}(\hat{f}_{\alg}(y)-f_1(y))^{2}dy}
    }\\
&\geq{}\frac{1}{2}\int_{-1}^{1}\brk*{\int_{\bbR^{4n}}\brk*{(\hat{f}_{\alg}(y)-f_0(y))^{2}+(\hat{f}_{\alg}(y)-f_1(y))^{2}}\min\crl*{P_0(S), P_1(S)}dS
    }dy\\
  &\geq{}\frac{1}{4}\int_{-1}^{1}(f_0(y)-f_1(y))^{2}dy\cdot{}
    \int_{\bbR^{4n}}\min\crl*{P_0(S), P_1(S)}dS\\
    &\geq{}\frac{1}{4}\int_{-1}^{1}(f_0(y)-f_1(y))^{2}dy\cdot{}\prn*{1-\frac{1}{2}
      \nrm*{P_{0}-P_{1}}_{L_1(\bbR^{4n})}}\\
  &=\frac{1}{4}\int_{-1}^{1}(f_0(y)-f_1(y))^{2}dy\cdot{}\prn*{1-\tv(\bbP_{S;0}\dmid\bbP_{S;1})}.
\end{align*}
If we choose $\beta=64\log^{1/2}n$ then our key technical lemma, \pref{lem:density_telescoping}, implies that
$\tv(\bbP_{S;0}\dmid\bbP_{S;1})=o(1)$.
\pref{lem:f_lower_bound} further implies that
$\int_{-1}^{1}(f_0(y)-f_1(y))^{2}dy\geq{}\frac{1}{8}\alpha^{2}\beta$,
so that when $n$ is sufficiently large we
have
\[
\max_{i\in\crl*{0,1}}\En_{i}\brk*{\int_{-1}^{1}(\hat{f}_{\alg}(y)-f_i(y))^{2}dx}\geq{}c\cdot{}\alpha^{2}\beta
= c\log^{-3/2}n.
\]

\qed
\subsection{Proofs for Supporting Lemmas}

\begin{proof}[\pfref{lem:m_derivatives}]
  We calculate that for $m\in\crl*{m_0, m_1}$, we have
  \[
    m'(z) = \gamma \pm
    \alpha\prn*{\frac{z}{\beta^{2}}e^{-\frac{z^2}{2\beta^{2}}}\cos(2\beta{}z)
      + 4\pi\beta{}e^{-\frac{z^2}{2\beta^{2}}}\sin(2\beta{}z)
      }.
  \]
Observe that $\abs*{\cos{}z},\abs*{\sin{}z},e^{-z^{2}}\leq{}1$, and 
\[
\abs*{\frac{z}{\beta^{2}}e^{-\frac{z^2}{2\beta^{2}}}} \leq{}
\frac{1}{\beta}\sup _{z}\abs*{ze^{-\frac{z^2}{2}}}\leq{} \frac{1}{\beta{}e^{1/2}}.
\]
It follows that
\[
f'(z) \in
\brk*{\gamma{}-\alpha\prn*{\frac{1}{\beta{}e^{1/2}}+4\pi{}\beta},
\gamma{}+\alpha\prn*{\frac{1}{\beta{}e^{1/2}}+4\pi\beta}}\subseteq{}\brk*{\gamma-14\alpha\beta,\gamma+14\alpha\beta},
\]
where we have used that $\beta\geq{}1$.
\end{proof}

\begin{lemma}
  \label{lem:density_telescoping}
If we choose $\beta=64\log^{1/2}n$, then for all $n\geq{}3$ we have  
  \[
    \tv^{2}(\bbP_{S;0}\dmid{}\bbP_{S;1}) \leq{} Cn^{-4}.
  \]
\end{lemma}
\begin{proof}[\pfref{lem:density_telescoping}]
  To begin, we apply Pinsker's inequality:
  \[
    \tv^{2}(\bbP_{S;0}\dmid{}\bbP_{S;1}) \leq{} \frac{1}{2}\kl(\bbP_{S;0}\dmid{}\bbP_{S;1}).
  \]
  Let $o\ind{j}=(y_0\ind{1},u_0\ind{1},y_1\ind{2},u_1\ind{2})$. We observe that then density $P_i(o\ind{1},\ldots,o\ind{n})$ factorizes as
  \begin{align*}
    &P_i(o\ind{1},\ldots,o\ind{n}) =\\
    &\prod_{j=1}^{n}p_{\maty_0;i}(y_0\ind{j})p_{\matu\ind{j}_0}(u_0\ind{j}\mid{}o\ind{1},\ldots,o\ind{t-1},y_0\ind{j})
    p_{\maty_1\mid\matu_0;i}(y_1\ind{j}\mid{}u_0\ind{j}) p_{\matu\ind{j}_1}(u_1\ind{j}\mid{}o\ind{1},\ldots,o\ind{t-1},y_0\ind{j},u_0\ind{j},y_1\ind{j}),
  \end{align*}
  where $p_{\maty_0;i}$ is the density for $\maty_0$ under instance
  $i$, $p_{\matu_0\ind{j}}$ and $p_{\matu_1\ind{j}}$
      are the conditional densities for $\matu_0\ind{j}$ and
      $\matu_1\ind{j}$ given all preceding observations, and
      $p_{\maty_1\mid{}\matu_0;i}$ is the conditional density for
      $\maty_1$ given $\matu_0$ under instance $i$. The densities $p_{\matu\ind{j}_0}$ and $p_{\matu_1\ind{j}}$ do not
  depend on the instance $i$, nor does the density $p_{\maty_0;i}$
  (recall that the choice of starting state $\matx_0=\frac{1}{8\beta}$
  guarantees $m_0(\matx_0)=m_1(\matx_0)$, so $\maty_0=\mateps_0$ in
  law for both instances). We conclude that the KL divergence
  telescopes as
  \begin{align*}
\kl(\bbP_{S;0}\dmid{}\bbP_{S;1}) &=  \sum_{j=1}^{n}\En_{0}\brk*{\kl\prn*{
        \bbP_{\maty_1\mid\matu_0;0}(\cdot{}\mid{}\matu_0\ind{j})\dmid\bbP_{\maty_1\mid\matu_0;1}(\cdot{}\mid{}\matu_0\ind{j})
                                   }}\\
    &\leq{}  \sum_{j=1}^{n}\En_{0}\brk*{\chisquared\prn*{
      \bbP_{\maty_1\mid\matu_0;0}(\cdot{}\mid{}\matu_0\ind{j})\dmid\bbP_{\maty_1\mid\matu_0;1}(\cdot{}\mid{}\matu_0\ind{j})
      }}.
  \end{align*}
  Since the algorithm satisfies $\abs{\matu_0\ind{j}},
  \abs{\matu_1\ind{j}}\leq{}\beta$ almost surely, we can apply
  \pref{lem:chisquared_ub} to each summand, which gives
  \[
    \kl(\bbP_{S;0}\dmid{}\bbP_{S;1})\leq{}Cn^{-9}.
  \]
\end{proof}

\begin{lemma}
  \label{lem:chisquared_ub}
  If we choose $\beta=64\log{}n$, then for all $n\geq{}3$ and all
  $\abs*{u}\leq{}\beta$, we have
  \begin{equation}
    \label{eq:chisquared_ub}
      \chisquared\prn*{
        \bbP_{\maty_1|\matu_0;0}(\cdot{}\mid{}u)\dmid\bbP_{\maty_1\mid\matu_0;1}(\cdot{}\mid{}u)
        }
        \leq{} Cn^{-10}.
  \end{equation}
\end{lemma}

\begin{proof}[\pfref{lem:chisquared_ub}]
  \newcommand{\pveps}{p_{\mateps}}
  \newcommand{\pw}{p_{\matw}}
  Recall that we let $p_i$ denote the conditional density for
  $\bbP_{\maty_1|\matu_0;i}(\cdot{}\mid{}u)$. Let $\pveps(\veps)=e^{-\frac{1}{2}\veps^{2}}$ denote the density of
$\mateps$ and $\pw(w)=e^{-\frac{1}{2}w^{2}}$ denote the density of
$\matw$. Observe that for each $i$, we have
  \[
    p_i(y\mid{}u) = \frac{1}{\sqrt{2\pi}}
    \int{}p_{\mateps}(y-f_i(u+w))p_{\matw}(w)dw.
  \]
  It follows that
  \begin{align*}
    &\chisquared\prn*{
        \bbP_{\maty_1\mid{}\matu_0;0}(\cdot{}\mid{}u)\dmid\bbP_{\maty_1\mid{}\matu_0;1}(\cdot{}\mid{}u)
      } \\
    &= \frac{1}{\sqrt{2\pi}}\int{}p^{-1}_1(y\mid{}u)\cdot\abs*{
    \int{}\brk*{p_{\mateps}(y-f_0(u+w))-p_{\mateps}(y-f_1(u+w))}p_{\matw}(w)dw
    }^2dy.
  \end{align*}
  By \pref{lem:density_lb} (with $\eta=1/5$), we have
  \begin{align*}
    p^{-1}_i(y\mid{}u) &\leq{}
                         3^{1/2}\exp\prn*{\frac{(1+1/5)(y-\gamma{}u)^{2}}{2}
                         + 5} \\
    &\leq{} 3^{1/2}\exp\prn*{\frac{(1+1/5)^2}{2}y^{2}
      +\frac{5(1+1/5)}{2}\gamma^{2}u^2+ 5}.
  \end{align*}
  Since$\abs*{u}\leq{}\beta$, $\gamma^{2}u^{2}\leq{}1$, so we can further
  simplify to
  \[
    p^{-1}_i(y\mid{}u) \leq{} C\cdot{}\exp\prn*{\frac{3}{4}y^{2}}.
  \]
  Consequently, we have
  \begin{align*}
    &\chisquared\prn*{
      \bbP_{\maty_1\mid{}\matu_0;0}(\cdot{}\mid{}u)\dmid\bbP_{\maty_1\mid{}\matu_0;1}(\cdot{}\mid{}u)
      } \\
    &\leq{} C\int{}e^{\frac{3}{4}y^2}\abs*{
    \int{}\brk*{p_{\mateps}(y-f_0(u+w))-p_{\mateps}(y-f_1(u+w))}p_{\matw}(w)dw
    }^2dy.
  \end{align*}
  Using the Taylor series representation for $\pveps$, we have
  \begin{align*}
    p_{\mateps}(y-f_i(u+w)) =  \sum_{k=0}^{\infty}\frac{1}{k!}p_{\mateps}\ind{k}(y)(-f_i(u+w))^{k},
  \end{align*}
  and so
  \begin{align*}
        &\chisquared\prn*{
      \bbP_{\maty_1\mid{}\matu_0;0}(\cdot{}\mid{}u)\dmid\bbP_{\maty_1\mid{}\matu_0;1}(\cdot{}\mid{}u)
      } \\
        &\leq{} C\int{}e^{\frac{3}{4}y^2}\abs*{
          \sum_{k=0}^{\infty}\frac{1}{k!}p_{\mateps}\ind{k}(y)
          \int{}\brk*{(-f_0(u+w))^{k} - (-f_1(u+w))^{k}}p_{\matw}(w)dw
    }^2dy.
  \end{align*}
  Applying the Cauchy-Schwarz inequality to the series, we can further
  upper bound by
  \begin{align*}
    &C\int{}e^{\frac{3}{4}y^2}\prn*{
    \sum_{k=0}^{\infty}\frac{2^{-2k}}{k!}(p_{\mateps}\ind{k}(y))^2}
    \prn*{\sum_{k=0}^{\infty}\frac{2^{2k}}{k!}\prn*{\int{}\brk*{(f_0(u+w))^{k} - (f_1(u+w))^{k}}p_{\matw}(w)dw}^2
      }dy\\
    &=C\prn*{
      \sum_{k=0}^{\infty}\frac{2^{-2k}}{k!}\int{}e^{\frac{3}{4}y^2}(p_{\mateps}\ind{k}(y))^2 dy}
    \prn*{\sum_{k=0}^{\infty}\frac{2^{2k}}{k!}\prn*{\int{}\brk*{(f_0(u+w))^{k} - (f_1(u+w))^{k}}p_{\matw}(w)dw}^2
    }.
  \end{align*}
  We first bound the left term involving the density $\pveps$. Let
$H_k(y)=(-1)^{k}e^{\frac{y^{2}}{2}}\frac{d^{k}}{dy^{k}}e^{-\frac{y^2}{2}}$
denote the probabilist's $k$th Hermite polynomial, so that
$\pveps^{(k)}(y) = (-1)^{k}H_k(y)e^{-\frac{1}{2}y^2}$. Then we have
\begin{align*}
  \int e^{\frac{3}{4}y^2}\abs*{p_{\mateps}\ind{k}(y)}^2dy &=
                                                          \int{}e^{\frac{3}{4}y^2}
                                                          \cdot{}H^2_k(y)e^{-y^{2}}dy\\
                                                        &=
                                                          \int{}H^2_k(y)e^{-\frac{1}{4}y^{2}}dy\\
                                                        &\overset{(i)}{=}
                                                          2^{k}\int{}H^2_k(y/\sqrt{2})e^{-\frac{1}{2}(y/\sqrt{2})^{2}}dy\\
                                                          &=
                                                            \sqrt{2}\cdot{}2^{k}\int{}H^2_k(y)e^{-\frac{1}{2}y^{2}}dy\\
                                                        &\leq{} C\cdot{}2^{k}k!,
\end{align*}
where $(i)$ uses that $H_k$ is a degree-$k$ polynomial. Applying this
inequality for each $k$, we have
\[
\sum_{k=0}^{\infty}\frac{2^{-2k}}{k!}\int{}e^{\frac{3}{4}y^2}(p_{\mateps}\ind{k}(y))^2
dy
\leq{} C\cdot{}\sum_{k=0}^{\infty}2^{-k}\leq{}C,
\]
and so
\begin{align*}
  \chisquared\prn*{
      \bbP_{\maty_1\mid{}\matu_0;0}(\cdot{}\mid{}u)\dmid\bbP_{\maty_1\mid{}\matu_0;1}(\cdot{}\mid{}u)
      }\leq{}
  C\cdot{}
    \sum_{k=0}^{\infty}\frac{2^{2k}}{k!}\prn*{\int{}\brk*{(f_0(u+w))^{k} - (f_1(u+w))^{k}}p_{\matw}(w)dw}^2.
\end{align*}
    Next, using the binomial theorem, for any $x\in\bbR$ we have
    \begin{align*}
      (f_0(x))^{k} - (f_1(x))^{k}
      &=       (r(x)+h(x))^{k} - (r(x)-h(x))^{k}\\
      &= \sum_{j=0}^{k}{k\choose j}r^{k-j}(x)h^{j}(x)(1-(-1)^{k})\\
      &= 2\sum_{j\leq{}k,\text{ odd}}{k\choose j}r^{k-j}(x)h^{j}(x),
    \end{align*}
    leading to the upper bound
    \begin{align*}
      &\chisquared\prn*{
      \bbP_{\maty_1\mid{}\matu_0;0}(\cdot{}\mid{}u)\dmid\bbP_{\maty_1\mid{}\matu_0;1}(\cdot{}\mid{}u)
      }\\
      &\leq{}
  C\cdot{}
    \sum_{k=0}^{\infty}\frac{2^{2k}}{k!}\prn*{\sum_{j\leq{}k,\text{ odd}}{k\choose j}
        \int{}
        r^{k-j}(u+w)h^{j}(u+w)p_{\matw}(w)dw}^2\\
      &\leq{}
        C\sum_{k=0}^{\infty}\frac{2^{2k}k}{k!}\sum_{j\leq{}k,\text{ odd}}{k\choose j}\abs*{
        \int{}
        r^{k-j}(u+w)h^{j}(u+w)p_{\matw}(w)dw
        }^2 \\
      &=
        C\sum_{k=0}^{\infty}\frac{2^{2k}k}{k!}\sum_{j\leq{}k,\text{ odd}}{k\choose j}\abs*{
        (r^{k-j}h^{j}\conv{}p_{\matw})(u)
        }^2\\
&\leq{}
        C\sum_{k=0}^{\infty}\frac{2^{2k}k}{k!}\sum_{j\leq{}k,\text{ odd}}{k\choose j}\sup_{u\in\bbR}\abs*{
              (r^{k-j}h^{j}\conv{}p_{\matw})(u)
              }^2\\
&\leq{}
        C\sum_{k=0}^{\infty}\frac{2^{3k}k}{k!}\max_{j\leq{}k,\text{ odd}}\sup_{u\in\bbR}\abs*{
              (r^{k-j}h^{j}\conv{}p_{\matw})(u)
      }^2,
    \end{align*}
    where the equality holds because $p_w$ is symmetric. We now appeal
    to \pref{lem:fourier_bound} for each term in the sum, which leads
    to an upper bound of 
    \begin{align*}
      &      C\sum_{k=0}^{\infty}\frac{2^{3k}k}{k!}\max_{j\leq{}k,\text{
      odd}}
        \prn*{\gamma^{k-j}\alpha^{j}\beta^{(k-j+1)/2}\cdot{}j\sqrt{(k-j)!}
      \cdot
        \exp\prn*{-2\pi^{2}\prn*{\frac{\beta^2}{j}\wedge{}1}\beta^{2}}}^2\\
      & \leq{} C\sum_{k=0}^{\infty}2^{3k}k^3\max_{j\leq{}k,\text{
      odd}}
        \prn*{\gamma^{k-j}\alpha^{j}\beta^{(k-j+1)/2}
      \cdot
                     \exp\prn*{-2\pi^{2}\prn*{\frac{\beta^2}{j}\wedge{}1}\beta^{2}}}^2.
    \end{align*}
Recalling the choice $\alpha=\frac{1}{12\beta^{2}}$ and
$\gamma=1/\beta$, we can upper bound
\[
\gamma^{k-j}\alpha^{j}\beta^{(k-j+1)/2} \leq{} \beta^{-k/2}
\]
for each term above, so we have
\begin{align*}
      &\leq{}    C\sum_{k=0}^{\infty}2^{3k}k^3\max_{j\leq{}k,\text{
      odd}}\beta^{-k}
      \cdot
      \exp\prn*{-4\pi^{2}\prn*{\frac{\beta^2}{j}\wedge{}1}\beta^{2}}.
\end{align*}
Since $\beta\geq{}64$ for $n\geq{}3$, we have
$\beta^{-k}\leq{}2^{-6k}$, so we can upper bound the sum
above as
\begin{align*}
      &\leq{}    C\sum_{k=0}^{\infty}2^{-2k}k^3\max_{j\leq{}k,\text{
      odd}}2^{-k}
      \cdot
        \exp\prn*{-4\pi^{2}\prn*{\frac{\beta^2}{j}\wedge{}1}\beta^{2}}.
\end{align*}
We now consider two cases for the term in the $\max$ above. First, if
$j\leq{}\beta^{2}$, then we have
$\exp\prn*{-4\pi^{2}\prn*{\frac{\beta^2}{j}\wedge{}1}\beta^{2}}\leq{}\exp\prn*{-4\pi^{2}\beta^{2}}$. Otherwise,
we have $k\geq{}j\geq{}\beta^{2}$, so
$2^{-k}\leq{}2^{-\beta^{2}}$. Putting the two cases together (using that
$\exp\prn*{-4\pi^{2}\beta^{2}}\leq{}2^{-\beta^2}$), we get the
following coarse upper bound:
  \begin{align*}
    C2^{-\beta^{2}}\sum_{k=0}^{\infty}2^{-k}k^{3}\leq{}C2^{-\beta^{2}}.
\end{align*}
The choice $\beta=64\log^{1/2}n$ implies that $2^{-\beta^{2}}\leq{}n^{-10}$.

\end{proof}
\begin{lemma}
  \label{lem:density_lb}
Let $\eta\leq{}1$ be given. Then for each $i\in\crl*{0,1}$, we have  
  \[
    p_i(y\mid{}u) \geq 3^{-1/2}\exp\prn*{-\prn*{\frac{(1+\eta)(y-\gamma{}u)^{2}}{2} + \frac{1}{\eta}}}.
  \]
\end{lemma}
\begin{proof}[\pfref{lem:density_lb}]
  We have
  \begin{align*}
    p_i(y\mid{}u) &= \frac{1}{\sqrt{2\pi}}
                    \int{}p_{\mateps}(y-f_i(u+w))p_{\matw}(w)dw\\\
    &= \frac{1}{\sqrt{2\pi}}
      \int{}\exp\prn*{-\frac{1}{2}(y-r(u+w) \pm h(u+w))^2}p_{\matw}(w)dw.
  \end{align*}
  Using the AM-GM inequality, we have that for any $\eta>0$, this is
  lower bounded by
  \begin{align*}
    \frac{1}{\sqrt{2\pi}}
    \int{}\exp\prn*{-\frac{1+\eta}{2}(y-r(u+w))^2}\exp\prn*{-\frac{1+1/\eta}{2}h^{2}(u+w)}p_{\matw}(w)dw.
  \end{align*}
  We will restrict to $\eta<1$. Since $\abs*{h}\leq{}\alpha<1$ everywhere, we can further lower bound
by
\begin{align*}
  &\frac{\exp\prn*{-\frac{1+1/\eta}{2}}}{\sqrt{2\pi}}
    \int{}\exp\prn*{-\frac{1+\eta}{2}(y-r(u+w))^2}p_{\matw}(w)dw\\
    &\geq{}\frac{e^{-\frac{1}{\eta}}}{\sqrt{2\pi}}
      \int{}\exp\prn*{-\frac{1+\eta}{2}(y-r(u+w))^2}p_{\matw}(w)dw\\
  &=\frac{e^{-\frac{1}{\eta}}}{\sqrt{2\pi}}
    \int{}\exp\prn*{-\frac{1+\eta}{2}(y-r(u+w))^2}\exp\prn*{-\frac{1}{2}w^{2}}dw.
\end{align*}
Define $\mu=y-\gamma{}u$, $\sigma^2=(1+(1+\eta)\gamma^2)^{-1}$, and
$\mu' = (1+\eta)\gamma\sigma^{2}\mu$. Then by completing the square,
we have
\[
  \exp\prn*{-\frac{1+\eta}{2}(y-r(u+w))^2}\exp\prn*{-\frac{1}{2}w^{2}}
  = \exp\prn*{-\frac{(1+\eta)\mu^{2}}{2(1+(1+\eta)\gamma^2)}}\cdot{}
  \exp\prn*{-\frac{(w-\mu')^2}{2\sigma^2}}.
\]
It follows that
\begin{align*}
  \int{}\exp\prn*{-\frac{1+\eta}{2}(y-r(u+w))^2}\exp\prn*{-\frac{1}{2}w^{2}}dw
  &=
  \exp\prn*{-\frac{(1+\eta)\mu^{2}}{2(1+(1+\eta)\gamma^2)}}\cdot\sqrt{2\pi\sigma^{2}}\\
  &\geq{} \exp\prn*{-\frac{(1+\eta)\mu^{2}}{2}}\cdot\sqrt{\frac{2\pi}{3}}.
\end{align*}

\end{proof}

\begin{lemma}
    \label{lem:fourier_bound}
    There is a universal constant $C>0$ such that for all $k$ and $j\leq{}k$
    with $j$ odd,
    \begin{equation}
      \label{eq:fourier_bound}
      \sup_{x\in\bbR}\abs*{(r^{k-j}h^{j}\conv{}p_{\matw})(x)}\leq{}
      C\cdot{}
      \gamma^{k-j}\alpha^{j}\beta^{(k-j+1)/2}\cdot{}j\sqrt{(k-j)!}
      \cdot
      \exp\prn*{-2\pi^{2}\prn*{\frac{\beta^2}{j}\wedge{}1}\beta^{2}}
    \end{equation}
  \end{lemma}
  \begin{proof}[\pfref{lem:fourier_bound}]
    Let $x\in\bbR$ be fixed. Then, using the Fourier inversion
    formula (using that both $r^{k-j}h^{j}$, $p_{\matw}$, and their respective
    Fourier transforms are $L_2$-integrable), we have
    \begin{align*}
      \abs*{(r^{k-j}h^{j}\conv{}p_{\matw})(x)}
      =
      \abs*{\int{}e^{i2\pi{}x\omega}\wh{(r^{k-j}h^{j})}(\omega)\wh{p_{\matw}}(\omega)d\omega}
      \leq{}\int{}\abs*{\wh{(r^{k-j}h^{j})}(\omega)\wh{p_{\matw}}(\omega)}d\omega.
    \end{align*}
    We proceed to compute the Fourier transform for
    $r^{k-j}(x)h^{j}(x)=\gamma^{k-j}\alpha^{j}x^{k-j}\phi^{j}(x)\psi^{j}(x)$. We
    first observe that $\phi^{j}(x) =
    \exp\prn{-\frac{j}{\beta^{2}}\cdot\frac{z^{2}}{2}}$. Let
    $b_1=\frac{\beta^{2}}{j}$. Then, using that the Fourier transform is
    self-dual for gaussians (specifically, that the Fourier transform
    of $e^{-cx^{2}}$ is $\sqrt{\frac{\pi{}}{c}}e^{-\frac{\pi^{2}}{c}\omega^{2}}$), we have
    \[
      \wh{\phi^{j}}(\omega) = \sqrt{2\pi{}b_1}e^{-2\pi^{2}b_1\omega^{2}}.
    \]
Next, we recall that for any $f$, the Fourier transform of $x^{n}f(x)$
is $\prn*{\frac{i}{2\pi}}^{n}\frac{d^n}{d\omega^{n}}\wh{f}(\omega)$,
so that
\begin{align*}
 \wh{x^{k-j}\phi^{j}}(\omega) &=
 \prn*{\frac{i}{2\pi}}^{k-j}\sqrt{2\pi{}b_1}\cdot{}\frac{d^{k-j}}{d\omega^{k-j}}e^{-2\pi^{2}b\omega^{2}}\\
                              &= \prn*{\frac{i}{2\pi}}^{k-j}\sqrt{2\pi{}b_1}b_2^{k-j}\cdot{}H_{k-j}(b_2\omega{})e^{-\frac{(b_2\omega)^2}{2}}.
\end{align*}
where $b_2\ldef{}2\pi\sqrt{b_1}$. Finally, we use that
\begin{align*}
  \psi^j(x) = (\cos(4\pi{}\beta{}x))^{j} &= \frac{1}{2^{j}}(e^{i4\pi{}\beta{}x}
                                      + e^{-i4\pi{}\beta{}x}) \\
                                    &=
                                      \frac{1}{2^{j}}\sum_{l=0}^{j}{j\choose{}l}e^{i4\pi{}\beta{}x\cdot{}(j-l)}\cdot{}e^{-i4\pi{}\beta{}x\cdot{}l}\\
  &= \frac{1}{2^{j}}\sum_{l=0}^{j}{j\choose{}l}e^{i4\pi{}\beta{}x\cdot{}(j-2l)}
\end{align*}
We now use that the Fourier transform of $e^{-icx}f(x)$ is
$\wh{f}(\omega-\frac{c}{2\pi})$ to derive
\begin{align*}
 \wh{x^{k-j}\phi^{j}\psi^{j}}(\omega) =
\frac{1}{2^j}\prn*{\frac{i}{2\pi}}^{k-j}\sqrt{2\pi{}b_1}b_2^{k-j}
  \sum_{l=0}^{j}{j\choose l}H_{k-j}(b_2(\omega{}-2\beta(j-2l))e^{-\frac{(b_2(\omega{}-2\beta(j-2l))^2}{2}}
\end{align*}
It follows that
\begin{align*}
  &  \int{}\abs*{\wh{(r^{k-j}h^{j})}(\omega)\wh{p_{\matw}}(\omega)}d\omega\\
  &\leq{}
    \gamma^{k-j}\alpha^{j}\frac{1}{2^j}\prn*{\frac{1}{2\pi}}^{k-j}\sqrt{2\pi{}b_1}b_2^{k-j}
  \sum_{l=0}^{j}{j\choose
    l}\int\abs*{H_{k-j}(b_2(\omega{}-2\beta(j-2l))e^{-\frac{(b_2(\omega{}-2\beta(j-2l))^2}{2}}\wh{p_w}(\omega)}d\omega\\
    &\leq
\gamma^{k-j}\alpha^{j}      \frac{1}{2^j}\prn*{\frac{1}{2\pi}}^{k-j}\sqrt{2\pi{}b_1}b_2^{k-j}
  \sum_{l=0}^{j}{j\choose l}\int\abs*{H_{k-j}(b_2(\omega{}-2\beta(j-2l))}e^{-\frac{(b_2(\omega{}-2\beta(j-2l))^2}{2}}e^{-2\pi^{2}\omega^{2}}d\omega
\end{align*}
Now, let $0\leq{}l\leq{}j$ be fixed. We bound
\begin{align*}
  &\int\abs*{H_{k-j}(b_2(\omega{}-2\beta(j-2l))}e^{-\frac{(b_2(\omega{}-2\beta(j-2l))^2}{2}}e^{-2\pi^{2}\omega^{2}}d\omega\\
&\leq{}\underbrace{\int_{(-\beta,\beta)}\abs*{H_{k-j}(b_2(\omega{}-2\beta(j-2l))}e^{-\frac{(b_2(\omega{}-2\beta(j-2l))^2}{2}}e^{-2\pi^{2}\omega^{2}}d\omega}_{(\star)}\\
&~~~~+\underbrace{\int_{\bbR\setminus(-\beta,\beta)}\abs*{H_{k-j}(b_2(\omega{}-2\beta(j-2l))}e^{-\frac{(b_2(\omega{}-2\beta(j-2l))^2}{2}}e^{-2\pi^{2}\omega^{2}}d\omega}_{(\star\star)}.
\end{align*}
For the integral in the term $(\star)$, we drop the $e^{-2\pi^{2}w^{2}}$ term (since
it is at most one), and apply Cauchy-Schwarz to bound by
\begin{align*}
&  \int_{(-\beta,\beta)}\abs*{H_{k-j}(b_2(\omega{}-2\beta(j-2l))}e^{-\frac{(b_2(\omega{}-2\beta(j-2l))^2}{2}}d\omega\\
  &\leq{}\sqrt{\int_{(-\beta,\beta)}H^2_{k-j}(b_2(\omega{}-2\beta(j-2l))e^{-\frac{(b_2(\omega{}-2\beta(j-2l))^2}{2}}d\omega}\cdot \sqrt{\int_{(-\beta,\beta)}e^{-\frac{(b_2(\omega{}-2\beta(j-2l))^2}{2}}d\omega}
\end{align*}
Observe that since $j$ is odd, $j-2l$ is also odd, and hence
$\abs*{j-2l}\geq{}1$. It follows that for $\omega\in(-\beta,\beta)$,
$\omega{}-2\beta(j-2l)\notin(-\beta,\beta)$, and so
\[
  \int_{(-\beta,\beta)}e^{-\frac{(b_2(\omega{}-2\beta(j-2l))^2}{2}}d\omega
  \leq{} \int_{(-\beta,\beta)}e^{-\frac{b_2^2}{2}\beta^{2}}d\omega
  \leq{} 2\beta{}e^{-\frac{b_2^2}{2}\beta^{2}}.
\]
Leaving the Hermite integral for a moment and moving to the second
term $(\star\star)$, we have
\begin{align*}
&\int_{\bbR\setminus(-\beta,\beta)}\abs*{H_{k-j}(b_2(\omega{}-2\beta(j-2l))}e^{-\frac{(b_2(\omega{}-2\beta(j-2l))^2}{2}}e^{-2\pi^{2}\omega^{2}}d\omega\\
  &\leq{}e^{-2\pi^{2}\beta^{2}}\int_{\bbR\setminus(-\beta,\beta)}\abs*{H_{k-j}(b_2(\omega{}-2\beta(j-2l))}e^{-\frac{(b_2(\omega{}-2\beta(j-2l))^2}{2}}d\omega\\
  &\leq{}e^{-2\pi^{2}\beta^{2}}\sqrt{\int_{\bbR\setminus(-\beta,\beta)}H_{k-j}^{2}(b_2(\omega{}-2\beta(j-2l))e^{-\frac{(b_2(\omega{}-2\beta(j-2l))^2}{2}}d\omega}\cdot\sqrt{\int_{\bbR\setminus(-\beta,\beta)}e^{-\frac{(b_2(\omega{}-2\beta(j-2l))^2}{2}}d\omega}\\
&\leq{}e^{-2\pi^{2}\beta^{2}}\sqrt{\int_{\bbR\setminus(-\beta,\beta)}H_{k-j}^{2}(b_2(\omega{}-2\beta(j-2l))e^{-\frac{(b_2(\omega{}-2\beta(j-2l))^2}{2}}d\omega}\cdot\sqrt{\frac{2\pi}{b_2}}.
\end{align*}
Putting both cases together, we have
\begin{align*}
  &\int\abs*{H_{k-j}(b_2(\omega{}-2\beta(j-2l))}e^{-\frac{(b_2(\omega{}-2\beta(j-2l))^2}{2}}e^{-2\pi^{2}\omega^{2}}d\omega\\
  &\leq{}    C\cdot{}(\sqrt{\beta}\vee{}1/\sqrt{b_2})\exp(-2\pi^{2}(b_1^2\wedge{}1)\beta^{2})\cdot{}\sqrt{\int{}H_{k-j}^{2}(b_2(\omega{}-2\beta(j-2l))e^{-\frac{(b_2(\omega{}-2\beta(j-2l))^2}{2}}d\omega},
\end{align*}
where $C$ is a numerical constant. Using a change of variables, we
have
\begin{align*}
  \sqrt{\int{}H_{k-j}^{2}(b_2(\omega{}-2\beta(j-2l))e^{-\frac{(b_2(\omega{}-2\beta(j-2l))^2}{2}}d\omega}  &=\frac{1}{\sqrt{b_2}}\sqrt{\int{}H_{k-j}^{2}(\omega)e^{-\frac{\omega^2}{2}}d\omega}\\
  &\leq{}\sqrt{\frac{2\pi(k-j)!}{b_2}}.
\end{align*}
Since this bound holds uniformly for all $l$ and
$\sum_{l=0}^{j}{j\choose l}=2^{j}$, we have
\begin{align*}
&  \int{}\abs*{\wh{(r^{k-j}h^{j})}(\omega)\wh{p_{\matw}}(\omega)}d\omega\\
  &\leq{} C\cdot{}
    \gamma^{k-j}\alpha^{j}\prn*{\frac{1}{2\pi}}^{k-j}\sqrt{2\pi{}b_1}b_2^{k-j}\cdot{}\sqrt{\frac{2\pi(k-j)!}{b_2}}
  \cdot
    (\sqrt{\beta}\vee{}1/\sqrt{b_2})\exp\prn*{-2\pi^{2}(b_1^2\wedge{}1)\beta^{2}}\\
    &\leq{} C'\cdot{}
      \gamma^{k-j}\alpha^{j} b_1^{(k-j)/2}\cdot{}(\sqrt{\beta}\vee{}1/\sqrt{b_1})\cdot\sqrt{(k-j)!}
      \cdot
      \exp\prn*{-2\pi^{2}(b_1^2\wedge{}1)\beta^{2}}\\
  &\leq{} C''\cdot{}
    \gamma^{k-j}\alpha^{j}\beta^{(k-j+2)/2}j\cdot{}\sqrt{(k-j)!}
      \cdot
      \exp\prn*{-2\pi^{2}(b_1^2\wedge{}1)\beta^{2}}.
\end{align*}
  \end{proof}

  \begin{lemma}
    \label{lem:loss_lower_bound}
    For any non-negative function $g:\bbR\to\bbR_{+}$ and any roll-in
    policy $\pi$ with $\abs*{\matu_0}\leq{}\beta$ almost surely,
    \[
      \En_{\pi;i}\brk*{g(\maty_1)}\geq{}
      c\cdot{}\int_{-1}^{1}g(y)dy\quad\text{for all $i\in\crl*{0,1}$},
    \]
    where $c$ is an absolute numerical constant.
  \end{lemma}
  \begin{proof}[\pfref{lem:loss_lower_bound}]
    Observe that we have
    \begin{align*}
      \En_{\pi;i}\brk*{g(\maty_1)} &= \En_{\matu_0;i}\brk*{
      \int_{-\infty}^{\infty}g(y)p_i(y\mid{}\matu_0)dy
                                     }\\
                                   &\geq{} \En_{\matu_0;i}\brk*{
                                     \int_{-1}^{1}g(y)p_i(y\mid{}\matu_0)dy
                                     }.
    \end{align*}
    \pref{lem:density_lb} (with $\eta=1$) implies that for all $y\in\brk*{-1,1}$ and
    $\abs*{u}\leq{}\beta$,
      \[
        p_i(y\mid{}u) \geq 3^{-1/2}\exp\prn*{-\prn*{(y-\gamma{}u)^{2}
            + 1}} \geq{} c.
      \]
      It follows that
      \[
        \En_{\matu_0;i}\brk*{
          \int_{-1}^{1}g(y)p_i(y\mid{}\matu_0)dy
        }
        \geq{} c\cdot{}\int_{-1}^{1}g(y)dy.
      \]
    \end{proof}

    \begin{lemma}
      \label{lem:f_lower_bound}
      If $\beta\geq{}1$ and $\alpha$ and $\gamma$ are chosen as in
      \pref{eq:alpha_gamma}, we have
      \[
        \int_{-1}^{1}(f_0(y)-f_1(y))^{2}dy\geq{}\frac{1}{8}\alpha^2\beta.
      \]
    \end{lemma}
    \begin{proof}[\pfref{lem:f_lower_bound}]
      Recall that $m_0=f_0^{-1}$ and $m_1=f_1^{-1}$. Throughout the
      proof we will use that
      \[
        \frac{1}{2\beta}\leq{}m'_i(z)\leq\frac{3}{2\beta},\quad\text{and}\quad
        \frac{2\beta}{3}\leq{}f'_i(y)\leq{}2\beta.
      \]As a first step, we have
      \begin{align*}
        \int_{-1}^{1}(f_0(y)-f_1(y))^{2}dy=
        \int_{-1}^{1}(f_0(y)-f_0(f_0^{-1}(f_1(y))))^{2}dy\geq{}
        \frac{4\beta^{2}}{9}\int_{-1}^{1}(y-f_0^{-1}(f_1(y)))^{2}dy,
      \end{align*}
      where we have used that $f'(y)\geq{}\frac{2\beta^2}{3}$
      everywhere. Next, using a change of variables, we have
      \begin{align*}
        \int_{-1}^{1}(y-f_0^{-1}(f_1(y)))^{2}dy
        &=
          \int_{f_1(-1)}^{f_1(1)}\frac{(f_1^{-1}(x)-f_0^{-1}(x))^{2}}{f_1'(f_1^{-1}(x))}dx\\
        &\geq{} \frac{1}{2\beta}\int_{f_1(-1)}^{f_1(1)}(f_1^{-1}(x)-f_0^{-1}(x))^{2}dx,
      \end{align*}
      where the inequality uses that $f'_1\leq{}2\beta$
      everywhere. Next, we observe that $f_1=m_1^{-1}$, and that
      \[
        m_1(1) \leq{} \gamma{} + \alpha<1,\quad\text{and}\quad{}m_1(-1)\geq{}-\gamma-\alpha>-1.
      \]
      It follows that $f_1(1)\geq{}1$ and $f_1(-1)\leq{}-1$, and
      consequently
      \begin{align*}
        \int_{f_1(-1)}^{f_1(1)}(f_1^{-1}(x)-f_0^{-1}(x))^{2}dx
        &\geq{}
          \int_{-1}^{1}(f_1^{-1}(x)-f_0^{-1}(x))^{2}dx\\
        &=
          \int_{-1}^{1}(m_1(x)-m_0(x))^{2}dx\\
        &=
          4\int_{-1}^{1}h^{2}(x)dx.
      \end{align*}
      Finally, we appeal to \pref{lem:h_lower_bound}, which implies
      that
      \[
        \int_{-1}^{1}h^{2}(x)dx \geq{}\frac{\alpha^{2}}{2e}.
      \]
    \end{proof}

  \begin{lemma}
    \label{lem:h_lower_bound}
    If we choose $\beta\geq{}1$, then the function $h$ in
    \eqref{eq:lb_defs} satisfies
    \begin{equation}
      \label{eq:3}
      \int_{-\beta}^{\beta}h^{2}(z)dz \geq{}\frac{\alpha^{2}\beta}{2e},\quad\text{and}\quad \int_{-1}^{1}h^{2}(z)dz \geq{}\frac{\alpha^{2}}{2e}.
    \end{equation}
  \end{lemma}
  \begin{proof}[\pfref{lem:h_lower_bound}]
    First, since we integrate only over the range $(-\beta,\beta)$,
    $e^{-\frac{z^2}{\beta^{2}}}\geq{}e^{-1}$, so we have
    \begin{align*}
      \int_{-\beta}^{\beta}h^{2}(z)dz=\alpha^{2}\int_{-\beta}^{\beta}e^{-\frac{z^2}{\beta^2}}\cos^2(4\pi\beta{}z)dz\geq{}
      \frac{\alpha^{2}}{e}\int_{-\beta}^{\beta}\cos^2(4\pi\beta{}z)dz.
    \end{align*}
Next, we recall that for any $a$, the indefinite integral of
$\cos^{2}(ax)$ satisfies $\int\cos^{2}(ax)=\frac{x}{2} +
\frac{1}{2a}\sin(ax)\cos(ax)$. Applying this above, we have
\begin{align*}
  \int_{-\beta}^{\beta}\cos^2(4\pi\beta{}z)dz =
\frac{x}{2}+\frac{1}{8\pi\beta}\sin(4\pi\beta{}x)\cos(4\pi\beta{}x)\Bigg\rvert_{-\beta}^{\beta}                                              \geq{} \beta-\frac{1}{4\pi\beta}.
\end{align*}
For $\beta>1$, this is at least $\frac{\beta}{2}$.

Similarly, since $\beta\geq{}1$, we have
\[
  \int_{-1}^{1}h^{2}(z)dz\geq{}
  \frac{\alpha^{2}}{e}\int_{-1}^{1}\cos^2(4\pi\beta{}z)dz,
\]
and
\begin{align*}
  \int_{-1}^{1}\cos^2(4\pi\beta{}z)dz =
\frac{x}{2}+\frac{1}{8\pi\beta}\sin(4\pi\beta{}x)\cos(4\pi\beta{}x)\Bigg\rvert_{-1}^{1} \geq{} 1-\frac{1}{4\pi\beta}\geq{}\frac{1}{2}.
\end{align*}

  \end{proof}

\newpage
\section{Learning Theory Tools}
\label{app:learning_theory}
\newcommand{\clscor}{c_{\mathrm{ls}}}
\newcommand{\somelogs}{\texttt{logs}}

	\newcommand{\delphi}{\updelta_{\phi}}
	\newcommand{\phihat}{\widehat{\phi}}
	\newcommand{\updelnot}{\updelta_{0}}
In this section, we state and prove basic learning-theoretic tools used throughout the proofs for our main results. \Cref{ssec:techtools_statement} gives the main statements and definitions for these results, and \Cref{ssec:techtools_proofs} proves the results in the order in which they appear. Our results are split into the following categories:
\begin{itemize}
\item \Cref{sssec:techtools_c_conc} introduces a convention for subexponential random variables (``$c$-concentrated'') used throughout our proofs and establishes key properties of random variables satisfying this condition (\Cref{lem:techtools_truncated_conc})
	\item \Cref{sssec:techtools_gaussian} gives a concentration properties for Gaussian vectors (\Cref{lem:techtools_gaussian_c_concentrated}) and establishes a useful change-of-measure lemma (\Cref{lem:techtools_Gaus_change_of_measure})
	\item \Cref{sssec:techtools_decodable} gives a generic template (\Cref{lem:techtools_condition_exp_idenity}) for computing conditional expectations for random variables we call \emph{decodable Markov chains} (\Cref{defn:techtools_decode}), which arise when analyzing the regression problems used in \arxiv{\pref{alg:main}}\neurips{\pref{alg:main2}}.
	\item \Cref{sssec:techtools_covering_num} presents \Cref{defn:techtools_covering_number}, which introduces the main notion of covering number used in our analysis, and provides bounds on covering numbers for these
          function classes used by \arxiv{\pref{alg:main}}\neurips{\pref{alg:main2}}.%
	\item \Cref{sssec:sq_loss_min} gives prediction error bounds for  square loss regression over a general function classes, subject to misspecification error. \Cref{prop:techtools_general_regression_with_truncation_and_error} provides guarantees based on a classical notion of misspecification error (which arises in Phase I of \arxiv{\pref{alg:main}}\neurips{\pref{alg:main2}}), while  \Cref{cor:techtools_function_dependend_error} gives guarantees under a stronger notion of \emph{function-dependent} misspecification error, which is used in the analysis of  Phase III.
	\item \Cref{sssec:techtools_pca} provides guarantees for a principal component analysis (PCA) setup which, in particular, subsumes the dimensionality reduction procedure used in Phase I of \arxiv{\pref{alg:main}}\neurips{\pref{alg:main2}}. It provides guarantees for estimating a covariance matrix under persistent error (\Cref{prop:techtools_general_pca}), and a corollary regarding overlap between eigenspaces (\Cref{cor:techtools_overlap_PCA}).
	\item \Cref{sssec:techtools_linear_reg} considers linear regression. \Cref{prop:techtools_LS_guarantee} gives bounds for parameter recovery under errors in variables, which is used to recover $\Aid$ and $\Bid$ in Phase II of \arxiv{\pref{alg:main}}\neurips{\pref{alg:main2}}. \Cref{prop:techtools_covariance_est} gives a guarantee for covariance estimation, which are used to estimate $\Sigwid$ in Phase II.
	\item Finally, \Cref{sssec:techtools_matrix_sensing} gives a parameter recovery bound for regression with measurements which are rank-one outer products of near-Gaussian vectors. This is used to recover the cost matrix $\Qid$ in Phase II.
\end{itemize}
\subsection{Statement of Guarantees \label{ssec:techtools_statement}}
\subsubsection{Generic Concentration \label{sssec:techtools_c_conc}}
	\begin{definition}[$c$-concentration]\label{defn:techtools_conc_prop} We
          say that a non-negative random variable $\bz$ is
          \emph{$c$-concentrated} if $\Pr[ \bz \ge c\ln(1/\delta)] \le
          \delta$ for
          all $\delta \in (0,1/e]$. For such random variables, we define $c_{n,\delta}:= c\ln(2n/\delta)$.
	\end{definition}
        This is one of many equivalent (up to numerical constants) definitions for sub-exponential concentration (e.g., \cite{wainwright2019high}). We opt for the term ``c-concentrated'' to make the dependence on the concentration parameter $c$ precise.

	\begin{lemma}[Truncated concentration]\label{lem:techtools_truncated_conc} Let $\bz$ be a non-negative $c$-concentrated random variable. Then, $\bz$ is $c'$-concentrated for  all $c' \geq c$, and $\alpha \bz + \beta$ is $\alpha c + \beta$-concentrated for all $\alpha,\beta > 0$.  Moreover, the the following bounds hold.
		\begin{enumerate}
                \item %
                  For any $\delta \in (0,1/e]$, we
                          have $\E [\bz\I\crl{\bz \ge c\ln(1/\delta)}] \le
                          c\delta$, and in particular, $\E\max\{c,\bz\}
                          \le 2c$. For any integer
                          $k\geq{}1$, and $\E [\bz^k]\leq{}2k!c^{k}$.
			\item Let $\veps^2 \ge \E\brk{\bz}$, and let $\delta \in (0,1)$. Suppose $n$ is large enough such that $\psi(n,\delta) \le \frac{\veps^2}{c}$, where we define
			\begin{align}
			\psi(n,\delta) := \frac{2\ln(2n/\delta)\ln(2/\delta)}{n}. \label{eq:techtools_psi_def}
			\end{align} 
			Then with probability at least $1 - \delta$, $\bz\supi \sim \bz$ satisfy $\frac{1}{n}\sum_{i}\bz^{(i)}\le 2\veps^2$, where $\bz\ind{i}\iidsim\bz$ for $1\leq{}i\leq{}n$.
			\item Consider the previous claim. Suppose we replace the hypothesis that $\bz$ is $c$-concentrated with the assumption that for a given $\delta\in (0,1)$,  $\bz \le c\ln(2n/\delta)$ almost surely. Then with failure probability at least $1 - 2\delta$, $\frac{1}{n}\sum_{i}\bz^{(i)}\le 2\veps^2$.
		\end{enumerate}
		\end{lemma}
\subsubsection{Gaussian Concentration and Change of Measure \label{sssec:techtools_gaussian}}
Our first lemma shows that norms of Gaussian vectors satisfy the $c$-concentration condition.
	\begin{lemma}\label{lem:techtools_gaussian_c_concentrated} Let $\matx \sim \cN(0,\Sigma)$. Then, $\|\matx\|^2$ is $c$-concentrated for  $c = 5\trace(\Sigma)$.
	\end{lemma}
	Next, we provide a change of measure argument, which is used to establish that \pref{alg:main} accurately estimates the system state.
	\begin{lemma}[Gaussian change of measure]\label{lem:techtools_Gaus_change_of_measure} Let $\Sigma_1,\Sigma_2 \succ 0 $ be matrices in $\bbR^{d\times{}d}$. Let $\matx_1 \sim \cN(0,\Sigma_1)$, $\matx_2 \sim \cN(0,\Sigma_2)$, and let $\by_1 \sim q( \cdot \mid \matx_1)$, $\by_2 \sim q(\cdot \mid \matx_1)$. Let $\hhat,\hst : \cY \to \R^d$ be two functions such that $\max\{\|\hhat(y)\|,\|\hst(y)\|\} \le L \|\fst(y)\|$. Suppose that 
	\begin{align*}
	\E_{\by_1}[\|\hhat(\by_1) - \hst(\by_1)\|^2] \le \veps^2, \quad \text{and} \quad \|I - \Sigma_1^{1/2}\Sigma_2^{-1}\Sigma_1^{1/2}\|_{\op} \le \frac{1}{14 d \ln \prn*{\frac{80 e L^2 (1+\|\Sigma_1\|_{\op})}{\veps^2}}},
	\end{align*}
        for some $\veps>0$. Then the following error bound holds:
	\begin{align*}
	\E_{\by_2}[\|\hhat(\by_2)-\hst(\by_2)\|^2] \le 2\veps^2.
	\end{align*}
	\end{lemma}
\subsubsection{Conditional Expectations for Decodable Markov Chains}
        \label{sssec:techtools_decodable}

	\begin{definition}[Decodable Markov chain]\label{defn:techtools_decode} 
          Let $\bu \in \cU$, $\bx \in \cX$, $\by \in \cY$ be random variables that form a Markov chain $\bu \to \bx \to \by$. We say $(\bu,\bx,\by)$ is a \emph{decodable Markov chain} if there exists some function $\fst:\cY\to\cX$ such that $\bx = \fst(\by)$ almost surely.
	\end{definition}

	\begin{lemma}[Characterization of square loss minimizer] \label{lem:techtools_condition_exp_idenity} 
          Let $(\bu,\bx,\by)$ be a decodable Markov chain. Then, $\E[\bu \mid \by = y] = \hst(y)$, where
\[\hst(y) := \E[\bu \mid \bx = \fst(y)].\] Moreover, for any class of functions $\Hclass$ with $\hst\in\Hclass$, for any $h \in \argmin_{h' \in \Hclass} \E\|h'(\by)- \bu\|^2$, we have 
		\begin{align*}
		h(\by) = \hst(\by)\;\; \text{ almost surely in }  \by.
		\end{align*}
	\end{lemma}

\subsubsection{Covering Numbers \label{sssec:techtools_covering_num}}
	\begin{definition}[Covering numbers]\label{defn:techtools_covering_number} 
          Let $(\cX,\dist)$ be a metric space with pseudometric $\dist$. The covering number $\eulN(\epsilon,\cX,\dist)$ is defined as the minimal cardinality of any set $\cX' \subseteq \cX$ such that
		\begin{align*}
		\max_{x \in \cX}\min_{x' \in \cX'} \dist(x,x') \le \epsilon.
		\end{align*}
		We say that $\cX'$ is a minimal $\epsilon$-cover of $\cX$ if it witnesses the condition above and has $|\cX'| = \eulN(\epsilon,\cX,\dist)$.
	\end{definition}

	\begin{lemma}\label{lem:Mcover} Let $\scrM := \{M \in \R^{d\times\dimx}: \|M\|_{\op} \le b\}$. Then, $\eulN(b\epsilon,\scrM,\|\cdot\|_{\op}) \le (1 + 2/\epsilon)^{d\dimx}$.
\end{lemma}
\subsubsection{Square Loss Regression \label{sssec:sq_loss_min}}

	\begin{proposition}[Square loss regression with misspecification error]\label{prop:techtools_general_regression_with_truncation_and_error} 
          Let $(\bu,\by)$ be a pair of random variables with $\bu\in\cU$, $\by\in\cY$, and let $\mate\in\cU$ be an arbitrary ``error'' random variable. %
          Suppose that $\Hclass$ is a function class that contains the function $\hst(y) := \E[ \bu  \mid \by = y]$.  Consider  empirical risk minimizer
		\begin{align*}
		\hhat_n := \argmin_{h \in \Hclass} \sum_{i=1}^n \|h(\by^{(i)}) + \mate^{(i)} - \bu^{(i)} \|^2,
		\end{align*}
                where $(\matu\ind{i},\maty\ind{i},\mate\ind{i})$ are drawn \iid from the law of $(\matu,\maty,\mate)$ for $1\leq{}i\leq{}n$. Suppose that there exists a constant $c > 0$ and function $\varphi : \cY \to \R_{+}$ such that the following properties hold: 
		\begin{itemize}
			\item $\|h(y)\|^2 \le \varphi(y)$ for all $h \in \Hclass$.
			\item The random variables $\varphi(\by)$ and $\|\mate - \bu\|^{2}$ are  $c$-concentrated. 
			\item For all $c' \ge c$ and all $\epsilon \le 1$, the $\sqrt{c'}\epsilon$-covering number of $\scrH$ in the pseudometric $d_{c',\infty}(h,h') := \sup_{y\in\cY}\{\|h(y) - h'(y)\|: \varphi(y) \le c'\}$ is bounded by a function $\eulN(\epsilon)$.
		\end{itemize}
		Then, with probability at least $1 - \frac{3\delta}{2}$,
		\begin{align*}
		\E\| \hhat_n(\by) - \hst(\by)\|^2 \le  \frac{270 c_{n,\delta} }{n}\ln (2\eulN(1/33n)\delta^{-1}) +  8 \E\|\mate\|^2, %
		\end{align*}
                where we recall that $c_{n,\delta} := c\ln(2n/\delta)$.
              \end{proposition}

	We now state two corollaries of the above regression. First, a simple corollary for structured function classes of the form $\{M \cdot f\}$, where $M$ are matrices of bounded operator norm, and $f \in \Fclass$ are elements of finite class which satisfy a growth condition like \Cref{asm:f_growth}.
	\begin{corollary}[Regression with Structured Function Class]\label{cor_techtools:reg_cor_simple}
	 Let $(\bu,\by)$ be a pair of random variables with $\bu\in\R^{d_u}$, $\by\in\cY$, and let $\mate\in\cU$ be an arbitrary ``error'' random variable. %
          Suppose that $\Hclass$ is a function class that contains the function $\hst(y) := \E[ \bu  \mid \by = y]$.  Consider  empirical risk minimizer
		\begin{align*}
		\hhat_n := \argmin_{h \in \Hclass} \sum_{i=1}^n \|h(\by^{(i)}) + \mate^{(i)} - \bu^{(i)} \|^2.
		\end{align*}
		where $(\matu\ind{i},\maty\ind{i},\mate\ind{i})$ are drawn \iid from the law of $(\matu,\maty,\mate)$ for $1\leq{}i\leq{}n$. Suppose $\Fclass:\cY \to \R^{\dimx}$ is a finite class of functions satisfying $f(y) \le L\max\{1,\|\fst(y)\|_2\}$ for all $f \in \Fclass$, where $L \ge 1$ without loss of generality. In addition, suppose  that $\Hclass$ takes the form
		\begin{align*}
		\scrH := \{h(y) = M \cdot f(y): f \in \Fclass ~~ M \in \R^{d_ud_x},~~\|M\|_{\op} \le b\}.
		\end{align*}
		Lastly, assume that the random variables $\varphi(\by)$ and $\|\mate - \bu\|^{2}$ are  $c$-concentrated, where $\varphi(y)^{1/2} := bL\max\{1,\|\fst(y)\|_2\}$. Then, with probability at least $1 - \frac{3\delta}{2}$,
		\begin{align*}
		\E\| \hhat_n(\by) - \hst(\by)\|^2 \le  \frac{c (d_ud_x +\ln|\Fclass|) \cdot \somelogs(n,\delta)}{n} +   8 \E\|\mate\|^2,
		\end{align*}
		where define $\somelogs(n,\delta):= 270 \ln(2n/\delta)\ln(330n/\delta) \lesssim \ln(n/\delta)^2$.
	\end{corollary}

	Second, we state a regression bound tailored to the structured regression problems that arise in Phase III of our algorithm. 
	\begin{corollary}\label{cor:techtools_function_dependend_error}
          Let $\cZ=\cY\times\cY$. Let $(\bv,\bz)\in \mathcal{V} \times \mathcal{Z}$ be a pair of random variables, and let $\mate\in\cV$ be a arbitrary ``error'' random variable defined on the same probability space. Let $\Hclass$ be a function class, and let  $\phi, \phihat \colon \scrG \times \cZ \rightarrow \reals^d$ be measurable functions. Suppose that the set $\scrH$ contains a function $h_\star$ satisfying $\phi(h_\star, z) := \E[ \bv  \mid \bz= z]$. Let $\{(\bz^{(i)}, \bv^{(i)}, \mate^{(i)})\}_{i=1}^{n}$ be i.i.d.~copies of $(\bz, \bv, \mate)$, and define 
		\begin{align*}
		\hat h := \argmin_{h \in \scrH} \sum_{i=1}^n \|\phihat(h, \bz^{(i)}) + \mate^{(i)} - \bv^{(i)} \|^2.
		\end{align*}
		Introduce $\delphi(z,h):= \phihat(h, \bz^{(i)}) - \phi(h,\bz^{(i)})$.
		Suppose that there exists a constant $c>0$ and a map $\psi \colon \cZ \rightarrow \reals_+$ such that the following properties hold:
		\begin{enumerate}
			\item $\sup_{h \in \scrH}\|\phi(h, z)\|^2 + \sup_{h \in \scrH}\|\delphi(z,h)\|^2 \le \psi(z)^2$. 
			\item For all $\delta \in (0,1/e]$, we have $\P[  \psi(\bz)^2 \vee \| \mate - \bv\|^2 \ge c \ln \delta^{-1}] \leq \delta$.
			\item $\scrH$ takes the form $\scrH = \{h(y) = M \cdot f(y): f \in \Fclass ~~ M \in \R^{d_1\times \dimx},~~\|M\|_{\op} \le b\}$ for some $b > 0$, where $\Fclass:\cY \to \R^{\dimx}$ is a finite class and .  Furthermore, there exists $L \ge 1$, matrices $X_1,X_2$ of appropriate dimension, and an arbitrary function $\updelnot : \cZ \to \cV$ (which does not depend on $h$) such that
			\begin{align*}
                          &\forall f \in \Fclass,\quad  \|f(y)\|_2 \le L\max\{1,\|\fst(y)\|\}\quad\text{for all $y\in\cY$}\\
                          &\forall h \in \Hclass, \quad \phihat(h,z):= X_1 (h(y_1) - X_2h(y_2) ) + \updelnot(z)\quad\text{for all $z=(y_1,y_2,y_3)$.}
		\end{align*}
		\item Finally, $c_{\psi} \ge 1$ satisifes the following for all all $z=(y_1,y_2,y_3)$
		\begin{align*}
		bL(\|X_1\|_{\op}+\|X_1 \cdot X_2\|_{\op})(2 + \|\fst(y_1)\| + \|\fst(y_2)\|) \le 2c_{\psi}\psi(z).
		\end{align*}

		\end{enumerate}
		Then, with probability at least $1 - \frac{3\delta}{2}$,
		\begin{align*}
		\E\|\phi(\hhat,\matz) - \phi(\hst,\matz)\|^2 &\le \frac{12c  (\ln|\Fclass| + d_1\dimx)\somelogs( c_{\psi}n,\delta)}{n} +  16 \E\|\mate\|^2 + 8\max_{h \in \scrH}\E\|\delphi(h,\bz)\|^2,
		\end{align*}
		where again we define $\somelogs(n,\delta):= 270 \ln(2n/\delta)\ln(330n/\delta)$, so that $\somelogs( c_{\psi}n,\delta) \lesssim \ln^2(c_{\psi} n/\delta)$.
	\end{corollary}

\subsubsection{Principal Component Analysis \label{sssec:techtools_pca}}

\begin{proposition}[PCA with errors]\label{prop:techtools_general_pca} Let $\Hclass\subseteq( \cY \to \reals^{d})$ be a function class, and let $\maty\in\cY$ be a random variable. Suppose that there exists a function $\varphi:\cY\to\bbR_{+}$ and constants $c$, $L$ such that the following properties hold:
	\begin{itemize}
        \item $\|h(\by)\|^2 \le \max\{c,\varphi(\by)\}$ for all $h \in \Hclass$.
        \item $\varphi(\maty)$ is $c$-concentrated.
                \end{itemize}
                Let $\hstar\in\Hclass$ be given, and let $\Lambda_{\star}\ldef\En\brk*{\hstar(\maty)\hstar(\maty)^{\trn}}$. Next, let
                $\hhat\in\Hclass$ be given and define
	\begin{align*}
	\Lamhat_{n} := \frac{1}{n}\sum_{i=1}^n \hhat(\by^{(i)})\hhat(\by^{(i)})^\top,
	\end{align*}
        where $\maty\ind{i}\iidsim\maty$.
	Then with probability $1 - \delta$, we have $\|\Lamhat_n - \Lambda_{\star}\|_{\op} \le \veps_{\mathrm{pca},n,\delta} $, where
	\begin{align*}
	\veps_{\mathrm{pca},n,\delta} := 3\sqrt{c\cdot\E[\|\hhat(\by) - \hst(\by)\|^2]} + 5 c n^{-1/2} \ln(2d n/\delta)^{3/2}.
	\end{align*}
	\end{proposition}

	\begin{corollary}[Significant basis overlap]\label{cor:techtools_overlap_PCA} Consider the setting of \Cref{prop:techtools_general_pca}, and suppose that $\Lambda_{\star} := \E[\hst(\by)\hst^\top(\by)]\in\bbR^{d\times{}d}$ has $\mathrm{rank}(\Lambda_{\star}) = \dimx$, so that $\lambda_{\dimx}(\Lambda_{\star}) > 0$. Let $V_{\star} \in \reals^{d\times \dimx}$ denote be a matrix with orthonormal columns that span the column space the image of $\Lambda_{\star}$. Likewise, let $\Vhat\in\bbR^{d\times\dimx}$ be a matrix with orthonormal columns span the eigenspace of the top $\dimx$ eigenvectors of $\wh{\Lambda}_n$. Suppose $\veps_{\mathrm{pca},n,\delta} \le \frac{\lambda_{\dimx}(\Lambda_{\star})}{4}$. Then on the good event for \pref{prop:techtools_general_pca}, we have $\sigma_{\dimx}(V_{\star}^\top \Vhat_n) \ge 2/3$.
	\end{corollary}

\subsubsection{Linear Regression \label{sssec:techtools_linear_reg}}
	\begin{restatable}[Linear regression with errors in variables]{proposition}{propTechtoolsLS}\label{prop:techtools_LS_guarantee} Let $(\bu,\by,\bw,\mate,\matdel)$ be a collection of random variables defined over a shared probability space, and let $\crl[\big]{(\bu\ind{i},\maty\ind{i} \bw\ind{i},\mate\ind{i},\matdel\ind{i})}_{i=1}^{n}$ be \iid copies. Suppose the following conditions hold:
	\begin{enumerate}
		\item $\by = \Mst \bu + \bw + \mate$ with probability $1$, where $\Mst\in\bbR^{\dimy\times\dimu}$.
		\item $\bw  \mid \bu,\matdel \sim \cN(0, \Sigma_w)$ and $\matu \sim \cN(0, \Sigma_u)$.
		\item We have $\E\|\mate\|^2 \le \veps_{\mate}^2$ and $\E\|\matdel\|^2 \le \veps^2_{\matdel}$. 
		\item $\mate$ is $c_{\mate}$-concentrated and $\matdel$ is $c_{\matdel}$-concentrated for $c_{\mate}\geq{}\veps_{\mate}^{2}$ and $c_{\matdel}\geq{}\veps_{\matdel}^{2}$.
		\item $\veps_{\matdel}^2 \le \frac{1}{16}\lambda_{\min}(\Sigu)$. 
	\end{enumerate}
	Let  $\delta \le 1/e$, and let $n\in\bbN$ satisfy
        \begin{enumerate}
          \item $\psi(n,\delta) \le \min\left\{\frac{\veps^2_{\mate}}{c_{\mate}} \frac{\veps^2_{\matdel}}{c_{\matdel}}\right\}$, where $\psi(n,\delta) := \frac{2\ln(2n/\delta)\ln(2/\delta)}{n}$.
            \item $n \ge c_1(\dimu + \ln(1/\delta))$, for some universal constant $c_1 > 0$.
      \end{enumerate}
      Then the solution to the least squares problem
	\begin{align*}
	\Mhat = \min_{M} \sum_{i=1}^n \| M  (\matu^{(i)} + \matdel^{(i)}) - \maty^{(i)}\|^2,
	\end{align*}
	satisfies the following inequality with probability at least $1-4\delta$:
	\begin{align}
          \| \Mhat - \Mst\|_{\op}^2 &\lesssim \lambda_{\min}(\Sigu)^{-1} \left(\|\Mst\|_{\op}^2\veps^2_{\matdel} + \veps^2_{\mate}  + \frac{\|\Sigw\|_{\op}(\dimy + \dimu + \ln(1/\delta))}{n}\right).
                                      \label{eq:techtools_linear_regression}
	\end{align}
	\end{restatable}

	\begin{restatable}{proposition}{propTechtoolsCovEst}\label{prop:techtools_covariance_est} Consider the setting of \Cref{prop:techtools_LS_guarantee}, and suppose we additionally require that $n \ge c_0(\dimy + \ln(1/\delta))$ for some (possibly inflated) universal constant $c_0$. Furthermore, suppose we have $\veps^2_{\matdel}\|\Mst\|_{\op}^2 + \veps^2_{\mate} \le 2\lambda_+$ for some $\lambda_+ \ge \lambda_{\max}(\Sigw)$. Then, with probability at least $1 - 7\delta$, \eqref{eq:techtools_linear_regression} holds, and moreover
		\begin{align*}
                  \nrm*{\frac{1}{n}\sum_{i=1}^n (\Mhat(\bu\supi+ \matdel\supi) - \maty\supi )^{\otimes 2} - \Sigw}_{\op}
                  &\lesssim \sqrt{\lambda_{+}(\veps^2_{\matdel}\|\Mst\|_{\op}^2 + \veps^2_{\mate}) + \frac{\lambda_{+}^2(\dimy + \log(1/\delta))}{n}}.
		\end{align*}
	\end{restatable}
\subsubsection{Regression with Matrix Measurements \label{sssec:techtools_matrix_sensing}}
\begin{restatable}[Regression with matrix measurements]{proposition}{propTechtoolsMatSes}\label{prop:techtools_matrix_sensing_regression}
  Let $\maty\in\cY$ be a random variable, and let $\maty\ind{i}\iidsim\maty$ for $1\leq{}i\leq{}n$.
	Fix two regression functions $\ghat,\gst : \cY \to \reals^d$, and suppose that $\matz := \max\{\|\ghat(\by)\|^2,\|\gst(\by)\|^2\}$ is $c$-concentrated, and that $\bx := \gst(\by) \sim \cN(0,\Sigma_x)$. Let $\Qst\psdgeq{}0$ be a fixed matrix, and consider the regression.
	\begin{align*}
	\Qtil \in \argmin_{M} \sum_{i=1}^n\left(\gst(\by^{(i)})^\top\Qst  \gst(\by^{(i)}) - \ghat(\by^{(i)})^\top M \ghat(\by^{(i)})^\top\right)^2.
	\end{align*}
	Set $\Qhat := (\frac{1}{2}\Qtil^\top + \frac{1}{2}\Qtil)_+$, where $(\cdot)_+$ truncates all negative eigenvalues to zero.  Then, there is a universal constant $c_0>0$ such that if the following conditions hold:
	\begin{align*}
	 \E\|\ghat(\maty) - \gst(\maty)\|^2 \le \veps^2, \quad \psi(n,\delta/2) \le \frac{\veps^2}{4c}, \quad  n \ge c_0(d^2 + \ln(1/\delta)), \mathand  \veps^2 \le \frac{\lambda_{\min}(\Sigma_x)^2}{64c\log(2n/\delta)},
	\end{align*}
	then with probability at least $1 - 2\delta$,
	\begin{align*}
	\|\Qhat - \Qst\|_{\fro}^2 \le \|\Qtil - \Qst\|_{\fro}^2 \le 64 c\veps^2 \ln(4n/\delta)\cdot \frac{\|\Qst\|_{\op}^2}{\lambda_{\min}(\Sigma_x)^2}. 
	\end{align*}
	\end{restatable}

\subsection{Proofs for Technical Tools \label{ssec:techtools_proofs}}

\subsubsection{Proof of \Cref{lem:techtools_truncated_conc}}
	
First observe that if $\bz$ is $c$-concentrated, then for $\delta \in (0,1/e]$, $\ln(1/\delta) \ge 1$, so that $\Pr[\alpha \bz + \beta \ge (\alpha c+\beta) \ln(1/\delta)] \le \Pr[\alpha \bz  \ge \alpha c \ln(1/\delta)] = \Pr[c \bz \ge c\ln(1/\delta)]$. This is at most $\delta$ by the definition of the $c$-concentrated property. We now turn to the enumerated points.

	\noindent\textbf{Point 1.} For $\delta \le 1/e$,  $\Pr[\bz \ge c
        \ln(1/\delta)] \le \delta$. Thus, for any $\delta \in (0,1/e)$, and $u
        \ge c_{\delta} \ge c$, we have $\Pr[\bz \ge u] \le e^{-u/c}$. It follows that for any $\delta \le 1/e$, 
	\begin{align*}
	\E [\bz\I(\bz \ge c\ln(1/\delta))]  &= \int_{u = c\ln(1/\delta)}^{\infty}\Pr[\bz \ge u]du\\
	&= \int_{u = c\ln(1/\delta)}^{\infty}e^{-u/c}du =
   ce^{-c\ln(1/\delta)/c} = c\delta.
	\end{align*}
        A similar calculation reveals that
        \begin{align*}
          \E [\bz^{k}]  &\leq{} c^{k} \int_{u = c^k}^{\infty}\Pr[\bz^{k} \ge u]du\\
                                                  &\leq{} c^{k} + \int_{u =
                                                    c^k}^{\infty}e^{-\frac{u^{1/k}}{c}}du\\
                                                  &= c^{k} + kc^{k}\int_{u=1}^{\infty}e^{-u}u^{k-1}du\\
                                                  &= c^{k}
                                                    (1+k!).
        \end{align*}
	
	\noindent\textbf{Point 2.}  Define the increments $\Delta_i = \bz^{(i)} - \E[\bz^{(i)}]$, and $ \Delta_{i,\delta} := \Delta_i \I( \bz^{(i)} \le c_{n,\delta} )$. By a union bound, $\Delta_i = \Delta_{i,\delta}$ for all $i \in [n]$ with probability at least $1-\delta/2$. Moreover, $-c \leq -\En\brk*{\bz}\le \Delta_{i,\delta}  \le c_{n,\delta}$, so that by Bennett's inequality (\cite{MaurerP09}, Theorem 3), it holds that with probability at least $1- \delta/2$, 
	\begin{align*}
	\frac{1}{n}\sum_{i}\Delta_{i,\delta} &\le \sqrt{\frac{2\Var[\Delta_{i,\delta}]\ln(2/\delta)}{n}} + \frac{c_{n,\delta}\ln(2/\delta)}{3n}\\
	&\le \frac{\tau}{2c_{n,\delta}}\Var[\Delta_{i,\delta}] + \prn*{\frac{1}{3}+\frac{1}{\tau}}\frac{c_{n,\delta}\ln(2/\delta)}{n},
	\end{align*}
        for any $\tau>0$. Moreover, we have $\Var[\Delta_{i,\delta}] = \E\brk{\Delta_{i,\delta}^2} \le c_{n,\delta}\E| \bz^{(i)} - \E[\bz^{(i)}]| \le 2c_{n,\delta}\E|\bz^{(i)}| = 2c_{n,\delta} \E\brk{\bz^{(i)}}$ by non-negativity of $\bz^{(i)}$. Hence, with total probability at least $1 - \delta$, we have
	\begin{align*}
	\frac{1}{n}\sum_{i}\Delta_{i} =  \frac{1}{n}\sum_{i}\Delta_{i} \le  \tau\E\brk{\bz^{(i)}} + \prn*{\frac{1}{3}+\frac{1}{\tau}}\frac{c_{n,\delta}\ln(2/\delta)}{n}.
	\end{align*}
	Recalling that $\Delta_{i} \le \bz^{(i)} - \E\brk{\bz^{(i)}}$ and taking $\tau=1/2$, we have that with total probability at least $1 - \delta$,
	\begin{align*}
	\frac{1}{n}\sum_{i}\bz^{(i)} \le \frac{3}{2}\E\brk{\bz^{(i)}} + \frac{5}{6}\frac{c_{n,\delta}\ln(1/\delta)}{n}.
	\end{align*}
	By assumption, $\E\brk{\bz^{(i)}} \le \veps^2$. Hence, for 
	\begin{align*}
	\frac{c_{n,\delta}\ln(2/\delta)}{n} \le \frac{6}{5}\cdot \frac{1}{2} \veps^2 \le \veps^2/2,
	\end{align*}
	we have that $\frac{1}{n}\sum_{i}\bz^{(i)} \le 2\veps^2$. In particular, since
	\begin{align*}
	\psi(n,\delta) =  \frac{2\ln(2n/\delta)\ln(2/\delta)}{n}  = \frac{2c_{n,\delta}\ln(2/\delta)}{cn},
	\end{align*}
	we have $\frac{1}{n}\sum_{i}\bz^{(i)} \le 2\veps^2$ for $\psi(n,\delta) \le \veps^2/c$.

        It is simple to verify that all the steps above go through if $\bz\le c_{\delta,n} = c \ln(2n/\delta)$ almost surely and $\En\brk*{\bz}\leq{}c$. Substituting in $c_{n,\delta} := c\ln(2n/\delta)$ concludes.

\qed

\subsubsection{Proof of \Cref{lem:techtools_gaussian_c_concentrated}}
First, observe that $\En\brk*{\nrm{\bx}^{2}}=\trace(\Sigma)$. Next, from \citet[Proposition 1]{hsu2012tail}, we have that
	\begin{align*}
	\Pr[\|\matx\|^2 \ge \trace(\Sigma) + 2 \sqrt{t}\|\Sigma\|_{\fro} + 2t\|\Sigma\|_{\op} ] \le e^{-t}.
	\end{align*}
	Setting $t = \log(1/\delta)\geq{}1$ and bounding $ \trace(\Sigma) + 2 \sqrt{t}\|\Sigma\|_{\fro} + 2t\|\Sigma\|_{\op} \le 5t\cdot \trace(\Sigma) = 5\trace(\Sigma) \log(1/\delta)$ concludes. \qed

\subsubsection{Proof of \Cref{lem:techtools_Gaus_change_of_measure}} 

Let $\cQ_1$ denote the law of $\by_1$, $\cQ_2$ the law of $\by_2$, $\cP_1$ the law of $\bx_1$, and $\cP_2$ the law of $\bx_2$. %
Let $q(y \mid x)$ denote the density of $y$ given $x$. We then have that
	\begin{align*}
	\E_{\by_2}[\|\hst(\by_2)- \hhat(\by_2)\|^2]  &= \int_y  \|\hst(y)- \hhat(y)\|^2\rmd \cQ_2(y)\\
	&=  \int_{x,y}  q(x \mid y) \|\hst(y)- \hhat(y)\|^2 \rmd \cP_2(x)\\
	&=  \int_{x,y}  q(x \mid y) \frac{\rmd\cP_2(x)}{\rmd\cP_1(x)}\|\hst(y)- \hhat(y)\|^2 \rmd \cP_1(x).
	\end{align*} 
	Using the standard expression for the density for the multivariate Gaussian distribution, we have the identity
	\begin{align*}
	\frac{\rmd\cP_2(x)}{\rmd\cP_1(x)} &= \det(\Sigma_1\Sigma^{-1}_2)^{1/2} \cdot \exp\prn*{\frac{1}{2} x^\top (\Sigma_1^{-1} - \Sigma_2^{-1}) x}\\
	&= \det( I + (\Sigma_1^{1/2}\Sigma^{-1}_2\Sigma^{1/2}- I)) ^{1/2} \cdot \exp\prn*{\frac{1}{2} x^\top \Sigma_1^{-1/2}(I - \Sigma_1^{1/2}\Sigma^{-1}_2\Sigma^{1/2}) \Sigma_1^{-1/2}x}.
	\end{align*}
	Hence, if we set $\eta = \|(\Sigma_1^{1/2}\Sigma^{-1}_2\Sigma^{1/2}- I)\|_{\op}$, we have
	\begin{align*}
	\det( I + (\Sigma_1^{1/2}\Sigma^{-1}_2\Sigma^{1/2}- I)) &= \prod_{i=1}^d \lambda_i(I + (\Sigma_1^{1/2}\Sigma^{-1}_2\Sigma^{1/2}- I))\\
	&\le \prod_{i=1}^d (1 + \eta)^d \le \exp(d \eta).
	\end{align*}
	Similarly, we may bound
	\begin{align*}
	\exp\prn*{ \frac{1}{2} x^\top \Sigma_1^{-1/2}(I - \Sigma_1^{1/2}\Sigma^{-1}_2\Sigma^{1/2}) \Sigma_1^{-1/2}x} \le \exp\prn*{ \frac{\eta}{2}x^\top \Sigma_1^{-1} x}.
	\end{align*}
	Thus, 
	\begin{align*}
	\frac{\rmd\cP_2(x)}{\rmd\cP_1(x)} &\le  \exp\prn*{ \frac{\eta}{2} (d + x^\top \Sigma_1^{-1} x)}.
	\end{align*}
	In particular, for any $B>0$, as long as
	\begin{align}
          x^\top \Sigma_1^{-1} x \le B,\mathand\eta \le \frac{2 \ln(3/2)}{d + B}, \label{eq:techtools_first_eta_bound}
	\end{align}
we have
	\begin{align*}
	\frac{\rmd\cP_2(x)}{\rmd\cP_1(x)} &\le 3/2.
	\end{align*}
	Henceforth, fix a bound parameter $B$ and assume $\eta \le \frac{2 \ln(3/2)}{d + B}<1$. We have 
	\begin{align*}
	\E_{\by_2}[\|\hst(\by_2)- \hhat(\by_2)\|^2]  &=  \int_{x,y}  q(x \mid y) \frac{\rmd\cP_2(x)}{\rmd\cP_1(x)}\|\hst(y)- \hhat(y)\|^2 \rmd \cP_1(x)\\
	&\le \underbrace{\frac{3}{2}\int_{x,y}  q(x \mid y) \|\hst(y)- \hhat(y)\|^2 \rmd \cP_1(x) \I(x^\top \Sigma_1^{-1} x \le B)}_{:= \Term_1} \\
	&\quad+ \underbrace{\int_{x,y}  q(x \mid y)  \exp\prn*{ \frac{\eta}{2} (d + x^\top \Sigma_1^{-1} x)}\|\hst(y)- \hhat(y)\|^2 \rmd \cP_1(x) \I(x^\top \Sigma_1^{-1} x > B)}_{:=\Term_2}.
	\end{align*} 
	To handle the first term, we use the assumed error bound between $\hst$ and $\hhat$:
	\begin{align}
	\Term_1&:= \frac{3}{2}\int_{x,y}  q(x \mid y) \|\hst(y)- \hhat(y)\|^2 \rmd \cP_1(x) \I(x^\top \Sigma_1 x \le B) \nonumber\\
	&\le \frac{3}{2}\int_{x,y}  q(x \mid y) \|\hst(y)- \hhat(y)\|^2 \rmd \cP_1(x) = \frac{3}{2}\E_{\by_1} \|\hst(\by_1)- \hhat(\by_2)\|^2  \le \frac{3 \veps^2}{2}. \label{eq:techtools_change_of_measure_first}
	\end{align}
	For $\Term_2$, we use the bound $\|\hst(y)- \hhat(y)\|^2  \le 4L\max\{1,\|\fst(y)\|^2\} = 4L^2\max\{1,\|x\|^2\}$ to bound 
	\begin{align*}
	&\int_{x,y}  q(x \mid y)  \exp\prn*{ \frac{\eta}{2} (d + x^\top \Sigma_1 x)}\|\hst(y)- \hhat(y)\|^2 \rmd \cP_1(x) \I(x^\top \Sigma_1^{-1} x > B)\\
	&\le 4L^2 e^{\frac{d\eta}{2}}\int_{x}  \exp\prn*{ \frac{\eta}{2} x^\top \Sigma_1^{-1} x }(1+\|x\|^2) \rmd \cP_1(x) \I(x^\top \Sigma_1^{-1} x > B).
	\end{align*}
	Let us change variables to $u = \Sigma^{-1/2}x$, and let $\cP_0$ denote the density of $u$, which is precisely the density of a standard normal $\cN(0,I)$ random variable. Then, using the formula the standard normal density,
	\begin{align*}
	&\int_{x}  \exp\prn*{\frac{\eta}{2} x^\top \Sigma_1^{-1} x }(1+\|x\|^2)\I(x^\top \Sigma_1^{-1} x > B) \rmd\cP_1(x) \\
	 &= \int_{u}  \exp\prn*{ \frac{\eta}{2} \|u\|^2 } \cdot (1+\|\Sigma^{1/2} u\|^2) \cdot \I(\|u\|^2 > B) \rmd \cP_0(u) \\
	&= \int_{u}  \frac{1}{(2\pi)^{d/2}}\exp\prn*{  -\frac{(1-\eta) }{2} \|u\|^2 } \cdot (1 + \|\Sigma^{1/2} u\|^2) \cdot \I(\|u\|^2 > B)\rmd u\\
	&\le \int_{u}  \frac{1}{(2\pi)^{d/2}}\exp\prn*{ -\frac{(1-\eta) }{2} \|u\|^2 } (1 + \|\Sigma_1\|_{\op}\| u\|^2) \I(\|u\|^2 > B)\rmd u.
	\end{align*}
	Again, let us rescale via $z \leftarrow (1 - \eta)^{-1/2}u$. The determinant of the Jacobian of this transformation is $(1-\eta)^{d/2}$, so that for $B \ge 1$, this is equal to 
	\begin{align*}
	 &(1-\eta)^{-d/2}\int_{z}  \frac{1}{(2\pi)^{d/2}}\exp\prn*{ -\frac{1}{2} \|u\|^2 } (1 + (1-\eta)\|\Sigma_1\|_{\op}\| z\|^2) \I(\|z\|^2 > B(1-\eta)^{-1})\rmd z \\
	 &= (1-\eta)^{-d/2}\E_{\bz\sim \cN(0,I)} \brk*{(1 + (1-\eta)\|\Sigma_1\|_{\op}\|\bz\|^2) \I(\|\bz\|^2 \ge B(1-\eta)^{-1})}\\
	 &\overset{(i)}{\le} e^{\eta{}d}\E_{\bz\sim \cN(0,I)}\brk*{ (1 + \|\Sigma_1\|_{\op}\|\bz\|^2) \I(\|\bz\|^2 \ge B(1-\eta)^{-1})}\\
	 &\le (1+\|\Sigma_1\|_{\op})e^{\eta{}d}\E_{\bz\sim \cN(0,I)} \brk*{\|\bz\|^2 \I(\|\bz\|^2 \ge B(1-\eta)^{-1})},
	\end{align*}
	where in $(i)$ we observe that $	(1-\eta)^{-1/\eta} \le e^{2} \text{ for } \eta \le 1/2$, and 
        where the last inequality uses that $B\geq{}1$. Now, from \Cref{lem:techtools_gaussian_c_concentrated}, we have that $\|\bz\|^2$ is $5d$-concentrated. Hence, for $\eta \le 1/2$, $(1-\eta)^{-1}\|\bz\|^2 $ is $10d$-concentrated. Thus $B = 10 d \ln(1/\delta) $  gives $\E_{\bz\sim \cN(0,I)} \brk*{\|\bz\|^2 \I((1-\eta)^{-1}\|\bz\|^2 \ge B)} \le 10d \delta$ by \Cref{lem:techtools_truncated_conc}., and therefore
	\begin{align*}
	\Term_2 = &\int_{x,y}  q(x \mid y)  \exp\prn*{ \frac{\eta}{2} (2d + x^\top \Sigma_1 x)}\|\hst(y)- \hhat(y)\|^2 \rmd \cP_1(x) \I(x^\top \Sigma_1^{-1} x > B) \\
	&\le 4(1+\|\Sigma_1\|_{\op})L^2  e^{\eta d} \cdot e^{\eta d/2} \cdot 10 d \delta = \delta \cdot (1+\|\Sigma_1\|_{\op}) 40L^2 e^{3d\eta/2 }.
	\end{align*}
	In particular, if $\eta \le 1/2d$ and $\delta = \frac{\veps^2}{80 L^2\|\Sigma_1\|_{\op}e }$, we have $\Term_2  \le \frac{\veps^2}{2}$, and thus $\Term_1 + \Term_2 \le 2\veps^2$. Gathering our conditions, we require $\eta \le 1/\max\{2,d\}$, $B = 10 d \ln (\frac{\veps^2}{80 L^2(1+\|\Sigma_1\|_{\op}) e})$, and---from \Cref{eq:techtools_first_eta_bound}---$\eta \le \frac{2 \ln(3/2)}{d + B}$. Altogether, it suffices to select
	\begin{align*}
	\eta \le \frac{2 \ln(3/2)}{11 d \ln \prn*{\frac{\veps^2}{80 L^2(1+\|\Sigma_1\|_{\op})e}}} \le \frac{1}{14 d \ln \prn*{\frac{\veps^2}{80 L^2(1+\|\Sigma_1\|_{\op})e}}}.
	\end{align*}
	\qed

\subsubsection{Proof of \Cref{lem:techtools_condition_exp_idenity}}
	By the tower rule and the fact that $\bu \to \bx \to \by$ is a Markov chain, $\E[\bu \mid \by = y] = \E[\E[\bu \mid \bx ,\by = y] \mid \by = y]  = \E[\E[\bu \mid \bx ] \mid \by = y]$. Moreover, from decodability, $ \E[\E[\bu \mid \bx ] \mid \by = y] = \E[ \E[\bu \mid \matx = \fst(y) ] \mid \by = y] = \E[\bu \mid \fst(y) = \matx] = \hst(y)$.

	For the second point, It is well know that any \emph{unrestricted} minimizer of $\|h(\by)- \bu\|^2$ over all measurable $h$ satisfies $h = h_0$ almost surely, where $h(y) := \E[\bu \mid \by = y]$. We verify above that $\hst(y) = \E[\bu \mid \by = y]$, proving the that any unrestricted minimizer $h$ coincideds with $\hst$. Since $\hst \in \Hclass$, the same holds for the function class constraint in the lemma statement. \qed

\subsubsection{Proof of \Cref{lem:Mcover}}

	Our task is to bound $\eulN(b\epsilon,\scrM,\|\cdot\|_{\op})$, where we recall $\scrM := \{M \in \R^{d\dimx}: \|M\|_{\op} \le b\}$. By rescaling, it suffices to bound  $\eulN(\epsilon,\frac{1}{b}\scrM,\|\cdot\|_{\op})$. We recognize $\scrM$ as the operator norm ball in $\R^{d\times \dimx}$ and appeal to the following standard lemma.
	\begin{lemma}[\cite{wainwright2019high}, Lemma 5.2] Let $\cB$ be the unit ball in $\R^{d}$ for an arbitrary norm. Then, if $\dist$ is the metric induced by the norm, $\eulN(\epsilon,\cB, \dist) \le (1 + \frac{2}{\eps})^d$.
	\end{lemma}
	\qed

\subsubsection{Proof of \Cref{prop:techtools_general_regression_with_truncation_and_error}}
Before diving into the meat of the proof, we first establish some basic concentration properties and state a number of definitions.	For each realization $(\maty,\mate,\matu)$, define 
	\begin{align*}
	\cE := \{\|\mate - \matu\|^2 \vee \varphi(\by) \le  c_{n,\delta} \},
	\end{align*}
        where we recall that $c_{n,\delta} := c\ln(2n/\delta)$. Let $\ell(h) := \|h(\maty) + \mate - \matu\|^2$, and let $\cL(h)=\En\brk*{\ell(h)}$. Furthermore, define $\ell_{\delta}(h)=\ell(h)\indic(\cE)$ and $\cL_{\delta}(h)=\En\brk*{\ell_{\delta}(h)}$. We first establish the following useful claim.
	\begin{claim}\label{claim:techtools_Eh_dff}Then on $\cE$, $ |\ell(h) - \ell(h')| \le 4\sqrt{c_{n,\delta}}   \|h(\maty) - h'(\maty)\|$. Moreover, defining $\matz = \|\mate - \matu\|^2 \vee \varphi(y)$, we have that $\ell(h) \le 4\matz$. In particular, on $\cE$, $\ell(h) \le 4c_{n,\delta}$.
	\end{claim}
	\begin{proof}[Proof of \Cref{claim:techtools_Eh_dff}]
	\begin{align*}
	 |\ell(h) - \ell(h')| &= | \|h'(\maty) + \mate - \matu\|^2 -  \|h(\maty) + \mate - \matu\|^2 | \nonumber\\
	&\le 2| \langle h(\maty) - h'(\maty),\mate - \matu\rangle | + (\|h(\maty)\| + \|h'(\maty)\|)(\|h(\maty)\| - \|h'(\maty)\|)\\
	&\le 4\sqrt{c_{n,\delta}} \|h(\maty) - h'(\maty)\|.
	\end{align*}
	This proves the first claim. The claim holds because $\ell(h) \le 2\|\mate - \matu\|^2 + 2\|h(\maty)^2\| \le 2\|\mate - \matu\|^2  + 2\max\{\varphi(\by),c\} \le 4\max\{\matz,c\}$.
	\end{proof}

	Next, Let $\cE^{(i)}$ denote the event that $\cE$ holds for the $i$th sample, and let $\cE_{1:n} = \bigcup_{i \in[n]} \cE^{(i)}$. Note that $\cE_{1:n}$ occurs with probability at least $1- 2\cdot \delta/2 = 1-\delta$ by the $c$-concentration property and a union bound. On this event, if we define
	\begin{align*}
           \cL_{n,\delta}(h) &:=   \sum_{i=1}^n \ell_{i,\delta}(h), \quad\text{where}\quad \ell_{i,\delta}(h) := \I(\cE^{(i)}) \|h(\by^{(i)}) + \mate^{(i)} - \bu^{(i)} \|^2,
	\end{align*}
        we have
        \begin{align*}
          \hhat_n :=\argmin_{h \in \Hclass} \cL_{n,\delta}(h).
        \end{align*}
	Lastly, define the excess risk with respect to the Bayes function $\hst(y) = \E[\matu \mid \maty = y]$:
	\begin{align*}
	\cR_{n,\delta}(h) = \cL_{n,\delta}(h) - \cL_{n,\delta}(\hst), \quad \cR_{\delta}(h) = \cL_{\delta}(h) - \cL_{\delta}(\hst).
	\end{align*}
	Finally, let $\Hclass_0 \subset \Hclass$ denote a finite cover for $\Hclass$ such that, for some $\epsilon > 0$ to be selected at the end of the proof,
	\begin{align}
	\sup_{h \in \Hclass}\inf_{h' \in \Hclass_0} \sup_{y: \varphi(y) \le c_{n,\delta}} \|h(y) - h'(y) \| \le \sqrt{c_{n,\delta}} \epsilon, \label{eq:techtools_H_0_def}
	\end{align}
	and let $\hhat_0 \in \Hclass_0$ denote the element that witnesses the covering inequality above for $\hhat_n$. Note that by \Cref{claim:techtools_Eh_dff} and \eqref{eq:techtools_H_0_def}, the differences on the truncated losses between $\hhat_n$ and $h_0$ satisfy
	\begin{align*}
	|\cL_{n,\delta}(\hhat_n) - \cL_{n,\delta}(\hhat_0)| \vee|\cL_{\delta}(\hhat_n) - \cL_{\delta}(\hhat_0)| \le 4 \epsilon  c_{n,\delta},
	\end{align*}
        whenever $\cE_{1:n}$ holds. Thus, on $\cE_{1:n}$, when $\hhat_n \in \argmin_{h \in \Hclass} \cR_{n,\delta}(h)$, we have
	\begin{align}
	\cR_{\delta}(\hhat_n) &= \cR_{\delta}(\hhat_n) - \cR_{n,\delta}(\hhat_n)  + \cR_{n,\delta}(\hhat_n) \nonumber\\
	&\overset{(i)}{\le} \cR_{\delta}(\hhat_n) - \cR_{n,\delta}(\hhat_n)  \nonumber\\
	&\le \cR_{\delta}(\hhat_0) - \cR_{n,\delta}(\hhat_0)  + 2\max\{\abs{\cL_{\delta}(\hhat_0) - \cL_{\delta}(\hhat_n)}, \abs*{\cL_{n,\delta}(\hhat_0) - \cL_{n,\delta}(\hhat_n)}\}\nonumber\\
	&\le \cR_{\delta}(\hhat_0) - \cR_{n,\delta}(\hhat_0)  + 8c_{n,\delta} \epsilon.\label{eq:techtools_Rhhat},
	\end{align}
      where $(i)$ uses that $\cR_{n,\delta}(\hhat_n)$ is non-positive for the empirical risk minimizer.

	\paragraph{Step 1: Bounding $\cR_{\delta}(n)$.} From the bound $\ell_{i,\delta}(h) \le 4 c_{n,\delta}$ (\Cref{claim:techtools_Eh_dff}), along with Bennett's inequality (see e.g. Theorem 3 of \cite{MaurerP09}) and a union bound over $\Hclass_0$, we have, for all $\delta \in (0,1)$, with probability at least $1-\delta/2$,
		\begin{align}
		 \cR_{\delta}(\hhat_0)- \cR_{n,\delta}(\hhat_0)  &  \leq \sqrt{2n^{-1} \Var[\ell_{i,\delta}(\hhat_0) -  \ell_{i,\delta}(\hst)]  \cdot \ln (2|\scrH_0| \delta^{-1})}+\tfrac{4c_{n,\delta}  }{3}  \ln (2|\scrH_0 |\delta^{-1})\nonumber\\
		&  \le \frac{\tau}{2c_{n,\delta} }\Var[\ell_{i,\delta}(\hhat_0) -  \ell_{i,\delta}(\hst)] +   \frac{ c_{n,\delta} }{n}\prn*{\frac{4}{3}   + \frac{1}{\tau}}\ln (2|\scrH_0 |\delta^{-1}), \label{eq:techtools_am_gm_last_step_R}
		\end{align}
		where the last step uses AM-GM and holds for all $\tau > 0$. Again, by \Cref{claim:techtools_Eh_dff}, we have
		\begin{align*}
		\Var[\ell_{i,\delta}(\hhat_0) -  \ell_{i,\delta}(\hst)] &\le  \E\left[\I(\cE) \left(\ell(\hhat_0) -  \ell(\hst)\right)^2\right]\\
		&\le  16c_{n,\delta} \E\left[\I(\cE) \|\hhat_0(\by) -  \hst(\maty)\|^2\right]\\
		&\le  32c_{n,\delta} \E\left[\I(\cE)\|\hhat_0(\maty) -  \hhat_n(\maty)\|^2\right] + 32c_{n,\delta} \E\left[\I(\cE)\|\hhat_n(\maty) -  \hst(\maty)\|^2\right].
		\end{align*}
		From \Cref{eq:techtools_H_0_def}, we have $\E\left[\I(\cE)\|\hhat_0(\maty) -  \hhat_n(\maty)\|^2\right] \le c_{n,\delta}  \epsilon^2$ Moreover, we can always upper bound $\E\left[\I(\cE)\|\hhat_n(\maty) -  \hst(\maty)\|^2\right]$ by removing the indicator.
		 This ultimately yields
		\begin{align*}
		\Var[\ell_{i,\delta}(\hhat_0) -  \ell_{i,\delta}(\hst)] \le 32c_{n,\delta} ^2 \epsilon^2 + 32c_{n,\delta} \E\left[\|\hhat_n(\maty) -  \hst(\maty)\|^2\right].
		\end{align*}
		 Thus, combining the above with \Cref{eq:techtools_Rhhat,eq:techtools_am_gm_last_step_R}, we have
		\begin{align}
		\cR_{\delta}(\hhat_n) \le 16\tau  \E \left[\|\hhat_n(\by) - \hst(\by)\|^2\right] + 16c_{n,\delta} (\epsilon  + \tau\epsilon^2 ) +   \frac{c_{n,\delta} }{n}\left(\frac{4}{3}   + \frac{1}{\tau}\right)\ln (2|\scrH_0 |\delta^{-1}). \label{eq:techtools_final_conc}
		\end{align}
		
	\paragraph{Step 2: Relating $\cR_{\delta}(h)$ to error against $\hst$.} 
	Recall that $\cL_{\delta}(h) \le \cL(h) $ due to truncation, so that
	\begin{align}
	\cR_{\delta}(h) = \cL_{\delta}(h) - \cL_{\delta}(\hst) &\ge \cL_{\delta}(h) - \cL(\hst)\nonumber\\
	&\ge \cL(h) - \cL(\hst) - |\cL_{\delta}(h) -\cL(h)|.\label{eq:techtools_init_R_bound}
	\end{align}
	We further develop
	\begin{align*}
	\cL(h) - \cL(\hst) &= \E[\|h(\maty) + \mate - \matu\|^2 - \|\hst(\maty) + \mate - \matu\|^2]\\
	&= \E[\|h(\maty)  - \matu\|^2 - \|\hst(\maty) - \matu\|^2] + 2\E\langle \mate, h(\maty) - \hst(\maty) \rangle\\
	&\ge \E[\|h(\maty)  - \matu\|^2 - \|\hst(\maty) - \matu\|^2] - 2\E\|\mate\|^2 -  \frac{1}{2}\E\|h(\maty) - \hst(\maty)\|^2,
	\end{align*}
	where the last line uses Cauchy-Schwartz and AM-GM
	Moreover, since $\hst = \E[\maty \mid \matu]$, we can see that $\E[\|h(\maty)  - \matu\|^2 - \|\hst(\maty) - \matu\|^2] = \E\|h(\maty) - \hst(\maty)\|^2$.  This yields
	\begin{align*}
	\cL(h) - \cL(\hst) \ge -2\E\|\mate\|^2   + \frac{1}{2}\E\|h(\maty) - \hst(\maty)\|^2.
	\end{align*}
	Hence, \Cref{eq:techtools_init_R_bound} yields that for all $h$, 
	\begin{align*}
	\E\|h(\maty) - \hst(\maty)\|^2 &\le  2\cR_{\delta}(h)  + 4 \E\|\mate\|^2   + 2|\cL_{\delta}(h) - \cL(h)|.
	\end{align*}
	Finally, recalling $\matz = \varphi(\maty)^2  \vee \|\mate - \matu\|^2$, we have
	\begin{align*}
	\sup_{h \in \Hclass} 2|\cL_{\delta}(h) - \cL(h)| &= \sup_{h \in \Hclass} 2\E[ \I(\cE^c) \ell(h)]\\
	&\le 8\E[ \I(\cE^c) \max\{c,\matz\}] \tag*{(\Cref{claim:techtools_Eh_dff}) }\\
	& \le \frac{8c\delta}{3 n}. \tag*{(\Cref{lem:techtools_truncated_conc})}
	\end{align*}
	Hence, the previous two displays give
	\begin{align*}
	\E\|h(\maty) - \hst(\maty)\|^2 &\le  2\cR_{\delta}(h)  + 4 \E\|\mate\|^2  + \frac{8c\delta}{3 n}.
	\end{align*}
	Thus, choosing $h=\hhat_n$ and combining with \Cref{eq:techtools_final_conc}, we have
	\begin{align*}
	\E\|\hhat_n(\maty) - \hst(\maty)\|^2 &\le 32\tau  \E \left[\|\hhat_n(\by) - \hst(\by)\|^2\right] + 32c_{n,\delta} (\epsilon  + \tau\epsilon^2 )\\
	&+   2\frac{c_{n,\delta} }{n}\left(\frac{4}{3}   + \frac{1}{\tau}\right) + \frac{8c\delta}{3 n} + 4\En\nrm*{\be}^{2}.
	\end{align*}
	Setting $\tau = \frac{1}{64}$ and using $\epsilon \le 1$ gives
	\begin{align*}
	\E\|\hhat_n(\maty) - \hst(\maty)\|^2 &\le \frac{1}{2} \E \left[\|\hhat_n(\by) - \hst(\by)\|^2\right] \\
	&\quad+ 33 c_{n,\delta} \epsilon  +  \frac{c_{n,\delta} }{n}\left(\frac{8}{3}   + 128\right)\ln (2|\scrH_0 |\delta^{-1}) + 4 \E\|\mate\|^2  + \frac{8c\delta}{3 n}\\
	&\le \frac{1}{2} \E \left[\|\hhat_n(\by) - \hst(\by)\|^2\right] + 33 c_{n,\delta} \epsilon  +  \frac{c_{n,\delta} }{n}134\ln (2|\scrH_0 |\delta^{-1})  + 4 \E\|\mate\|^2,
	\end{align*}
	where in the last line we folded the $8c\delta/3n$ term into the term with the log, bounding $8/3 + 8\delta/3n \le 16/3 \le 6$. Rearranging the above yields
	\begin{align*}
	\E\|\hhat_n(\maty) - \hst(\maty)\|^2 &\le 66 c_{n,\delta} \epsilon  +  \frac{268c_{n,\delta} }{n}\ln (2|\scrH_0 |\delta^{-1}) +  8 \E\|\mate\|^2.
	\end{align*}
	Taking $\epsilon = 1/33n$ concludes the proof.

	\qed

\subsubsection{Proof of \Cref{cor_techtools:reg_cor_simple}}
	We verify Conditions 1-3 of \Cref{prop:techtools_general_regression_with_truncation_and_error} in succession:
	\begin{enumerate}
		\item Condition 1: By assumption $1$ of the corollary, $f(y) \le  L\max\{1,\|\fst(y)\|_2\}$, then $h(y) \le bL\max\{1,\|\fst(y)\|_2\} := \varphi(y)$ for all $h \in \Hclass$.
		\item Condition 2: This is satisfied by assumption $2$ of the corollary,   $\varphi(\by)^{1/2}$ and $\|\mate - \bu\|^{2}$ are  $c$-concentrated. 
		\item Condition 3: We bound the covering number. Let $\scrM := \{M \in \R^{d_ud_x}: \|M\|_{\op} \le b\}$. Then, for an $\epsilon > 0$ to be chosen, let $\eulN(b\epsilon,\scrM,\|\cdot\|_{\op}) \le (1 + 2/\epsilon)^{d_ud_x}$ from \Cref{lem:Mcover}, so we may take a $b\epsilon$- cover $\scrM_{\epsilon}$ of $\scrM$ to have cardinality  $ (1 + 2/\epsilon)^{d_ud_x}$. Define the induced cover $\scrH_{\epsilon}:= \{Mf : M \in \scrM_{\epsilon}, f \in \Fclass\}$, which has $|\scrH_{\epsilon}| \le |\Fclass|(1 + 2/\epsilon)^{d_ud_x}$. Given $h = Mf\in \Hclass$, let $h' : M' f$, where $M' \in \scrM_{\epsilon}$ satisfies $\|M - M'\|_{\op} \le b \epsilon$. Then, 
		\begin{align*}
		d_{c',\infty}(h,h') &:= \sup_{y\in\cY}\{\|h(y) - h'(y)\|: \varphi(y)^{1/2} \le \sqrt{c'}\}\\
		&:= \sup_{y\in\cY}\{\|(M-M')f(y)\|: \varphi(y)^{1/2} \le \sqrt{c'}\}\\
		&\le \sup_{y\in\cY}\{b \epsilon \cdot \|f(y)\|:  \varphi(y)^{1/2} \le \sqrt{c'}\} \tag*{($\|M-M'\|_{\op} \le b\epsilon$)}\\
		&\le \sup_{y\in\cY}\{b \epsilon \cdot L \max\{1,\|\fst(y)\|\}: \varphi(y)^{1/2} \le \sqrt{c'} \} \tag*{(Assumption 1 of Corollary)}\\
		&\le \sup_{y\in\cY}\{b \epsilon \cdot L \max\{1,\|\fst(y)\|\}:  bL \max\{1,\|\fst(y)\|\} \le \sqrt{c'}\} \tag*{(Definition of $\varphi$)}\}\\
		&= \epsilon\sqrt{c'}.
		\end{align*}
		Hence, the $\sqrt{c'}\epsilon$ cover of $\scrH$ in the metric $d_{c',\infty}(h,h')$ is at most the cardinality of $\scrH_{\epsilon}$, which is at most $|\Fclass|(1 + 2/\epsilon)^{d_ud_x}$. Thus, we can take $\ln \eulN(\epsilon) = \ln|\Fclass| +d_ud_x\ln(1 + 2/\epsilon)$ in Condition 3 of \Cref{prop:techtools_general_regression_with_truncation_and_error}.  For $\epsilon \le 1$, this may be upper bounded by $\ln \eulN(\epsilon) = \ln|\Fclass| + d_ud_x\ln(5/\epsilon)$. 
	\end{enumerate}

	Hence, the conclusion of \Cref{prop:techtools_general_regression_with_truncation_and_error} entails that, with probability at least $1 - \frac{3\delta}{2}$,
		\begin{align*}
		&\E\| \hhat_n(\by) - \hst(\by)\|^2 \\
		&\le  \frac{270  c\ln(2n/\delta)(\ln(2/\delta) + \ln|\Fclass| + d_ud_x\ln(5\cdot 33n)) }{n} +  8 \E\|\mate\|^2.\\
		\end{align*}
		Recalling that $\somelogs(n,\delta):= 270 \ln(2n/\delta)\ln(330n/\delta)$, we simplify 
 	 	\begin{align}
 	 	 270 \cdot \ln(2n/\delta) \cdot (\ln(2/\delta) + \ln|\Fclass| + d_1 \dimx \ln(3\cdot 55 n))  &\le 270 \ln(2n/\delta)(\ln|\Fclass| + d_1\dimx \ln(330 n /\delta)| \nn\\
 	 	 &\le (\ln|\Fclass| + d_1\dimx)270 \ln(2n/\delta)\ln(330n/\delta) := (\ln|\Fclass| + d_1\dimx)\somelogs(n,\delta)\label{eq:logs_simplif}
 	 	\end{align}
 	 	which yields our final bound of 
 	 	\begin{align*}
		\E\| \hhat_n(\by) - \hst(\by)\|^2 &\le \frac{c (d_ud_x +\ln |\Fclass|) \cdot \somelogs(n,\delta)}{n} +   8 \E\|\mate\|^2,
		\end{align*}
		as needed.

	\qed

\subsubsection{Proof of \Cref{cor:techtools_function_dependend_error}}
	We consider
	\begin{align*}
		\hhat := \argmin_{h \in \scrH} \sum_{i=1}^n \|\phihat(h, \bz^{(i)})  + \mate^{(i)} - \bv^{(i)} \|^2.
	\end{align*}
	Recall the assumption that, for some $\hst \in \scrH$, $\phi(\hst,\bz) = \E[\bu \mid \matz]$, and that $\delphi(h,z) := \phi(h,z) - \phi(h,z)$. Let us set up a correspondence with \Cref{prop:techtools_general_regression_with_truncation_and_error}. 
	\begin{itemize}
		\item $g_h(\bz) := \phihat(h,\bz) - \delphi(\hst,\bz)$. Let $\scrG$ denote the resulting class of functions $\{g_h :  h \in \Hclass\}$.
		\item $\matetil = \delphi(\hst,\bz) + \mate$.
	\end{itemize}
 	Then, we have
	\begin{align*}
	\phihat(\hhat)  - \delphi(\hst,\bz) = g_{\hhat}(\maty)  := \argmin_{g \in \scrG} \sum_{i=1}^n \|g(\bz) + \matetil^{(i)} - \bv^{(i)} \|^2.
	\end{align*}
	\newcommand{\tilhst}{g_{\star}}
	\newcommand{\tilh}{g}
	\newcommand{\tilhhat}{\widehat{g}}
	\newcommand{\scrHtil}{\scrG}

	Define $\tilhst := g_{\hst} = \phi(\hst,\matz)$ and $\tilhhat = g_{\hhat}$. We then have
	\begin{align}
	\E\|\phi(\hhat,\matz) - \phi(\hst,\matz)\|^2 &= \E\|\tilhhat(\matz) - \tilhst(\matz) + \delphi(\hhat,\bz) - \delphi(\hst,\bz) \|^2\nn\\
	&=2\E\|\tilhhat(\matz) - \tilhst(\matz)\|^2  + 2\E\|\delphi(\hhat,\bz) - \delphi(\hst,\bz) \|^2\nn\\
	&\le 2\E\|\tilhhat(\matz) - \tilhst(\matz)\|^2  + 8\max_{h \in \scrH}\E\|\delphi(h,\bz)\|^2 \label{eq:phi_ub_cor_2}
	\end{align}
	It remains to bound $\E\|\tilhhat(\matz) - \tilhst(\matz)\|^2$.

	Note that  $\tilhst \in \scrG$, and moreover $\tilhst(z) = \phi(\hst,\bz)$, which is equal to $\E[\bu \mid \matz]$ by assumption. Considering the function class $\scrG = \{\tilh_h: h \in \scrH\}$ as the function class, $\tilhst$ as the Bayes regressor, $\matv$ as the target, and $\matetil$ as the residual noise, let us verify with conditions of \Cref{prop:techtools_general_regression_with_truncation_and_error}, albeit with slightly inflated constants. We have
	 \begin{enumerate}
	 	\item Define $\tilde{\varphi}(z) = 6\psi^2(z)$. We bound $\|\tilh_h(z)\|^2 \le \tilde{\varphi}(z)$ via 
	 	\begin{align*}
	 	\|\tilh_h(z)\|^2 &\le (\|\phi(h,z)\| + \|\delphi(h,z)\| + \|\delphi(\hst,z)\|)^2 \\
	 	&\le (\sup_{h \in \scrH}\|\phi(h,z)\| + 2 \sup_{h \in \scrH}\|\delphi(h,z)\|)^2 \le 2\sup_{h \in \scrH}\|\phi(h,z)\|^2 + 6\sup_{h \in \scrH}\|\delphi(h,z)\|^2 \\
	 	&\le 6(\sup_{h \in \scrH}\|\phi(h,z)\|^2 + \sup_{h \in \scrH}\|\delphi(h,z)\|^2) \le 6\psi(z)^2 := \tilde{\varphi}(z),
	 	\end{align*}
	 	where the last inequality follows by the first assumption of the lemma. 
	 	\item Next, we establish the concentration property for $\tilde{\varphi}(\matz) \vee \|\matetil_t - \matv_t\|^2$ that, for $\tilde{c}= 6c$, we have
	 	\begin{align}
	 	\Pr[\tilde{\varphi}(\matz) \vee \|\matetil_t - \matv_t\|^2 \ge \tilde{c} \ln(1/\delta)] \le 1/\delta.  \label{eq:varphi1_cor_bound}
	 	\end{align}

	 	We have that 
	 	\begin{align*}
	 	\tilde{\varphi}(z) \vee \|\matetil_t - \matv_t\|^2 &= \tilde{\varphi}(z) \vee \|\delphi(z,\hst) + \mate_t - \matv_t\|^2\\
	 	&\le \tilde{\varphi}(z) \vee \left(2\|\delphi(\hst,z)\|^2 + 2\|\mate_t - \matv_t\|^2\right)\\
	 	&\le (\tilde{\varphi}(z) \vee 2\|\delphi(\hst,z)\|^2) + 2\|\mate_t - \matv_t\|^2).
	 	\end{align*} 
	 	Now, by assumption, we have that $2\|\delphi(\hst,z)\|^2 \le \tilde{\varphi}(z) = 6\psi(z)^2$, so we may drop the $\delphi$-term. Substituting in the definition of $\tilde{\varphi}(z)$ and bounding $2 \le 6$ gives
	 	\begin{align*}
	 	\tilde{\varphi}(z) \vee \|\matetil_t - \matv_t\|^2  \le 6\left( \psi(z)^2 \vee  (\|\mate_t - \matv_t\|^2) \right). 
	 	\end{align*}
	 	Hence, the desired inequality \Cref{eq:varphi1_cor_bound} follows from the second condition of our corollary. 
	 	\item Lastly, it remains to verify the covering property from \Cref{prop:techtools_general_regression_with_truncation_and_error}. Let $\scrM := \{M \in \R^{d_1\dimx}:\|M_{\op}\| \le b\}$, let $\scrM_{\epsilon}$ denote a $b\epsilon$-cover of $\scrM$ in $\|\cdot\|_{\op}$, let $\scrH_{\epsilon} := \{M\cdot f: M \in \scrM_{\epsilon},f \in \Fclass\}$, and finally set $\scrHtil_{\epsilon} := \{\tilh_{h}: h \in \scrH_{\epsilon}\}$. Our goal will be to show that, for $\epsilon$ adequately chosen, $\scrHtil_{\epsilon}$ is an adequate cover of $\scrHtil$. 

	 	Let $\tilh \in \scrHtil$. Then, $\tilh = \tilh_{h}$, where $h = M \cdot f$ for some $f \in \Fclass$ and $M \in \scrM$. Let $\htil_{\epsilon} \in \scrH_{\epsilon}$ be selected by selecting $M_{\epsilon} \in \scrM_{\epsilon}$ such that $\|M - M_{\epsilon}\|_{\op} \le b\epsilon$,  $h_{\epsilon} := M_{\epsilon} \cdot f$, and $\tilh_{\epsilon} := \tilh_{h_{\epsilon}}$. Then, for any $z$, we have
	 	\begin{align*}
	 	 \tilh(z) - \tilh_{\epsilon}(z) &= \phihat(h,z) + \delphi(\hst,z) - (\phihat(h_{\epsilon},z)+ \delphi(\hst,z))\\
	 	 &= \phihat(h,z)-\phihat(h_{\epsilon},z)\\
	 	 &\overset{(i)}{=} X_1(h(y_1) - X_2 h(y_2)) + \updelnot(z) - \left(X_1(h_{\epsilon}(y_1) - X_2 h_{\epsilon}(y_2)) + \updelnot(z)\right)\\
	 	 &= X_1(h(y_1) - h_{\epsilon}(y_1)) - X_1 X_2 (h(y_2) - h_{\epsilon}(y_2))\\
	 	 &= X_1(M - M_{\epsilon})f(y_1) - X_1 X_2 (M - M_{\epsilon}) f(y_2),
	 	\end{align*}
	 	where in $(i)$ we use the functional form of $\phihat$ assumed by the lemma. Since $\|M-M_{\epsilon}\|_{\op} \le b\epsilon$, and $f(y) \le L \max\{1,\|\fst\|\}$
	 	\begin{align*}
	 	 \tilh(z) - \tilh_{\epsilon}(z) &\le b \epsilon(\|X_1\|_{\op}\|f(y_1)\| + \|X_1X_2\|_{\op}\|f(y_2)\|)\\
	 	 &\le b L \epsilon(\|X_1\|_{\op} + \|X_1\cdot X_2\|_{\op})(\max\{1,\|\fst(y_1)\|\} + \max\{1,\|\fst(y_2)\|\})\\
	 	 &\le b L \epsilon(\|X_1\|_{\op} + \|X_1\cdot X_2\|_{\op})(2 \vee \|\fst(y_1)\| + \|\fst(y_2)\|).
	 	\end{align*}
	 	Finally, by assumption, we have that $bL(\|X_1\|_{\op} + \|X_1\cdot X_2\|_{\op})(2 + \|\fst(y_1)\| + \|\fst(y_2)\|) \le 2c_{\psi}\psi(z)$. Thus, recalling $\tilde{\varphi}(z) = 6\psi(z)^2$, we have
	 	\begin{align*}
	 	 \tilh(z) - \tilh_{\epsilon}(z) &\le c_{\psi}\epsilon \sqrt{(2\psi(z))^2} \le c_{\psi}\epsilon \tilde{\varphi}(z)^{1/2}.
 	 	\end{align*}
 	 	It therefore follows that, for all $c'$, and all $\epsilon \le 1$, $\scrHtil_{\epsilon/c_{\psi}}$ is a $\sqrt{c'}\epsilon$-covering number of $\scrHtil$ in the pseudometric $d_{c',\infty}(h,h') := \sup_{z\in\cZ}\{\|\tilh(z) - \tilh'(z)\|: \varphi(z) \le c'\}$. Hence, we can take $\eulN(\epsilon) = |\scrHtil_{\epsilon}|$ is applying \Cref{prop:techtools_general_regression_with_truncation_and_error}.
	 \end{enumerate}
	 Hence, \Cref{prop:techtools_general_regression_with_truncation_and_error}  implies the bound
	\begin{align*}
	\E\|\ghat - \gst \|^2  &\le   \frac{270 \tilde{c}_{n,\delta} }{n}\ln (2|\scrHtil_{1/33c_{\psi}n}|\delta^{-1}) +  8 \E\|\mate\|^2, %
	\end{align*}
	where we have $\tilde{c}_{n,\delta} = \tilde{c}\ln(2n/\delta) = 6c\ln(2n\delta)$. Combining the above with \Cref{eq:phi_ub_cor_2}
	\begin{align*}
	\E\|\phi(\ghat,\matz) - \phi(\gst,\matz)\|^2 \le  \frac{12 \cdot 270 \ln(2n/\delta)}{n}\ln (2|\scrHtil_{1/33c_{\psi}n}|\delta^{-1}) +  16 \E\|\mate\|^2 + 8\max_{h \in \scrH}\E\|\delphi(h,\bz)\|^2.
	\end{align*}
	Finally, let us bound $|\scrHtil_{1/33c_{\psi}n}|$. From \Cref{lem:Mcover}, we have
 	 	\begin{align*}
 	 	\ln|\scrHtil_{\epsilon}| = \ln|\scrH_{\epsilon}| = \ln(|\Fclass||\scrM_{\epsilon}|) \le \ln(|\Fclass|) + d_1 \dimx \ln(5/\epsilon).
 	 	\end{align*}
 	 	Thus, repeating the computation \Cref{eq:logs_simplif} in the proof of \Cref{cor_techtools:reg_cor_simple},
 	 	\begin{align*}
 	 	 270 \ln(2n/\delta)\ln (2|\scrHtil_{1/33c_{\psi}n}|\delta^{-1}) &\le 270 \cdot \ln(2n/\delta) \cdot (\ln(2/\delta) + \ln|\Fclass| + d_1 \dimx \ln(3\cdot 55 c_{\psi} n))\\
 	 	 &\le  (\ln|\Fclass| + d_1\dimx)\somelogs(c_{\psi}n,\delta).
 	 	\end{align*}
 	 	Thus,
 	 	\begin{align*}
	\E\|\phi(\ghat,\matz) - \phi(\gst,\matz)\|^2 \le  \frac{12c  (\ln|\Fclass| + d_1\dimx)\somelogs(c_{\psi} n,\delta)}{n} +  16 \E\|\mate\|^2 + 8\max_{h \in \scrH}\E\|\delphi(h,\bz)\|^2,
	\end{align*}
	concluding the corollary. \qed

\subsubsection{Proof of \Cref{prop:techtools_general_pca}}
	Define the matrix $\Lamhat = \E[\hhat(\maty)\hhat(\maty)^\top]$.
        To begin, we have
	\begin{align*}
	\|\Lambda_{\star} - \Lamhat_{n}\|_{\op} \le \|\Lambda_{\star} - \Lamhat\|_{\op} + \|\Lamhat - \Lamhat_n\|_{\op}.
	\end{align*}
	We now bound the terms on the right-hand side one by one. First, 
	\begin{align*}
	&\|\Lambda_{\star} - \Lamhat\|_{\op}\\
	&=\|\E[\hhat(\maty)\hhat(\maty)^\top - \hst(\maty)\hst(\maty)^\top \|_{\op}  \\
	&\le \E[\|\hhat(\maty)\hhat(\maty)^\top - \hst(\maty)\hst(\maty)^\top\|_{\op}]  \\
	&\le \E[(\|\hhat(\maty)\| + \|\hst(\maty)\|)\|\hst(\maty) - \hhat(\maty)\|]\\
	&\le \E[(\|\hhat(\maty)\| + \|\hst(\maty)\|)^2]^{1/2}\E\brk*{\|\hst(\maty) - \hhat(\maty)\|^{2}}^{1/2}.
	\end{align*}
	Moreover, $(\|\hhat(\maty)\| + \|\hst(\maty)\|)^2 \le 4\max\{c,\varphi(\by)\}$ by assumption, so this is at most
	\begin{align}
	\|\E[\hhat(\maty)\hhat(\maty)^\top - \hst(\maty)\hst(\maty)^\top] \|_{\op}  \le 2(\E\max\{c,\varphi(\by)\})^{1/2})\E\brk{\|\hst(\maty) - \hhat(\maty)\|^{2}}^{1/2}. \label{eq:techtools_pca_exp_diff}
	\end{align}
	Finally, using \Cref{lem:techtools_truncated_conc}, one can bound $\E\max\{c,\varphi(\by)\} \le 2c$, so we can further bound by $3\sqrt{c} \E[\|\hst(\maty) - \hhat(\maty)\|]^{1/2}$. 
	
	For the second term, we appeal to truncation. Let $\cE$ denote the event $\{\varphi(\by) \le c \ln(2n/\delta)\}$, and let $\cE^{(i)}$ denote the analogous event for $\maty\ind{i}$. By construction $\cE\ind{1},\ldots,\cE\ind{n}$ occur simultaneously with probability at least $1-\delta/2$, so that we may bound
	\begin{align}
	&\|\Lamhat - \Lamhat_n\|_{\op}\\
	 &=\left\|\frac{1}{n}\sum_{i=1}^n  \E[\hhat(\maty)\hhat(\maty)^\top] - \indic(\cE\ind{i})\hhat(\maty^{(i)})\hhat(\maty^{(i)})^\top \right\|_{\op} \nonumber\\
	&\le \left\|\E[(1- \indic(\cE))\hhat(\maty)\hhat(\maty)^\top]\right\|_{\op}\nonumber\\
	&\quad+\left\|\frac{1}{n}\sum_{i=1}^n  \E[\hhat(\maty)\hhat(\maty)^\top \I(\cE)] - \I(\cE^{(i)})\hhat(\maty^{(i)})\hhat(\maty^{(i)})^\top\right\|_{\op}. \label{eq:techtools_pca_empirical_term_trunc}
	\end{align}
        We bound the first term above by
	\begin{align}
	\left\|\E[(1 - \cE)\hhat(\maty)\hhat(\maty)^\top\right\|_{\op} &\le \E\brk*{\left\|(1 - \cE)\hhat(\maty)\hhat(\maty)^\top \right\|_{\op}} \nonumber\\
	&= \E[(1 - \cE)\|\hhat(\by)\|^2]\nonumber \\
	&\le \E[ \I(\varphi(\by) >  c \ln(2n/\delta))\max\{c,\varphi(\by)\}]\nonumber\\
	&\le  \E[ \I(\varphi(\by) >  c \ln(2n/\delta))\varphi(\by)] \le \frac{ c \delta}{2n},\label{eq:techtools_pc_trunc_exp}
	\end{align}
	where the last line uses \Cref{lem:techtools_truncated_conc}.

        To conclude, let us bound the last term in \Cref{eq:techtools_pca_empirical_term_trunc}. Define the symmetric matrices $\bM^{(i)} := \E[\hhat(\maty)\hhat(\maty)^\top \I(\cE)]  - \hhat(\maty^{(i)})\hhat(\maty^{(i)})^\top \I(\cE^{(i)})$. Then $\E\bM^{(i)} = 0$, we can see that $\|\bM^{(i)}\| \le c\ln(2n/\delta)$ almost surely (indeed, if $X,Y \succeq 0$ , then $\|X - Y\|_{\op} \le \max\{\|X\|_{\op},\|Y\|_{\op}\}$),  and thus
	\begin{align*}
	(\bM^{(i)})^2 \preceq (c\ln(2n/\delta))^2 I.
	\end{align*}
	Hence, by Theorem 1.3 of \cite{tropp2012user},
	\begin{align*}
	\Pr\brk*{ \left\|\sum_{i=1}^n \bM^{(i)} \right\| \ge t} \le 2d e^{-t^2/8\sigma^2}, \quad\text{where}\quad \sigma^2 := n(c\ln(2n/\delta))^2.
	\end{align*}
	Rearranging, we have that 
	\begin{align*}
	\Pr\left[ \nrm[\bigg]{\frac{1}{n}\sum_{i=1}^n \bM^{(i)} } \ge 2c\ln(2n/\delta)\sqrt{2 \ln(2d/\delta )/n}\right] \le \frac{\delta}{2}.
	\end{align*}
	Simplifying $2c\ln(2n/\delta)\sqrt{2 \ln(2d/\delta )/n} \le 4cn^{-1/2} \ln(2d n/\delta)^{3/2}$, we have that with probabilitiy $1 - \delta/2$, $\frac{1}{n}\sum_{i=1}^n  \E[\hhat(\maty)\hhat(\maty)^\top \I(\cE)] - \I(\cE)^{(i)}\hhat(\maty^{(i)})\hhat(\maty^{(i)})^\top]  = \|\sum_{i=1}^n \bM^{(i)}\|_{\op} \le 4cn^{-1/2} \ln(2d n/\delta)^{3/2}$. Hence, combining with \Cref{eq:techtools_pca_exp_diff,eq:techtools_pc_trunc_exp}, we conclude that with probability $1 - \delta$, 
	\begin{align*}
	\|\Lambda_{\star} - \Lamhat_{n}\|_{\op} &\le \|\Lambda_{\star} - \Lamhat\|_{\op} + \|\Lamhat - \Lamhat_n\|_{\op}\\
	&\le 3  \sqrt{ c\E[\|\hhat(\by) - \hst(\by)\|^2]} +  4cn^{-1/2} \ln(2d n/\delta)^{3/2} +  \frac{\delta c}{2n}\\
	&\le 3\sqrt{ c\E[\|\hhat(\by) - \hst(\by)\|^2]} + 5 c n^{-1/2} \ln(2d n/\delta)^{3/2} := \veps^2_{\mathrm{pca},n} .
	\end{align*}
	\qed

\subsubsection{Proof of \Cref{cor:techtools_overlap_PCA}}
	By assumption, $\Lambda_{\star}$ is rank $\dimx$ and $\lambda_{\dimx}(\Lambda_{\star}) > 0$.  Let $V_{\star}$ be and eigenbasis for the top $\dimx$ eigenvalues of $\Lambda_{\star}$, and let $\Vhat_n$ be an eigenbasis for the top $\dimx$ eigenvalues of $\Lamhat_n$. From the Davis-Kahan sine theorem \citep{davis1970rotation}, we have that for any $\alpha \in (0,1)$,
	\begin{align}
          \|(I- V_{\star}V_{\star}^\top)\Vhat_n \|_{\op} \le (1 - \alpha)^{-1}\lambda_{\dimx}(\Lambda_{\star})^{-1}\|\Lambda_{\star} - \Lamhat_{n}\|_{\op}, \nonumber
        \end{align}
        whenever
        \begin{align}
          \|\Lambda_{\star} - \Lamhat_{n}\|_{\op} \le \alpha\lambda_{\dimx}(\Lambda_{\star}) . \label{eq:techtools_pca_davis_kahan}
	\end{align}
	In particular, if $n$ is sufficiently large that for
	\begin{align*}
	\|\Lambda_{\star} - \Lamhat_{n}\|_{\op} \le \veps_{\mathrm{pca},n,\delta} \le  \frac{1}{4}\lambda_{\dimx}(\Lambda_{\star}),
	\end{align*}
	then from \Cref{eq:techtools_pca_davis_kahan},
	\begin{align*}
	\|(I- V_{\star}V_{\star}^\top)\Vhat_n \|_{\op} \le \frac{4\veps_{\mathrm{pca},n,\delta}}{3\lambda_{d}(\Lambda_{\star})^{-1}} \le \frac{1}{3}.
	\end{align*}
	And thus, 
	\begin{align*}
	\sigma_{\dimx}(V_{\star}^\top \Vhat_n) \geq \sigma_{\dimx}(V_{\star} V_{\star}^\top \Vhat_n) &\ge \sigma_{\dimx}(\Vhat_n) - \|(I- V_{\star}V_{\star}^\top)\Vhat_n \|_{\op}\\
	&\ge 1 - 1/3 = 2/3,
	\end{align*} 
	where the previous display uses that $\sigma_{\dimx}(\Vhat_n) = 1$ since $\Vhat_n$ has orthonormal columns, and that $\|(I- V_{\star}V_{\star}^\top)\Vhat_n \|_{\op} \le 1/3$.%
	\qed

\subsubsection{Proof of \Cref{prop:techtools_LS_guarantee}}
Let $\bU$ be the matrix with $\crl{\bu^{(i)}}_{i=1}^{n}$ as rows, and let $\matDel,\bY,\bW,\bE$ be defined analogously for $\matdel$, $\maty$,$\matw$, and $\mate$ respectively. Let us assume for now that $(\bU+\matDel)$ has full row rank; this will be justified momentarily in \pref{claim:techtools_Lowner_LB_gaussian}. Then we have
	\begin{align}
	\Mhat^\top &= (\bU + \matDel)^{\dagger} \bY \nonumber\\
	&= (\bU + \matDel)^{\dagger} ( \bU \Mst^\top + \bW + \bE)\nonumber\\
	&= (\bU + \matDel)^{\dagger} (\bU  \Mst^\top+ \bW + \bE)\nonumber\\
	&= \Mst^\top + (\bU + \matDel)^{\dagger} (-\matDel\Mst^\top+ \bW + \bE). \label{eq:techtools_LS_identity}
	\end{align}
	Thus, 
	\begin{align}
	\| \Mhat - \Mst\|_{\op} \le \frac{\|\matDel\|_{\op}\|\Mst\|_{\op} + \|\bE\|_{\op} }{\sigma_{\min}(\bU + \matDel)} + \|(\bU + \matDel)^{\dagger}\bW\|_{\op}. \label{eq:techtools_ls_error_decomp}
	\end{align}
	\paragraph{Handling the Gaussian Noise.}
        We first handle the term $\|(\bU + \matDel)^{\dagger}\bW\|_{\op}$. Observe that $\bW$ is Gaussian conditioned on $\bU$ and $\matDel$. Fix a matrix $U$ with $U^\top U \succ 0$. Fix a vector $v \in \bbR^{\dimy}$ with $\|v\| = 1$, and observe that $\langle v, \mate^{(i)} \rangle$ are $\|\Sigw\|_{\op}$-subgaussian. Thus, for any matrix $\Lambda \succ 0$, we have from \citet[Theorem 3]{abbasi2012online} that conditioned on $\bU$ and $\matDel$, with probability at least $1-\delta$,
	\begin{align*}
	\|U^\top \bW v\|_{(\Lambda + U^\top U)^{-1}} \le \sqrt{2 \|\Sige\|_{\op}\ln\prn*{ \frac{\det(\Lambda + U^\top U)^{1/2}\det(\Lambda)^{-1/2}}{\delta}}}.
	\end{align*}
	Since we are taking $U$ to be fixed (conditioned on $\bU$ and $\matDel$), we can take $\Lambda = U^\top U$. This gives, with probability at least $1 - \delta$,
	\begin{align*}
	  \|U^\top \bW v\|_{(U^\top U)^{-1}} \le  2\sqrt{2 \|\Sigw\|_{\op}\ln\prn*{ \frac{2^{\dimu/2}}{\delta}}}
\le 2\sqrt{ \|\Sigw\|_{\op}(\dimu + 2\ln(1/\delta))}. 
	\end{align*}
It follows that
	\begin{align*}
	\|U^\dagger \bW v\|_{2} \le \sigma_{\min}(U)^{-1} \|U^\top \bW v\|_{(U^\top U)^{-1}} \le 2\sigma_{\min}(U)^{-1}\sqrt{ \|\Sigw\|_{\op}(\dimu + 2\ln(1/\delta))}. 
	\end{align*}
	By a standard covering argument (see, e.g., \citet[Section 4.2]{vershynin2018high}), we find that with probability at least $1-\delta$,
	\begin{align*}
	 \|U^{\dagger} \bW \|_{\op} = \sup_{v \in \bbR^{\dimy}:\|v\| = 1}\|U^{\dagger} \bW v\| \lesssim \sigma_{\min}(U)^{-1}\sqrt{ \|\Sigw\|_{\op}(\dimu + \dimy + \ln(1/\delta))},
	\end{align*}
	Taking $U = \bU + \matDel$, this implies that with probability at least $1 - \delta$,
	\begin{align*}
	 \|(\bU + \matDel)^{\dagger} \bW \|_{\op}  \lesssim \sigma_{\min}(\bU + \matDel)^{-1}\sqrt{ \|\Sigw\|_{\op}(\dimu + \dimy + \ln(1/\delta))},
	\end{align*}
	\paragraph{Error Terms.}
	We have
	\begin{align*}
	\|\matDel\|_{\op}\|\Mst\|_{\op} + \|\bE\|_{\op} \le \|\Mst\|_{\op}\|\matDel\|_{F} + \|\bE\|_{F} &= \left(\|\Mst\|_{\op}\sqrt{\sum_{i=1}^n \|\matdel^{(i)}\|^2}\right) + \sqrt{\sum_{i=1}^n \|\mate^{(i)}\|^2.}
	\end{align*}
	Recall that  1) $\psi(n,\delta) := \frac{2\ln(2n/\delta)\ln(2/\delta)}{n}$, 2)  $\|\matdel^{(i)}\|^2$ is $c_{\matdel}$-concentrated and $\|\mate^{(i)}\|^2$ are $c_{\mate}$ concentrated (\Cref{defn:techtools_conc_prop}), and 3) $\E\|\mate^{(i)}\|^2 \le \veps^2_{\mate}$ and 
	 $\E\|\matdel^{(i)}\|^2 \le \veps^2_{\matdel}$. \Cref{lem:techtools_truncated_conc} thus implies that for
         \begin{align*}
	\psi(n,\delta) \le \min\crl*{\frac{\veps^2_{\mate}}{c_{\mate}},  \frac{\veps^2_{\matdel}}{c_{\matdel}}}, %
	\end{align*}
	the following event holds with probability at least $1-2\delta$:
	\begin{align}
	\cE_{\mathrm{ls},1} := \crl*{\sum_{i=1}^n \|\matdel^{(i)}\|^2 \le 2n\veps^2_{\matdel}} \cap \crl*{\sum_{i=1}^n \|\mate^{(i)}\|^2 \le 2 n \veps^2_{\mate}}. \label{eq:techtools_event_Eone_ls}
	\end{align}
	Clearly, on $\cE_{\mathrm{ls},1}$ we have
	\begin{align*}
	\|\matDel\|_{\op}\|\Mst\|_{\op} + \|\bE\|_{\op} \lesssim n^{1/2}(\|\Mst\|_{\op}\veps_{\matdel} + \veps_{\mate}).
	\end{align*}
	\paragraph{Bounding the least eigenvalue.}
	Summarizing the development so far, we have for $\psi(n,\delta) \le \max\crl{\frac{\veps^2_{\mate}}{c_{\mate}},  \frac{\veps^2_{\matdel}}{c_{\matdel}}}$, with probability at least $1 - 3\delta$,
	\begin{align*}
	\| \Mhat - \Mst\|_{\op} \lesssim \frac{n^{1/2}(\|\Mst\|_{\op}\veps_{\matdel} + \veps_{\mate}) + \sqrt{\|\Sigw\|_{\op}(\dimy + \dimu + \ln(1/\delta))}}{\sigma_{\min}(\bU + \matDel)} .
	\end{align*}
	Finally, let us lower bound $\sigma_{\min}(\bU + \matDel) $. We start with the following self-contained result.
	\begin{claim}\label{claim:techtools_lowner_lb} Consider matrices $U,\Delta$, and suppose $\|\Delta\|_{\op}^2 \le \frac{1}{4}\lambda_{\min}(U^\top U)$. Then,
	\begin{align*}
	(U + \Delta)^\top (U + \Delta) \succeq \frac{1}{4} U^\top U.
	\end{align*}
	\end{claim}
	\begin{proof} 
		By Cauchy-Schwarz and AM-GM, we have the elementary inequality that for two vectors $v,w$ of the same dimension, $\|v+w\|^2 = \|v\|^2 +\|w\|^2 + 2 \langle v, w \rangle \ge  \frac{1}{2}\|v\|^2 - \|w\|^2$. This entails
		\begin{align*}
		(U + \Delta)^\top (U + \Delta) &\succeq \frac{1}{2}U^\top U - \Delta^\top \Delta\\
		&\succeq \frac{1}{2}U^\top U - I\|\Delta\|_{\op}^2 \\
		&= \frac{1}{4}U^\top U +(\frac{1}{4}U^\top U - \|\Delta\|_{\op}^2I) \\
		&\succeq \frac{1}{4}U^\top U,
		\end{align*}
		where the last line uses the that $\frac{1}{4}U^\top U \succeq \frac{1}{4}\lambda_{\min}(U^\top U)$, and the assumption $\|\Delta\|_{\op}^2 \le \frac{1}{4}\lambda_{\min}(U^\top U)$.
	\end{proof}

	\begin{claim}\label{claim:techtools_Lowner_LB_gaussian}
	There is a universal constant $c_1 > 0$ such that the following holds.  Let $\bU\in\bbR^{n\times\dimu}$ be a matrix with rows drawn i.i.d. from $\cN(0,\Sigma_u)$ where $\Sigu \succeq 0 $, and let $\matDel$ be a matrix of the same dimension with $\|\matDel\|_{\op}^2 \le \frac{n}{8}\lambda_{\min}(\Sigu)$. Then, for $n \ge c_1(\dimu  + \ln(1/\delta))$, the following holds with probability $1 - \delta$:
	\begin{align}
	(\bU + \matDel)^\top(\bU + \matDel) \succeq \frac{n \Sigu}{8}  \succeq \frac{n \lambda_{\min}(\Sigu)}{8} I.
	\end{align}
	\end{claim}
	\begin{proof}[Proof of \Cref{claim:techtools_Lowner_LB_gaussian}]
		From \Cref{claim:techtools_lowner_lb}, we have that if $\|\matDel\|_{\op}^2 \le \frac{1}{4}\lambda_{\min}(\bU^\top \bU)$, we have
		\begin{align*}
		(\bU + \matDel)^\top (\bU + \matDel) \succeq \frac{1}{4} \bU^\top \bU. 
		\end{align*}
		If this holds, we have
		\begin{align*}
		\frac{1}{4} \bU^\top \bU = \frac{1}{4} \Sigu^{1/2}\left(\Sigma^{-1/2}\bU^\top \bU \Sigu^{-1/2}\right)\Sigu^{1/2} \succeq \lambda_{\min}(\Sigu^{-1/2}\bU^\top \bU \Sigu^{-1/2}) \Sigu.
		\end{align*}
         
		Note that $\bU\Sigu^{-1/2}$ has standard Gaussian rows, and its number of rows exceeds its number of columns. Thus, from Theorem~5.39 of \cite{vershynin2010introduction}, we have that
		\begin{align}
		\P\left[ \lambda_{\min}(\Sigu^{-1/2}{\bU}^\top {\bU}\Sigu^{-1/2})^{1/2} \geq \sqrt{n} - \cO(\sqrt{\dimu}+\sqrt{\ln (1/\delta)})  \right] & \geq 1 -\delta. \label{eq:techtools_Verhsynin}
		\end{align}
		In particular, for $n \ge c_1(\dimu + \ln(1/\delta))$ for some universal $c_1$, we have that with probability $1 - \delta$, $\lambda_{\min}(\Sigu^{-1/2}{\bU}^\top {\bU}\Sigu^{-1/2})^{1/2} \ge \sqrt{n/2}$, and thus when this occurs, and when $\|\matDel\|_{\op}^2 \le \frac{n\lambda_{\min}(\Sigu)}{8} \le \frac{1}{4}\lambda_{\min}(\bU^\top \bU)$, we have 
		\begin{align*}
		(\bU + \matDel)^\top (\bU + \matDel)\succeq \frac{1}{4} \bU^\top \bU \succeq \frac{n}{8}\Sigu \succeq \frac{n}{8}\lambda_{\min}(\Sigu) I.
		\end{align*}
	\end{proof}
	Hence, for  $n \ge c_1(\ln(1/\delta) + \dimu)$,  $\psi(n,\delta) \le \max\{\frac{\veps^2_{\mate}}{c_{\mate}},  \frac{\veps^2_{\matdel}}{c_{\matdel}}\}$, and $2\veps^2_{\matdel} \le \frac{n}{8}\lambda_{\min}(\Sigu)$ (or equivalently, $\veps^2_{\matdel} \le \frac{n}{16}\lambda_{\min}(\Sigu)$), we find that with total failure probability at least $1 - 4\delta$,
	\begin{align*}
	\| \Mhat - \Mst\|_{\op} &\lesssim \frac{n^{1/2}(\|\Mst\|_{\op}\veps_{\matdel} + \veps_{\mate}) + \sqrt{\|\Sigw\|_{\op}(\dimy + \dimu + \ln(1/\delta))}}{\sigma_{\min}(\bU + \matDel)} \\
	&\lesssim \frac{(\|\Mst\|_{\op}\veps_{\matdel} + \veps_{\mate}) + \sqrt{n^{-1/2}\|\Sigw\|_{\op}(\dimy + \dimu + \ln(1/\delta))}}{\sqrt{\lambda_{\min}(\Sigu)}}.
	\end{align*}
	Hence, under these conditions, with probability $1 - 4\delta$,
	\begin{align*}
	\| \Mhat - \Mst\|_{\op}^2 &\lesssim \lambda_{\min}(\Sigu)^{-1} \left(\|\Mst\|_{\op}^2\veps^2_{\matdel} + \veps^2_{\mate}  + \frac{\|\Sigw\|_{\op}(\dimy + \dimu + \ln(1/\delta))}{n}\right).
	\end{align*}
	\qed

\subsubsection{Proof of \Cref{prop:techtools_covariance_est}}
Assume that the events of the proof of \Cref{prop:techtools_general_pca} above; this contributes a failure probability of $4\delta$.To begin, we have that
	\begin{align*}
	&\sum_{i=1}^n (\Mhat(\bu\supi+ \matdel\supi) - \maty\supi )^{\otimes 2} \\
	&= \sum_{i=1}^n ((\Mhat - \Mst)(\bu\supi+ \matdel\supi) - \bw\supi - \mate\supi)^{\otimes 2} \\
	&= \left((\bU + \matDel)(\Mhat - \Mst)^\top + \matDel \Mst^\top - \bW - \bE\right)^{\top}\left((\bU + \matDel)(\Mhat - \Mst)^\top + \matDel \Mst^\top - \bW - \bE\right).
	\end{align*}
	From \Cref{eq:techtools_LS_identity}, and the fact that $\bU + \matDel$ has full rank under the high probability events of \Cref{prop:techtools_general_pca},  we have 
	\begin{align*}
		\Mhat^\top  &= \Mst^\top + (\bU + \matDel)^{\dagger} (-\matDel  \Mst^\top+ \bW + \bE).
	\end{align*}
	This yields
	\begin{align*}
	(\bU + \matDel)(\Mhat^\top - \Mst^\top) = P_{\bU + \matDel} (-\matDel \Mst^{\trn} + \bW + \bE),
	\end{align*}
	where $P_{\bU + \matDel} := (\bU + \matDel)(\bU + \matDel)^{\dagger}\in\bbR^{n\times{}n}$ is the projection onto the row space of $\bU + \matDel$, which has dimension $\dimu$. 
	Thus, we find that
	\begin{align*}
	&\sum_{i=1}^n (\Mhat(\bu\supi+ \matdel\supi) - \maty\supi )^{\otimes 2} \\
	&= \left((I-P_{\bU + \matDel})(-\matDel \Mst^{\trn} + \bW + \bE)\right)^{\top}\left((I-P_{\bU + \matDel})(-\matDel \Mst^{\trn} + \bW + \bE)\right).
	\end{align*}
	Rearranging, and using that $I - P_{\bU + \matDel}$ is a projection operator, we have that
	\begin{align*}
	\nrm*{\frac{1}{n}\sum_{i=1}^n (\Mhat(\bu\supi+ \matdel\supi) - \maty\supi )^{\otimes 2} - \frac{1}{n}\bW^\top (I-P_{\bU + \matDel}) \bW}_{\op} &\le 2\|\bW\|_{\op}\| \bE -\matDel \Mst^{\trn}\|_{\op} + \| \bE -\matDel \Mst^{\trn}\|^2_{\op}.
	\end{align*}

          We can now bound this quantity using the following claim.
	\begin{claim}\label{claim:techtools_cov_est_claim_1} Suppose that $\lambda_+ \ge \lambda_{\max}(\Sigw)$, and $ \veps^2_{\matdel}\|\Mst\|_{\op}^2 + \veps^2_{\mate} \le 2\lambda_+$. Suppose the event $\cE_{\mathrm{ls},1}$ of \Cref{eq:techtools_event_Eone_ls} holds, and $n \ge c'\sqrt{\dimy + \log(1/\delta)}$ for $c'$ sufficiently large. Then, 
	\begin{align*}
	\nrm*{\frac{1}{n}\sum_{i=1}^n (\Mhat(\bu\supi+ \matdel\supi) - \maty\supi )^{\otimes 2} - \frac{1}{n}\bW^\top (I-P_{\bU + \matDel}) \bW}_{\op} \lesssim \sqrt{\frac{\lambda_{+}(\veps^2_{\matdel}\|\Mst\|_{\op}^2 + \veps^2_{\mate})}{n}}.
	\end{align*}
	\end{claim}
	\begin{proof}
	On the event $\cE_{\mathrm{ls},1}$ of \Cref{eq:techtools_event_Eone_ls}, recall that 
	\begin{align*}
	\|\bE -\matDel \Mst^{\trn}\|_{\op} \le \veps_{\matdel}\|\Mst\|_{\op} + \veps_{\mate}.
	\end{align*}
	In addition, for $n \ge c'(\dimy + \log(1/\delta))$ for some sufficiently large numerical constant $c'$, a suitable analogue of \Cref{eq:techtools_Verhsynin} implies that with an additional probability $1- \delta$,
	\begin{align*}
	\|\bW\|_{\op} \le \lambda_{\max}(\Sigw)^{1/2}\|\Sigw^{-1/2}\bW\|_{\op} \le 2\sqrt{n}\lambda_+^{1/2}.
	\end{align*}

	Hence, for $ \veps^2_{\matdel}\|\Mst\|_{\op}^2 + \veps^2_{\mate} \le \lambda_+$, we have that with total probability at least $1 - 5\delta$ (including events from the previous proposition),
	\begin{align*}
	\nrm*{\frac{1}{n}\sum_{i=1}^n (\Mhat(\bu\supi+ \matdel\supi) - \maty\supi )^{\otimes 2} - \frac{1}{n}\bW^\top (I-P_{\bU + \matDel}) \bW}_{\op}  \lesssim \sqrt{\lambda_{+}(\Sigw)(\veps^2_{\matdel}\|\Mst\|_{\op} + \veps^2_{\mate})}.
	\end{align*}
	\end{proof}
	To conclude the proof, we bound
	\begin{align}
	&\nrm*{\frac{1}{n}\sum_{i=1}^n (\Mhat(\bu\supi+ \matdel\supi) - \maty\supi )^{\otimes 2} - \Sigw}_{\op}\nonumber\\
	&\le \nrm*{\frac{1}{n}\sum_{i=1}^n (\Mhat(\bu\supi+ \matdel\supi) - \maty\supi )^{\otimes 2} - \frac{1}{n}\bW^\top (I-P_{\bU + \matDel}) \bW}_{\op}\nonumber \\
	&\quad+\nrm*{\Sigw- \frac{1}{n}\bW^\top (I-P_{\bU + \matDel}) \bW}_{\op}. \label{eq:techtools_final_err_decomp_covariance_est}
	\end{align}
	The following claim bounds the second term.
	\begin{claim}\label{claim:techtools_cov_est_claim_2} Suppose  $n \ge c'' (\dimu + \log(1/\delta))$, where $c'' > 0$ is a suitably large numerical constant.  For any upper bound $\lambda_+\ge \lambda_{\max}(\Sigw)$, with probability $1 - 2\delta$, 
	\begin{align*}
	\nrm*{ \Sigw- \frac{1}{n}\bW^\top (I-P_{\bU + \matDel}) \bW}_{\op} \le \lambda_+\sqrt{\frac{\dimy + \log(1/\delta)}{n}}. 
	\end{align*}
	\end{claim}
	\begin{proof} Define $P_{\bU + \matDel}^c  := I-P_{\bU + \matDel} $. Then we have
	\begin{align*}
	\nrm*{ \Sigw- \frac{1}{n}\bW^\top P_{\bU + \matDel}^c \bW}_{\op} &\le \frac{1}{n}\lambda_{\max}(\Sigw)\| n I- \underbrace{(\bW\Sigw^{-1/2})^\top P_{\bU + \matDel}^c (\bW\Sigw^{-1/2})}_{:=\bM}\|_{\op}.
	\end{align*}
	Now, observe that since $P_{\bU + \matDel}^c \in \R^{n\times{}n}$ is a projection matrix with rank $n - \dimu$, and $\bW\Sigw^{-1/2} \in \R^{n \times \dimy}$, the matrix $\bM = (\bW\Sigw^{-1/2})^\top P^c_{\bU + \matDel} (\bW\Sigw^{-1/2})$ is identical in distribution to $\bG^\top \bG$, where $\bG \in \R^{(n- \dimu)\times\dimy}$ has i.i.d. unit Gaussian entries. Theorem~5.39 of \cite{vershynin2010introduction} guarantees that with probability at least $1 - 2\delta$, 
	\begin{align*}
	\sqrt{n - \dimu} - \BigOh\prn*{\sqrt{\dimy + \log(1/\delta)}} \le \sigma_{\min}(\bG) \le \sigma_{\max}(\bG) \le \sqrt{n - \dimu} + \BigOh\prn*{\sqrt{\dimy + \log(1/\delta)}}.
	\end{align*}
	This implies that for $n \ge c'' (\dimu + \dimy + \log(1/\delta))$ for some universal constant $c''$, we have that
	\begin{align*}
	(n-\dimu) - \BigOh\prn*{\sqrt{\dimy + \log(1/\delta)}}\sqrt{n-\dimu} \le \lambda_{\min}(\bM) \le \lambda_{\max}(\bM) \le (n-\dimu) + \BigOh\prn*{\sqrt{\dimy + \log(1/\delta)}}\sqrt{n-\dimu}.
	\end{align*}
	Hence, on this event (and again for $n \ge c'' (\dimu + \log(1/\delta))$ for $c''$ suitably large), 
	\begin{align*}
	\|nI - \bM\|_{\op} \le \dimu + \BigOh\prn*{\sqrt{\dimy + \log(1/\delta)}}\sqrt{n-\dimu} \lesssim \sqrt{n(\dimy + \log(1/\delta))}. 
	\end{align*}
	as needed.
	\end{proof}
	
	In total, combining \Cref{claim:techtools_cov_est_claim_1,claim:techtools_cov_est_claim_2} and \Cref{eq:techtools_final_err_decomp_covariance_est}, we conclude that on the events of the previous proposition, and with an additional $3\delta$ failure probability, 
	\begin{align}
	&\nrm*{\frac{1}{n}\sum_{i=1}^n (\Mhat(\bu\supi+ \matdel\supi) - \maty\supi )^{\otimes 2} - \Sigw}_{\op}\nonumber\\
	&\lesssim \sqrt{\lambda_{\max}(\Sigw)(\veps^2_{\matdel}\|\Mst\|_{\op}^2 + \veps^2_{\mate}) + \frac{\lambda_{\max}(\Sigw)^2(d + \log(1/\delta))}{n}},
	\end{align}
	provided that $ \veps^2_{\matdel}\|\Mst\|_{\op}^2 + \veps^2_{\mate} \le \lambda_{\max}(\Sigw)$, and $n \ge c(\dimy + \log(1/\delta))$ for some universal constant $c$.

\subsubsection{ Proof of \Cref{prop:techtools_matrix_sensing_regression}}
We observe that since $\Qst$ lies in the convex PSD cone,  $\|\Qhat - \Qst\|_{\fro} \le \|(\frac{1}{2}\Qtil^\top + \frac{1}{2}\Qtil) - \Qst\|_{\fro}$ by the Pythagorean theorem. In more detail, we have the following result.
\begin{claim}
  Let $A,B\in\bbR^{d\times{}d}$, and let $B\psdgeq{}0$. Then $\nrm*{A_{+}-B}_F\leq{}\nrm*{A-B}_{F}$.
\end{claim}
\begin{proof}
  Let $A=A_{+} + A_{-}$, so that $A_{+}\psdgeq{}0$ and $A_{-}\psdleq{}0$. Then we have
  \[
    \nrm*{A-B}_{F}^{2}-\nrm*{A_{+}-B}_{F}^{2} =
        \nrm*{A_{+} + A_{-}-B}_{F}^{2}-\nrm*{A_{+}-B}_{F}^{2} = \nrm*{A_-}_F^2 + \tri*{A_-, A_+-B}.
      \]
      Now, note that $\tri*{A_-, A_+-B}=\tri*{-A_-,B}\geq{}0$, since $\tri*{X,Y}\geq{}0$ whenever $X,Y\psdgeq{}0$.
\end{proof}

Moreover, $\Qst = (\Qst)^\top$, so $ \|(\frac{1}{2}\Qtil^\top + \frac{1}{2}\Qtil) - \Qst\|_{\fro} \le \|\Qtil - \Qst\|_{\fro}$ by the triangle inequality. Thus, we conclude
	\begin{align*}
	\|\Qhat - \Qst\|_{\fro} \le  \|\Qtil - \Qst\|_{\fro}.
	\end{align*}
	Next, let us introduce $\bv_{i} := \mathrm{vec}(\gst(\by^{(i)})\gst(\by^{(i)})^\top)\in\bbR^{d^{2}}$ and $\bvhat_{i} := \mathrm{vec}(\ghat(\by^{(i)})\ghat(\by^{(i)})^\top)$. Let $\bV\in\bbR^{n\times{}d^{2}}$ denote the matrix whose rows are $\bv_{i}$ and $\bVhat$ analogouly for $\bvhat$. Then, we have that
	\begin{align*}
	\mathrm{vec}(\Qhat) &= \bVhat^{\dagger}\bV\vec(\Qst)\\
	&= \bVhat^{\dagger}\bVhat\vec(\Qst) + \bVhat^{\dagger}(\bV - \bVhat)\mathrm{vec}(\Qst)\\
	&= \mathrm{vec}(\Qst) + \bVhat^{\dagger}(\bV - \bVhat)\mathrm{vec}(\Qst),
	\end{align*}
	provided that $\bVhat$ is full rank (which we ultimately verify), where we recall that $\bVhat^{\dagger}=(\bVhat^{\trn}\bVhat)^{-1}\bVhat^{\trn}$ in this case. Next, we bound
	\begin{align*}
	\|\mathrm{vec}(\Qhat) - \mathrm{vec}(\Qst)\|^2 &\le \frac{1}{\sigma_{\min}(\bVhat)^2} \| (\bV - \bVhat)\mathrm{vec}(\Qst)\|_{2}^2\\
	&= \frac{1}{\lambda_{\min}(\bVhat^\top \bVhat)} \sum_{i=1}^n \left\langle \mathrm{vec}(\gst(\by^{(i)})\gst(\by^{(i)})^\top) - \mathrm{vec}(\ghat(\by^{(i)})\ghat(\by^{(i)})^\top, \mathrm{vec}(\Qst)\right\rangle^2\\
	&= \frac{1}{\lambda_{\min}(\bVhat^\top \bVhat)} \sum_{i=1}^n \left\langle \gst(\by^{(i)})\gst(\by^{(i)})^\top  -\ghat(\by^{(i)})\ghat(\by^{(i)}) ^{\trn}, \Qst\right\rangle^2\\
	&= \frac{1}{\lambda_{\min}(\bVhat^\top \bVhat)} \sum_{i=1}^n \left\langle \gst(\by^{(i)})\gst(\by^{(i)})^\top  -\ghat(\by^{(i)})\ghat(\by^{(i)}) ^{\trn}, \Qst\right\rangle^2\\
	&\overset{(a)}{\le} \frac{\|\Qst\|_{\op}^2}{\lambda_{\min}(\bVhat^\top \bVhat)} \sum_{i=1}^n \|\gst(\by^{(i)})\gst(\by^{(i)})^\top  -\ghat(\by^{(i)})\ghat(\by^{(i)}) ^{\trn}\|_{\mathrm{nuc}}^2\\
	&\overset{(b)}{\le} \frac{2\|\Qst\|_{\op}^2}{\lambda_{\min}(\bVhat^\top \bVhat)} \sum_{i=1}^n \|\gst(\by^{(i)})\gst(\by^{(i)})^\top  -\ghat(\by^{(i)})\ghat(\by^{(i)}) ^{\trn}\|_{\fro}^2,
		\end{align*}
		where $(a)$ uses \Holder's inequality ($|\langle A, B \rangle| \le \|A\|_{\op}\|B\|_{\mathrm{nuc}}$), and $(b)$ uses that $\gst(\by^{(i)})\gst(\by^{(i)})^\top  -\ghat(\by^{(i)})\ghat(\by^{(i)})^{\trn}$ has rank $2$, so its nuclear norm is at most $\sqrt{2}$ times its Frobenius norm. Recognizing $\sum_{i=1}^n \|\gst(\by^{(i)})\gst(\by^{(i)})^\top  -\ghat(\by^{(i)})\ghat(\by^{(i)})^{\trn}\|_{\fro}^2 = \|\bVhat - \bV\|_{\fro}^2$, we obtain 
	\begin{align}
	\|\mathrm{vec}(\Qtil) - \mathrm{vec}(\Qst)\|^2 &\le
	\frac{4\|\bV - \bVhat\|_{\fro}^2\|\Qst\|_{\op}^2}{\lambda_{\min}(\bVhat^\top \bVhat)}. \label{eq:techtools_mtil_error_cost}
	\end{align}
	 Next, we give the following bound.
	\begin{claim}\label{claim:techtools_bv_bvhat_dif} Suppose that $\psi(n,\delta) \le \frac{\veps^2}{4c}$. Then, with probability $1 - \delta$, we have the bound  $\|\bV - \bVhat\|_{\op}^2 \le \|\bV - \bVhat\|_{\fro}^2 \le 8c_{n,\delta} \veps^2$, where $c_{n,\delta} = c\ln(2n/\delta)$.
	\end{claim}
	\begin{proof}
          To begin, observe that
		\begin{align*}
		\|\bV - \bVhat\|_{\op}^2 &\le \|\bV - \bVhat\|_{\fro}^2\\
		&= \sum_{i=1}^n \|\mathrm{vec}(\gst(\by^{(i)})\gst(\by^{(i)})^\top) -  \mathrm{vec}(\ghat(\by^{(i)})\ghat(\by^{(i)})^\top)\|_{2}^2\\
		&\le \sum_{i=1}^n \|\gst(\by^{(i)})\gst(\by^{(i)})^\top -  \ghat(\by^{(i)})\ghat(\by^{(i)})^\top\|_{\fro}^2\\
		&\le \sum_{i=1}^n \left((\|\gst(\by^{(i)}) \| + \|\gst(\by^{(i)})\|)^2\| \gst(\by^{(i)}) - \ghat(\by^{(i)})\|^2\right).
		\end{align*}
		Introduce the event $\cE := \max\{\|\gst(\by) \|^2, \|\ghat(\by) \|^2 \} \le c_{n,\delta}  := c\ln(2n/\delta)$, and let $\cE^{(i)}$ denote the analogous event for $\maty\ind{i}$. Let $\cE\ind{1:n}=\bigcap_{i=1}^{n}\cE\ind{i}$. Then $\cE^{(1:n)}$ holds with probability at least $1 - \delta/2$, and on this event the above display is at most
		\begin{align*}
		\|\bV - \bVhat\|_{\op}^2 \le 4c_{n,\delta} \sum_{i=1}^n \I(\cE^{(i)})\| \gst(\by^{(i)}) - \ghat(\by^{(i)})\|^2.
		\end{align*}
		Next, define the random variable $\matdel_i := \I(\cE^{(i)})\| \gst(\by^{(i)}) - \ghat(\by^{(i)})\|^2$, and we observe that $\matdel_i \le 4c\ln(2n/\delta)$ with probability $1$. Thus, by applying \Cref{lem:techtools_truncated_conc} with $c \leftarrow 4c$, we have that for any $\veps^2 \ge \E[\I(\cE^{(i)})\| \gst(\by^{(i)}) - \ghat(\by^{(i)})\|^2]$, with probability at least $1-\delta/2$,
		\begin{align*}
		\sum_{i=1}^n \I(\cE^{(i)})\| \gst(\by^{(i)}) - \ghat(\by^{(i)})\|^2 \le 2 \veps^2,
		\end{align*}
		as soon as $\psi(n,\delta) \le \frac{\veps^2}{4c}$. In paricular, since $\| \gst(\by^{(i)}) - \ghat(\by^{(i)})\|^2 \ge 0$, it is valid to select $\veps^2 \ge \E[\| \gst(\by^{(i)}) - \ghat(\by^{(i)})\|^2]$. Hence, for such $\veps^2$, we conclude that with total probability at least $1 - \delta$,
		\begin{align*}
		\|\bV - \bVhat\|_{\op}^2  \le 8c_{n,\delta} \veps^2.
		\end{align*}
	\end{proof}

	Denote the event of \Cref{claim:techtools_bv_bvhat_dif} by $\cE_1$. Then on $\cE_1$,  \Cref{eq:techtools_mtil_error_cost} implies
		\begin{align}
		\|\Qtil - \Qst\|_{\fro}^2  \le \|\Qst\|_{\fro}^2 \frac{8 c_{n,\delta}\veps^2}{\lambda_{\min}(\bVhat^\top \bVhat)}. \label{eq:techtools_matrix_sensing_upper_bound_in_terms_bVhat_lam_min}
		\end{align}
		Next, from \Cref{claim:techtools_lowner_lb}, we have that 
		\begin{align*}
		\|\bVhat-\bV\|_{\op}^2 \le \frac{1}{4}\lambda_{\min}(\bV^\top \bV)\quad \text{ implies }\quad \lambda_{\min}(\bVhat^\top \bVhat) \ge \frac{1}{4}\lambda_{\min}(\bV^\top \bV) .
		\end{align*}
		And thus, on $\cE_1$, we have that
		\begin{align}
		 \veps^2 \le \frac{\lambda_0}{32c_{n,\delta}} \le \frac{1}{4}\lambda_{\min}(\bV^\top \bV)\quad \text{ implies }\quad \lambda_{\min}(\bVhat^\top \bVhat) \ge \frac{1}{4}\lambda_{\min}(\bV^\top \bV). \label{eq:techtools_vhat_covariance_lb_sensing}
		\end{align}
		Let us now lower bound $\lambda_{\min}(\bV^\top \bV)$ with high probability. We observe that $\lambda_{\min}(\bV^\top \bV) \ge \lambda_0$ for some $\lambda_0>0$ if and only if
		\begin{align}
                  &\forall M \in \reals^{d \times d}:\quad \sum_{i=1}^n \langle \gst(\by^{(i)})  \gst(\by^{(i)})^\top, M \rangle^2 \ge \lambda_0\|M\|_{\fro}^2,\nonumber
                    \intertext{if and only if}
		&\forall M \in \reals^{d \times d}:\quad \sum_{i=1}^n \langle \bx^{(i)}  \bx^{(i)\top}, M \rangle^2 \ge \lambda_0\|M\|_{\fro}^2,
		 \label{eq:techtools_lambda_0_matrix_sensing}
		\end{align}
                where $\bx^{(i)}\ldef\gst(\by\ind{i}) \iidsim \cN(0, \Sigma_x)$ by assumption.
		\begin{claim}\label{claim:techtools_matrix_sensing_conv_lb} Let $\bx^{(i)} \iidsim \cN(0, \Sigma_x)$. Then, for $n \ge c_1 d$, the following holds with probability $1 - e^{-c_2 n}$:
		\begin{align*}
		\forall M \in \R^{d \times d}: \quad \sum_{i=1}^n \langle \bx^{(i)} \bx^{(i)\top}, M \rangle^2  \ge  \frac{1}{2}\lambda_{\min}(\Sigma_x)^2 \|M\|_{\fro}^2,
		\end{align*}
                where $c_1$ is a numerical constant
		\end{claim}
	\begin{proof}

		Define $\tilde{\bx}^{(i)} :=\Sigma_x^{-1/2}\bv^{(i)} $. Note that $\tilde{\bx}^{(i)} \sim \cN(0,I)$. From \citet[Theorem 10.12]{wainwright2019high}, we find that for any $\alpha>0$, for all $n \ge c_1 \alpha d$, with probability $1 - e^{-c_2 n}$, the following holds simultaneously for all matrices $M$ satisfies $\|M\|_{\mathrm{nu}}^2 \le \alpha \|M\|_{\fro}^2$
		\begin{align}
		\sum_{i=1}^n \langle \tilde{\bx}^{(i)} \tilde{\bx}^{(i)\top}, M \rangle^2  \ge \frac{1}{2}\|M\|_{\fro}^2, \label{eq:techtools_lb_matrix_sensing}
		\end{align}
		where $\|M\|_{\mathrm{nuc}} = \sum_{i=1}^n \sigma_i(M)$ denotes the matrix nuclear norm.  By Cauchy-Schwartz, $\|M\|_{\mathrm{nu}}^2 \le d \|M\|_{\fro}^2$ for all matrices $M \in \R^{d \times d}$. This means that we capture \emph{all} matrices $M$ by setting $\alpha = d$, and thus, for $n \ge c_1 d^2$, then with  probability $1 - e^{-c_2 n}$, \Cref{eq:techtools_lb_matrix_sensing} holds for all $M \in \R^{d \times d}$ simultaneously. When this holds, we have that for all such $M$,
		\begin{align*}
		\sum_{i=1}^n \langle \bx^{(i)} \bx^{(i)\top}, M \rangle^2   = \sum_{i=1}^n \langle \tilde{\bx}^{(i)} \tilde{\bx}^{(i)\top}, \Sigma_x^{1/2} M\Sigma_x^{1/2} \rangle^2 \ge\frac{1}{2} \|\Sigma_x^{1/2} M \Sigma_x^{1/2}\|_{\fro}^2.
		\end{align*}
		Moreover, we have that 
		\begin{align*}
		|\Sigma_x^{1/2} M \Sigma_x^{1/2}\|_{\fro}^2 &= \trace(\Sigma_x^{1/2} M \Sigma_x M^\top \Sigma_x^{1/2}) \\
		&\geq \lambda_{\min}(\Sigma_x) \trace(\Sigma_x^{1/2} M M^\top \Sigma_x^{1/2})\\
		&= \lambda_{\min}(\Sigma_x) \trace(M^\top\Sigma_x M ) \geq  \lambda_{\min}(\Sigma_x)^2 \trace(M^\top M) = \lambda_{\min}(\Sigma_x)^2 \|M\|_{\fro}^2.
		\end{align*}
	\end{proof}
	Denote the event of \Cref{claim:techtools_matrix_sensing_conv_lb} by $\cE_2$. Then, on $\cE_2$, we can take $\lambda_0 = \frac{1}{2}\lambda_{\min}(\Sigma_x)^2$ in \Cref{eq:techtools_lambda_0_matrix_sensing}, and thus on $\cE_1$,  \Cref{eq:techtools_vhat_covariance_lb_sensing} yields that
	\begin{align*}
		 \veps^2 \le \frac{\lambda_{\min}(\Sigma_x)^2}{64c_{n,\delta}} \le \frac{1}{4}\lambda_0 \quad\text{ implies }\quad \lambda_{\min}(\bVhat^\top \bVhat) \ge \frac{\lambda_{\min}(\Sigma_x)^2}{8}. 
		\end{align*}
		Thus, by \Cref{eq:techtools_matrix_sensing_upper_bound_in_terms_bVhat_lam_min}, we have that  on $\cE_1 \cap \cE_2$, 
		\begin{align*}
		\|\Qtil - \Qst\|_{\fro}^2 = \|\mathrm{vec}(\Qtil) - \qst\|^2 \le 4\cdot 16 c_{n,\delta} \veps^2 \cdot \frac{\|\Qst\|_{\op}^2}{\lambda_{\min}(\Sigma_x)^2},
		\end{align*}
		giving us the desired inequality.
		Since $\Pr(\cE_1 \cap \cE_2)  \ge 1 - \delta - e^{-c_2 n}$ for $n \ge c_1 d^2$, we have that if $n \ge c (d^2 + \ln(1/\delta))$ for some universal constant $c$, $\Pr(\cE_1 \cap \cE_2)  \ge 1 - 2\delta$, yielding our desired failure probability. Recalling that $c_{n,\delta} := c\ln(2/\delta)$ concludes.
	\qed

\section{Linear Control Theory}
\label{sec:linear_control}
In this section we recall some basic results for the classical
LQR problem in the fully observed setting with known dynamics. The
main result for this section is \pref{thm:performance_difference},
which bounds the regret of any policy for the \richlqr in terms of
decoding errors. Proofs are deferred to the end of the section.

\subsection{Basic Technical Results}

  \begin{lemma}
    \label{lem:strong_stability}
    Let $X$ be any matrix with $\rho(X)<1$. Then for any $Y\psdgt{}0$,
    there exists a unique solution $P\psdgt{}0$ to the Lyapunov equation
    \begin{equation}
      \label{eq:dlyap}
      P = X^{\trn}PX+Y.
    \end{equation}
    Moreover, $X$ is $(\alpha,\gamma)$-strongly stable for 
    $\alpha=\nrm{P^{1/2}}_{\op}\nrm{P^{-1/2}}_{\op}$ and $\gamma=\nrm*{I-P^{-1/2}YP^{-1/2}}_{\op}^{1/2}$.
  \end{lemma}
This lemma immediately implies the following strong stability
guarantees for the closed-loop and open-loop dynamics for LQR.
    \begin{proposition}
    \label{prop:closed_loop_ss}
    $\Aclinf\ldef{}A+B\Kinf$ is $(\alphainf,\gammainf)$-strongly stable, where $\alphainf\ldef{}\nrm{\Pinf^{1/2}}_{\op}\nrm{\Pinf^{-1/2}}_{\op}$ and $\gammainf\ldef{}\nrm{I-\Pinf^{-1/2}\Rx\Pinf^{-1/2}}_{\op}^{1/2}<1$.\footnote{\pref{prop:closed_loop_ss} and \pref{prop:open_loop_ss} are immediate consequences of \pref{lem:strong_stability}, proven in \pref{sec:linear_control}.}
  \end{proposition}
\begin{proposition}
  \label{prop:open_loop_ss}
  If we define
  $\alphaa\ldef{}\nrm{\Sigmaa^{1/2}}_{\op}\nrm{\Sigmaa^{-1/2}}_{\op}$ and $\gammaa\ldef{}  \| I_{\dimx} -
                               \Sigma^{-1}_{A}\|^{1/2}<1$, then $A$ is
                               $(\alphaa,\gammaa)$-strongly stable,
                               where $\Sigmaa$ is the unique solution
                               to the Lyapunov equation
                               \begin{align}
                                 \Sigma = A \Sigma A^\top + I_{\dimx}. \label{eq:Lyapunov}
                               \end{align}
                             \end{proposition}
We also make use of the following bound on the operator norm for the
infinite-horizon covariance matrix.
  \begin{proposition}
    \label{prop:inf_bounds}
    We have
      $\nrm*{\Siginf}_{\op}\leq{}
      \nrm*{R}_{\op}+\nrm*{B}_{\op}^{2}\nrm*{\Pinf}_{\op} \leq{}
      2\Psistar^{3}$.
  \end{proposition}
  \subsection{Value Functions}
  Toward proving our main regret decomposition, in this section we
  establish some basic technical results regarding the value functions
  and Q-functions for the fully observed LQR problem. Our first result
  concerns finite-horizon value functions for linear controller.
    \begin{lemma}
    \label{lem:lqr_q_functions}
    Consider the \richlqr setting \eqref{eq:dynamics} under \pref{ass:perfect}, and consider a
    state feedback controller $\pi_K(y) =
    K\fstar(y)$, where $\fstar$ is the true decoder. Define
    \begin{equation}
    \begin{aligned}
      &\Vf_{t:T}^{K}(x) =
      \En_{\pi_K}\brk*{\sum_{s=t}^{T}\matx_{s}^{\trn}\Rx\matx_s+\matu_{s}^{\trn}\Ru\matu_s\mid{}\matx_t=x},\\
      &\Qf_{t:T}^{K}(x,u) = \En_{\pi_K}\brk*{\sum_{s=t}^{T}\matx_{s}^{\trn}\Rx\matx_s+\matu_{s}^{\trn}\Ru\matu_s\mid{}\matx_t=x,\matu_t=u}.
    \end{aligned}\label{eq:lqr_q}
  \end{equation}
  Then we have
      \begin{equation}
    \begin{aligned}
      &\Vf_{t:T}^{K}(x) =
      \sum_{s=t}^{T}\nrm*{(A+BK)^{s-t}x}_{Q+K^{\trn}RK}^{2} +
      F_{t:T}(A,B,Q,R,K,\Sigw),\\
      &\Qf_{t:T}^{K}(x,u) = \nrm*{x}_{Q}^{2}+\nrm*{u}_{R}^{2} +
      \Vf^{K}_{t+1:T}(Ax+Bu)
      +       G_{t:T}(A,B,Q,R,K,\Sigw),
    \end{aligned}\label{eq:lqr_q2}
    \end{equation}
    where $F_{t:T}$ and $G_{t:T}$ are functions that depend on the system
    parameters and time horizon, but not the state or control inputs.

  \end{lemma}
In light of \pref{lem:lqr_q_functions}, it will be convenient to
define
\begin{equation}
  \label{eq:vbar}
  \Vbar_{t:T}^{K}(x) = \sum_{s=t}^{T}\nrm*{(A+BK)^{s-t}x}_{Q+K^{\trn}RK}^{2},
\end{equation}
which is simply the value function in
\pref{eq:lqr_q2} in the absence of noise. Our next result concerns the
infinite-horizon value functions that arise in the noiseless setting.
\begin{lemma}[\cite{bertsekas2005dynamic}]
  \label{lem:pinf_properties}
  Consider the optimal infinite horizon controller
  $\piinf(x)=\Kinf{}x$, and define
  $\Vinf(x)=\nrm*{x}_{\Pinf}^{2}$. Then $\Vinf$ is the
  infinite-horizon cost for playing $\piinf$ starting from $\matx_1=x$
  under the noiseless dynamics
  \[
\matx_{t+1}=A\matx_{t}+B\matu_t.
\]
Moreover, if we define
$\Qinf(x,u)=\nrm*{x}_{Q}^{2}+\nrm*{u}_{R}^{2}+\nrm*{Ax+Bu}_{\Pinf}^{2}$,
we have
\[
  \piinf(x) = \argmin_{u\in\bbR^{\dimu}}\Qinf(x,u).
\]
Finally, we have
\[
  \Pinf = \sum_{k=0}^{\infty}((A+B\Kinf)^{\trn})^{k}(Q+\Kinf^{\trn}R\Kinf)(A+B\Kinf)^{k}
\]
\end{lemma}
The following lemma shows that the infinite-horizon value functions
are well-approximated by their finite-horizon counterparts.
\begin{lemma}
  \label{lem:vinf_approx}
  For all $x\in\bbR^{\dimx}$ and all $t\leq{}T$, we have
  \[
    \abs*{\Vbar^{\Kinf}_{t:T}(x)
      - \Vinf(x)} \leq{} \bigoh(\alphainf^{2}(1-\gammainf^{2})^{-1}\Psistar^{3})\cdot{}\gammainf^{2(T-t+1)}\nrm*{x}^{2}_2
  \]
  and
    \[
    \abs*{\Qbar^{\Kinf}_{t:T}(x,u)
      - \Qinf(x,u)} \leq{} \bigoh(\alphainf^{2}(1-\gammainf^{2})^{-1}\Psistar^{5})\cdot{}\gammainf^{2(T-t)}(\nrm*{x}^{2}_2+\nrm*{u}_{2}^{2}).
  \]
  
\end{lemma}
Lastly, we establish a Lipschitz property for the finite-horizon $Q$-functions.
\begin{lemma}
  \label{lem:qinf_lipschitz}
For all $x\in\bbR^{\dimx}$ and $u,u'\in\bbR^{\dimu}$,
  \[
    \abs*{\Qbar^{\Kinf}_{t:T}(x,u)- \Qbar^{\Kinf}_{t:T}(x,u')} \leq{}    \bigoh(\Psistar^{3})\cdot{}(\nrm*{x}_2\vee\nrm*{u}_2\vee\nrm*{u'}_2)\nrm*{u-u'}_2.
  \]
\end{lemma}

\subsection{Perturbation Bound for the Optimal Controller}
To analyze the quality of the certainty-equivalent controller used in
\richidce, we use the following perturbation bound.
\begin{theorem}[\citet{mania2019certainty}]
  \label{thm:perturbation_bound}
      Suppose we have matrices $(\Ahat,\Bhat,\Qhat)$ for which there
    exists an invertible transformation $G$ such that
    \[
      \nrm*{\Ahat-GAG^{-1}}_{\op}\vee\nrm*{\Bhat-GB}_{\op}\vee\nrm*{\Qhat-G^{-\trn}QG^{-1}}_{\op}\leq{}\veps.
    \]
    Suppose that $\nrm{G}_{\op}\vee\nrm{G^{-1}}_{\op}\leq{}\Csim$.
    Let $\Khat$ be the optimal infinite-horizon controller for
    $(\Ahat,\Bhat,\Qhat,R)$.   Then once   $\veps\leq{}c_{\mathrm{stable}}\cdot{}\gammainf\cdot{}\Csim^{-15}\alphainf^{-4}(1-\gammainf^{2})^{-2}\Psistar^{-11}$,
    where $c_{\mathrm{stable}}$ is a sufficiently small numerical
  constant,
    \begin{equation}
    \label{eq:khat_perturbation}
    \nrm*{\Khat-\Kinf{}G^{-1}}_{\op}\leq{} \bigoh(\Csim^{11}\alphainf^{2}(1-\gammainf^{2})^{-1}\Psistar^{9})\cdot{}\veps,
  \end{equation}
  and we are guaranteed that $A+B\Khat$ is
  $(\alphainf,\gammainfb)$-strongly stable, where $\gammainfb=(1+\gammainf)/2$.
  \end{theorem}

  \subsection{Regret Decomposition}
The following theorem is the main result from this section, and shows
that any policy of the form $\pihat_t(\maty_{1:t}) =
\Khat\fhat_t(\maty_{1:t})$ (in particular, the policy returned by
Phase III of \richidce), has low regret whenever $\Khat$ accurately
approximates $\Kinf$ and $\fhat_t$ has low prediction error on the
state distribution induced by $\pihat_{1:t-1}$.
\begin{theorem}
  \label{thm:performance_difference}
  Consider a randomized policy of the form
  $\pihat_t(\maty_{1:t}) = \Khat\fhat_t(\maty_{1:t})+\bnu_t$, where $\En\brk*{\bnu_t\mid{}\maty_{1:t}}=0$.
  Suppose we are guaranteed that
  \[
\nrm[\big]{\Khat-\Kinf}_{\op}\leq{}\vepsk\leq\nrm*{\Kinf},\quad\text{and}\quad\En_{\pihat}\nrm*{\fhat_t(\maty_{1:t})-\fstar(\maty_t)}^{2}_2\leq{}\vepsf^2\quad\text{for
      all $t$.}
  \]
  Suppose that $\nrm[\big]{\fhat_t}_2\leq{}\Bf$ almost surely, that
  $\En\nrm*{\bnu_t}_2^{2}\leq\sigmanu^{2}$, and that
 $\En_{\pihat}\nrm*{\matx_t}_2^2\leq\cx^2$, where $\bclip,\cx\geq{}1$. Then for any $0\leq{}\tau\leq{}T$, we have
  \begin{align}
    &\cost(\pihat) - \cost(\piinf) \\&\leq{}
    C_1\cdot{}(\vepsf^2+\vepsk^2+\signu^2)(T-\tau)/T
    + C_2(\vepsf+\cx\cdot{}\vepsk+\signu)\tau/T
    + C_3 \exp(-2\log(1/\gammainf)\tau)(T-\tau)/T,\notag
  \end{align}
  where $C_1\leq{}\bigoh\prn*{\Psistar^{5}\cx^{2}}$, $C_2\leq{}\bigoh(\bclip\Psistar^{5}\cx(1\vee\signu))$, and  $C_3\leq{}\bigoh\prn*{\alphainf^{2}\Psistar^{7}(\cx^{2}\vee\sigmanu^2\vee\bclip^2)}$.

\end{theorem}
\subsection{Proofs for Linear Control Theory Results}
  \begin{proof}[\pfref{lem:strong_stability}]
    Existence of a unique solution to the Lyapunov equation is a
    standard result \citep{bertsekas2005dynamic}. Now, define $L=P^{1/2}XP^{-1/2}$. Then the Lyapunov equation
    \eqref{eq:dlyap} is equivalent to
    \[
      L^{\trn}L + P^{-1/2}YP^{-1/2} = I.
    \]
    This implies that
    \[
      \nrm*{L}_{\op}^{2} = \nrm[\big]{L^{\trn}L}_{\op} \leq{} \underbrace{\nrm[\big]{I-P^{-1/2}YP^{-1/2}}_{\op}}_{=\gamma^{2}}<1.
    \]
    Moreover, since $L=P^{1/2}YP^{-1/2}$, we may take $\alpha=\nrm{P^{1/2}}_{\op}\nrm{P^{-1/2}}_{\op}$.
  \end{proof}

  \begin{proof}[\pfref{lem:lqr_q_functions}]
  Since we have perfect decodability, $\pi_K$ operates directly on the true
  state, and so we may overload $\pi_K(x)=Kx$. To begin, we observe
  that if we begin at $\matx_t=x$ and follow $\pi_K$, we have
  \[
\matx_{s}=(A+BK)^{s-t}x + \sum_{i=t}^{s-1}(A+BK)^{s-i-1}\matw_i.
\]
It follows that
\[
  \Vf_{t:T}^{K}(x) =\En\brk*{
    \sum_{s=t}^{T}\nrm*{(A+BK)^{s-t}x + \sum_{i=t}^{s-1}(A+BK)^{s-i-1}\matw_i}_{Q+K^{\trn}RK}^{2}
  }.
\]
However, since $\matw_t$ are zero-mean and independent, we can expand
the norm and cancel the cross terms, which allows us to write this as
\begin{align}
  \Vf_{t:T}^{K}(x) =\sum_{s=t}^{T}\nrm*{(A+BK)^{s-t}x}_{Q+K^{\trn}RK}^{2} + F_{t:T}(A,B,Q,R,K,\Sigw).
\end{align}
The expression for $\Qf^K_{t:T}$ immediately follows, since we have
\[
\Qf^K_{t:T}(x,u) = \En_{\pi_K}\brk*{\nrm*{x}_{Q}^{2}+\nrm*{u}_{R}^{2} + \Vf_{t+1:T}(Ax+Bu+\matw_t)}.
\]
The fact that $\Vf^K_{t+1:T}$ is quadratic and $\matw_t$ is
zero-mean again allows us to factor out the noise.
  
\end{proof}

\begin{proof}[\pfref{lem:vinf_approx}]
  Since $\Vbar_{t:T}^{K}=\Vbar_{1:T-t+1}^{K}$, we focus on the case
  $t=1$ without loss of generality. Observe that we have
  \[
    \Vbar^{\Kinf}_{1:T}(x) =
    \sum_{k=0}^{T-1}\nrm*{(A+B\Kinf)^{k}x}_{Q+\Kinf^{\trn}R\Kinf}^{2}
    = \tri*{x, \sum_{k=0}^{T-1}((A+B\Kinf)^{\trn})^{k}(Q+\Kinf^{\trn}R\Kinf)(A+B\Kinf)^{k}x}.
  \]
  Using the expression for $\Pinf$ from \pref{lem:pinf_properties}, it
  follows that
  \begin{align*}
    \abs*{\Vbar^{\Kinf}_{1:T}(x)
    - \Vinf(x)}
    &\leq{}
      \nrm*{x}_{2}^{2}\cdot{}\nrm*{\sum_{k=T}^{\infty}((A+B\Kinf)^{\trn})^{k}(Q+\Kinf^{\trn}R\Kinf)(A+B\Kinf)^{k}}_{\op}\\
    &\leq{} \nrm*{x}_{2}^{2}\cdot{}2\Psistar^{3}\sum_{k=T}^{\infty}\nrm*{(A+B\Kinf)^{k}}^{2}_{\op}.
  \end{align*}
  Now, using \pref{prop:closed_loop_ss}, we are guaranteed that
  $\nrm*{(A+B\Kinf)^{k}}_{\op}\leq{}\alphainf\gammainf^{k}$, so we
  have
  \[
    \sum_{k=T}^{\infty}\nrm*{(A+B\Kinf)^{k}}^{2}_{\op}
    \leq{} \alphainf^{2}\sum_{k=T}^{\infty}\gammainf^{2k}
    =\alphainf^{2}\gammainf^{2T}(1-\gammainf^{2})^{-1}.
  \]
This is establishes the bound on the error to $\Vinf$. The error bound
for the $Q$-functions follows immediately, since
\[
    \abs*{\Qbar^{\Kinf}_{t:T}(x,u)
      - \Qinf(x,u)}
    =     \abs*{\Vbar^{\Kinf}_{t+1:T}(Ax+Bu)
    - \Vinf(Ax+Bu)}.
  \]
\end{proof}
\begin{proof}[\pfref{lem:qinf_lipschitz}]
  We first compute that for any $x,x'$,
  \begin{align*}
    \abs*{\Vbar^{\Kinf}_{t+1:T}(x)-\Vbar^{\Kinf}_{t+1:T}(x')}
    &\leq{}2(\nrm*{x}_2\vee\nrm*{x'}_2)\nrm*{x-x'}_{2}\nrm*{\sum_{s=t}^{T}((A+B\Kinf)^{\trn})^{s-t}
      (Q+\Kinf^{\trn}R\Kinf)(A+B\Kinf)^{s-t}}_{\op}\\
        &\leq{}2(\nrm*{x}_2\vee\nrm*{x'}_2)\nrm*{x-x'}_{2}\nrm*{\sum_{s=0}^{\infty}((A+B\Kinf)^{\trn})^{s-t}
          (Q+\Kinf^{\trn}R\Kinf)(A+B\Kinf)^{s-t}}_{\op}\\
    &=2(\nrm*{x}_2\vee\nrm*{x'}_2)\nrm*{x-x'}_{2}\nrm*{\Pinf}_{\op}.
  \end{align*}
  As a consequence, for all $x$ and $u,u'$, we have
  \begin{align*}
    \abs*{\Qbar^{\Kinf}_{t:T}(x,u)- \Qbar^{\Kinf}_{t:T}(x,u')} &\leq{}
    \abs*{\nrm*{u}_{R}^{2}-\nrm*{u'}_{R}^{2}} +
                                                                 \abs*{\Vbar^{\Kinf}_{t+1:T}(Ax+Bu)-\Vbar^{\Kinf}_{t+1:T}(Ax+Bu')}\\
                                                               &\leq{}
                                                                 2\Psistar(\nrm*{u}_2\vee\nrm*{u'}_2)\nrm*{u-u'}_{2}
                                                                 +
                                                                 \abs*{\Vbar^{\Kinf}_{t+1:T}(Ax+Bu)-\Vbar^{\Kinf}_{t+1:T}(Ax+Bu')}\\
                                                                   &\leq{}
                                                                 2\Psistar(\nrm*{u}_2\vee\nrm*{u'}_2)\nrm*{u-u'}_{2}
                                                                     +
                                                                     2\nrm*{\Pinf}(\nrm*{Ax+Bu}_2\vee\nrm*{Ax+Bu'}_2)\nrm*{B(u-u')}_2\\
                                                                       &\leq{}
\bigoh(\Psistar^{3})\cdot{}(\nrm*{x}_2\vee\nrm*{u}_2\vee\nrm*{u'}_2)\nrm*{u-u'}_2.
  \end{align*}
\end{proof}

\begin{proof}[\pfref{thm:perturbation_bound}]
  We first consider the case where $G$ is the identity matrix. We apply Proposition 2 of \cite{mania2019certainty}, which
    implies that\footnote{To apply the proposition as stated in their
      paper, we use that $\rho(\Aclinf)\leq{}\gammainf$ by Gelfand's
      formula, and that their parameter $\tau(\Aclinf,\gammainf)$ is
      bounded by $\alphainf$.}
    \[
      \nrm*{\Phat-\Pinf}_{\op}\leq{} \bigoh(\alphainf^{2}(1-\gammainf^{2})^{-1}\Psistar^{6})\cdot{}\veps,
    \]
    as long as
    $\veps\leq{}c\cdot{}(1-\gammainf^2)^2\alphainf^{-4}\Psistar^{-11}$,
      where $c$ is a sufficiently small numerical constant. Proposition 1 of \cite{mania2019certainty} now implies that
    \[
      \nrm*{\Khat-\Kinf}_{\op}\leq{} \bigoh(\alphainf^{2}(1-\gammainf^{2})^{-1}\Psistar^{9})\cdot{}\veps.
    \]
    The strong stability result follows by observing that
    \begin{align*}
      \nrm*{\Pinf^{1/2}(A+B\Khat)\Pinf^{-1/2}}_{\op}
&      \leq{}
      \nrm*{\Pinf^{1/2}(A+B\Kinf)\Pinf^{-1/2}}_{\op}
                                                       +\nrm*{\Pinf^{1/2}B(\Khat-\Kinf)\Pinf^{-1/2}}_{\op}\\
      &      \leq{}
\gammainf
        +\alphainf\Psistar\nrm*{\Khat-\Kinf}_{\op}.
    \end{align*}
In the general case, we apply the reasoning above with $A'=BAG^{-1}$,
$B'=GB$, and $Q'=G^{-T}QG^{-1}$, and $R$, and observe that the optimal
controller for this system is $\Kinf{}G^{-1}$. The same perturbation
bound holds, but with $\Psistar$ scaled up by at most $\Csim^2$ and
$\alphainf$ scaled up by at most $\Csim$.
    
\end{proof}

\begin{proof}[\pfref{thm:performance_difference}]
  Before beginning the proof, we collect some helpful norm bounds. We
  have:
  \begin{align}
    &\En_{\pihat}\nrm*{\matx_t}_{2}^{2} \leq{} \cx^{2},
    \label{eq:xt_bound}\\
    &\En_{\pihat}\nrm*{\pihat_t(\maty_{1:t})}_{2}^{2}\leq{}
    \En_{\pihat}\nrm[\big]{\Khat\fhat_t(\maty_{1:t})}_{2}^{2}+\sigmanu^{2}\leq{}
    4 2\Psistar^{2}\Bf^{2}+\sigmanu^{2},\label{eq:pihat_bound}\\
    &\En_{\pihat}\nrm*{\pistar_t(\maty_{t})}_{2}^{2}\leq{}
    \nrm*{\Kinf}_{\op}^{2}\En_{\pihat}\nrm*{\matx_t}_{2}^{2}\leq{}\Psistar^{2}\cx^{2},
  \end{align}
  where \eqref{eq:pihat_bound} follows because $\nrm{\Khat}_{\op}\leq{}2\nrm*{\Kinf}_{\op}$, so that
  $\nrm[\big]{\Khat\fhat(\maty_{1:t})}\leq{}2\nrm*{\Kinf}_{\op}\Bf\leq{}2\Psistar\Bf$
  almost surely.

  As a first-step, using the standard performance difference lemma \citep{kakade2003sample}, we have
  \begin{align*}
    \cost(\pihat) - \cost(\piinf)
    &=\En_{\pihat}\brk*{\frac{1}{T}\sum_{t=1}^{T}\Qf_{t:T}^{\Kinf}(\matx_t,\pihat(\maty_{1:t}))
      - \Qf_{t:T}^{\Kinf}(\matx_t,\piinf(\maty_{t}))
      },
  \end{align*}
where we have used \pref{ass:perfect}, which implies that the
Q-functions for $\piinf$ have the
form in \pref{eq:lqr_q2}.

Let $T_0=T-\tau$. We handle the timesteps before and after $T_0$ separately. For
the first case, where $t\leq{}\tau$, we apply \pref{lem:vinf_approx},
which implies that
\begin{align*}
  &\En_{\pihat}\brk*{\sum_{t=1}^{\Tnot}\Qf_{t:T}^{\Kinf}(\matx_t,\pihat(\maty_{1:t}))
    - \Qf_{t:T}^{\Kinf}(\matx_t,\piinf(\maty_{t}))}\\
  &\leq{}\En_{\pihat}\brk*{\sum_{t=1}^{\Tnot}\Qinf(\matx_t,\pihat(\maty_{1:t}))
    - \Qinf(\matx_t,\piinf(\maty_{t}))} \\
  &~~~~+
\bigoh(\alphainf^{2}(1-\gammainf^{2})^{-1}\Psistar^{5})\cdot{}\sum_{t=1}^{\Tnot}\gammainf^{2(T-t)}(\En_{\pihat}\nrm*{\matx_t}^{2}_2+\En\nrm*{\pihat(\maty_{1:t})}_{2}^{2}+\En\nrm*{\pistar(\maty_{1:t})}_{2}^{2})
\end{align*}
We simplify the error term above to
\begin{align*}
  \gammainf^{2\tau}(T-\tau)\cdot{}\bigoh\prn*{\alphainf^{2}\Psistar^{7}(\cx^{2}\vee\sigmanu^2\vee\bclip^2)}.
\end{align*}
To handle the summands, we observe that since
$\piinf(\maty_t)=\argmin_{u\in\bbR^{\dimu}}\Qinf(\matx,u)$, and since
$\Qinf$ is a strongly convex quadratic with Hessian $\Pinf +
B^{\trn}\Pinf{}B=\Siginf$, the first-order conditions for optimality
imply that
\[
\Qinf(\matx_t,\pihat(\maty_{1:t}))
- \Qinf(\matx_t,\piinf(\maty_{t})) = \nrm*{\pihat(\maty_{1:t})-\piinf(\maty_t)}_{\Siginf}^{2}.
\]
Thus, since $\nrm*{\Siginf}_{\op}\leq{} 2\Psistar^{3}$
(\pref{prop:inf_bounds}), we have
\begin{align*}
  \En_{\pihat}\brk*{\sum_{t=1}^{\Tnot}\Qinf(\matx_t,\pihat(\maty_{1:t}))
  - \Qinf(\matx_t,\piinf(\maty_{t}))}
  \leq2\Psistar^{3}\sum_{t=1}^{\Tnot}\En_{\pihat}\nrm*{\pihat(\maty_{1:t})-\pistar(\maty_t)}_2^2.
\end{align*}
Now, for each $t$, we have
\begin{align*}
\En_{\pihat}\nrm*{\pihat(\maty_{1:t})-\pistar(\maty_t)}_2^2=  &\En_{\pihat}\nrm*{\Khat\fhat_t(\maty_{1:t})+\bnu_t-\Kinf\fstar(\maty_t)}_2^2\\
  &\leq{}\En_{\pihat}\nrm*{\Khat\fhat_t(\maty_{1:t})-\Kinf\fstar(\maty_t)}_2^2
    + \sigmanu^{2}\\
  &\leq{}2\En_{\pihat}\nrm*{\Khat\fhat_t(\maty_{1:t})-\Khat\fstar(\maty_t)}_2^2
    + 2\En_{\pihat}\nrm*{(\Khat-\Kinf)\fstar(\maty_t)}_2^2
    + 2\sigmanu^{2}\\
  &\leq{}8\Psistar^{2}\En_{\pihat}\nrm*{\fhat_t(\maty_{1:t})-\fstar(\maty_t)}_2^2
    + 24\cx^{2}\nrm*{\Khat-\Kinf}_{\op}^{2}+2\sigmanu^{2}\\
  &\leq{}  8\Psistar^{2}\vepsf^{2}
                                                + 24\cx^{2}\vepsk^{2}+2\sigmanu^{2}. \tag{\theequation}\label{eq:per_step_bound}
\end{align*}
Collecting terms, this gives a coarse bound of
\[
  \En_{\pihat}\brk*{\sum_{t=1}^{\Tnot}\Qinf(\matx_t,\pihat(\maty_{1:t}))
  - \Qinf(\matx_t,\piinf(\maty_{t}))}\leq{} \bigoh\prn*{\Psistar^{5}\cx^{2}(T-\tau)(\vepsf^{2}+\vepsk^2+\signu^{2})}.
\]
We now bound the terms after time $T_0$. Using
\pref{lem:qinf_lipschitz}, we have
\begin{align*}
  &\En_{\pihat}\brk*{\sum_{t=\Tnot}^{T}\Qf_{t:T}^{\Kinf}(\matx_t,\pihat(\maty_{1:t}))
    - \Qf_{t:T}^{\Kinf}(\matx_t,\piinf(\maty_{t}))}\\
&\leq{}\bigoh(\Psistar^{3})\En_{\pihat}\brk*{\sum_{t=\Tnot}^{T}(\nrm*{\matx_t}_2+\nrm*{\pihat_t(\maty_{1:t})}_2+\nrm*{\piinf(\maty_t)}_2)\nrm*{\pihat(\maty_{1:t})-\piinf(\maty_t)}_2}\\
&\leq{}\bigoh(\bclip\Psistar^{4})\En_{\pihat}\brk*{\sum_{t=\Tnot}^{T}(\nrm*{\matx_t}_2+\nrm*{\bnu_t}_2)\nrm*{\pihat(\maty_{1:t})-\piinf(\maty_t)}_2}\\
&\leq{}\bigoh(\bclip\Psistar^{4})\sum_{t=\Tnot}^{T}\prn*{\sqrt{\En_{\pihat}\nrm*{\matx_t}^2_2}+\sqrt{\En\nrm*{\bnu_t}_2^{2}}}\sqrt{\En_{\pihat}\nrm*{\pihat(\maty_{1:t})-\piinf(\maty_t)}_2^{2}}\\
&\leq{}\bigoh(\bclip\Psistar^{4}(\cx+\signu))
\sum_{t=\Tnot}^{T}\sqrt{\En_{\pihat}\nrm*{\pihat(\maty_{1:t})-\piinf(\maty_t)}_2^2}\\
  &\leq{}\bigoh(\bclip\Psistar^{4}(\cx+\signu)(\Psistar\vepsf+\cx\vepsk+\signu)\cdot\tau)\\
  &\leq{}\bigoh(\bclip\Psistar^{5}\cx(1\vee\signu)(\vepsf+\cx\vepsk+\signu)\cdot{}\tau),
\end{align*}
where the second-to-last inequality uses \pref{eq:per_step_bound}.

\end{proof}

\newpage
\section{Proofs for \richid Phase I and II}
\label{sec:phase12_proofs}

The section is organized as follows. 
\begin{itemize}
\item \Cref{ssec:sysid_prelim} contains
  preliminaries. \Cref{sssec:sysid_marg_cond} establishes the relevant
  Gaussian marginals and conditionals, \Cref{sssec:sysid_kapnot}
  specifies the burn-in parameter $\kapnot$, and \Cref{app:sysid_hid_properties} addresses relevant properties of the function class $\Hid$.
\item \Cref{ssec:sysid_thm_phase_one} provides proofs for Phase I, in particular \Cref{thm:phase_one} and its more granular statement, \Cref{thm:phase_one_a}.
\item \Cref{sssec:thm_phase_two} provides proofs for Phase II, including \Cref{thm:phase_two}/\Cref{thm:phase_two_a}.
\end{itemize}
\subsection{Preliminaries \label{ssec:sysid_prelim}}

	Recall that in the identification phase, for each $t \ge 0$,
        we take $\matu_t \sim \cN(0,I_{\dimu})$. We recall that the
        controllability matrices are given by $\cont_k= [A^{k-1}B \mid \dots \mid B ] $, and define the following matrices:
	\begin{align}
	\Sigkid&:= A^k \Sigxnot (A^k)^\top + \sum_{s=0}^{k-1} (A^s)(\Sigw + BB^\top) (A^s)^\top.\\
	\Sigstid &:= \sum_{s=0}^{\infty} (A^s)(\Sigw + BB^\top ) (A^s)^\top. \label{eq:sigstid}
	\end{align}
	We also recall the definition of $\kapnot$ and $\kapone$:
	\begin{align}
	\kapnot &:= \ceil*{\frac{1}{1-\gammastar} \ln \left(\frac{84\Psistar^5 \alphastar^4 \dimx \ln(1000 \nid)}{(1-\gammastar)^2}\right)} \label{eq:sysid_kap_def_app}.\\
	\kappa_1 &:= \kapnot + \kappa.
	\end{align}
	Finally, we define
	\begin{align*}
	\bv :=  (\bu_{\kapnot}^\top, \dots, \bu_{\kapone-1}^\top)^\top\in\bbR^{\kappa\dimu},
	\end{align*}
        and we recall the definition of the function class used in the
        regression problem for Phase I:
	\begin{align}
	\Hid := \left\{ M  f(\cdot)\mid  f \in \Fclass,~~ M \in
	  \R^{\kappa\dimu \times \dimx}, ~~\|M\|_{\op} \le \sqrt{\Psistar}
	  \right\},\label{eq:hid}
        \end{align}
        which corresponds to choosing $r_{\id}=\sqrt{\Psist}$.
\subsubsection{Marginals and Conditions \label{sssec:sysid_marg_cond}} 
To compute the Bayes regression function for Phase I we use the following results, which are readily verified.
	\begin{fact}[Marginals for Phase I]\label{fact:marginals}
	Fix $\kappa,\kapnot$ and define. $\kapone := \kapnot + \kappa$. Then  $\bv,\bxkapone$ are jointly Gaussian are jointly gaussian and mean zero. Moreover, $\bxkapone \sim \cN(0,\Sigkid[\kapone])$,  $\bv \sim \cN(0, I_{k\dimu})$, and $\E[\bv \bxk^\top] = \contkap^\top$. 
	\end{fact}
	\begin{fact}[Gaussian
          Expectation]\label{fact:gaussian_expectation} Let $(U,X)$ be
          jointly Gaussian random variables with distribution
	\begin{align*}
	(U,X) \sim \cN\left(0, \begin{bmatrix} \Sigma_{UU} & \Sigma_{UX} \\ \Sigma_{XU} & \Sigma_{XX} \end{bmatrix}\right).
	\end{align*}
	Then we have $\E[ U \mid X = x] = \Sigma_{UX} \Sigma_{XX}^{-1} x$.
	\end{fact}
\subsubsection{Selecting the Burn-In Time \label{sssec:sysid_kapnot}}
	
	\begin{lemma}\label{lem:sysid_Sigma_bound} Fix an integer
          $\nid \in \bbN$. Then as long as $\kapnot$ satisfies \pref{eq:sysid_kap_def_app},
        we have that for any $k, k' \ge \kapnot$
        (including $k = \infty$), the following properties hold.
	\begin{enumerate}
        \item The following bounds hold with respect to the
                  PSD ordering:
		\begin{align*}
		\frac{9}{10}\Sigst \preceq \Sigkid[k] \preceq \frac{11}{10}\Sigst \preceq \frac{11}{5}\Psistar^2\alphastar^2(1-\gammastar)^{-1} \cdot I.
		\end{align*}
		\item Fix $\veps>0$. For any $h_1,h_2 \in \Hid$ with $\E\|h_1(\byk) - h_2(\byk)\|^2 \le \veps^2$, we have \[\E\|h_1(\by_{k'}) - h_2(\by_{k'})\|^2 \le 2\max\{\veps^2,\Psistar L^2/\nid\}.\]
		\item  The controllability matrices satisfy the
                  following bounds:
		\begin{align}
		&1 \wedge \sigma_{\min}(\contk^\top\Sigkid[k']^{-1/2}) \ge \sigma_{\dimx}(\contk)\sqrt{\frac{5(1-\gammastar)}{11\Psistar^2\alphastar^2}}, \label{eq:sysid_sigma_mins_alt}\intertext{and}
		&1 \wedge \sigma_{\min}(\contk^\top\Sigkid[k']^{-1/2}) \cdot \sigma_{\min}(\Sigkid[k']^{-1/2}) \ge \frac{5(1-\gammastar)\sigma_{\dimx}(\contk)}{11\Psistar^2\alphastar^2}. \label{eq:sysid_sigma_mins}
		\end{align}
		\item $\nrm[\big]{\contk^\top\Sigkid[k']^{-1}}_{\op} \le \sqrt{\Psistar}$, provided $k \ge k'$ (but in fact, not requiring $k \ge \kapnot$).
		\end{enumerate} 
	\end{lemma}
	\begin{proof} 
		Our proof starts with the following claim, which
                shows that the covariance matrices for $k$ and $k'$
                are very close under the conditions of the lemma.
		\begin{claim}\label{claim:sysid_mix_claim} Fix
                  $\epsilon\in(0,1/2)$. For all $k,k'  \ge \frac{1}{1-\gammastar} \ln \frac{6\epsilon^{-1}\Psistar^2 \alphastar^2}{1-\gammastar}$, we have that 
		\begin{align*}\max\{\|I- \Sigkid[k]^{-1/2}\Sigkid[k']\Sigkid^{-1/2}\|_{\op},\|I- \Sigkid^{1/2}\Sigkid[k']^{-1}\Sigkid^{1/2}\|_{\op}\} \le \epsilon.
		\end{align*}
		\end{claim}
		\begin{proof}[Proof of \Cref{claim:sysid_mix_claim}]
		Since $\Sigkid \succeq \Sigw \succeq \Psistar^{-1} I$ by definition, we have that 
		\begin{align*}
		\|I- \Sigkid^{-1/2}\Sigkid[k']\Sigkid^{-1/2}\|_{\op} &= \|\Sigkid^{-1/2}(\Sigkid - \Sigkid[k'])\Sigkid^{-1/2}\| \\
		&\le \|\Sigkid^{-1}\|\|\Sigkid[k'] -\Sigkid\| \le \Psistar \|\Sigkid[k'] -\Sigkid\|.
		\end{align*}
		 By the same token,
		\begin{align*}
		&\|I- \Sigkid^{1/2}\Sigkid[k']^{-1}\Sigkid^{1/2}\|_{\op} \le \epsilon \\
		&\quad \iff (1 - \epsilon) I \preceq \Sigkid^{1/2}\Sigkid^{-1}\Sigkid^{1/2} \preceq (1+\epsilon) I\\
		&\quad\iff (1 - \epsilon) \Sigkid^{-1} \preceq \Sigkid[k']^{-1} \preceq (1+\epsilon) \Sigkid^{-1} \tag{conjugation}\\
		&\quad\iff (1 +\epsilon)^{-1} \Sigkid \preceq \Sigkid[k'] \preceq (1-\epsilon)^{-1} \Sigkid \tag{inversion}\\
		&\quad\iff (1 +\epsilon)^{-1} I \preceq \Sigkid^{-1/2}\Sigkid[k']\Sigkid^{-1/2} \preceq (1+\epsilon)^{-1} \Sigkid^{-1/2}\\
		&\quad\iff \|I -\Sigkid^{-1/2}\Sigkid[k']\Sigkid^{-1/2}\|_{\op} \le \max\{1 - (1 +\epsilon)^{-1}, (1-\epsilon)^{-1} -1 \}.
		\end{align*}
		In particular, for $\epsilon \le 1/2$,
		$\|I- \Sigkid^{1/2}\Sigkid[k']^{-1}\Sigkid^{1/2}\|_{\op} \le \epsilon$ as long as $\|I -\Sigkid^{-1/2}\Sigkid[k']\Sigkid^{-1/2}\|_{\op} \le 2 \epsilon$. Combining with the above,
		\begin{align}
		&\max\{\|I -\Sigkid^{-1/2}\Sigkid[k']\Sigkid^{-1/2}\|_{\op},\|I- \Sigkid^{-1/2}\Sigkid[k']\Sigkid^{-1/2}\|_{\op}\} \le \epsilon \label{eq:sysid_sig_error_shit} \\
		&\text{if } \|\Sigkid -\Sigkid[k']\| \le \frac{\epsilon}{2\Psistar}\le  \epsilon \le 1/2. \nonumber
		\end{align}

		Next, for any $k,k'$, using strong stability implies
		\begin{align*}
		\|\Sigkid - \Sigkid[k']\|_{\op} &\le \|A^k(\Sigxnot)(A^k)^\top - A^{k'}(\Sigxnot)(A^{k'})^\top\|_{\op} + \left\|\sum_{i=\min\{k,k'\}+1}^{\max\{k,k'\}} (A^{i}) \Sigw (A^{i})^\top\right\|_{\op}\\
		&\le \Psistar \alphastar^2\left(2\gammastar^{2\min\{k,k'\}} + \sum_{i=\min\{k,k'\}+1}^{\max\{k,k'\}} \gammastar^{2s}\right)\\
		&\le \frac{3\Psistar \alphastar^2\gammastar^{2\min\{k,k'\}}}{1 - \gammastar}.
		\end{align*}
		Hence, for a given $\epsilon > 0$, we have
		\begin{align*}
		\|\Sigkid- \Sigkid[k']\|_{\op} \le \epsilon \text{ for } \min\{k,k'\}  \ge \frac{1}{1-\gammastar} \ln \frac{3\epsilon^{-1}\Psistar \alphastar^2}{1-\gammastar}.
		\end{align*}
		The bound now follows by combining with \Cref{eq:sysid_sig_error_shit}, and shrinking $\epsilon$ by a factor of $2$.
		\end{proof}
                Next, we require a basic operator norm bound for $\Sigstid$.
		\begin{claim}\label{claim:sysid_Sigstid_ub} $\|\Sigstid\|_{\op} \le 2\Psistar^2\alphastar^2(1-\gammastar)^{-1}$.
		\end{claim}
		\begin{proof}[Proof of \Cref{claim:sysid_Sigstid_ub}]
                  $\|\Sigstid\|_{\op} = \|\sum_{i = 0}^{\infty}
                  (A^i)(\Sigw + BB^\top) (A^i)^\top \|_{\op} \le
                  (\|\Sigw\|_{\op} +
                  \|B\|_{\op}^2)\sum_{i=0}^{\infty}\|A^i\|_{\op}^2$. We
                  can bound $(\|\Sigw\|_{\op} + \|B\|_{\op}^2) \le
                  2\Psistar^2$ and $\|A^i\|_{\op}^2 \le
                  \alphastar^2\gammastar^{2i}$, so that $\|\Sigstid\|_{\op}\le 2\Psistar^2 \alphastar^2 \sum_{i\ge 0}\gammastar^{2i} \le \Psistar^2\alphastar^2(1-\gammastar)^{-1}$.
		\end{proof}
                We now proceed with the proof of the lemma. We prove
                points 1 through 4 in order.
		\begin{enumerate}
		\item  We have that $\frac{9}{10}\Sigstid \preceq \Sigkid \preceq \frac{11}{10}\Sigstid$ if and only $\|I - \Sigstid^{-1/2}\Sigkid\Sigstid^{-1/2}\| \le 1/10$. Hence, the bounds hold by selecting $k \leftarrow \infty$, $k' \leftarrow k$, and invoking \Cref{claim:sysid_mix_claim} for our choice of $\kappa_0$. Moreover, by \Cref{claim:sysid_Sigstid_ub}, $\Sigstid \preceq  2I\Psistar^2\alphastar^2(1-\gammastar)^{-1} $, yielding the last inequality.
		\item For point $2$, every $h \in \Hid$ satisfies
                  $\|h(y)\| \le L \sqrt{\Psistar}
                  \max\{1,\|\fst(y)\|\}$; see the definition of the class $\Hid$ in \Cref{eq:hid}. Hence, given two elements $h,h' \in \Hid$ with  $\E_{\by_k \sim \cN(0,\Sigkid)}\|h(\by_k) - h'(\by_k)\| \le \veps^2$, \Cref{lem:techtools_Gaus_change_of_measure} ensures that
		\begin{align*}
		&\E_{\by_{k'} \sim \cN(0,\Sigkid[k'])}\|h(\by_{k'}) - h'(\by_{k'})\| \le 2\max\{\veps^2,\Psistar L^2/\nid\},\\
		&\text{provided that } \|I - \Sigkid^{1/2}\Sigkid[k']^{-1}\Sigkid^{1/2}\|_{\op} \le \frac{1}{14 \dimx \ln (80 e \nid  (1+\|\Sigkid\|_{\op}))}.
		\end{align*}
		Using that $\Psistar,\alphastar \ge 1$ and the
                previous bound, $(1+\|\Sigkid\|_{\op}) \le
                \frac{22}{5}\Psistar^2\alphastar^2(1-\gammastar)^{-1}
                $, we have that as long as 
		\begin{align}
                  \label{eq:burnin_desired_prec}
		\|I - \Sigkid^{1/2}\Sigkid[k']^{-1}\Sigkid^{1/2}\|_{\op} \le \frac{1}{14 \dimx \ln (16 \cdot 22 e\nid  \Psistar^2\alphastar^2(1-\gammastar)^{-1})},
		\end{align}
		we obtain the desired inequality: $\E_{\by_{k'} \sim \cN(0,\Sigkid[k'])}\|h(\by_{k'}) - h'(\by_{k'})\| \le 2\max\{\veps^2,\Psistar L^2/\nid\}$. Finally to obtain the guarantee in \pref{eq:burnin_desired_prec}, we require
		\begin{align*}
		\min\{k,k'\} \ge \frac{1}{1-\gammastar} \ln \left(\frac{6\cdot 14\Psistar^3 \alphastar \dimx \ln(16 \cdot 22 e \nid  \Psistar^2\alphastar^2(1-\gammastar)^{-1}))}{1-\gammastar}\right).
		\end{align*}
                Simplifying constants, a sufficient condition is that
		\begin{align*}
		\min\{k,k'\}\ge \frac{1}{1-\gammastar} \ln \left(\frac{84\Psistar^3 \alphastar^2 \dimx \ln(1000 \nid  \Psistar^2\alphastar^2(1-\gammastar)^{-1})}{1-\gammastar}\right).
		\end{align*}
		Finally, since $\ln(xy) = \ln(x)+\ln(y) \leq y\ln(x)$
                for $x \ge e$ and $y\geq{}1$, we can further simplify to the
                sufficient condition
		\begin{align*}
		\min\{k,k'\}\ge \frac{1}{1-\gammastar} \ln \left(\frac{84\Psistar^5 \alphastar^4 \dimx \ln(1000 \nid)}{(1-\gammastar)^2}\right) := \kappa_0.
		\end{align*}
		which is precisely the condition in \Cref{eq:sysid_kap_def_app}.
		\item For the third point, we start with $1 $
		\begin{align}
		\lambda_{\min}(\Sigkid[k']^{-1/2}) = \sqrt{\frac{1}{\lambda_{\max}(\Sigkid[k'])}} \ge \sigma_{\dimx}(\contk)\sqrt{\frac{5(1-\gamma)}{11\Psistar^2\alphastar^2}}\label{eq:Sigkid_third_point_1}
		\end{align}
		 To prove the first point of \Cref{eq:sysid_sigma_mins}, we bound 
		\begin{align}
		\sigma_{\dimx}(\contk^\top \Sigkid[k']^{-1/2}) &\ge \sigma_{\dimx}(\contk)\lambda_{\min}(\Sigkid[k']^{-1/2})  \ge \sigma_{\dimx}(\contk)\sqrt{\frac{5(1-\gamma)}{11\Psistar^2\alphastar^2}},\label{eq:Sigkid_third_point_2}
		\end{align}
		where we use the first point of the lemma in the last
        step (\Cref{eq:Sigkid_third_point_1}). To see that this lower bound (the RHS of \Cref{eq:Sigkid_third_point_2}) is less than $1$ (acounting for the $\wedge 1$ in the LHS of \Cref{eq:sysid_sigma_mins}), we observe  $\Sigkid[k'] \succeq
                \contk\contk^\top$ for $k' \ge k$, and thus  $\sigma_{\dimx}(\contk^\top \Sigkid[k']^{-1/2}) \le 1$.

		Proving the second part of
                \Cref{eq:sysid_sigma_mins} follows by combining \Cref{eq:Sigkid_third_point_1,eq:Sigkid_third_point_2}. The resultant lower bound is also less than $1$, since the $\text{(RHS of \Cref{eq:Sigkid_third_point_1})} \le 1$, and  $\sigma_{\dimx}(\contk^\top \Sigkid[k']^{-1/2}) \le 1$ as well. %
		\item Finally, $\|\contk^\top \Sigkid[k']^{-1}\|_{\op} \le \|\contk^\top (\Sigkid[k']^{-1/2})\|_{\op}\|\Sigkid[k']^{-1/2}\|_{\op}$. Since $\Sigkid[k'] \succeq \contk \contk^\top$ for $k' \ge k$, $\|\contk^\top (\Sigkid[k']^{-1/2})\|_{\op} \le 1$. Moreover, since $\Sigkid[k'] \succeq \Sigw \succeq \Psistar^{-1}I$,  $\|\Sigkid[k']^{-1/2}\|_{\op} \le \sqrt{\Psistar}$, as needed.
		\end{enumerate}
	\end{proof}
\subsubsection{Properties of the Class $\Hid$}
\label{app:sysid_hid_properties}

	\begin{lemma}\label{lem:sysid_cconcid} %
	Let $\kappa_0$ satisfy \Cref{eq:sysid_kap_def_app}, and define
	\begin{align}
	\cconcid := \frac{12 L^2  \Psistar^3\alphastar^2}{1-\gammastar} \label{eq:sysid_cconcid_def}.
	\end{align}
	Then, for all $k \ge \kappa_0$: %
	\begin{enumerate}
		\item $\max_{h \in \Hid}\|h(\by_k)\|^2\leq
                  L^2\Psistar\max\{1,\|\fst(\by_k)\|^2\}$, and both are $(\dimx\cdot\cconcid)$-concentrated.
		\item For any matrix $V$ with $\|V\|_{\op} \le 1$
                  (e.g., any $V$ with orthonormal columns) and any
                  $h,h' \in \Hid$, the random variable $\|V^\top h(\by_k)\|^2$ is $(\dimx \cconcid)$ concentrated, and $\|V^\top h(\by_k) - V^\top h'(\by_k)\|^2$ is $(4\dimx \cconcid)$-concentrated.
	\end{enumerate}
	\end{lemma}
	\begin{proof}
		Let us first reason about the concentration of
                $\|\fst(\byk)\|^2$. Under perfect decodability,
                $\|\fst(\byk)\|^2 = \|\bxk\|^2$, which is
                $5\trace(\Sigkid) \le 5\dimx
                \|\Sigkid\|_{\op}$-concentrated by
                \Cref{lem:techtools_gaussian_c_concentrated}. Moreover,
                from \Cref{lem:sysid_Sigma_bound}, we have that
                $5\dimx \|\Sigkid\|_{\op} \le
                11\Psistar^2\alphastar^2(1-\gammastar)^{-1}$. 

To finish proving the first point, observe that that $\max_{h \in \Hid}\|h(y)\|^2 \le L^2\Psistar\max\{1,\|\fst(y)\|^2\}$ (\Cref{eq:hid}). From \Cref{lem:techtools_truncated_conc}, we recall that if a random variable $\bz$ is $c$-concentrated, then $\alpha (\bz + \beta)$ is $\alpha(c + \beta)$ concentrated for $\beta,\alpha > 0$. Hence, $\max_{h \in \Hid}\|h(y)\|^2$ is $L^2 \Psistar(1 + 11\dimx\Psistar^2\alphastar^2(1-\gammastar)^{-1} ) \le \dimx \cdot 12 L^2  \Psistar^3\alphastar^2(1-\gammastar)^{-1}= \dimx \cconcid$-concentrated, as needed.

		The proof of the second point is analogous. First, we
                note that for $\nrm{V}_{\op}\leq{}1$,  $\|V^\top h(\by_k)\|^2 \le \|h(\by_k)\|^2$
                and $\|V^\top (h(\by_k)
                -h'(\by_k)) \|^2 \le
                4L^2\Psistar\max\{1,\|\fst(y)\|^2\}$. Combined with
                the concentration result for $\nrm*{\fst(\byk)}^{2}$
                above, this yields the result.
	\end{proof}

	\begin{lemma}\label{lem:sysid_bayes_opt} Let $\kapnot$ satisfy
          \Cref{eq:sysid_kap_def_app}, let $\kappa \in \bbN$, and
          define $\kappa_1 := \kapnot + \kappa$.
	Then for all $x \in \reals^{\dimx}$ and $y \in \supp q(\cdot
        \mid x)$ we have:
	\begin{align*}
	\E[\bv \mid \bykapone = y] = \contkap^\top\Sigkaponeid^{-1} x \eqqcolon  \hstid(y), 
	\end{align*}
and $\hstid \in \Hid$.
	\end{lemma}
	\begin{proof} 

		Since $\bv \to \bxkapone \to \bykapone$ forms a Markov
                chain, and $\bxkapone= \fst(\bykapone)$ almost surely,
                we have that $(\bv,\bxkapone,\bykapone)$ is
                \emph{decodable} in the sense of \Cref{defn:techtools_decode}. Thus, by \Cref{lem:techtools_condition_exp_idenity}, $\E[\bv \mid \bykapone = y] = \E[\bv \mid \bxkapone = \fst(y)]$. By \Cref{fact:marginals}, $(\bv,\bxkapone)$ are jointly Gaussian and mean zero, and $\E\brk{\bxkapone\bxkapone^\top} = \Sigkaponeid$ and $\E\brk{\bv\bxkapone^\top} = \contkap^\top$. Thus, $\E[\bv \mid \bxkapone = x] = \contkap^\top\Sigkaponeid^{-1} x$ (\Cref{fact:gaussian_expectation}), giving that $\E[\bv \mid \bykapone = y] =   \contkap^\top\Sigkaponeid^{-1}\fst(y) = \hstid(y)$, as needed.

		To see that $\hstid \in \Hid$, we observe that $\hstid = M\fst$ for $M = \contkap^\top\Sigkaponeid^{-1}$. By \Cref{lem:sysid_Sigma_bound} part 4, we have $\|M\|_{\op} \le \Psistar^{1/2}$. Thus, from the definition of $\Hid$ in \Cref{eq:hid} and the fact that $\fst \in \Fclass$ by the realizability assumption, we conclude that $\hstid \in \Hid$.
	\end{proof}

\subsection{Proof of Decoder Recovery
  (\Cref{thm:phase_one}) \label{ssec:sysid_thm_phase_one}}
We first state the full version of \pref{thm:phase_one}, which asserts
that Phase I recovers a decoder that accurately predicts the state
under Gaussian roll-in, up to a well-conditioned similarity transformation.
	\begin{thmmod}{thm:phase_one}{a} 
		\label{thm:phase_one_a}
	For a universal constant $\cbaridone \ge 8$, define
	\begin{align}\label{eq:sysid_epsidh_def}
	\epsidh^2 &=  \cbaridone\frac{\ln^2(\frac{\nid}{\delta})(\dimu \kappa + \dimx \cconcid)(\ln |\Fclass|+ \dimu \dimx \kappa)}{\nid},
	\end{align}
	and assume  that $\nid$ is sufficiently large such that
	\begin{align*}
	\epsidh	\sqrt{\dimx\cconcid} \le  \frac{(1-\gammastar)\sigma_{\dimx}(\contkap)^2}{71\alphastar^2\Psistar^2}.
	\end{align*}
	Then, with probability at least $1 - 3\delta$, there exists an invertible matrix $\Sid \in \R^{\dimx^2}$ satisfying 
	\begin{align*}
		1 \wedge \sigma_{\min}(\Sid) \ge \sigidmin :=
          \frac{\sigma_{\min}(\contkap)(1-\gammastar)}{4\Psistar^2\alphastar^2},
          \mathand 1 \vee \|\Sid\|_{\op} \le \sigidmax := \sqrt{\Psistar},
	\end{align*}
	such that the function $\fstid(y) := \Sid \fst(y)$ and the learned decoder $\fhatid$ satisfy
	\begin{align*}
	\E\|\fstid(\bykapone) - \fhatid(\bykapone)\|^2 \le \epsidh^2.
	\end{align*} 
	In particular, for 
			\begin{align*}
			\nid = \Omega_{\star}\prn*{\dimx\dimu \kappa(\ln |\Fclass| + \dimu \dimx \kappa) },
			\end{align*}
			we have that
			\begin{align*}
			\E\|\fstid(\bykapone) - \fhatid(\bykapone)\|^2 &\le \bigohs\prn*{\frac{\dimu \kappa(\ln |\Fclass| + \dimu \dimx \kappa)\ln^2(\frac{\nid}{\delta})}{\nid}}.
			\end{align*}
	\end{thmmod}	
	\begin{proof} The proof of this theorem follows from two
          propositions which we establish in the sequel. The first, \Cref{prop:sysid_h_estimate}, demonstrates that the learned function $\hhatid$ satisfies the following bound with probability $1 - 3\delta/2$:
	\begin{align*}
	\E_{\bykapone}[\|\hhatid(\bykapone) - \hstid(\bykapone)\|^2 ] \le \epsidh^2,
	\end{align*}
	where $\hstid(y) :=
        \contkap^\top\Sigkaponeid^{-1}\fst(y)$. Now recall that the
        function $\fhatid(y)$ is constructed as $\Vhatid^\top
        \hstid(y)$, where  $\Vhatid$ has orthonormal columns. Defining
        $\Sid = \Vhatid^\top \contkap^\top\Sigkaponeid^{-1}$, we see
        that $\Vhatid^\top \hstid(y) = \Sid \fst(y)=\fstid(y)$. Thus, since
        $\Vhatid$ has operator norm $1$,
	\begin{align*}
	\E_{\bykapone}[\|\fhatid(\bykapone) - \fstid(\bykapone)\|^2 ] &= \E_{\bykapone}[\|\Vhatid^\top(\hhatid(\bykapone) - \hstid(\bykapone))\|^2 ]\\
	&\leq \E_{\bykapone}[\|\hhatid(\bykapone) - \hstid(\bykapone)\|^2 ] \le \epsidh^2.
	\end{align*}
	To conclude, the norm bounds for the matrix $\Sid$ are
        provided by \Cref{prop:sysid_pca}, which hold with probability
        at least $1 - \delta$.
	\end{proof}
\subsubsection{Prediction Error Guarantee for $\hhatid$}
	\begin{proposition}\label{prop:sysid_h_estimate} Let
          $(\bykapone^{(i)},\bv^{(i)})_{i =1}^{\nid}$ be as described
          in \pref{alg:phase1} and \pref{sec:phase1}, and let 
	\begin{align*}
	\hhatid \in \argmin_{h \in \Hid}\sum_{i=1}^{\nid} \|h(\bykapone\ind{i}) - \bv\ind{i}\|^2.
	\end{align*}
	Then there is a universal constant $\cbaridone \ge 8$ such
        that with probability at least $1 - \frac{3}{2}\delta$, we have 
	\begin{align*}
	\E_{\bykapone}[\|\hhatid(\bykapone) - \hstid(\bykapone)\|^2 ] \le \epsidh^2 
	\le   \bigohs\prn*{\frac{\ln^2(\frac{\nid}{\delta})\dimu \kappa (\ln |\Fclass| + \dimu \dimx \kappa)}{\nid}}.
	\end{align*}
	We let $\eventidh$ denote the event that this inequality holds.
	\end{proposition}
	\begin{proof} 
		From \Cref{lem:sysid_bayes_opt}, we have $\E[\bv \mid
                \bykapone = y] = \hstid(y)$, so $\hstid\in\Hid$. To
                prove the result, we simply apply our general-purpose
                error bound for least-square regression,
                \Cref{cor_techtools:reg_cor_simple},
                with $\maty = \bykapone$, $\bu = \bv$, and $\mate =
                0$. We verify that each precondition for the
                proposition holds.
		\begin{itemize}
		\item \textbf{Structure of function classes}: $\Fclass$ is finite, and by \Cref{asm:f_growth}, $\|f(y)\| \le L \max\{1,\|\fst(y)\|\}$. Moreover, $\Hid:= \{M \cdot f: \|M\| \le \sqrt{\Psistar}, f \in \Fclass\}$. 
		\item \textbf{Concentration Property.} By
                  \Cref{lem:sysid_cconcid}, defining $\varphi(y) :=L^2\Psistar\max\{1,\|\fst(y)\|^2\}$
                  we see that
                  $\varphi(\by_{\kapone})$ is
                  $\dimx\cconcid$-concentrated. Moreover, since $\bv
                  \sim \cN(0,I_{\dimu \kappa})$, we have that
                  $\|\bv\|^2$ is $5 \dimu \kappa$-concentrated by
                  \Cref{lem:techtools_gaussian_c_concentrated}. Hence,
                  $c= 5\dimu \kappa + \dimx\cconcid$ is a valid choice
                  for the concentration constant $c$ in \Cref{prop:techtools_general_regression_with_truncation_and_error}.
		\end{itemize}
		Thus, \Cref{cor_techtools:reg_cor_simple} with $c \lesssim \dimu \kappa + \dimx \cconcid$ implies that
		\begin{align*}
		\E_{\bykapone}[\|\hhatid(\bykapone) - \hstid(\bykapone)\|^2 ] &\lesssim \frac{\ln^2(\tfrac{\nid}{\delta})(\dimu \kappa + \dimx \cconcid)(\ln |\Fclass| + \dimu \dimx \kappa)}{\nid}\\
		&= \bigohs\prn*{\frac{\ln^2(\frac{\nid}{\delta})(\dimu \kappa + \dimx )(\ln |\Fclass|+ \dimu \dimx \kappa)}{\nid}}\\
		&= \bigohs\prn*{\frac{\ln^2(\frac{\nid}{\delta})\dimu \kappa (\ln |\Fclass| + \dimu \dimx \kappa)}{\nid}},
		\end{align*}
		where the last simplification uses that controllability requires $\dimu \kappa \ge \dimx$.
	\end{proof}

\subsubsection{Dimension Reduction}
	\begin{proposition}\label{prop:sysid_pca} Suppose that
          $\kapnot$ satisfies \Cref{eq:sysid_kap_def_app}. Let $\Vhatid\in\bbR^{\kappa\dimu\times\dimx}$ be an eigenbasis for the top $\dimx$ eigenvalues of $\Lamhat_n$, where we define
		\begin{align*}
		 \Lamhat_n := \frac{1}{n}\sum_{i=1}^n \hhatid(\bykapone\supi)\hhatid(\bykapone\supi)^\top.
		\end{align*}
		Further, let $ \Sid \coloneqq  \Vhatid^\top
		\contkap^\top\Sigkaponeid^{-1} \in \R^{\dimx\times\dimx}$. Then if
		\begin{align}
		\sqrt{\dimx\cconcid}\epsidh \le  \frac{(1-\gammastar)\sigma_{\dimx}(\contkap)^2}{71\alphastar^2\Psistar^2},\label{eq:sysid_pca_cond}
		\end{align}
		we have that with probability at least $1 - \delta$,
                an event $\eventidpca$ occurs such that on $\eventidh\cap\eventidpca$,
		\begin{align*}
		1 \wedge \sigma_{\min}(\Sid) \ge \sigidmin := \frac{\sigma_{\min}(\contkap)(1-\gammastar)}{4\Psistar^2\alphastar^2}, \quad 1 \vee \|\Sid\|_{\op} \le \sigidmax := \sqrt{\Psistar}.
		\end{align*}
		\end{proposition}

	\begin{proof}
		Introduce $\Lambda_{\star}  := \contkap \Sigkaponeid^{-1}\contkap^\top$, and let $\Vid\in\bbR^{\kappa\dimu\times\dimx}$ be an eigenbasis for its $\dimx$ non-zero eigenvectors. From \Cref{lem:sysid_bayes_opt} and \Cref{fact:marginals}, we have
		\begin{align*}
		\E[\hstid(\bykapone)\hstid(\bykapone)^\top] &= \contkap\Sigkaponeid^{-1}\E[\fst(\bykapone)\fst(\bykapone)^\top]\Sigkaponeid^{-1}\contkap^{\top} \\
		&= \contkap\Sigkaponeid^{-1}\E[\bxkapone\bxkapone^\top]\Sigkaponeid^{-1}\contkap^{\top} \\
		&= \contkap\Sigkaponeid^{-1}\Sigkaponeid\Sigkaponeid^{-1}\contkap^{\top} = \contkap \Sigkaponeid^{-1}\contkap^\top := \Lambda_{\star}.
		\end{align*}

		 We apply
                 \Cref{prop:techtools_general_pca,cor:techtools_overlap_PCA},
                 with $\Lamhat_n$ and $\Lambda_{\star}$ as above. To
                 apply this proposition, first observe
                 that $\nrm{\hhatid(\bykapone)}^2$ is
                 $c=\dimx\cconcid$-concentrated by \Cref{lem:sysid_cconcid}. Morever, on $\eventidh$, we have $\E[\|\hhatid(\bykapone) - \hstid(\bykapone)\|^2 \le \epsidh^2$. Thus, the term $\veps_{\mathrm{pca},n,\delta}$ in  \Cref{prop:techtools_general_pca}, specializes to 
			\begin{align*}
			\epsidpca := 3\sqrt{\dimx\cconcid}\epsidh + 5 \dimx\cconcid \nid^{-1/2} \ln(2\kappa\dimu \nid/\delta)^{3/2}.
			\end{align*}
			By the fact that $\nid \ge  \kappa\dimu$ (it can be
                        verified that this is required to ensure the
                        upper bound on $\epsidh$), we can
                        bound \[\dimx\cconcid\nid^{-1/2}
                          \ln(2\kappa\dimu \nid/\delta)^{3/2}
                          \le\sqrt{\dimx\cconcid} \epsidh,\] where
                        $\epsidh$ is as in
                        \Cref{prop:sysid_h_estimate}. Hence, we can bound
			\begin{align*}
			\epsidpca = 3\sqrt{\dimx\cconcid}\epsidh + 5 \dimx\cconcid \nid^{-1/2} \ln(2\kappa\dimu \nid/\delta)^{3/2} \le 8\sqrt{\dimx\cconcid} \epsidh.
			\end{align*}
			Thus, if we denote the event above by $\eventidpca$, we that have on $\eventidh\cap\eventidpca$,
			\begin{align*}
			\|\Vhatid^\top(\Lamhat_n - \Lambda_{\star})\Vhatid\|_{\op} =  \|\Lamhat_n - \Lambda_{\star}\|_{\op} \le8\sqrt{\dimx\cconcid\epsidh} .
			\end{align*}
		From the fourth point of \Cref{lem:sysid_Sigma_bound}, we have that $\lambda_{\min}(\Lambda_{\star}) \ge \frac{5(1-\gammastar)\sigma_{\dimx}(\contkap)^2}{11\alphastar^2\Psistar^2}$. Hence,  \Cref{cor:techtools_overlap_PCA} ensures that under the $\eventidh\cap\eventidpca$, we have
		\begin{align}
                  \sigma_{\min}(\Vhatid^\top \Vid) \ge
                  \frac{2}{3}, \label{eq:ovlap_guarantee}
                  \end{align}as long as\begin{align}
                  \sqrt{\dimx\cconcid}\epsidh \le  \frac{(1-\gammastar)\sigma_{\dimx}(\contkap)^2}{71\alphastar^2\Psistar^2}  \le \frac{1}{4 \cdot 8 } \cdot \frac{5(1-\gammastar)\sigma_{\dimx}(\contkap)^2}{ 11\alphastar^2\Psistar^2}  \nonumber,
		\end{align}
		which is precisely the condition
                \Cref{eq:sysid_pca_cond} required by the theorem. 
		To conclude, let us bound the singular values of $\Sid
                := \Vhatid^\top \contkap^\top\Sigkaponeid^{-1}$ under
                the assumption that the bound above holds. Since $\Vhatid$ has orthonormal columns, have that 
		\begin{align*}
		\|\Sid\| &\le \|\contkap^\top\Sigkaponeid^{-1}\| \le \sqrt{\Psistar} \tag{\Cref{lem:sysid_Sigma_bound}}.
		\end{align*}
		On the other hand, we can lower bound 
		\begin{align}
		\sigma_{\min}(\Sid) = \sigma_{\min}(\Vhatid^\top \contkap^\top\Sigkaponeid^{-1}) \ge \sigma_{\min}(\Vhatid^\top \contkap^\top\Sigkaponeid^{-1/2}) \sigma_{\min}(\Sigkaponeid^{-1/2}), \label{eq:Sid_pca_lb}
		\end{align}
                where we have used that $\Sigkaponeid^{-1/2}$ and
                $\Vhatid^\top \contkap^\top\Sigkaponeid^{-1/2}$ are square. We now prove the following claim, which is also reused
                in a number of subsequent proofs.
		\begin{claim}\label{claim:sysid_eventpca_sigmin_claim} On $\eventidh \cap \eventidpca$, we have 
		\begin{align*}
		\sigma_{\min}(\Vhatid^\top \contkap^\top\Sigkaponeid^{-1/2}) \ge \frac{2}{3}\sigma_{\dimx}(\contkap^\top\Sigkaponeid^{-1/2}).
		\end{align*}
		\end{claim}
		\begin{proof}[Proof of \Cref{claim:sysid_eventpca_sigmin_claim}]
			Since $\Vid$ is an eigenbasis for $\contkap^\top\Sigkaponeid^{-1}\contkap$, we have that $\contkap^\top\Sigkaponeid^{-1/2} = \Vid\Vid^\top\contkap^\top\Sigkaponeid^{-1/2}$, and $\sigma_{\dimx}(\Vid^\top\contkap^\top\Sigkaponeid^{-1/2}) = \sigma_{\dimx}(\contkap^\top\Sigkaponeid^{-1/2})$. Thus, 
			\begin{align*}
			\sigma_{\min}(\Vhatid^\top \contkap^\top\Sigkaponeid^{-1/2}) &= \sigma_{\min}(\Vhatid^\top \Vid\Vid^\top\contkap^\top\Sigkaponeid^{-1/2})\\
			&\ge \sigma_{\min}(\Vhatid^\top \Vid)\sigma_{\min}(\Vid^\top\contkap^\top\Sigkaponeid^{-1/2})\\
			&\ge
                   \frac{2}{3}\sigma_{\min}(\Vid^\top\contkap^\top\Sigkaponeid^{-1/2})
                   \tag{by \Cref{eq:ovlap_guarantee}},
			\end{align*}
                        where the first inequality uses that
                        $\sigmamin(XY)\geq\sigmamin(X)\sigmamin(Y)$
                        for $X,Y\in\bbR^{\dimx\times\dimx}$. Again, since $\Vid$ is an eigenbasis for the non-zero eigenvalues of $\contkap^\top\Sigkaponeid^{-1/2}$, we have 
			\begin{align*}
			\sigma_{\min}(\Vid^\top\contkap^\top\Sigkaponeid^{-1/2}) &= \sigma_{\dimx}(\contkap^\top\Sigkaponeid^{-1/2}).
			\end{align*}
		\end{proof}

		Combining the above claim with \Cref{eq:Sid_pca_lb}, we have
		\begin{align*}
		\sigma_{\min}(\Sid)  &\ge \frac{2}{3}\sigma_{\min}(\contkap^\top\Sigkaponeid^{-1/2}) \sigma_{\min}(\Sigkaponeid^{-1/2})\\
		&\ge \frac{5\cdot 2\sigma_{\min}(\contkap)(1-\gammastar)}{11\cdot 3\Psistar^2\alphastar^2} \tag{\Cref{eq:sysid_sigma_mins} in \Cref{lem:sysid_Sigma_bound}}\\
		&\ge \frac{\sigma_{\min}(\contkap)(1-\gammastar)}{4\Psistar^2\alphastar^2},
		\end{align*}
		as needed.
	\end{proof}

\subsection{Estimating Costs and Dynamics (\Cref{thm:phase_two}) \label{sssec:thm_phase_two}}
To begin this section, we state the full version of
\pref{thm:phase_two}, which shows that Phase II (\pref{alg:phase2})
accurately recovers the system matrices, noise covariance, and state
cost up to a similarity transformation.
	\begin{thmmod}{thm:phase_two}{a}\label{thm:phase_two_a} 
		For a possibly inflated numerical constant $\cbaridone$ in \Cref{eq:sysid_epsidh_def}, suppose $\epsidh$ satisfies
		\begin{align*}
		\epsidh \sqrt{\ln(2\nid/\delta)} \le  \frac{\sigma_{\dimx}(\contkap)^2(1-\gammastar)^2}{80\cdot12 L^2\Psistar^5\alphastar^4\dimx}.
		\end{align*}
		Then, with probability at least $1 - 11\delta$ over
                both Phase I and Phase II, the following bounds hold:
		\begin{align*}
			\|\Qhatid - \Qid\|_{\op} &\lesssim \epsidh \cdot \frac{\alphastar^6\Psistar^7\sqrt{\cconcid \dimx \ln(\nid/\delta)} }{(1-\gammastar)^3\sigma_{\min}(\contkap)^4},\\
			\| [\Ahatid;\Bhatid] - [\Aid;\Bid]\|_{\op} &\lesssim   \epsidh \cdot \frac{\Psistar^{7/2}\alphastar^3}{\sigma_{\min}(\contkap)^2(1-\gammastar)^{3/2}}, \\
			\|\Sigwidhat - \Sigwid\|_{\op} &\lesssim  \epsidh \cdot \frac{\Psistar^{5/2}\alphastar^4}{\sigma_{\min}(\contkap)(1-\gammastar)}.
			\end{align*}
			In particular, for 
			\begin{align*}
			\nid = \Omega_{\star}\prn*{\dimx^2\dimu \kappa (\ln |\Fclass| + \dimu \dimx \kappa)\max\{1,\sigma_{\min}(\contkap)^{-4}\} },
			\end{align*}
			we have that
			\begin{align*}
			&\|\Qhatid - \Qid\|_{\op} \vee\| [\Ahatid;\Bhatid] - [\Aid;\Bid]\|_{\op} \vee\|\Sigwidhat - \Sigwid\|_{\op} \\
			&\le \bigohs\prn*{\nid^{-1/2}\cdot\sqrt{\dimx\dimu \kappa (\ln |\Fclass| + \dimu \dimx \kappa)\ln(\nid/\delta)^4 }}
			\end{align*}
	\end{thmmod}

\subsubsection{Preliminaries for \Cref{thm:phase_two}}
Before proceeding with the proof of \pref{thm:phase_two}, we recall some notation. First, following \pref{sssec:thm_phase_two}, we let 
\[
\eventidh\mathand\eventidpca
\]
denote the events from \pref{prop:sysid_h_estimate} and
\pref{prop:sysid_pca}, respectively. Next, we introduce some functions used throughout the proof and prove some
basic facts about them.
		\begin{definition} 
			Recall that $\Sid := \Vhatid^\top
                        \contkap^\top\Sigkaponeid^{-1}$
                        (\Cref{prop:sysid_pca}). Define functions
                        $\fstid,\fhatid,\errfid: \cY \to \R^{\dimx}$ via:
			\begin{align*}
			\fhatid := \Vhatid^\top \hhatid, \quad \fstid := \Sid \fst, \quad \errfid := \fhatid - \fstid.
			\end{align*}
		\end{definition}

		\begin{lemma}\label{lem:sysid_concentration_and_error_bounds_fid} 
			
			Let $\kappa$ satisfy the conditions of \Cref{lem:sysid_Sigma_bound}. Then, under the good event $\eventidh \cap \eventidpca$,
			\begin{enumerate}
				\item $\E\|\errfid(\bykapone)\|^2 \le \epsidh^2$.%
				\item For any $k \ge \kapnot$ (in
                                  particular, for $k = \kapone + 1$), $\E\|\errfid(\byk)\|^2 \le 2\epsidh^2$
				\item For any $k \ge \kapone$ (in
                                  particular, for $k \in \{\kapone,\kapone+1\}$), $\max\{\|\fhatid(\by_k)\|,\|\fstid(\by_k)\|^2\}$ is $\dimx \cconcid$-concentrated, and $\|\errfid(\byk)\|^2$ is $4\dimx\cconcid$-concentrated. 
				\item We have that $\fstid(\bykapone)
                                  = \Sid \bxkapone$ is zero-mean Gaussian, with
				\begin{align*}
				\E\brk[\big]{\fstid(\bykapone)\fstid(\bykapone)^\top} = \Sid\prn[\big]{\E\brk[\big]{\bxkapone \bxkapone^\top}}\Sid^\top  \succeq I\cdot \frac{\sigma_{\dimx}(\contkap)^2(1-\gammastar)}{10\Psistar^2\alphastar^2} .
				\end{align*}
			\end{enumerate}
		\end{lemma}
		\begin{proof} 
			For point $1$, we have that $\fstid := \Sid \fst = \Vhatid^\top \contkap^\top\Sigkaponeid^{-1} \fst = \Vhatid^\top \hstid$. Hence,
			\begin{align*}
			\E\|\errfid(\bykapone)\|^2 =\E\|\Vhatid^\top (\hhatid - \hstid)\|^2 \overset{(i)}{\le} \E\|\hhatid - \hstid\|^2 \overset{(ii)}{\le} \epsidh^2,
			\end{align*}
			where inequality $(i)$ uses that $\Vhatid$ has
                        orthonormal columns, and inequality $(ii)$
                        uses the definition of $\eventidh$
                        (\Cref{prop:sysid_h_estimate}). Points 2 and 3
                        follow from \Cref{lem:sysid_Sigma_bound} and
                        \Cref{lem:sysid_cconcid}, respectively.

			Finally, for point $4$, we use that $\Sid =
                        \Vhatid^\top \contkap^\top\Sigkaponeid^{-1}$
                        to write
			\begin{align*}
			\E\brk[\big]{\fstid(\bxkapone)\fstid(\bxkapone)^\top} &= \Sid\E\brk[\big]{\bxkapone\bxkapone^\top}\Sid^\top \\
			&= \Sid(\Sigkaponeid)\Sid^\top\\
			&= \Vhatid^\top \contkap^\top\Sigkaponeid^{-1}\Sigkaponeid \Sigkaponeid^{-1}\contkap \Vhatid\\
			&= \Vhatid^\top \contkap^\top\Sigkaponeid^{-1}\contkap\Vhatid.
			\end{align*}
			Hence,
			\begin{align*}
			\lambda_{\min}(\Sid\E\brk[\big]{\bxkapone\bxkapone^\top}\Sid^\top) &\ge \sigma_{\min}(\Vhatid^\top \contkap^\top\Sigkaponeid^{-1/2})^2\\
			&\ge\frac{4}{9} \sigma_{\min}(\contkap^\top\Sigkaponeid^{-1/2})^2\tag{\Cref{claim:sysid_eventpca_sigmin_claim}}\\
			&\ge\frac{4}{9} \cdot \frac{5\sigma_{\dimx}(\contkap)^2(1-\gammastar)}{11\Psistar^2\alphastar^2} \tag{\Cref{lem:sysid_Sigma_bound}}\\
			&\ge\frac{\sigma_{\dimx}(\contkap)^2(1-\gammastar)}{10\Psistar^2\alphastar^2},
			\end{align*}
			as needed.
		\end{proof}
\subsubsection{Estimation of $\Aid$, $\Bid$, and $\Sigwid$}
We first show that Phase II recovers the system matrices and system
noise covariance.
	\begin{proposition}\label{prop:sysid_estimate_ABSig} Define
		\begin{align*}
		(\Ahatid,\Bhatid) &\in \argmin_{(A,B)} \frac{1}{\nid}\sum_{i=2\nid + 1}^{\nid} \|\fhatid(\by_{\kapone+1}\supi) - A\fhatid(\bykapone\supi) - B\bu_{\kapone}\supi \|^2,\\
		\Sigwidhat &= \frac{1}{\nid}\sum_{i=2\nid + 1}^{\nid} (\fhatid(\by_{\kapone+1}\supi) - \Ahatid\fhatid(\bykapone\supi) - \Bhatid\bu_{\kapone}\supi )^{\otimes 2}.
		\end{align*}
		Further, define the matrices $\Aid := \Sid
                A\Sid^{-1}$, $\Bid := \Sid B$, and $\Sigwid := \Sid
                \Sigw \Sid^\top$. Suppose $\epsidh$ satisfies
                \Cref{eq:sysid_pca_cond} (which is the preqrequisite
                of $\eventidpca$ of \Cref{prop:sysid_pca}), and $n \ge
                c_0(\dimx  + \log(1/\delta))$ for some universal
                constant $c_0$. Then on $\eventidpca \cap \eventidh$, the following event, designated $\eventidls$, holds with probability $1 - 7\delta$:
		\begin{align*}
			\| [\Ahatid;\Bhatid] - [\Aid;\Bid]\|_{\op} &\lesssim \frac{\Psistar^{7/2}\alphastar^3}{\sigma_{\min}(\contkap)^2(1-\gammastar)^{3/2}} \epsidh, \\
			\|\Sigwidhat - \Sigwid\|_{\op} &\lesssim \frac{\Psistar^{5/2}\alphastar^4}{\sigma_{\min}(\contkap)(1-\gammastar)} \epsidh.
			\end{align*}
	\end{proposition}
	
	\begin{proof} 

		We cast the regressions above as an instance of
                error-in-variable regression, then apply our general
                guarantees for this problem (\Cref{prop:techtools_LS_guarantee,prop:techtools_covariance_est}). We restate the guarantees here:
		\propTechtoolsLS*
		\propTechtoolsCovEst*

		To distinguish between our present notation and the
                notation of these propositions, we mark the terms to
                which we apply the proposition with a tilde. Define
		\begin{align*}
		\bytil &:= \fhatid(\by_{\kapone+1}\supi),\\
		\betil &:= \errfid(\by_{\kapone+1}\supi) = (\fstid - \fhatid)(\by_{\kapone+1}\supi),\\
		\matdeltil &:= [ -\errfid(\by_{\kapone})^\top ; 0_{\dimu}^\top]^\top,\\
		\butil &:= [ \fstid(\by_{\kapone})^\top ; \bu_{\kapone}^\top]^\top,\\
		\bwtil &:= -\Sid \bw_{\kapone},\\
		d_{\tilde{y}} &:= \dimx ,\\
		d_{\tilde{u}} &:= \dimx + \dimu,\\
		\Mtil &:= [\Aid;\Bid].
		\end{align*}
To proceed, we verify that this correspondence satisfies the
conditions of the propositions above.
		\begin{claim}\label{claim:ls_id_struc_eqs} It holds
                  that   $\Mtil\butil = \bytil  + \betil  + \bwtil$ and
		\begin{align*}
                  [\Ahatid;\Bhatid] \in \argmin_{M}\sum_{i=1}^n\|\bytil\supi - M(\butil\supi + \matdeltil\supi)\|^2.
		\end{align*}
		\end{claim}
		\begin{proof} Observe that we have the dynamics
			\begin{align*}
			A\bx_{\kapone} + B\bu_{\kapone} &= \bx_{\kapone+1} - \bw_{\kapone}\intertext{and}
			A\fst(\by_{\kapone}) + B\bu_{\kapone} &= \fst(\by_{\kapone+1} )- \bw_{\kapone}.
			\end{align*}
			Thus, recalling that $\fstid(y) = \Sid \fst(y)$, we have
			\begin{align*}
			\Sid A\Sid^{-1}\fstid(\by_{\kapone}) + \Sid B\bu_{\kapone} &= \fstid(\by_{\kapone+1}) - \Sid\bw_{\kapone}.
			\end{align*}
                        In our new notation, this implies that
			\begin{align*}
			[\Aid;\Bid]\butil &= \fstid(\by_{\kapone+1}) + \bwtil.
			\end{align*}
			Finally, writing $\fstid(\by_{\kapone}) =
                        \bytil + \matetil$ yields the first part of the claim. The second part follows from similar manipulations.
		\end{proof}

Next, we check the Gaussianity and covariance properties of $\butil,\bwtil$.

		\begin{claim}\label{claim:sysid_ls_gaussians} The
                  following properties hold:
		\begin{itemize}
			\item  $\bwtil$, $\butil$, and $\matdeltil$ are mutually independent.
			\item $\bwtil \sim \cN(0,\Sigwid)$, where $\lambda_{\max}(\Sigwid) \le 1$.
			\item $\butil \sim \cN(0,\Sigutil)$, where $\lambda_{\min}(\Sigutil) \ge \frac{\sigma_{\dimx}(\contkap)^2(1-\gammastar)}{10\Psistar^2\alphastar^2}$.
		\end{itemize}
		\end{claim}
		\begin{proof}[Proof of \Cref{claim:sysid_ls_gaussians}]  

			The first point of the claim follows because
                        $\bwtil$ is determined by $\bw_{\kapone}$ and
                        $\butil$ and $\matdeltil$ are determined by
                        $\by_{\kapone}$ and $\bu_{\kapone}$, respectively.

			For the second claim, we have that $\bwtil  = -\Sid \bw$. Since $\E\brk[\big]{\bw\bw^\top} = \Sigw$, $\E\brk[\big]{\bwtil\bwtil^\top} = \Sid \Sigw \Sid^\top$. Recalling that $\Sid = \Vhatid^\top \contkap^\top\Sigkaponeid^{-1}$ for some orthonormal $\Vhatid$, we find that
                        \begin{align*}
\lambda_{\max}(\Sid \Sigw \Sid^\top) &\le  \lambda_{\max}(\contkap^\top\Sigkaponeid^{-1}\Sigw \Sigkaponeid^{-1}\contkap) \\&= \sigma_{\max}^2(\contkap^\top\Sigkaponeid^{-1}\Sigw^{1/2}) \le \prn*{\sigma_{\max}(\contkap^\top\Sigkaponeid^{-1/2})\sigma_{\max}(\Sigkaponeid^{-1/2}\Sigw^{1/2})}^{2}.
                        \end{align*}
			Since  $\Sigkaponeid \succeq \Sigw$ and since
                        $\Sigkaponeid \succeq \contkap\contkap^\top$,
                        we have that
                        $\sigma_{\max}(\contkap^\top\Sigkaponeid^{-1/2}),\sigma_{\max}(\Sigkaponeid^{-1/2}\Sigw^{1/2})
                        \leq 1$, and we conclude that $\lambda_{\max}(\Sid \Sigw \Sid^\top) \le 1$ as needed.

			For the second-to-last claim, we have 
			\begin{align*}
			\butil = \begin{bmatrix}\fstid(\by_{\kapone})\\
			\bu_{\kapone}
			\end{bmatrix} = \begin{bmatrix}\Sid\bx_{\kapone}\\
			\bu_{\kapone}
			\end{bmatrix}.
			\end{align*}
			Since $\bx_{\kapone} \perp \bu_{\kapone}$, we have that
			\begin{align*}
			\E\brk[\big]{\butil\butil^{\top}} = \begin{bmatrix} \Sid\E\brk[\big]{\bx_{\kapone}\bx_{\kapone}^\top}\Sid^{\top} & 0 \\
			0 & I
			\end{bmatrix} \succeq  \begin{bmatrix} I \cdot \frac{\sigma_{\dimx}(\contkap)^2(1-\gammastar)}{10\Psistar^2\alphastar^2}  & 0 \\
			0 & I
			\end{bmatrix},
			\end{align*}
			where the last inequality uses part 4 of \Cref{lem:sysid_concentration_and_error_bounds_fid}. One can verify that the lower bound on the upper left block is less than $1$. Thus,
			\begin{align*}
			\lambda_{\min}\prn*{\E\brk[\big]{\butil\butil^{\top}}} \ge \frac{\sigma_{\dimx}(\contkap)^2(1-\gammastar)}{10\Psistar^2\alphastar^2}.
			\end{align*}
			
		\end{proof}
		Lastly, we check the relevant concentration properties
                for the errors $\betil$ and $\matdeltil$. 
		\newcommand{\cleastsq}{c_{\mathrm{ls}}}
		\newcommand{\epsleastsq}{\veps_{\mathrm{ls}}}
		\begin{claim}\label{claim:sysid_ls_adv_errs} The following bounds hold:
		\begin{itemize}
			\item $\matdeltil$ and $\betil$ are both $\cleastsq \ldef 4\cconcid$-concentrated, and satisfy $\E\|\matdeltil\|^2 \vee \E\|\betil\|^2 \le  2 \epsidh^2 \rdef \epsleastsq^2$. 
			\item For $n \ge \nid$, we have that $\epsleastsq/\cleastsq \ge \psi(n,\delta) = \frac{2\ln(2n/\delta)\ln(2/\delta)}{n}$.
			\item $\|\Mtil\|_{\op} \le \frac{4\Psistar^{5/2}\alphastar^2}{\sigma_{\min}(\contkap)(1-\gammastar)}$, and thus 
			\begin{align*}
			(1+\|\Mtil\|_{\op}^2)\epsleastsq^2 \le \frac{65\Psistar^{5}\alphastar^4}{\sigma_{\min}(\contkap)^2(1-\gammastar)^2} \epsleastsq^2 \le \frac{130\Psistar^{5}\alphastar^4}{\sigma_{\min}(\contkap)^2(1-\gammastar)^2} \epsidh^2.
			\end{align*}
			\item For $\epsidh$ satisfying
                          \Cref{eq:sysid_pca_cond}, we have $\epsleastsq^2 \le \frac{1}{16}\lambda_{\min}(\Sigutil)$ and thus $(1+\|\Mtil\|_{\op}^2)\epsleastsq^2  \le \lambda_+ \rdef 1$
		\end{itemize}
		\end{claim}
		\begin{proof}

			The first claims follows from \Cref{lem:sysid_concentration_and_error_bounds_fid}. 

			The second claim uses that, examining the
                        definition of $\epsidh$ in
                        \Cref{prop:sysid_h_estimate}, we have $\epsidh
                        \le 4\psi(\nid,\delta)/\cconcid$, implying
                        $\epsleastsq/\cleastsq \ge \psi(\nid,\delta)$.
                        Lastly use that $n \ge \nid$ and that $n \mapsto \psi(n,\delta)$ is decreasing.

			For the third point,
			\begin{align*}
			\|\Mtil\|_{\op} = \|[\Aid;\Bid]\|_{\op}
                          =\|[\Sid A \Sid^{-1};\Sid B]\|_{\op}
                          \leq{}
                          2(1\vee\sigmamin^{-1}(\Sid))(\nrm*{\Sid{}A}_{\op} \vee \nrm*{\Sid{}B}_{\op}).
			\end{align*}
                        Recalling $\Sid = \Vhatid^\top \contkap^\top
                        \Sigkaponeid^{-1}$, we use that
                        $\|\Vhatid\|_{\op} \le 1$ and $\Sigkaponeid
                        \succeq \contkap^\top$ to
                        bound
                        \[
                          \nrm*{\Sid{}A}_{\op} \vee
                          \nrm*{\Sid{}B}_{\op}
                          \leq{} \nrm[\big]{\Sigkaponeid^{-1/2}A}_{\op} \vee \nrm[\big]{\Sigkaponeid^{-1/2}B}_{\op}.
                        \]
                        Since $\Sigkaponeid \succeq A\Sigw
                        A^\top \succeq AA^\top/\lambda_{\min}(\Sigw)
                        \succeq AA^\top\Psistar^{-1}$ and $\Sigkaponeid
                        \succeq BB^\top$, we conclude that $                          \nrm*{\Sid{}A}_{\op} \vee
                          \nrm*{\Sid{}B}_{\op}\leq{}\Psist^{1/2}$.
Lastly, using $\sigma_{\min}(\Sid) \ge \sigidmin = \frac{\sigma_{\min}(\contkap)(1-\gammastar)}{4\Psistar^2\alphastar^2}$ from \Cref{prop:sysid_pca}, we conclude  that 
			\begin{align*}
			\|\Mtil\|_{\op} \le \sqrt{\Psistar}\frac{8\Psistar^2\alphastar^2}{\sigma_{\min}(\contkap)(1-\gammastar)} = \frac{8\Psistar^{5/2}\alphastar^2}{\sigma_{\min}(\contkap)(1-\gammastar)},
			\end{align*}
			as needed. The following inequality follows directly:
			\begin{align}
			(\|\Mtil\|_{\op}^2+1)\epsleastsq^2 \le \frac{65\Psistar^{5}\alphastar^4}{\sigma_{\min}(\contkap)^2(1-\gammastar)^2} \epsleastsq^2 \le \frac{130\Psistar^{5}\alphastar^4}{\sigma_{\min}(\contkap)^2(1-\gammastar)^2} \epsidh^2. \label{eq:Mtil_epslst_ineq}
			\end{align}

			For the fourth point, we examine the condition
                        in \Cref{eq:sysid_pca_cond},
                        $\sqrt{\dimx\cconcid}\epsidh \le
                        \frac{(1-\gammastar)\sigma_{\dimx}(\contkap)^2}{71\alphastar^2\Psistar^2}$,
                        which is equivalent to
			\begin{align*}
			\epsidh^2 \le  \frac{(1-\gammastar)^2\sigma_{\dimx}(\contkap)^2}{\dimx \cconcid 71^2\alphastar^2\Psistar^2} \cdot (\sigma_{\dimx}(\contkap)^2/\alphastar^2\Psistar^2).
			\end{align*}
			Using the definition of $\cconcid = 12
                        \Psistar^3 \alphastar^2L^2/(1-\gammastar)$ and
                        that $\dimx,L  \ge 1$, this further implies that
			\begin{align*}
			\epsidh^2 &\le  \frac{(1-\gammastar)^2\sigma_{\dimx}(\contkap)^2}{12 \cdot 71^2\alphastar^2\Psistar^2} \cdot (\sigma_{\dimx}(\contkap)^2/\alphastar^2\Psistar^2) \cdot \frac{(1-\gammastar)}{\Psistar^3\alphastar^2}.
			\end{align*}
			One can verify that
                        $(\sigma_{\dimx}(\contkap)^2/\alphastar^2\Psistar^2)
                        \le 1$ using the same arguments as in \Cref{lem:sysid_Sigma_bound}. Thus, under the above condition,
			\begin{align*}
			\epsidh^2 &\le  \frac{(1-\gammastar)^2\sigma_{\dimx}(\contkap)^2}{12 \cdot 71^2\alphastar^2\Psistar^2}  \cdot \frac{(1-\gammastar)}{\Psistar^3\alphastar^2}.
			\end{align*}
			Recalling $\lambda_{\min}(\Sigutil) \ge
                        \frac{\sigma_{\dimx}(\contkap)^2(1-\gammastar)}{10\Psistar^2\alphastar^2}
                        $ from \Cref{claim:sysid_ls_gaussians}, we can
                        directly verify that this implies $\epsleastsq^2 = 2\epsidh^2
                        \le
                        \frac{1}{16}\lambda_{\min}(\Sigutil)$. Similarly, using the bound in \Cref{eq:Mtil_epslst_ineq}
                        we can check that $(\|\Mtil\|_{\op}^2+1)\epsleastsq^2  \le \lambda_+ := 1$.
		\end{proof}
		To summarize, the claims above verify that, for
                $\epsidh$ satisfying \Cref{eq:sysid_pca_cond}, and $\nid
                \ge c_0(d_{\tilde{u}} + d_{\tilde{y}} +
                \log(1/\delta))$, the conditions for
                \Cref{prop:techtools_LS_guarantee,prop:techtools_covariance_est}
                hold. It follows that with probability at least $1 - 7\delta$,
		\begin{align*}
                  \| [\Ahatid;\Bhatid] - [\Aid;\Bid]\|_{\op}^2
                  \lesssim \lambda_{\min}(\Sigutil)^{-1}
                  \left((1+\|\Mtil\|_{\op}^2)\epsleastsq^2 +
                  \frac{\|\Sigwid\|_{\op}(d_{\tilde{y}} +
                  d_{\tilde{u}} +
                  \ln(1/\delta))}{\nid}\right),
                \end{align*}
                and
                \begin{align*}
                  \|\Sigwidhat- \Sigwid\|_{\op}\nonumber \lesssim \sqrt{\lambda_{+}(1+\|\Mtil\|_{\op}^2)\epsleastsq^2 + \frac{\lambda_{+}^2(\dimy + \log(1/\delta))}{\nid}}.
		\end{align*}
		Finally, to conclude, we note that $d_{\tilde{y}} +
                d_{\tilde{u}} = 2\dimx + \dimu$, so that (also
                recalling $\dimu \le \dimx$) it suffices to ensure
                $\epsidh$ satisfying \Cref{eq:sysid_pca_cond}, and $\nid
                \ge c_0(\dimx  + \log(1/\delta))$ for a larger universal constant $c_0$. Moreover, we can check that $(\dimx+ \log(1/\delta))/\nid \lesssim \epsidh^2$, which together with the bound $\|\Sigwid\|_{\op}\le \lambda_+ = 1$ means that the dominant terms above are the terms $(1+\|\Mtil\|_{\op}^2)\epsleastsq^2 \lesssim  \frac{\Psistar^{5}\alphastar^4}{\sigma_{\min}(\contkap)^2(1-\gammastar)^2} \epsidh^2$.  Thus, recalling $\lambda_{\min}(\Sigutil) \gtrsim \frac{\sigma_{\dimx}(\contkap)^2(1-\gammastar)}{\Psistar^2\alphastar^2}$ from {\Cref{claim:sysid_ls_gaussians}}, we find
		\begin{align*}
		\| [\Ahatid;\Bhatid] - \Mtil\|_{\op} &\lesssim \sqrt{\frac{\Psistar^2\alphastar^2}{\sigma_{\dimx}(\contkap)^2(1-\gammastar)}  \cdot \frac{\Psistar^{5}\alphastar^4}{\sigma_{\min}(\contkap)^2(1-\gammastar)^2} \epsidh^2}\\
		&\le \frac{\Psistar^{7/2}\alphastar^3}{\sigma_{\min}(\contkap)^2(1-\gammastar)^{3/2}} \epsidh. \intertext{and}
		\|\Sigwidhat - \Sigwid\|_{\op} &\lesssim \frac{\Psistar^{5/2}\alphastar^4}{\sigma_{\min}(\contkap)(1-\gammastar)} \epsidh.
		\end{align*}
	\end{proof}
        \subsubsection{Recovering the Cost Matrix $\Qid$}
        We now show that Phase II successfully recovers the system
        cost matrix $Q$ up to a similarity transform.
	\begin{proposition}[Guarantee for Recovery of $\Qid$]\label{prop:sysid_Qhatid}
	Recall the estimator
	\begin{align*}
          \Qhatid = \left(\frac{1}{2}\Qtilid  +
            \frac{1}{2}\Qtilid^\top\right)_{+},\end{align*}
        \begin{align*}
          \Qtilid= \min_{Q}\sum_{i=2\nid + 1}^{3\nid} \left(\matc_{\kapone}\ind{i} - (\bu_{\kapone}\ind{i})^\top R \bu_{\kapone}\ind{i} -  \fhatid(\bykapone\ind{i})^\top Q\fhatid(\bykapone\ind{i})\right)^2.
	\end{align*}
Suppose that the following conditions hold for a sufficiently large
numerical constant $c_0$:
\begin{align*}
  \nid \ge c_0 (\dimx^2 + \ln(1/\delta)),\quad\text{and}\quad
		\epsidh \sqrt{\ln(2\nid/\delta)} \le \frac{\sigma_{\dimx}(\contkap)^2(1-\gammastar)}{8 \cdot 10\cconcid\Psistar^2\alphastar^2\dimx}.
	\end{align*}
	Then with probability at least $1 - 2\delta$, we have
	\begin{align*}
		\|\Qhatid - \Qid\|_{\op} &\lesssim \frac{\alphastar^6\Psistar^7\sqrt{\cconcid \dimx \ln(\nid/\delta)} }{(1-\gammastar)^3\sigma_{\min}(\contkap)^4}\cdot \epsidh.
		\end{align*}
	\end{proposition}
	\begin{proof}
          We first rewrite the regression as a special case of
          \Cref{prop:techtools_matrix_sensing_regression}, which
          offers a generic guarantee for matrix regression
          with rank-one measurements. Observe that we have
		\begin{align*}
                  \matc_{\kapone} - \bu_{\kapone}^\top R \bu_{\kapone}  &= \left(\fst(\bykapone)^\top Q \fst(\bykapone) + \bu_{\kapone}^\top R \bu_{\kapone}\right) - \bu_{\kapone}^\top R \bu_{\kapone} \\
		&= \fst(\bykapone)^\top Q \fst(\bykapone) \\
		&= \fstid(\bykapone)^\top \Sid^{-\top} Q \Sid^{-1} \fstid(\bykapone) \\
		&= \fstid(\bykapone)^\top \Qid \fstid(\bykapone).
		\end{align*}
		Thus, the above regression is equivalent to solving
		\begin{align*}
		\Qtilid &= \min_{Q}\sum_{i=2\nid + 1}^{3\nid} \left(\fstid(\bykapone\ind{i})^\top \Qid \fstid(\bykapone\ind{i}) -  \fhatid(\bykapone\ind{i})^\top Q\fhatid(\bykapone\ind{i})\right)^2.
		\end{align*}

		We now recall the statement of \Cref{prop:techtools_matrix_sensing_regression}.
		\propTechtoolsMatSes*
		To apply the proposition, we make the following substitutions:
		\begin{align*}
		\Qst \leftarrow \Qid, ~~\Qhat \leftarrow \Qhatid, ~~\gst \leftarrow \fstid, ~~ \ghat \leftarrow \fhatid,~~ \Sigma_x \leftarrow \E\brk[\big]{\fstid(\by_{\kapone})\fstid(\by_{\kapone})^\top},~~n \leftarrow \nid.
		\end{align*}
We now verify that the conditions for the proposition are satisfied.
		\begin{enumerate}
			\item From \Cref{lem:sysid_concentration_and_error_bounds_fid},	$\max\{\|\fhatid(\bykapone)\|^2,\|\fstid(\bykapone)\|^2\}$ is $\dimx \cconcid$-concentrated, and $\E\|\fhatid(\bykapone) - \fstid(\bykapone)\|^2 \le \epsidh^2$.
			\item We have that $\psi(\nid,\delta/2)\leq \frac{\epsidh^2}{\dimx \cconcid}$ by examining the definition of $\epsidh$ in \Cref{prop:sysid_h_estimate}. 
			\item We have that $\fstid(\bykapone) = \Sid
                          \bxkapone$ is zero-mean Gaussian, with
			\begin{align*}
			\E\brk[\big]{\fstid(\bykapone)\fstid(\bykapone)^\top}  \succeq I\cdot \frac{\sigma_{\dimx}(\contkap)^2(1-\gammastar)}{10\Psistar^2\alphastar^2}
			\end{align*}
                        by
                        \Cref{lem:sysid_concentration_and_error_bounds_fid}. 	Hence,
                        the conditions
			\begin{align*}
			\epsidh \sqrt{\ln(2\nid/\delta)} \le \frac{\sigma_{\dimx}(\contkap)^2(1-\gammastar)}{8 \cdot 10\cconcid\Psistar^2\alphastar^2\dimx}, \quad\text{and}\quad \nid \ge c_0 (\dimx^2 + \ln(1/\delta)),
			\end{align*}
			suffice to satisfy the third condition of the proposition.
		\end{enumerate}
                We conclude that when the conditions above hold, with probability at least $1 - 2\delta$,
			\begin{align*}
			\|\Qhatid - \Qid\|_{\op} &\le \|\Qhatid - \Qid\|_{\fro}
			 \\
			&\lesssim \frac{\alphastar^2\Psistar^2\sqrt{\cconcid \dimx\ln(\nid/\delta)} \epsidh}{(1-\gammastar)\sigma_{\min}(\contkap)^2}\|\Qid\|_{\op}.
			\end{align*}
			Finally, we bound
			\begin{align*}
			\|\Qid\|_{\op} &= \|\Sid^{-1}\Qid \Sid^{-1}\|_{\op} \le \|\Qid\|_{\op}\sigma_{\min}^{-2}(\Sid)\\
			&= \Psistar \sigma^{-2}_{\min}(\Sid)\\
			&\le \Psistar\left(\frac{\sigma_{\min}(\contkap)(1-\gammastar)}{4\Psistar^2\alphastar^2}\right)^{-2}\\
			&\lesssim \frac{\Psistar^5\alphastar^4}{\sigma_{\min}(\contkap)^2(1-\gammastar)^2}.
			\end{align*}
			Thus, altogether,
			\begin{align*}
			\|\Qhatid - \Qid\|_{\op} &\le \|\Qhatid - \Qid\|_{\fro}
			 \\
			&\lesssim \frac{\alphastar^6\Psistar^7\sqrt{\cconcid \dimx\ln(\nid/\delta)}
			}{(1-\gammastar)^3\sigma_{\min}(\contkap)^4} \epsidh.
			\end{align*}
	\end{proof}
\subsubsection{Concluding the Proof of \Cref{thm:phase_two_a}}
	In total, by combining
        \Cref{prop:sysid_estimate_ABSig,prop:sysid_Qhatid} and
        conditioning on the probability $1 - 4\delta$ event from
        \Cref{prop:sysid_h_estimate,prop:sysid_pca}, we have that as long as (for some universal $c_0$),
	\begin{align*}
		&\epsidh \sqrt{\ln(2\nid/\delta)} \le \frac{\sigma_{\dimx}(\contkap)^2(1-\gammastar)}{8 \cdot 10\cconcid\Psistar^2\alphastar^2\dimx}, \\& \epsidh \text{ satisifies \Cref{eq:sysid_pca_cond},} \\
		&\text{and}\quad \nid \ge c_0 (\dimx^2 + \ln(1/\delta)),
	\end{align*}
	then with
        total failure probability at most $1 - 9 \delta - 4\delta$,
	\begin{align*}
		\|\Qhatid - \Qid\|_{\op} &\lesssim \frac{\alphastar^6\Psistar^7\sqrt{\cconcid \dimx\ln(\nid/\delta)} }{(1-\gammastar)^3\sigma_{\min}(\contkap)^4}\epsidh,\\
		\| [\Ahatid;\Bhatid] - [\Aid;\Bid]\|_{\op} &\lesssim \frac{\Psistar^{7/2}\alphastar^3}{\sigma_{\min}(\contkap)^2(1-\gammastar)^{3/2}} \epsidh,\\
		\|\Sigwidhat - \Sigwid\|_{\op} &\lesssim \frac{\Psistar^{5/2}\alphastar^4}{\sigma_{\min}(\contkap)(1-\gammastar)} \epsidh.
		\end{align*}
To simplify the conditions slightly, we observe that since $\epsidh \ge
        \cbaridone\frac{\dimx^2}{\nid}$ for some universal constant
        $\cbaridone \ge 4$, by inflating this constant, we can ensure
        that $ \nid \ge c_0 (\dimx^2 + \ln(1/\delta))$. Next let us
        consolidate the conditions
	\begin{align*}
		\epsidh \sqrt{\ln(2\nid/\delta)} \le \frac{\sigma_{\dimx}(\contkap)^2(1-\gammastar)}{8 \cdot 10\cconcid\Psistar^2\alphastar^2\dimx}, \quad\text{ and } \quad\epsidh \text{ satisifies  \Cref{eq:sysid_pca_cond}}.
	\end{align*}
	Restating \Cref{eq:sysid_pca_cond}, we require that
	\begin{align*}
          \epsidh\sqrt{\dimx\cconcid} \le  \frac{(1-\gammastar)\sigma_{\dimx}(\contkap)^2}{71\alphastar^2\Psistar^2}.
	\end{align*}
	Since $\cconcid \dimx \ge 1$, it suffices to take
	\begin{align*}
	\epsidh \sqrt{\ln(2\nid/\delta)} &\le \frac{\sigma_{\dimx}(\contkap)^2(1-\gammastar)}{80\cconcid\Psistar^2\alphastar^2\dimx}.
	\end{align*}
	Recalling that $\cconcid=\frac{12 L^2
          \Psistar^3\alphastar^2}{1-\gammastar}$, the final condition,
	\begin{align*}
	\epsidh \sqrt{\ln(2\nid/\delta)} \le  \frac{\sigma_{\dimx}(\contkap)^2(1-\gammastar)^2}{80\cdot12 L^2\Psistar^5\alphastar^4\dimx}.
	\end{align*}
	\qed

\section{Proofs for \richid Phase III}
\label{sec:phase3_proofs}
\newcommand{\cEroll}{\mathscr{E}}
\newcommand{\logterms}{\mathsf{logs}}
\newcommand{\gambarinf}{\bar{\gamma}_{\infty}}
\newcommand{\epsw}{\veps_w}
\newcommand{\epszero}{\veps_\init}

\newcommand{\PsiSig}{\Psi_{\Sigma}}
\newcommand{\PsiM}{\Psi_{M}}
\newcommand{\Honpo}{\scrH_{\onpo}}
\newcommand{\epsys}{\veps_{\sys}}
\newcommand{\alphA}{\alpha_A}
\newcommand{\dev}{\mathsf{dev}}
\newcommand{\devx}{\dev_x}
\newcommand{\delfrac}{\frac{1}{\delta}}
\newcommand{\pitil}{\widetilde{\pi}}
	\newcommand{\cwphi}{c_{w,\phi}}
	\newcommand{\cwupphi}{c_{w,\upphi}}
	\newcommand{\Lonpo}{L_{\onpo}}
\newcommand{\Phihat}{\widehat{\upphi}}

\paragraph{Section organization.}

This section is dedicated to the proof of \pref{thm:phase_three}, which is the main result concerning Phase III of \mainalg (cf. \pref{sec:phase3}). We state a number of intermediate results, leading up to the proof of theorem. In \pref{sec:learningafterinitial}, we present a performance bound for the state decoders $(\hat f_t)_{t \geq 1}$ as a function of the decoding error of the initial state decoder$\hat f_1$. \pref{sec:learningintial} is dedicated to the perfomance of $\hat f_1$, as this requires extra steps to the decode the initial state $\bx_0$. In \pref{sec:phaseiii_master}, we combine these results to prove \pref{thm:phase_three}. Finally, \pref{sec:proofs} contains the proofs of all the intermediate results.

We recall that the definition of the decoders $(\hat f_\tau)$ requires a clipping step (see \eqref{eq:decoder_simple}). Performing this step allows us to use standard concentration tools to bound the decoding error for the predictors that come out of the regression problems solved in Phase III (see \pref{lem:decoder}). The impact of clipping on the prediction error is low: In \pref{thm:clippingprob} we show that the probability of ever clipping is very small, so long as the clipping parameter $\bclip$ is chosen appropriately.

\subsection{Preliminaries}
Before proceeding to the main results, we first provide additional notation and definitions, as well as some basic lemmas which will be used in subsequent proofs.

\paragraph{Additional notation.}
For $t\geq0$ and a policy $\pi \colon \bigcup_{\tau=1}^\infty \cY^{\tau} \rightarrow \reals^{\dimu}$, where $\pi(\by_{0:t})$ maps past and current observations $\by_{0:t}$ to the current action $\bu_t$, we let $\mathbb{P}_{\pi}$ and $\E_{\pi}$ be the probability and expectation with respect to the system's dynamics and policy $\pi$. We will use $\cE$ to denote events which hold over the randomness in the learning procedure, and $\cEroll$ to denote events which hold under a given rollout from, say, $\E_{\pi}$. 

Throughout this section we let $\pihat$ denote the policy returned by \pref{alg:phase3}.

\paragraph{Basic definitions for Phase III.}
 To simplify presentation, we assume going forward that
$\Sid=I_{\dimx}$ at the cost of increasing problem-dependent parameters
such as $\Psistar$ and $\alpha_\star$ by a factor of $\nrm*{\Sid}_{\op}\vee\nrm{\Sid^{-1}}_{\op}$---we make this reasoning precise in the proof of \pref{thm:main_full}. We therefore drop the subscript $\mathrm{id}$, so that the system parameters we take as a given are $(\Ahat,\Bhat,\Qhat,\Sigwhat)$. We will consider the following function class 
\begin{align}
\scrH_{\onpo} \coloneqq \left\{M \cdot f(\cdot) \mid  f \in \Fclass,~~ M \in
\R^{\dimx \times \dimx}, ~~\|M\|_{\op} \le \Psistar^{3} \label{eq:funclass}
\right\},
\end{align}
that is, we take $r_{\op}=\Psist^{3}$ (note that the final value for $r_{\op}$ when \pref{alg:phase3} is invoked within \pref{alg:main} will be inflated to account for the similarity transformation above).

In what follows, we will construct a sequence of functions $(\hat{f}_t\colon  \cY^{t+1}\rightarrow \reals^{\dimx})$ which map observations $(\by_{0:t})$ to estimates of the true states $(\bx_t = f_\star(\by_t))$. We will denote by $\what \pi$ the randomized policy defined by $\what \pi(\by_{0:t}) \coloneqq \what K \hat{f}_t (\by_{0:t}) + \bnu_t$, for all $t\geq 0$, where $\bnu_t \sim \cN(0,\sigma^2 I_{\dimu})$ for some $\sigma \in (0,1]$ to be determined later.%
 Furthermore, for $t\geq 0$, we define the policy $\wtilde\pi_t$ which satisfies 
\begin{align}\wtilde{\pi}_t(\by_{0:\tau})= \left \{ \begin{array}{ll} \what \pi(\by_{0:\tau}), & \text{if } \tau\leq t; \\ \bnu_\tau \sim \cN(0, \sigma^2 I_{\dimu}), & \text{otherwise}. \end{array}\right.\label{eq:tilpolicy}
\end{align}

\paragraph{Additional problem parameters.}
Our final results for this section are stated in $\bigohs(\cdot)$, but we state many of our intermediate results with precise dependence on the problem parameters. To simplify these statements, we use the following definitions.
\begin{definition}[Aggregated problem parameters]
\begin{align}
\PsiSig &\coloneqq  \alpha_A^2 \|\Sigma_0\|_\op +\| \Sigma_\idinf\|_{\op}  + 1, \label{eq:PsiSig_def} \\
\devx &\coloneqq  \frac{3  \alpha_A^2}{1 -\gamma_A} \cdot \Psist^2 \|\what K\|^2_{\op}, \label{eq:devx_def}\\
\PsiM &\coloneqq \max\left\{1,\epsys,\|\calM\|_{\op}, \max_{k\in \kappa}\|M_k\|_{\op} \right\}, \nn \\
\Lonpo &\coloneqq \Psist^3 L.
\end{align}
\end{definition}
We simplify our intermediate results to get the final $\bigohs({\cdot})$-based bound for \pref{thm:phase_three} in \Cref{sec:phaseiii_master}.  

\subsubsection{Approximation Error for Plug-In Estimators.}
Recall that Phase III uses the estimates for $\Ahat$, $\Bhat$, and so forth from Phase II to form plug-in estimates for a number of important system parameters. Before proceeding, we give some guarantees on the error of these estimates as a function of the error from Phase II.

For $k \in [\dimk]$, recall that we define the matrices $M_k \in \reals^{k \dimu \times \dimx}$ and $\calM \in \reals^{(\dimk +1)\dimk/2 \times \dimk}$ by
\begin{gather} 
M_k \coloneqq   \cC_{k}^\top (   \cC_k \cC_k^\top+ \sigma^{-2}   \Sigma_{w} + \dots + \sigma^{-2} A^{k-1} \Sigma_w  (A^{k-1})^\top)^{-1},\nn \\
\calM \coloneqq [M_1^\top, (M_2A)^{\trn},\dots, (M_\dimk A^{k-1})^\top]^\top   \quad 
\text{where} \quad \cC_k \coloneqq [A^{k -1} B \mid \dots \mid B] \in \reals^{ \dimx  \times k \dimu }. \label{eq:Ck}
\end{gather}
We also let $\what M_k$ and $\what \calM$ be the \emph{plug-in estimators} of $M_{k}$, and $\calM$ respectively, obtained by replacing $A$, $B$, and $\Sigma_w$ in the definitions of $M_k$ and $\calM$ by the previously derived estimators $\what A, \what B$, and $\what \Sigma_w$, respectively (see \pref{sec:phase2}). %

For $0<\varepsilon_{\sys} \leq1\wedge\nrm{\Sigma_w}_{\op}\wedge \|\Sigma_w\|^{-1}_{\op}$, let $\cE_{\sys}$ be the event %
\begin{align}
  \cE_{\sys}&\coloneqq \left\{ \max_{k \in [\dimk ]} \left\{ \begin{matrix}  \|\what K - \Kinf\|_{\op},  ~\|\what M_k \what A^k-  M_k A^k \|_{\op}, \\  \|\what M_k \what A^k\what B-  M_k A^k B\|_{\op},~\|\what B \what K - B \what K \|_{\op}, \\ \|I_{\dimx}- \what \Sigma_w \Sigma^{-1}_w \|_{\op},  ~\| \what \Sigma_w^{-1} - \Sigma^{-1}_w \|_{\op}, \\ \|\what A-  A \|_{\op},~\|\what B-  B \|_{\op} , \|\what \calM -\calM\|_{\op} \end{matrix}  \right\} \leq \varepsilon_{\sys}\right\} \bigcap \cE_{\mathrm{stab}}, \label{eq:event0}%
\end{align}
where
\begin{align}
   \cE_{\mathrm{stab}}\coloneqq   \left\{
               \begin{matrix}
                 (A+B\Khat) \text{ is } (\alphainf,\gambarinf)\text{-strongly stable with }\gambarinf\ldef(1+\gammainf)/2,\\
                 \text{and } \Ahat \text{ is } (\alphaa,\gammaab)\text{-strongly stable with }\gammaab\ldef(1+\gammaa)/2.
\end{matrix}
\right\}
    .\label{eq:stab}
\end{align}
\newcommand{\csys}{c_{\mathrm{sys}}}

\begin{lemma}
\label{lem:id_to_sys}
Suppose that $\sigma^2=\bigohs(1)$ and 
\begin{align*}
\nrm*{\Ahat-A}_{\op}\vee\nrm*{\Bhat-B}_{\op}\vee\nrm*{\Qhat-Q}_{\op}\vee\nrm*{\wh{\Sigma}_w-\Sigma_w}_{\op}\leq\vepsid.
\end{align*}
Then once $\vepsid\leq{}\csys \ldef \mathrm{poly}(\gammast(1-\gammast),\alphast^{-1},\Psist^{-1})=\bigohch(1)$, we have that $\cE_{\sys}$ holds for 
\begin{align*}
\vepssys\leq{}\bigohs(\vepsid).
\end{align*}
\end{lemma}
\subsubsection{Conditioning for the $\cM$ Matrix}
	\pref{ass:m_matrix} is central to the results in this section. In particular, we will use the following implication of this assumption.%
	\begin{lemma}
		\label{lem:eiglem}
Let $\cM_{\sigma^2}$ denote the value of the matrix $\cM$ in \eqref{eq:Ck} for noise
parameter $\sigma^{2}$. Then $\nrm*{\cM_{\sigma^{2}}}_{\op}=\bigohs(1)$ whenever $\sigma^{2}=\bigohs(1)$. Moreover, suppose \pref{ass:m_matrix} holds. Then there exists $\bar\sigma=\bigohch(\lambdam)$, such that for all $\sigma^2 \leq \bar{\sigma}^2$, we have 
\begin{equation}
\eigmin^{1/2}(\cM_{\sigma^2}^{\trn}\cM_{\sigma^2})\geq{}\lambda_{\cM}\cdot\sigma^{2} /2>0,\label{eq:sigma_small}
\end{equation}
where $\lambda_\cM$ is as in \pref{ass:m_matrix}.
	\end{lemma}

        Throughout this section, we make the following assumption, which will eventually be justified by the choice of $\sigma$ in \mainalg.
        \begin{assumption}
          \label{ass:sigma_small}
          $\sigma\leq{}1$ is sufficiently small such that \pref{eq:sigma_small} holds.
        \end{assumption}
        In particular, this assumption implies that the matrix $\cM$ in \eqref{eq:Ck} has full row rank.

\subsection{Learning State Decoders for Rounds $t\geq 1$}
\label{sec:learningafterinitial}
We now prove that Phase III successfully learns decoders for $t\geq{}1$ with high probability, up to an error term determined by the auxiliary predictor $\hat{f}_{A,0}$ produced during the separate initial state learning phase; the error of this predictor is handled in the next subsection. For the rest of this subsection, we assume the iteration $t\geq{}0$ of \pref{alg:phase3} is fixed, meaning we already have $\fhat_t$ and our goal is to compute $\fhat_{t+1}$. We introduce the following quantities.
\begin{itemize} 
	\item Let $(\by_{\tau})_{\tau\geq{}0}$ be the observations induced by following the policy $\wtilde \pi_t$ defined in \eqref{eq:tilpolicy}. 
	\item Let $\{(\by_{\tau}^{(i)}, \bx_\tau^{(i)}, \bnu_\tau^{(i)})\}_{i\in [2n]}$ be i.i.d.~copies of $(\by_\tau, \bx_{\tau}, \bnu_\tau)$, where $(\bnu_\tau)$ are the random Gaussian vectors used by the policy $\wtilde \pi_t$. This is simply the data collected by the $t$th iteration of the loop in \pref{alg:phase3}.
        \end{itemize}
        We also adopt the shorthand $n = n_{\onpo}$.

\paragraph{Learning the decoder at a single step.} Let us recall some notation. For round $t$, we already have a state decoder $\hat{f}_t\colon \cY^{t+1}\rightarrow \reals^{\dimx}$ produced by the previous iteration. As the first step, for each $k \in [\dimk]$, with $\scrH_{\onpo}$ as in \eqref{eq:funclass}, \pref{alg:phase3} solves
\begin{align}
\hhat_{t,k} \in \argmin_{h \in \scrH_{\onpo}} \sum_{i=1}^n \left\|\what M_{k}  \left(h(\by_{t+k}^{(i)})-\what A^{k} h(\by_t^{(i)}) - \what A^{k-1}\what B \what K \hat{f}_t(\by^{(i)}_{0:t})\right) - \bnu^{(i)}_{t:t+k-1}\right\|^2\label{eq:gtk}.
\end{align}
Using the solutions of the above regressions for $k\in [\kappa]$, the algorithm constructs the stacked vector
\begin{align}
\what\upphi_t(\by_{0:t+\dimk})&\coloneqq [ \what\phi_{t,1}(\hat{h}_{t,1},\by_{0:t}, \by_{t+1})^\top, \dots,  \what\phi_{t,\dimk}(\hat{h}_{t,\dimk},\by_{0:t}, \by_{t+\dimk})^\top]^\top \in \reals^{(1+\dimk)\dimk/2}, \nn \\
\shortintertext{where}
 \phihat_{t,k}(h, \by_{0:t}, \by_{t+k}) &\coloneqq \what M_{k}  \left(h(\by_{t+k})-\what A^{k}  h(\by_t) - \what A^{k-1}\what B \what K \fhat_t(\by_{0:t})\right), \  k \in[\dimk] \label{eq:phi}.
\end{align}
Finally, the algorithm computes the intermediate estimator $\hhat_t$:
\begin{align}
\hhat_{t} &\in \argmin_{h \in \scrH_{\onpo}} \sum_{i=n+1}^{2n} \left\|\what \calM \left(h(\by^{(i)}_{t+1}) - \what A h(\by^{(i)}_{t}) - \what B \what K \hat{f}_t (\by_{0:t}^{(i)})  \right) - \what\upphi_t(\by_{0:t+\dimk}^{(i)})  \right\|^2. \label{eq:gt}
\end{align}
Our first guarantee for this section shows that the function $\hhat_t$ estimates the system's noise $\bw_t$ up to a linear transformation given by the matrix $\calM$.
\begin{theorem}
	\label{thm:monster}
	Let $t\geq 0$ and $\bclip>0$ be given. For $h\in \scrH_{\onpo}$ and $\hat{f}_t\colon \cY^{t+1}\rightarrow \cX$, let \[\phi_t(h, {\by_{0:t+1}})\coloneqq \calM  ( h(\by_{t+1})-  A h(\by_t) -   B \what K \hat{f}_t(\by_{0:t})).\] If the event $\cE_{\sys}$ holds and $\|\hat{f}_t(\by_{0:t})\|\leq \bclip$ a.s., then for $\hat{h}_{t}$ as in \eqref{eq:gt}, with probability at least $1-\delta$, 
	\begin{gather}
	\E_{\what \pi} \left[  \left\| \phi_t(\hat{h}_t, \by_{0:t+1}) -\calM (\bw_t + B \bnu_t) \right\|^2  \right]  \leq \veps_{\noise}^2(\delta),
	\end{gather}
	where
	\begin{align}
	\veps_{\noise}^2(\delta)  &\lesssim \kappa(1 + \ln(\kappa)) \left(\frac{\cwphi  (\ln|\Fclass| + \dimx^2)\ln^2(n\kappa/\delta)}{n} +  \Lonpo^2 \epsys^2 \left(\dimx \PsiSig +   \devx  \bclip^2\right)\right) \label{eq:epsw}\intertext{with}
	\cwphi &:=  30\kappa \dimu \sigma^2 + 18\alphA^2 \Lonpo^2 \PsiM^2 \left(32 \dimx \PsiSig + 3  \bclip^2   \cdot \devx\right) \label{eq:cwphi_def}.
	\end{align}
      \end{theorem}
For the remainder of the subsection, we let $\delta \in(0,e^{-1}]$ be fixed and define
\begin{align}
\cE_{\noise} \coloneqq \left \{ \E_{\what \pi} \left[\left\| \phi_\tau (\hat{h}_\tau, \by_{0:\tau+1}) -\calM (\bw_\tau + B \bnu_\tau)  \right\|^2  \right]  \leq \veps_{\noise}^2(\delta), \ \text{for all  }  0 \leq \tau  \leq t  \right\}.
\end{align}

\paragraph{From noise estimate to state estimate.} Now that we can estimate the noise at round $t$ using $\hat h_t$, we build a state decoder $\hat{f}_{t+1}$ for round $t+1$ by combining $\hhat_t$ with the decoder $\hat{f}_t$. Recall that $\tilde{f}_0\equiv\hat{f}_0 \equiv 0$, and that \pref{alg:phase3} forms $\fhat_t$ for all $t\geq 1$ via
\begin{align}
\hspace{-0.0cm}\hat{f}_{t+1}(\cdot) \coloneqq  \tilde{f}_{t+1}(\cdot) \Ind\{\|\tilde{f}_{t+1}(\cdot) \| \leq \bclip\} , \  \text{where}  \   \tilde{f}_{t+1}(\by_{0:t+1}) \coloneqq  \hat{h}_t(\by_{t+1}) + \what A\hat{f}_t(\by_{0:t})  -   \what A  \hat{h}_t(\by_{t}),\label{eq:decoder}
\end{align}
where $\bclip>0$ is the clipping parameter. Note that we treat $\bclip$ as a free parameter throughout this section unless explicitly specified. The case $t=0$ needs special care as it requires decoding the initial state; we set
\begin{align}
	\hat{f}_{1}(\cdot) \coloneqq  \tilde{f}_{1}(\cdot) \Ind\{\|\tilde{f}_{1}(\cdot) \| \leq \bclip\}, \ \ \text{where} \ \  \tilde{f}_1(\by_{0:1}) \coloneqq \hat{h}_1(\by_1) + \hat{f}_{A, 0}(\by_0) - \what A \hat{h}_0(\by_0), \label{eq:decod1}
\end{align}
and $\hat{f}_{A,0}(\by_0)$ is the estimator for $A f_\star(\by_0)$ which we will construct in the next subsection.

Our goal now is to prove that the $\fhat_{t+1}$ is good whenever $\hhat_0,\ldots,\hhat_t$ are good. To this end, we first give a guarantee on the unprojected decoder $\tilde{f}_{t+1}$, which shows that it has low prediction error for trajectories in which the event
\begin{gather}
\cEroll_{0:t} \coloneqq  \left\{ 	  \tilde{f}_{\tau}(\by_{0:\tau})= \hat{f}_{\tau}(\by_{0:\tau}), \text{for all }  0 \leq \tau \leq t  \right\} \label{eq:clipevent}
\end{gather}
occurs.

\begin{lemma}
	\label{lem:decoder}
	Let $t\geq 0$ be given. Let $(\tilde{f}_{\tau})_{\tau \in [t+1]}$ and $\cEroll_{0:t}$ be defined as in \eqref{eq:decoder}, and \eqref{eq:clipevent}, respectively. If the events $\cE_{\sys}$ and $\cE_{\noise}$ hold, then for $\veps_{\noise}$ as in \eqref{eq:epsw}, we have
	\begin{align}
          &\E_{\what \pi} \left[ \max_{ 0\leq  \tau \leq t} \|  \tilde{f}_{\tau+1}(\by_{0:\tau+1}) - f_\star(\by_{\tau+1}) \|^2 \cdot \Ind\{\cEroll_{0:t}\}  \right] \leq    \veps_{\dec,t}^2  ,\label{eq:guarantee}
        \end{align}
        where
        \begin{align}
           \veps_{\dec,t}^2 \coloneqq 3 \alpha^2_A  (1-\gamma_A)^{-2} \left(\veps^2_{\sys} \bclip^2  +  \epszero^2 + \sigma_{\min}(\calM)^{-2}  \veps^2_{\noise}t\right),\mathand
          \epszero^2 := \E_{\what \pi} [  \|  \hat{f}_{A,0}(\by_{0}) - A f_\star(\by_{0}) \|^2].\label{eq:epsf} 
	\end{align}
\end{lemma}
For the next theorem, we show that the even $\cEroll_{0:t}$ occurs with overwhelming probability whenever the clipping parameter $\bclip$ is selected appropriately. We need the following definitions. For $t\geq 0$ and $\eta>0$, let \begin{align*}
\bz_t &\coloneqq \sum_{\tau=0}^t (A + B \what K)^{t-\tau} (B\bnu_\tau + \bw_\tau)
\end{align*}
denote the contribution of the process noise and Gaussian inputs to the state $\matx_{t+1}$. The associated covariance of this random variable when $t\to \infty$ is given by 
\begin{align}
\Sigma_{z,\infty}&\coloneqq \sum_{\tau=0}^{\infty} (A +B \what K)^{\tau} (\sigma^2 B B^\top + \Sigma_w)((A +B \what K)^{\tau})^\top. \label{eq:thecovar}
\end{align}
The sum in \eqref{eq:thecovar} converges under the event $\cE_\sys$, since in this case $\|(A  +  B \what K )^t\|_{\op} \leq  \alpha_{\infty}\gambarinf^t$, for all $t\geq 0$, and $\gambarinf<1$; see \pref{eq:stab}. Finally, we consider the following useful event:
\begin{align}
\cEroll_{0:t}' &\coloneqq \left\{\alphainf^2 \|\bx_0\|^2_{2} +\|\bz_\tau\|^2 \leq (\dimu \ \nn \alphainf^2  \|\Sigma_0\|_{\op} + \dimx \|\Sigma_{z,\infty}\|_{\op}) \ln (2\eta), \ \ \text{for all } 0\leq \tau \leq t \right\}. \nonumber 
\end{align}
Lastly, we define the following term which guides how we select the clipping in the definition of $(\hat f_t)$ in \eqref{eq:decoder}: 
\begin{align}
\bclip_{\infty} \coloneqq  \frac{6(1 -  \gamma_\infty)^{-1}\alpha_\infty\Psist \veps_{\dec,t}\sqrt{\eta}+\sqrt{2(\dimu \alphainf^2 \|\Sigma_0\|_{\op} + \dimx \|\Sigma_{z,\infty}\|_{\op}) \ln (2\eta)}}{ 1 -2\alpha_\infty \veps_\sys (1 -\gamma_\infty)^{-1}}, \label{eq:bclip}
\end{align}
where $\eta>e$ is a free parameter.

We now show that if the clipping parameter $\bclip$ in \eqref{eq:decoder} is chosen sufficiently large, then under a given execution of $\pihat$, the clipping operator is never actived (i.e.~$\cEroll_{0:t}$ holds) with high enough probability, provided that the clipping operator is not activated at $t = 1$.
\begin{theorem}
	\label{thm:clippingprob}
	Let $t\geq 0$, $\eta >0$, and $\bclip>0$ be given. Let $\veps_{\dec,t}$, $\Sigma_{z,\infty}$, and $\bclip_\infty$ be defined as in \eqref{eq:epsf}, \eqref{eq:thecovar}, and \eqref{eq:bclip}, respectively. If \Imark the events $\cE_{\sys}$ and $\cE_{\noise}$ hold; \IImark $\veps_{\sys}  < (1 - \gamma_\infty)(2 \alpha_\infty)^{-1}$; and \IIImark $\bclip \geq \bclip_\infty$, then
	\begin{align}
	\P_{\what \pi}[\cEroll_{0:t}\wedge \cEroll_{0:t}' ] \geq \P_{\what \pi}[\tilde{f}_1(\by_{0:1}) = \hat{f}_1(\by_{0:1}) ] - 2(t+1)/\eta.\nn
	\end{align}
\end{theorem}

\paragraph{Concluding the guarantee for the state decoders.} We now put together the preceding results to give the main guarantee for our state decoders $(\hat{f}_t)$ for $t\geq{}1$.
\begin{theorem}
	\label{thm:minusinit}
	Let $T\geq 0$, $\eta >0$, and $\bclip>0$ be given. Under the conditions $\Imark$, $\IImark$, and $\IIImark$ of \pref{thm:clippingprob}, we have
	\begin{align}
	 \E_{\what \pi} \left[ \max_{ 0\leq  t \leq T} \|  \hat{f}_t(\by_{0:t}) - f_\star(\by_{t}) \|^2   \right]   \leq \veps^2_{\dec,t} + (4T^{1/2}\cx+2 \bclip^2)\left(\frac{4T}{\eta}+1-\P_{\what \pi}[\{\tilde{f}_0(\by_0) = \hat{f}_0(\by_0)\} \wedge \cEroll_0' ]\right),\nn 
	 \end{align}
where $\cx\coloneqq 30\dimx\PsiSig  +  2\bclip^2   \cdot \devx.$
\end{theorem}

\subsection{Learning the Initial State}
\label{sec:learningintial}
\pref{thm:minusinit} ensures that the decoders $\fhat_0,\ldots,\fhat_T$ have low error only if the initial error $\epszero^2 := \E_{\what \pi} [  \|  \hat{f}_{A,0}(\by_{0}) - A f_\star(\by_{0}) \|^2]$ is small. In this subsection, we show that the extra initial state learning procedure in \pref{alg:phase3} ensures that this happens with high probability.

Recall that $n_\init \in \mathbb{N}$ denotes the sample size used by \pref{alg:phase3} for learning the initial state. During the initial state learning phase (\pref{line:start} through \pref{line:ned} of \pref{alg:phase3}), the algorithm gathers data by following a policy we denote $\pi_{\ol}$ which plays random noise $(\bnu_{\tau})$, where $\bnu_\tau \sim \cN(0, \sigma^2 I_{\dimu})$, for $\tau\geq 0$ and $\sigma \in(0,1]$.

Let $\hat{h}_{\ol, 0} \coloneqq \hat{h}_{0}$ (recall that the ``ol'' subscript refers to \emph{open loop}), where we recall that $\hat{h}_0$ is computed on \pref{line:hhat_t} of \pref{alg:phase3} prior to the initial state learning phase, using the procedure analyzed \pref{sec:learningafterinitial}. In particular, 
by instantiating the result of \pref{thm:monster} with $\hat{f}_0 \equiv 0$, we get that under the event $\cE_{\sys}$, for any $\delta \in(0,1/e]$, with probability at least $1 - \delta$, 
\begin{gather}
\E_{\pi_{\ol}} \left[ \left\| \hat{h}_{\ol,0}(\by_1) - A \hat{h}_{\ol,0}(\by_0) - B \bnu_0  - \bw_0 \right\|^2   \right] \leq   \sigma_{\min}(\calM)^{-2} \cdot  \varepsilon^2_{\noise}(\delta). \label{eq:baseguarantee}
\end{gather}
We recall that the minimum singular value of $\cM$ is bounded away from zero for all sufficiently small $\sigma>0$ (see \pref{lem:eiglem}). It follows from \eqref{eq:baseguarantee} that $\hat{h}_{\ol,0}(\by_1) - A \hat{h}_{\ol,0}(\by_0) - B \bnu_0$ can be used as an estimator for the noise vector $\bw_0$. Using this estimator, we solve the following regression problem in \pref{line:start+1}:
\begin{align}
\hat{h}_{\ol,1} \in \argmin_{h \in \scrH_{\onpo}} \sum_{i=1}^{\ninit} \left\| h(\by^{(i)}_1)  - \left(\hat{h}_{\ol,0}(\by^{(i)}_1) - \what A \hat{h}_{\ol,0}(\by^{(i)}_0) -  \what B \bnu_0^{(i)}\right)  \right\|^2,\label{eq:g1check}
\end{align}
where $\{(\by_{\tau}^{(i)},\bx_{\tau}^{(i)},\bnu_{\tau}^{(i)})\}_{1 \leq i\leq \ninit}$, are the fresh i.i.d.~trajectories generated by the policy $\pi_{\ol}$ on \pref{line:start}.

We first show that up to a linear transformation, this regression recovers the vector $A\matx_0$ (our target), plus a linear combination $B\bnu_0+\matw_0$ of the system noise and injected noise for $t=0$. This guarantee is quite useful: Since we can already predict $B\bnu_0+\matw_0$ well via \pref{eq:baseguarantee}, we will be able to extract $A\matx_0$ from this representation.
\begin{lemma}
	\label{lem:stateestimate}
	Let $\hat{h}_{\ol,1}$ be defined as in \eqref{eq:g1check}, and let $\Sigma_1 \coloneqq \sigma^2 B B^\top  + A \Sigma_0 A^\top + \Sigma_w$. If $\cE_{\sys}$ holds, then for any $\delta\in(0,1/e]$, with probability at least $1-5\delta/2$, we have
	\begin{align}
	& \E_{\pi_{\ol}} \left[ \| \hat{h}_{\ol,1}(\by_1) -  \Sigma_w \Sigma_1^{-1} ( \bw_0+ B \bnu_0+  A \bx_0 ) \|^2    \right]\leq \veps^2_{\ol,1}, \label{eq:guaran}
	\end{align}
	where we have
\begin{align}
 \veps^2_{\ol,1}  &\lesssim  \frac{ c_1(\dimx^2 +\ln|\Fclass|) \ln^2\tfrac{n_{\init}}{\delta}}{n_{\init}} + \sigma_{\min}(\calM)^{-2} \veps^2_{\noise}(\delta) +  \veps_{\sys}^2\Lonpo^2 (1 + \dimx \|\Sigma_0\|_{\op} + \sigma^2 \dimu),\nn  \intertext{and}
	c_1 &\coloneqq  \Lonpo^2\Psist^{2}( 1 + \dimu \sigma^2 + \dimx (\|\Sigma_1\|_{\op} + \|\Sigma_0\|_{\op})). \label{eq:the_c1_thm}
	\end{align}
      \end{lemma}
      To make use of this lemma, we must invert the linear transformation $\Sigma_w\Sigma_1^{-1}$. In fact, the prediction error guarantee from \pref{lem:stateestimate} implies that we can estimate $\Sigma_\cv \coloneqq \Sigma_w \Sigma_1^{-1} \Sigma_w$ (where $\Sigma_1$ is as in \pref{lem:stateestimate}) by computing
\begin{align}
\what \Sigma_\cv \coloneqq  \frac{1}{n}\sum_{i=\ninit+1}^{2\ninit} \hat{h}_{\ol,1}(\by_1^{(i)})  \hat{h}_{\ol,1}(\by_1^{(i)})^\top,  \label{eq:newestim}
\end{align}
where $\{(\by_{\tau}^{(i)},\bx_{\tau}^{(i)},\bnu_{\tau}^{(i)})\}_{\ninit<i\leq 2\ninit}$, are fresh i.i.d.~trajectories generated by the policy $\pi_{\ol}$. 
To see this, observe that by \eqref{eq:guaran} implies that up to the error $\veps_{\ol,1}$, \eqref{eq:newestim} is an estimator for the covariance matrix of the Gaussian vector $\Sigma_w \Sigma_1^{-1} ( \bw_0+ B \bnu_0+  A \bx_0)$ which is just $\Sigma_w \Sigma_1^{-1} \Sigma_w$. The following lemma gives a guarantee for the estimated covariance $\what \Sigma_\cv$.
\begin{lemma}
\label{lem:covar}
Let $c_{\cv} \coloneqq \Lonpo^2   (1 + (3 \dimx +2)\|\sigma^2 B B^\top + A \Sigma_0 A^\top + \Sigma_w\|_{\op} )$ and
\begin{equation}
\veps'_\cv \coloneqq 3  \veps_{\ol,1}  \sqrt{c_\cv}+ 5 c_{\cv}  {\ln(2 \dimx \ninit/\delta)^{3/2}}{\ninit^{-1/2}}.\label{eq:vepscov}
\end{equation}
Suppose that $\ninit$ is large enough such that
\begin{align}
\veps_\cv' < \sigma_{\min}(\Sigma_\cv)/2, \quad \text{where} \quad  \Sigma_\cv \coloneqq \Sigma_w \Sigma_1^{-1} \Sigma_w \preceq \Sigma_w,
\label{eq:covevent}
\end{align}
and $\Sigma_1$ is as in \pref{lem:stateestimate}. Then under the event $\cE_{\sys}$, with probability at least $1 - (3 \dimk + 4)\delta$, 
\begin{align}
\|I_{\dimx}- \what \Sigma_w \what \Sigma^{-1}_\cv \Sigma_w \Sigma_1^{-1} \|_{\op} \leq  \veps_\cv, 
\quad \|\what \Sigma_\cv \|_{\op} \leq 2 \|\Sigma_\cv\|_{\op}, \mathand \sigmamin(\wh{\Sigma}_{\cv})\geq{}\sigmamin(\Sigma_{\cv})/2\label{eq:thebound}
\end{align}
where
\begin{align}
\veps_{\cv} \coloneqq  2\Psist\prn*{\veps_{\sys} + \frac{2\|\Sigma_w^{-1}\|_{\op} \veps_\cv' }{\sigma_{\min}(\Sigma_w\Sigma_1^{-1} \Sigma_w) }}. \label{eq:cv}
\end{align}
\end{lemma}
\pref{lem:covar} shows that $\wh{\Sigma}_w\wh{\Sigma}_{\cv}^{-1}\approx(\Sigma_w\Sigma_1^{-1})^{-1}$, which is exactly what we require to invert the linear transformation in \pref{eq:guaran}. To finish up, we solve the regression problem (\pref{line:start+3})
\begin{align}
\tilde{h}_{\ol,0} \in \argmin_{h \in \scrH_{\onpo}} \sum_{i=\ninit+1}^{2\ninit} \left\|h(\by_0^{(i)})  -    \hat{h}_{\ol,1}(\by^{(i)}_1) \right\|^2\label{eq:g0tilde}.
\end{align}
Note that the argument to $h$ in \pref{eq:g0tilde} is $\maty_0$, while the argument to $\hhat_{\ol,1}$ is $\maty_1$, so that the Bayes predictor, by \pref{eq:condexp}, is equal to $h(\maty_0)=\Sigma_w\Sigma_1^{-1}A\matx_0$. Motivated by this observation, the final step is to set
\[
\hat{f}_{A,0}(\by_0)= \what \Sigma_w \what \Sigma_\cv^{-1} \tilde{h}_{\ol,0} (\by_0).
\]
Our main theorem for this subsection gives the desired prediction error guarantee for this predictor.
\begin{theorem}
	\label{thm:stateestimate2}
	Let $\tilde{h}_{\ol,0}$ be as in \pref{eq:g0tilde}, and set $\hat{f}_{A,0}(\by_0)\coloneqq \what \Sigma_w \what \Sigma^{-1}_\cv \tilde{h}_{\ol,0} (\by_0)$. If $\cE_{\sys}$ holds and \pref{eq:covevent} is satisfied, then for any $\delta\in(0,1/e]$, with probability at least $1 - (3\kappa+9)\delta$, the following properties hold
	\begin{enumerate}
	\item The estimator $\hat{f}_{A,0}$ satisfies
	\begin{align}
          &\E_{\what \pi} \left[ \left\| \hat{f}_{A,0} (\by_0)-  A f_\star(\by_0) \right\|^2  \right] \approxleq  \veps^2_{\init}, \label{eq:firstb}%
        \end{align}
        where
        \begin{align}
          \veps^2_{\init} \ldef
\| \Sigma^{-1}_w\|^2_{\op} \|\Sigma_\cv\|^2_{\op} \left(\frac{ c_0(\dimx^2 +\ln|\Fclass|) \ln\prn{\tfrac{n_{\init}}{\delta}}^{2}}{n_{\init}} + \veps^2_{\ol,1}\right)  + \dimx \veps^2_\cv \|A\|_\op^2 \|\Sigma_0\|,\nn
	\end{align}
        with $\veps_{\ol,1}$ as in \pref{lem:stateestimate}, $\veps_\cv$ as in \pref{eq:cv}, and $c_0\coloneqq  32\Lonpo^{2}\Psist^{3}\dimx$.
	\item Let $\eta>e$ be given, and let $\tilde{f}_1$ and $\hat{f}_1$ be defined as in \pref{line:clip} of \pref{alg:phase3}. If
	\begin{align}
          \bclip^2 \ge \bclip_0^2 \ln(2\eta),\quad \text{where}\quad
          \bclip_0^{2} \ldef 10^{4}\dimx\Lonpo^{2}\Psist^{12}(1+\nrm{\Sigma_0}_{\op}+\nrm{\Sigma_1}_{\op})\label{eq:bclip0}.
	\end{align}
	then it holds that
\begin{align}
\P_{\what \pi} [\tilde{f}_1(\by_{0:1}) = \hat{f}_1(\by_{0:1}) ] \geq  1 - \eta^{-1}. \label{eq:secondb}
	\end{align}
\end{enumerate}
\end{theorem}

\subsection{Master Theorem for Phase III}
\label{sec:phaseiii_master}
By combining \pref{thm:minusinit,thm:stateestimate2}, we derive the proof of \pref{thm:phase_three}.

    \begin{proof}[\pfref{thm:phase_three}]
      Let $\eta>e$ and $\sigma^{2}\leq{}1$ be fixed. Introduce the shorthand $\lambda=\eigmin(\cM^{\trn}\cM)$, and recall from \pref{lem:eiglem} that $\lambda=\bigohs(1)$ whenever $\sigma^{2}=\leq1$. Lastly, let us set $\ninit=\nref$.

      Let us begin with some initial parameter choices. First, we set $\ninit=\nref$. Following \pref{lem:id_to_sys}, we assume that $\vepsid=\bigohch(1)$ is sufficiently small such that $\cE_{\sys}$ holds and $\vepssys\leq \frac{1-\gammainf}{8\alphainf}\leq\bigohs(\vepsid)=\bigohs(1)$. Next, following \pref{lem:eiglem}, we assume that $\sigma\leq{}1$ is chosen such that $\sigma=\bigohch(\lambdam)$ and $\lambda=\eigmin(\cM^{\trn}\cM)\geq\lambda_{\cM}^{2}\sigma^{4}/4$.

      Since $\sigma^2=\bigohs(1)$, we observe from  \eqref{eq:bclip} that 
      \[
        \bclip_{\infty}^2\leq{}\bigohs(\veps_{\mathrm{dec},T}^2\eta + (\dimu+\dimx)\log(\eta)),
      \]
      and from \eqref{eq:bclip0} we have
      \[
        \bclip_0^2\log(2\eta) \leq{} \bigohs(\dimx\log(\eta)).
      \]
      Let us assume for now that $\bclip^{2}=\Omega_{\star}(\dimx+\dimu)$; we will specify a precise choice at the end. Note that choosing $\bclip\geq\bclip_{\infty}\vee\bclip_0\log(2\eta)$ is non-trivial, since our bound on $\bclip_{\infty}$ depends on $\veps_{\dec,T}$, which itself depends on $\bclip$. Nonetheless, we will show that an appropriate choice of $\bclip$ solves this recurrence.

      Let $\delta\leq{}1/e$ be given. Define $\logterms=(\dimx^{2}+\log\abs{\Fclass})\log^{2}(\nref/\delta)\vee\log^{3}(\nref/\delta)$. As a first step, we simplify the various parameters defined in this section using the $\bigohs(\cdot)$ notation. In particular, we have
      \newcommand{\logs}{\logterms}
      \begin{align*}
        &\cwphi= \bigohs(\kappa\dimu+\bclip^{2})=\bigohs(\kappa\bclip^{2}),\\
        & 	\veps_{\noise}^2(\delta) = \bigohs\prn*{
          \kappa\prn*{\frac{\cwphi\logterms}{\nop} + \vepssys(\dimx+\bclip^2)
          }} =  \bigohs\prn*{\frac{\kappa^2\bclip^2\logterms}{\nref}  + \dimk \bclip^2 \veps_{\sys}^2
              },\\
        &\cx = \bigohs(\dimx + \bclip^{2})=\bigohs(\bclip^{2}),\\
        &c_1=\bigohs(\dimx+\dimu),\\
        &\veps_{\ol,1}^{2} = \bigohs\prn*{
          \frac{c_1\logs}{\nref} + \lambda^{-1}\veps_{\noise}^{2}(\delta) + \vepssys^{2}(\dimx+\dimu)
          } = \bigohs\prn*{\lambda^{-1}\cdot\prn*{\frac{\kappa^2\bclip^2\logterms}{\nref}  + \dimk \bclip^2 \veps_{\sys}^2}
          },\\
        &c_{\cv} = \bigohs(\dimx),\\
        &(\veps'_\cv)^{2} = \bigohs\prn*{\veps_{\ol,1}^2 \dimx+ \frac{\dimx^2\ln(\nref/\delta)^{3}}{\nref}
                           }
          =\bigohs\prn*{\lambda^{-1}\cdot\prn*{\frac{\kappa^2\dimx\bclip^2\logterms}{\nref}  + \dimk\dimx \veps_{\sys}^2\bclip^2}},\\
        &\veps_{\cv}^{2} = \bigohs(\vepssys^{2} + (\veps_{\cv}')^2).
      \end{align*}
      We now appeal to \pref{thm:stateestimate2}. Simplifying the upper bounds, we are guaranteed that with probability at least $1-\bigoh(\kappa\delta)$, we have
      \[
        \P_{\what \pi} [\tilde{f}_1(\by_{0:1}) = \hat{f}_1(\by_{0:1}) ] \geq  1 - \eta^{-1}
      \]
      and 
      \[
        \E_{\what \pi} \left[ \left\| \hat{f}_{A,0} (\by_0)-  A f_\star(\by_0) \right\|^2  \right] \approxleq  \veps^2_{\init}
      \]
      where
      \[
        \veps_{\init}^{2}=\bigohs\prn*{\frac{\dimx\logterms}{\nref} + \veps_{\ol,1}^{2} + \dimx\veps_{\cv}^{2}
          } =\bigohs\prn*{\lambda^{-1}\cdot\prn*{\frac{\kappa^2\dimx^2\bclip^2\logterms}{\nref}  + \dimk\dimx^2 \veps_{\sys}^2\bclip^2}}.
        \]
        We now appeal to \pref{thm:monster} and \pref{thm:minusinit}. By the union bound, and in light of \pref{eq:secondb}, we have that with probability at least $1-\bigoh(\kappa{}T\delta)$,
      \begin{align}
	\E_{\what \pi} \brk*{\max_{ 1\leq  t \leq T} \|  \fhat_t(\by_{0:t}) - \fstar(\by_{t}) \|^2}
         \leq{} \bigohs\prn*{
          \veps_{\dec,T}^{2} + (T^{1/2}\cx+\bclip^{2})T/\eta
          }
      \end{align}
      where
      \begin{align*}
                \veps_{\dec,T}^2 &=   \bigohs\prn*{\vepsinit^{2}+ \lambda^{-1}T\veps^2_{\noise}  +  \veps^2_{\sys} \bclip^2
}=\bigohs\prn*{\lambda^{-1}T\cdot\prn*{\frac{\kappa^2\dimx^2\bclip^2\logterms}{\nref}  + \dimk\dimx^2 \veps_{\sys}^2\bclip^2}}.
      \end{align*}
      Hence, we can simplify to
      \begin{align}
        &\E_{\what \pi} \brk*{\max_{ 1\leq  t \leq T} \|  \fhat_t(\by_{0:t}) - \fstar(\by_{t}) \|^2}\notag\\
        &\leq \bigohs\prn*{
          T^{3/2}\bclip^2\eta^{-1}
+ \lambda^{-1}T\cdot\prn*{\frac{\kappa^2\dimx^2\bclip^2\logterms}{\nref}  + \dimk\dimx^2 \veps_{\sys}^2\bclip^2}
          }\notag\\
        &\leq \bigohs\prn*{
          T^{3/2}\bclip^2\eta^{-1}
+ \veps_0^{2} + \lambda^{-1}T\dimk\dimx^2 \veps_{\sys}^2\bclip^2
          },
          \label{eq:thm4_step1}
      \end{align}
      where $\veps_0^2\ldef \lambda^{-1}T\cdot\frac{\kappa^2\dimx^2\bclip^2\logterms}{\nref}$.         

      It remains to choose $\eta$ and ensure that the condition on $\bclip$ is satisfied. We choose $\eta=\veps_0^{-2}$. Since $\veps_{\dec,T}^{2}\leq\bigohs(\veps_0^2 + \lambda^{-1}T\dimk\dimx^2\vepssys^2\bclip^2)$, this implies 
        \begin{align*}
          \bclip_{0}^2\log(2\eta)\vee\bclip_{\infty}^2 &= \bigthetas(\veps_{\dec,T}^{2}\eta + (\dimu+\dimx)\log(\eta))\\
                                            &= \bigohs(1+ \veps_0^{-2}\cdot\lambda^{-1}T\dimk\dimx^2\vepssys^2\bclip^2 + (\dimu+\dimx)\log(\veps_0^{-2})).
        \end{align*}
        It follows that if $\vepssys^2\leq{}\bigohch(\frac{\veps_0^2\lambda}{\dimk\dimx^2\bclip^2T})$, we have
        \[
          \bclip_{0}^2\log(2\eta)\vee\bclip_{\infty}^2\leq{} \bigohs((\dimu+\dimx)\log(\veps_0^{-2}))=\bigohs((\dimu+\dimx)\log(\nref)).
        \]
        Hence, we can satisfy the constraint that $\bclip_{0}\vee\bclip_{\infty}\leq\bclip$ by choosing $\bclip=\Theta_{\star}((\dimu+\dimx)\log(\nref))$. Returning to the final error bound, we have
        \begin{align*}
          &\E_{\what \pi} \brk*{\max_{ 1\leq  t \leq T} \|  \fhat_t(\by_{0:t}) - \fstar(\by_{t}) \|^2}\\
        &\leq \bigohs\prn*{
          \bclip^2T^{3/2}\eta^{-1}
+ \veps_0^{2} + \lambda^{-1}T\dimk\dimx^2 \veps_{\sys}^2\bclip^2
          }\\
          &= \bigohs\prn*{
            \bclip^{2}T^{3/2}\veps_0^{2} + \lambda^{-1}T\dimk\dimx^2 \veps_{\sys}^2\bclip^2
            }\\
          &= \bigohs\prn*{
            \lambda^{-1}T^3\kappa^2(\dimx+\dimu)^4\log^{2}(\nref)\cdot\prn*{\frac{\logs}{\nref}+\vepsid^2}
            }.
        \end{align*}
        To simplify, we recall that (1) $\lambda^{-1}\leq4\lambdam^{-2}\sigma^{-4}$, and (2) $\logs=\bigohs((\dimx^2+\log\abs{\Fclass})\log^{3}(\nref/\delta))$. Moreover, our condition on $\vepssys$ above implies that
        \[
\vepssys^2\leq{}\bigohch(\logs/\nref),
\]
which means it suffices to take $\vepsid^{2}=\bigohch(\vepssys^2)=\bigohch(\logs/\nref)$ as well.
Hence, for the final bound, we can simplify to
\[
  \bigohs\prn*{
\frac{\lambdam^{-2}}{\sigma^{4}}T^3\kappa^2(\dimx+\dimu)^4\cdot\frac{(\dimx^2+\log\abs{\Fclass}) \log^{5}(\nref/\delta)}{\nref}
            }.
        \]

   \end{proof}

\newcommand{\upphist}{\upphi^{\star}}
\newcommand{\upphihat}{\what\upphi}

\section{Supporting Proofs for \pref{sec:phase3_proofs}}
\label{sec:proofs}

\subsection{A Truncation Bound for the Iterates}
Before proceeding with the main proofs in this section, we state a
lemma which bounds the magnitudes of the states under the
event that $\fhat_{\tau}$ returns state estimates bounded by
$\bclip$. This bound is used by a number of subsequent proofs.
	\begin{lemma}
		\label{lem:xconcentration}
	Let $\bclip>0$ and $t\geq 0$. If $\|\hat{f}_\tau
        (\by_{0:\tau})\| \leq \bclip$ a.s. for all $ \tau \geq 0$, then for all $\delta \in (0,1/e]$ and all $ \tau \geq  0$, we have that
	\begin{align}
	&\Pr_{\pitil_t} \left[ \|\bx_\tau \|^2 \geq  30\dimx\PsiSig  +  2\bclip^2   \cdot \devx  \ln \frac{2}{\delta}    \right] \leq \delta, \quad \text{and}
	 \label{eq:theprobpart}\\
	&\E_{\pitil_t} \left[ \|\bx_\tau\|^2  \right] \leq 3 \dimx \PsiSig +   \devx  \bclip^2.  \label{eq:theexppart}
	\end{align} 
	Moreover, both displays above also hold with $\pitil_t$ replaced by $\pihat_t$.
		\end{lemma}
	\begin{proof}
	Let $\tau \geq 0$ be fixed. By the system's dynamics and the definition
        of $\wtilde\pi_t$ (cf. \pref{eq:tilpolicy}), we have
		\begin{align}
		\bx_{\tau} = A^{\tau} \bx_0 + \sum_{s=0}^{\tau-1}   A^{\tau-s-1}( B \what K \hat{f}_s(\by_{0:s}) \Ind_{s\leq t} + B\bnu_s   + \bw_{s}). \label{eq:0rolleddynamics}
		\end{align}
	Thus, by Jensen's inequality, and using strong stability of $A$,
		\begin{align}
		\| \bx_{\tau}\|^2 &\leq  3 \|A^\tau \bx_0\|^2 + 3\left\|  \sum_{s=0}^{\tau-1}   A^{\tau -s-1} ( B\bnu_s   + \bw_{s})  \right\|^2 + 3 \left\| \sum_{s=0}^{t\wedge ( \tau -1)}   A^{\tau -s-1}  B \what K \hat{f}_s(\by_{0:s})\right\|^2, \nn \\
		& \leq 3 \alpha_A^2 \|\bx_0\|^2  +  3 \left\|   \bxi_\tau \right\|^2 +  \frac{3  \alpha_A^2}{1 -\gamma^2_A} \cdot \|B \what K\|^2_{\op}\cdot  \bclip^2 , \label{eq:xbound}
		\end{align}
		where $\bxi_\tau \coloneqq \sum_{s=0}^{\tau-1}
                A^{\tau -s-1} ( B\bnu_s   + \bw_{s}) \sim \cN(0,
                \Sigma_\xi)$, with $\Sigma_\xi \preceq  \Sigma_\idinf$
                since $\sigma \leq 1$ under \pref{ass:sigma_small} (cf. \pref{eq:sigstid}). By
                \pref{lem:thechis}, the expression above implies that
		\begin{gather}
		\P_{\wtilde\pi_t}\left[ \|\bx_{\tau}\|^2 \geq \left( (\alpha_A^2 \|\Sigma_0\|_\op +\| \Sigma_\idinf\|_{\op}   )(9 \dimx +6)  +   \frac{3  \alpha_A^2}{1 -\gamma^2_A} \cdot \|B \what K\|^2_{\op}  \bclip^2 \right) \log(2/\delta) \right]\leq \delta. \label{eq:probb}
              \end{gather}
              which we simplify to
              \begin{align*}
                		\P_{\wtilde\pi_t}\left[ \|\bx_{\tau}\|^2 \geq \left( 2(\alpha_A^2 \|\Sigma_0\|_\op +\| \Sigma_\idinf\|_{\op}   )(9 \dimx +6)  +   \frac{3  \alpha_A^2}{1 -\gamma^2_A} \cdot \|B \what K\|^2_{\op}  \bclip^2 \right) \log(1/\delta) \right]\leq \delta.
              \end{align*}
		Substituting in the definition of $\devx$ and $\PsiSig$ (\Cref{eq:PsiSig_def,eq:devx_def}), with $\gamma^2_A \le 1$ establishes \eqref{eq:theprobpart}. We now show \pref{eq:theexppart}. by \eqref{eq:xbound} and \pref{lem:theexpchi}, we have 
		\begin{align}
		\E_{\pitil_t} \left[ \|\bx_\tau\|^2  \right] &\leq 2 \alpha_A^2 \E_{\pitil_t}[\|\bx_0\|^2]+ 3	\E_{\wtilde\pi_t} [ \left\|   \bxi_\tau \right\|^2] +   \frac{3  \alpha_A^2}{1 -\gamma^2_A} \cdot \|B \what K\|^2_{\op}  \bclip^2, \nn \\
		& 	\leq 3 \dimx(\|\Sigma_0\|_\op + \| \Sigma_\idinf\|_{\op} ) +   \frac{3  \alpha_A^2}{1 -\gamma^2_A} \cdot \|B \what K\|^2_{\op}  \bclip^2.\label{eq:expb}
		\end{align}
		The second part of the lemma follows from \eqref{eq:probb} and \eqref{eq:expb} by the fact that $\wtilde\pi_t$ and $\what \pi$ coincide up to round $t$ (inclusive).
		\end{proof}

\subsection{Proof of \pref{lem:id_to_sys}}

  \newcommand{\Chat}{\wh{\cC}}

\begin{proof}[\pfref{lem:id_to_sys}]
  The bounds on $\nrm{\Ahat-A}_{\op}$ and $\nrm{\Bhat-B}_{\op}$ immediately follow from the conditions of the lemma. To show that $\Ahat$ is $(\alphaa,\gammaab)$-strongly stable, we observe that if $S$ is the matrix that witnesses strong stability for $A$, we have
\begin{align*}
    \nrm{S^{-1}\Ahat{}S}_{\op}
    \leq{}     \nrm{S^{-1}(\Ahat-A)S}_{\op} +     \nrm{S^{-1}AS}_{\op}
  \leq{}     \alphaa\vepsid + \gammaa.
\end{align*}
Hence, once $\vepsid\leq{}\frac{1-\gammaa}{2\alphaa}$, we have $    \nrm{S^{-1}\Ahat{}S}_{\op}\leq{}\gammaab$.

  Next, we appeal to \pref{thm:perturbation_bound}, which implies that
  once
  $\vepsid\leq{}c\cdot\alphainf^{-4}(1-\gammainf^{2})^{2}\Psist^{-11}$
  for a sufficiently small numerical constant $c$, we have
\begin{align*}
    \nrm{\Khat-\Kinf}_{\op} \leq{} \bigohs(\vepsid),
\end{align*}
  and $A+B\Khat$ is $(\alpha_{\infty},\gammab)$-strongly stable for $\gammab=(1+\gammainf)/2$. In particular, for $\vepsid$ sufficiently small we have $\nrm{\Khat}_{\op}\leq2\nrm{\Kinf}_{\op}$, so that
  \[
    \nrm{\what B \what K - B \what K}_{\op}\leq{}\bigohs(\vepsid).
  \]
  Next, we observe that once $\vepsid\leq{}\Psistar^{-1}/2\leq\eigmin(\Sigmaw)/2$, we have $\eigmin(\wh{\Sigma}_w)\geq{}\eigmin(\Sigmaw)/2$, and so we can apply \pref{prop:matrix_inverse_error} to deduce that
  \[
    \|I_{\dimx}- \what \Sigma_w \Sigma^{-1}_w \|_{\op}\vee  \| \what \Sigma_w^{-1} - \Sigma^{-1}_w \|_{\op}\leq{}\bigohs(\vepsid).
  \]
  Finally, we bound the errors for the terms involving $\wh{M}_k$ and $\wh{\cM}$. We first show that to do this, it suffices to bound $\max_{1\leq{}k\leq\kappa}\nrm{\wh{M}_k-M_k}_{\op}$. First, as long as $\vepsid=\bigohs(1)$, we have
  \[
    \nrm{\what M_k \what A^k\what B-  M_k A^k B}_{\op}
    \leq{} \bigohs(\nrm{\what M_k \what A^k-  M_k A^k}_{\op} + \vepsid),
  \]
  by triangle inequality. Next, we have
  \begin{align*}
    \nrm{\what M_k \what A^k-  M_k A^k}_{\op}
&\leq{}     \nrm{\Ahat^{k}}_{\op}\nrm{\Mhat_k - M_k}_{\op}
                                                + \nrm{M_k}_{\op}\nrm{\Ahat^{k}-A^k}_{\op}\\
    &\leq{}     \nrm{\Ahat^{k}}_{\op}\nrm{\Mhat_k - M_k}_{\op}
      + \bigohs\prn*{\nrm{\Ahat^{k}-A^k}_{\op}}.
  \end{align*}
  By \pref{lem:Ahat_power_error}, once $\vepsid\leq{}\frac{(1-\gammaa)}{2\alphaa}$, this is upper bounded by
  \begin{align*}
    \bigohs(\gammaab^{k-1}k\cdot(\nrm{\Mhat_k - M_k}_{\op}+\vepsid)).
  \end{align*}
  Note that $\max_{k\geq{}1}\gammaab^{k-1}k\leq{}\sum_{k=1}^{\infty}\gammaab^{k-1}k\leq{}1/(1-\gammaab)^{2}=\bigohs(1)$, so the bound bove further simplifies to
  \[
    \bigohs(\nrm{\Mhat_k - M_k}_{\op}+\vepsid).
  \]
  Finally, by similar reasoning, we have
  \begin{align*}
    \nrm{\wh{\cM}-\cM}_{\op}
    &\leq{} \sum_{k=1}^{\kappa}\nrm{\Mhat_k\Ahat^{k}-M_kA^k}_{\op}\\
    &\leq{} \sum_{k=1}^{\kappa}\bigoh\prn*{\gammaab^{k-1}k\prn*{\nrm{\Mhat_k-M_k}_{\op}+\vepsid}}\\
    &\leq{} \bigoh\prn*{\max_{1\leq{}k\leq\kappa}\nrm{\Mhat_k-M_k}_{\op}+\vepsid}\cdot\sum_{k=1}^{\infty}\gammaab^{k-1}k\\
    &\leq{} \bigohs\prn*{\max_{1\leq{}k\leq\kappa}\nrm{\Mhat_k-M_k}_{\op}+\vepsid}.
  \end{align*}
  Finally, we appeal to \pref{lem:mk_error_bound}, which implies that $\max_{1\leq{}k\leq\kappa}\nrm{\Mhat_k-M_k}_{\op}=\bigohs(\vepsid)$.
\end{proof}

\subsubsection{Supporting Results}

\begin{proposition}
  \label{prop:matrix_inverse_error}
  Let $X,Y\in\bbR^{d\times{}d}$ be positive definite matrices with $\nrm*{X-Y}_{\op}\leq\veps$. Then we have $\nrm{I-XY^{-1}}_{\op}\leq{}\nrm{Y^{-1}}\cdot\veps$ and $\nrm{X^{-1}-Y^{-1}}_{\op}\leq{}\nrm{X^{-1}}_{\op}\nrm{Y^{-1}}_{\op}\cdot\veps$.
  
\end{proposition}
\begin{proof}[\pfref{prop:matrix_inverse_error}]
  The result follows by the inequalities
  \[
    \nrm{X^{-1}-Y^{-1}}_{\op}\leq\nrm{X^{-1}}_{\op}\cdot\nrm{I-XY^{-1}}_{\op},
  \]
  and
  \[
    \nrm{I-XY^{-1}}\leq{}\nrm{Y^{-1}}_{\op}\cdot\nrm{Y-X}_{\op}.
  \]
\end{proof}

\begin{lemma}
  \label{lem:Ahat_power_error}
  Suppose $\nrm{\Ahat-A}_{\op}\leq{}\frac{(1-\gammaa)}{2\alphaa}$. Then for all $k\geq{}1$,
  \[
    \nrm{\Ahat^{k}-A^{k}}_{\op}\leq{}\alphaa^2\gammaab^{k-1}k\nrm{\Ahat-A}_{\op},
  \]
  where $\gammaab=(1+\gammaa)/2$. Furthermore, we have $    \nrm{\Ahat^{k}}_{\op}\leq{}2\alphaa\gammaab^{k-1}k$.
\end{lemma}
\begin{proof}[\pfref{lem:Ahat_power_error}]
  Using Lemma 5 of \citet{mania2019certainty}, we are guaranteed that\footnote{In the notation of \citet{mania2019certainty}, we can take $\rho\leq\gammaa$ and $\tau(A,\gammaa)\leq\alphaa$.}
  \[
\nrm{\Ahat^{k}-A^{k}}_{\op}\leq{}\alphaa^2\prn*{\alphaa\nrm{\Ahat-A}_{\op}+\gammaa}^{k-1}k\nrm{\Ahat-A}_{\op}.
  \]
  The condition in the lemma statement ensures that $\alphaa\nrm{\Ahat-A}_{\op}+\gammaa\leq\gammaab$, leading to the first result. As a consequence, we also have
  \begin{align*}
    \nrm{\Ahat^{k}}_{\op}&\leq{}\nrm{A^k}_{\op}+\nrm{\Ahat^{k}-A^{k}}_{\op}\\
                          &\leq{}\alphaa\gammaa^{k}+\alphaa^2\gammaab^{k-1}k\nrm{\Ahat-A}_{\op}\\
                          &\leq{}\alphaa\gammaa^{k}+\alphaa(1-\gammaab)\gammaab^{k-1}k\\
    &\leq{}2\alphaa\gammaab^{k-1}k.
  \end{align*}
\end{proof}

\begin{lemma}
  \label{lem:mk_error_bound}
If $\vepsid\leq{}\frac{(1-\gammaa)}{2\alphaa}\wedge\frac{\Psist^{-1}}{2}$ and $\sigma^{2}=\bigohs(1)$, then for all $1\leq{}k\leq\kappa$, $\nrm{\Mhat_k-M_k}_{\op}\leq{}\bigohs(\vepsid)$.
\end{lemma}
\begin{proof}[\pfref{lem:mk_error_bound}]
  Let $k$ be fixed. Define
  \[
        \Sigma_k = \sum_{i=1}^{k}A^{i-1}\Sigma_w(A^{\trn})^{i-1},\mathand         \wh{\Sigma}_k = \sum_{i=1}^{k}\wh{A}^{i-1}\wh{\Sigma}_w(\Ahat^{\trn})^{i-1},
      \]
      so that we have
      \[
        M_k=\cC_k^{\trn}(\cC_k\cC_k^{\trn}+\sigma^{-2}\Sigma_k)^{-1},\mathand
        \wh{M}_k=\wh{\cC}_k^{\trn}(\wh{\cC}_k\wh{\cC}_k^{\trn}+\sigma^{-2}\wh{\Sigma}_k)^{-1},
      \]
      where $\wh{\cC}_k\ldef{}\brk{\Ahat^{k-1}\Bhat\mid\cdots\mid{}\Bhat}$.

      As a starting point, we have by \pref{lem:ck_error_bound} that $\nrm{\cC_k-\wh{\cC}_k}_{\op}\leq\bigohs(\vepsid)$ once $\vepsid\leq{}\frac{(1-\gammaa)}{2\alphaa}$. As such our task will mainly boil down to relating the error of $\wh{M}_k$ to that of $\wh{\cC}_k$. We will use going forward that $\nrm{\cC_k}_{\op}\vee\nrm{\Chat_k}_{\op}=\bigohs(1)$.

      For the first step, by the triangle inequality we have
      \begin{align*}
        &\nrm{\Mhat_k-M_k}_{\op} \\&\leq{} \nrm{\cC_k-\Chat_k}_{\op}\nrm[\big]{(\cC_k\cC_k^{\trn}+\sigma^{-2}\Sigma_k)^{-1}}_{\op}
        + \nrm{\Chat_k}_{\op}\nrm[\big]{(\cC_k\cC_k^{\trn}+\sigma^{-2}\Sigma_k)^{-1}-(\wh{\cC}_k\wh{\cC}_k^{\trn}+\sigma^{-2}\wh{\Sigma}_k)^{-1}}_{\op}.
      \end{align*}
      Now, note that $\cC_k\cC_k^{\trn}+\sigma^{-2}\Sigma_k\psdgeq\sigma^{-2}\Sigw$, so $\nrm{(\cC_k\cC_k^{\trn}+\sigma^{-2}\Sigma_k)^{-1}}_{\op}=\bigohs(\sigma^2)$. Similarly, as long as $\vepsid\leq{}\Psistar^{-1}/2\leq\eigmin(\Sigmaw)/2$, we have $\eigmin(\Sigwhat)\geq{}\eigmin(\Sigmaw)/2>0$, so we have $\nrm{(\wh{\cC}_k\wh{\cC}_k^{\trn}+\sigma^{-2}\wh{\Sigma}_k)^{-1}}_{\op}=\bigohs(\sigma^2)$. This leads allows us to simplify the bound above to
            \begin{align*}
              \nrm{\Mhat_k-M_k}_{\op} &\leq{} \bigohs(\sigma^{2}\vepsid)
                                         + \bigohs\prn*{\nrm[\big]{(\cC_k\cC_k^{\trn}+\sigma^{-2}\Sigma_k)^{-1}-(\wh{\cC}_k\wh{\cC}_k^{\trn}+\sigma^{-2}\wh{\Sigma}_k)^{-1}}_{\op}},
                                         \intertext{and moreover, by invoking \pref{prop:matrix_inverse_error} with the aforementioned operator norm bounds for the inverse matrices, we can further upper bound by}
              &\leq{} \bigohs(\sigma^{2}\vepsid)
                + \bigohs\prn*{\sigma^{4}\nrm[\big]{(\cC_k\cC_k^{\trn}+\sigma^{-2}\Sigma_k)-(\wh{\cC}_k\wh{\cC}_k^{\trn}+\sigma^{-2}\wh{\Sigma}_k)}_{\op}} \\
                                       &\leq{} \bigohs(\sigma^{2}\vepsid)
                                         + \bigohs\prn*{\sigma^{4}\nrm{\cC_k\cC_k^{\trn} - \wh{\cC}_k\wh{\cC}_k^{\trn}}_{\op}+\sigma^{2}\nrm{\Sigma_k - \wh{\Sigma}_k}_{\op}},\\
              &\leq{} \bigohs(\vepsid + \nrm{\Sigma_k - \wh{\Sigma}_k}_{\op}),
            \end{align*}
            where the final step uses that $\sigma^{2}=\bigohs(1)$ to simplify. Finally, we bound
            \begin{align*}
              \nrm{\Sigma_k -\wh{\Sigma}_k}_{\op} & \leq{} \sum_{i=1}^{k}\nrm{A^{i-1}-\wh{A}^{i-1}}_{\op}\nrm{\Sigmaw(A^{\trn})^{i-1}}_{\op}
                + \nrm{\wh{A}^{i-1}}_{\op}\nrm{\Sigmaw-\Sigwhat}_{\op}\nrm{(A^{\trn})^{i-1}}_{\op} \\ & \qquad 
                + \nrm{\Ahat^{i-1}\Sigwhat}_{\op}\nrm{A^{i-1}-\Ahat^{i-1}}_{\op}.
            \end{align*}
            By \pref{lem:Ahat_power_error}, once $\vepsid\leq{}\frac{(1-\gammaa)}{2\alphaa}$, we have $\nrm{\Ahat^{i}}_{\op}\leq\bigohs(\gammaab^{i-1}i)$ and $\nrm{\Ahat^{i}-A^{i}}_{\op}=\bigohs(\gammaab^{i-1}i\vepsid)$. We also have $\nrm{A^{i}}_{\op}=\bigohs(\gammab^{i})$ and $\nrm{\Sigw}_{\op}\vee\nrm{\Sigwhat}_{\op}=\bigohs(1)$, so we can bound the sum above as
            \begin{align*}
                            &\nrm{\Sigma_k -\wh{\Sigma}_k}_{\op} \leq{} \bigohs\prn*{\vepsid\cdot\prn*{1+\sum_{i=2}^{k}
                              \gammaab^{2(i-2)}i^{2}
                              }} \leq{} \bigohs\prn*{\vepsid\cdot\prn*{1+\sum_{i=2}^{\infty}
                              \gammaab^{2(i-2)}i^{2}
                              }}= \bigohs(\vepsid).
            \end{align*}

    \end{proof}

    \begin{lemma}
  \label{lem:ck_error_bound}
If $\vepsid\leq{}\frac{(1-\gammaa)}{2\alphaa}$ then for all $1\leq{}k\leq\kappa$, $\nrm{\Chat_k-\cC_k}_{\op}\leq{}\bigohs(\vepsid)$.
\end{lemma}
\begin{proof}[\pfref{lem:ck_error_bound}]
  Let $k$ be fixed. As a first step, we use the block structure to bound
  \begin{align*}
    \nrm{\Chat_k-\cC_k}_{\op}
    &\leq{} \sum_{i=1}^{k}\nrm{\Ahat^{i-1}\Bhat-A^{i-1}B}_{\op}\\
    &\leq{} \sum_{i=1}^{k}\nrm{\Bhat}_{\op}\nrm{\Ahat^{i-1}-A^{i-1}}_{\op}
      + \nrm{A^{i-1}}_{\op}\nrm{\Bhat-B}_{\op}.
  \end{align*}
  By \pref{lem:Ahat_power_error}, once $\vepsid\leq{}\frac{(1-\gammaa)}{2\alphaa}$, we have $\nrm{\Ahat^{i}}_{\op}\leq\bigohs(\gammaab^{i-1}i)$ and $\nrm{\Ahat^{i}-A^{i}}_{\op}=\bigohs(\gammaab^{i-1}i\cdot{}\vepsid)$. We further have $\nrm{A^{i}}_{\op}=\bigohs(\gammab^{i})$ and $\nrm{B}_{\op}\vee\nrm{\Bhat}_{\op}=\bigohs(1)$, since  $\vepsid=\bigohs(1)$. Plugging in these bounds above and simplifying, we have
  \begin{align*}
    \nrm{\Chat_k-\cC_k}_{\op}
    \leq{} \bigohs\prn*{
    \vepsid\prn*{1 + \sum_{i=2}^{k}\gammaab^{i-2}i
    }}
        \leq{} \bigohs\prn*{
    \vepsid\prn*{1 + \sum_{i=2}^{\infty}\gammaab^{i-2}i
    }} = \bigohs(\vepsid).
  \end{align*}
  
\end{proof}

\subsection{Proof of \pref{lem:eiglem}}
Let $\cMbar$ be as in \pref{eq:calM}. For the first point, it is easy
to see that $\nrm*{M_k}=\bigohs(1)$ whenever $\sigma^{2}=\bigohs(1)$
using strong stability. It follows that
\[
  \nrm*{\cM_{\sigma^2}}_{\op}\leq{}\sum_{k=1}^{\kappa}\nrm[\big]{A^{k-1}}\nrm*{M_k}
  \leq{} \bigohs\prn*{\sum_{k=1}^{\kappa}\gammaa^{k-1}}=\bigohs(1).
\]
To prove the second point, we first recall the following result.
          \begin{lemma}[\cite{phien2012some}, Theorem 2.2]
            \label{lem:inverse_lipschitz}
            Let $X,Y\in\bbR^{d\times{}d}$. If $X$ is non-singular and $r\ldef\nrm*{X^{-1}Y}_{\op}<1$, then $X+Y$ is non-singular and $\nrm*{(X+Y)^{-1}-X^{-1}}_{\op}\leq{}\nrm*{Y}_{\op}\nrm*{X^{-1}}^{2}_{\op}/(1-r)$.
          \end{lemma}
          Let $k$ be fixed. We set $X=\sum_{i=1}^k A^{i-1} \Sigma_w (A^{i-1})^\top$ and $Y=\sigma^2\cC_k \cC_k^\top$. Since $X\psdgeq{}\Sigw\psdgt{}0$, we have that $\nrm{X^{-1}}_{\op}=\bigohs(1)$ and $\nrm*{Y}_{\op}=\bigohs(\sigma^{2})$. Moreover, $\nrm*{\cC_k}_{\op}=\bigohs(1)$. This implies that for any fixed $\veps>0$, there exists $\bar\sigma = \bigohch(\veps)$ such that for all $\sigma^2 \leq \bar{\sigma}^2$,
		\begin{align}
\veps & \geq   	\left\|	\cC_k^\top \left(\sigma^2\cC_k \cC_k^\top  + \sum_{i=1}^k A^{i-1} \Sigma_w (A^{i-1})^\top\right)^{-1} - \cC_k^\top \left(\sum_{i=1}^k A^{i-1} \Sigma_w (A^{i-1})^\top\right)^{-1}  \right\|_\op,\nn \\
	& = 	\left\|	M_k/\sigma^2 - \cC_k^\top \left(\sum_{i=1}^k A^{i-1} \Sigma_w (A^{i-1})^\top\right)^{-1}  \right\|_\op. \label{eq:last}
		\end{align}
                Define $\Mbar_k= \cC_k^\top \left(\sum_{i=1}^k A^{i-1}
                  \Sigma_w (A^{i-1})^\top\right)^{-1}$. Then using the
                definitions of $\cM$ and $\cMbar$, we have that 
                \begin{align*}
                  \|\cM_{\sigma^2}^\top \cM_{\sigma^2} /\sigma^4
                  -\cMbar^\top \cMbar \|_{\op} &\leq
                                                 2\prn*{\nrm*{\cM_{\sigma^{2}}/\sigma^{2}}_{\op}\vee\nrm*{\cMbar}_{\op}}\cdot\|
                                                 \cM_{\sigma^2}
                                                 /\sigma^2 -\cMbar
                                                 \|_{\op} \\
                  &=\bigohs\prn*{\nrm*{\cM_{\sigma^{2}}/\sigma^2 - \cMbar}_{\op}}
                \end{align*}
                and
                \begin{align*}
                  \nrm*{\cM_{\sigma^{2}}/\sigma^2 - \cMbar}_{\op}
                  &\leq{}
                  \sum_{k=1}^{\kappa}\nrm[\big]{A^{k-1}}_{\op}\nrm*{M_k/\sigma^2-\Mbar_k}_{\op}\\
                  &\leq\bigohs\prn*{\max_{1\leq{}k\leq\kappa}\nrm*{M_k/\sigma^2-\Mbar_k}_{\op}}.
                \end{align*}
                Together with \eqref{eq:last}, this implies that 
		\begin{align}
                  \|\cM_{\sigma^2}^\top \cM_{\sigma^2} /\sigma^4 -\cMbar^\top \cMbar \|_{\op} \leq \cO_\star(\veps), \quad \forall \sigma^2 \leq \bar \sigma^2.\label{eq:froblim}
		\end{align}
Now, note that
                \begin{align*}
                  |\eigmin (\cM_{\sigma^2}^\top \cM_{\sigma^2}/\sigma^4) -  \eigmin(\cMbar^\top \cMbar)| \leq  \|\cM_{\sigma^2}^\top \cM_{\sigma^2} /\sigma^4 -\cMbar^\top \cMbar \|_{\op}.
		\end{align*}
		Combining this with \eqref{eq:froblim} implies that, for all $\sigma^2 \leq \bar\sigma^2$, \begin{align}\left|\eigmin^{1/2} (\cM_{\sigma^2}^\top \cM_{\sigma^2})/\sigma^2- \lambda_{\cM}  \right|\leq \cO_\star(\veps).\label{eq:preeps}
		\end{align} 
		Finally, by definition of the $\cO_\star(\cdot)$
                notation, there exists $\veps = \bigohch(\lambdam)$
                such that the right-hand side of \pref{eq:preeps} is
                at most $\lambdam/2$, which yields the desired result.

                \qed

\subsection{Proof of \Cref{thm:monster}}
\subsubsection{Regression Bound for $\phi_{t,k}$ \label{sec:phitk_reg}}
Let $t\geq{}0$ and $k\in\brk*{\kappa}$ be fixed and introduce the shorthand
	\begin{align}
	\matz_{t,k} \coloneqq (\by_{0:t},\by_{t+k}). \label{eq:thez}
	\end{align}
	Define the ``true'' $\phi$-function
	\begin{align}
		\phi_{t,k}(h,\matz_{t,k}) & \coloneqq  M_{k}  \left(h(\by_{t+k})-  A^{k} h(\by_t) -   A^{k-1}  B \what K \hat{f}_t(\by_{0:t})\right), \label{eq:thetruephi}
	\end{align}
for $h \in \cH_\op$, and its plug-in estimate analogue:
	\begin{align*}
		\phihat_{t,k}(h, \matz_{t,k}) &\coloneqq \what  M_{k}  \left(h(\by_{t+k})- \what A^{k} h(\by_t) -  \what A^{k-1} \what B \what K \hat{f}_t(\by_{0:t})\right). 
		\end{align*}
                Finally, define their difference by
	\begin{align*}
		\updelta_{t,k}(h, \matz_{t,k})& \coloneqq \what\phi_{t,k}(h,\matz_{t,k})-  \phi_{t,k}(h,\matz_{t,k}).
	\end{align*}

	\begin{lemma}
		\label{lem:condsatisf1}
		For $k \in [\dimk]$ define the error bound
                \begin{align}
		\psi_{t,k}(\matz_{t,k})^2 & \coloneqq   6 \alphA^2 \Lonpo^2 \PsiM^2 \left(2 + \|\bx_t\|^2_{2} +\|\bx_{t+k}\|^2_{2} + \Psist^2 \| \what K\|^2_{\op} \bclip^2\right), \label{eq:psi_tk_def}
		\end{align} 
		which is well defined since $\matx_\tau = \fst(\maty_\tau)$ due to the decodability assumption. Further, introduce the error constant
		\begin{align}
		\cwphi :=  50k \dimu \sigma^2 + 30\alphA^2 \Lonpo^2 \PsiM^2 \left(62 \dimx \PsiSig + 5  \bclip^2   \cdot \devx\right).   \label{eq:cwphi_def}
		\end{align}
		 Then, recalling $\bv \coloneqq [\bnu_t^\top, \dots,
                 \bnu_{t+k}^\top]^\top$, for all $\delta \in (0,1/e]$ and $h \in \cH_\op$,
                 the following results hold.
        \begin{enumerate}
        	\item We have the bound
		\begin{align}
		&\sup_{h \in \Honpo}\| \phi_{t,k}(h, \matz_{t,k})\|^2 + \sup_{h \in \Honpo}\|\updelta_{t,k}(h,\matz_{t,k})\|^2 \leq \psi_{t,k}(\matz_{t,k})^2.\label{eq:fineq}
		\end{align}
		\item We have the bound
		\begin{align}
		&\Lonpo (\|\widehat{M}_k\|_{\op} + \|\what M_k \what A^{k} \|_{\op})(2 + \|\fst(\by_t)\|_{2} +\|\fst(\by_{t+k})\|_{2}) \lesssim \psi_{t,k}(\matz_{t,k}).  \label{eq:psi_tk_ub_pihat}
		\end{align}
		\item For all $\delta \in (0,1/e)$,
		\begin{align}
		&\mathbb{P}_{\pitil_t}\left[ \psi_{t,k}(\matz_{t,k})^2 + \|\updelta_{t,k}(h,\matz_{t,k})- \bv \|^2 \geq \frac{5}{3}\cwphi \ln (1/\delta)\right] \leq \delta. \label{eq:sineq}
		\end{align}
		\item We have
		\begin{align}
		&\E_{\wtilde \pi_t} \brk[\bigg]{\sup_{h\in \scrH_{\onpo}} \|\updelta_{t,k}(h, \matz_{t,k}) \|^2} \le 24 \Lonpo^2 \epsys^2 \left(\dimx \PsiSig +   \devx  \bclip^2\right). \label{eq:pitil_zeta_bound}
		\end{align}
	\end{enumerate}
	\end{lemma} 

	\begin{proof}
		For notational convenience, we will drop the
                subscripts $t,k$ in the expressions of $\phi_{t,k},
                \what\phi_{t,k}, \psi_{t,k}$, $\updelta_{t,k}$, and $
                \matz_{t,k}$. Let $h \in \cH_\op$ be fixed throughout.
		\begin{enumerate}
			\item \textbf{Proof of \Cref{eq:fineq}} By Jensen's inequality and Cauchy-Schwarz, we have 
		\begin{align} 
	 \| \phi(h, \matz) \|^2	& \leq  3 \| M_k\|^2_{\op} \left( \|h(\by_{t+k})\|^2 + \alpha_A^2\gamma_A^{2k} \|h(\by_t)\|^2 +\alpha_A^2\gamma_{A}^{2k-2}  \|B \what K\|^2_{\op} \bclip^2\right),  \nn\\
		& \overset{(i)}{\le} 3 \| M_k\|^2_{\op} \left(\Lonpo^2 \cdot(1+ \alpha^2_A+\|\bx_{t+k}\|^2 + \alpha_A^2\|\bx_t\|^2) +\alpha_A^2  \|B \what K\|^2_{\op} \bclip^2\right),\nn\\
		& \overset{(ii)}{\le} 3 \PsiM^2 \alpha_A^2 \Lonpo^2 \left( 2+ \|\bx_{t+k}\|^2 + \|\bx_t\|^2 +  \Psist^2\| \what K\|^2_{\op} \bclip^2\right)
		\label{eq:thephibound},
		\end{align}
	where inequality $(i)$ follows by the definition of the
        function class $\scrH_{\onpo}$, and $(ii)$ uses that $\alpha_A,\Lonpo \ge 1$,  $\|M_k\|_{\op} \le \PsiM$, and $\|B\| \le \Psist$. Similarily, we also have 
		\begin{align}
		\|\updelta(h, \matz)\|^2 &\leq  3 \Lonpo^2 \|\what M_k - M_k \|_{\op}^2   (1 + \|\bx_{t+k}\|^2) +3 \Lonpo^2\|\what M_k \what A^k - M_k A^k \|^2_{\op} (1+ \|\bx_t\|^2) \nn \\
		& \quad + 3 \|\what M_k \what A^{k-1} \what B - M_k A^{k-1} B \|^2_{\op} \|\what K\|_{\op}^2 \bclip^2, \nn \\
		& \leq 3 \Lonpo^2\veps_{\sys}^2(1 + \|\bx_{t+k}\|^2) +3 \Lonpo^2\veps_{\sys}^2 (1+ \|\bx_t\|^2) + 3 \veps_{\sys}^2\|\what K\|_{\op}^2 \bclip^2, \nn\\
		& \leq 3 \Lonpo^2\veps_{\sys}^2\left(2 + \|\bx_{t+k}\|^2 + \|\bx_t\|^2 + \|\what K\|_{\op}^2 \bclip^2\right), \label{eq:simple}
		\end{align}
		where the second-to-last inequality follows since we
                have assumed that the event $\cE_{\sys}$ holds, and
                the last uses $\Lonpo^2 \ge 1$. Combining  \eqref{eq:thephibound} and \eqref{eq:simple}, with $\alpha_A,\Psist \ge 1$ and $\PsiM \ge \epsys$,
		\begin{align*}
			&\sup_{h \in\Honpo}\| \phi(h, \matz) \|^2 + \sup_{h \in \Honpo} \| \delta(h, \matz) \|^2 \le 6 \alphA^2 \Lonpo^2 \PsiM^2 \left(2 + \|\bx_t\|^2_{2} +\|\bx_{t+k}\|^2_{2} + \Psist^2 \| \what K\|^2_{\op} \bclip^2\right) \rdef \psi(\matz)^2,
		\end{align*}
		where we use the simplification $\alphA \ge 1$, and the definitions of $\PsiM$ and $\Psist$, followed by $\Lonpo \ge 1$. This shows \eqref{eq:fineq}. 
		\item \textbf{Proof of \Cref{eq:psi_tk_ub_pihat}} Using the bounds $\|\what M_k \what A^{k} - M_k A^k \|_{\op} , \|\what M_k - M_k \|_{\op} \le \epsys \le 1$, and $\|M_k\|_{\op} \le \PsiM$, $\|A^k\|_{\op} \le \alphA$, we have
		\begin{align*}
		&\Lonpo (\|\widehat{M}_k\|_{\op} + \|\what M_k \what A^{k} \|_{\op})(2 + \|\fst(\by_t)\|_{2} +\|\fst(\by_{t+k})\|_{2}) \\
		&= \Lonpo (\|\widehat{M}_k\|_{\op} + \|\what M_k \what A^{k} \|_{\op})(2 + \|\bx_t\|_{2} +\|\bx_{t+k}\|_{2})\\
		&= \Lonpo (2\epsys + \|M_k\|_{\op} + \| M_k A^{k} \|_{\op})(2 + \|\bx_t\|_{2} +\|\bx_{t+k}\|_{2})\\
		&\le \Lonpo (2\epsys + \PsiM (1+ \alpha_A)) (2 + \|\bx_t\|_{2} +\|\bx_{t+k}\|_{2}).
		\end{align*}
		Since $\epsys \le 1 \le \PsiM$, and $\alpha_A \ge 1$, the bound follows.
		\item \textbf{Proof  of \Cref{eq:sineq}} By \pref{lem:xconcentration,lem:thechis} we have, for all $\delta \in (0,1/e]$, and any $\tau \le t+k$,
	\begin{align*}
	\P_{\wtilde{\pi}_t}\brk*{\|\bv\|^2 \geq \sigma^2 \cdot (3 k \dimu +2) \ln \delta^{-1}} \leq \delta,\mathand
          \Pr_{\pitil_t} \left[ \|\bx_\tau \|^2 \geq  15\dimx\PsiSig  +  \bclip^2   \cdot \devx  \ln \frac{2}{\delta}    \right] \leq \delta,
	\end{align*} 
	and so by a union bound, with probability at least $1 -
        \delta$,
	\begin{align*}
          &\left(\ln \frac{5}{\delta}\right)^{-1} \cdot
            \left(\psi(\matz)^2 \vee \|\updelta(\hst
            , \matz) - \bv \|^2\right)\\
          &\le\left(\ln \frac{5}{\delta}\right)^{-1} \cdot
            \left(2\psi(\matz)^2 +2\|\bv \|^2\right) 
          \\&\le \underbrace{\left(2\sigma^2 \cdot (3 k \dimu +2)\right)}_{\le 10k \dimu \sigma^2} 
          +  6 \alphA^2 \Lonpo^2 \PsiM^2 \left(2 + 30\dimx\PsiSig  +  2 \bclip^2   \cdot \devx  + \Psist^2 \| \what K\|^2_{\op} \bclip^2 \right).
	\end{align*} 
	Finally, since $\devx \ge \Psist^2 \| \what K\|^2_{\op}$ by definition (see \Cref{eq:devx_def}), and $\PsiSig \ge 1$, the above is at most
	\begin{align*}
	\left(\ln \frac{5}{\delta}\right)^{-1} \cdot \left(\psi(\matz)^2 \vee \|\updelta(f_\star, \matz) - \bv \|^2\right) &\le  10k \dimu \sigma^2 + 6\alphA^2 \Lonpo^2 \PsiM^2 \left(32 \dimx \PsiSig + 3  \bclip^2   \cdot \devx\right) := \cwphi/5.
	\end{align*}
	Finally, for $\delta \le 1/e$, we have that $\left(\ln \frac{5}{\delta}\right) \le  5\ln(1/\delta)$. This bound follows by the fact that $\ln (5/\delta)=\ln(5) +\ln (1/\delta)\leq (\ln 5 +1)\ln (1/\delta)\leq 5 \ln (1/\delta)$, for all $\delta \in(0,1/e]$.

	\item \textbf{Proof of \Cref{eq:pitil_zeta_bound}}. We bound
	\begin{align}
	\E_{\wtilde \pi_t} \left[ \sup_{h\in \scrH_{\onpo}} \|\updelta(h, \matz) \|^2 \right] &= \E_{\wtilde \pi_t} \left[ \sup_{h\in \scrH_{\onpo}} \|\updelta_{t,k}(h, \by_{0:t}, \by_{t+k}) \|^2 \right],\nn \\
	&\leq 3\Lonpo^2 \epsys^2 \left(2 + \|\what K\|_{\op}^2 \bclip^2 + \E_{\wtilde{\pi}_t} \left[ \|\bx_{t+k}\|^2  + \|\bx_t\|^2\right]\right)\,\label{eq:step},
	\end{align}
	where we use \Cref{eq:simple} in the last step.  From \pref{lem:xconcentration}, we have
	\begin{align*}
	3\Lonpo^2 \epsys^2 \left(2 + \|\what K\|_{\op}^2 \bclip^2 + \E_{\wtilde{\pi}_t} \left[ \|\bx_{t+k}\|^2  + \|\bx_t\|^2\right]\right)\le 3\Lonpo^2 \epsys^2 \left(2 + \|\what K\|_{\op}^2 \bclip^2 + 6 \dimx \PsiSig +   2\devx  \bclip^2\right),
	\end{align*}
Using the above two displays together with $\dev_x \ge \|\what K\|_{\op}^2$ and $\PsiSig \ge 1$ yields
	\begin{align}
	\E_{\wtilde \pi_t} \left[ \sup_{h\in \scrH_{\onpo}} \|\updelta(h, \matz) \|^2 \right] \le 3\Lonpo^2 \epsys^2 \left( 8 \dimx \PsiSig +   3\devx  \bclip^2\right) \le 24 \Lonpo^2 \epsys^2 \left(\dimx \PsiSig +   \devx  \bclip^2\right)\nn.
	\end{align}
\end{enumerate}
	\end{proof}

	\newcommand{\evphi}[1][k]{\cE_{\phi;t,#1}}

	\begin{lemma}
		\label{lem:firstregression}
	Let $t\geq 0$, $k \in [\dimk]$. For $\hhat_{t,k}$ and $\phi_{t,k}$ as in \eqref{eq:gtk} and \eqref{eq:thetruephi}, respectively, we have with probability at least $1-3\delta/2$,
	\begin{align}
	&\E_{\wtilde \pi_t} \left[  \left\| \phi_{t,k}(\hat{h}_{t,k}, \by_{0:t}, \by_{t+k})  - M_k ( A^{k-1} \bw_t + \dots + \bw_{t+k-1}+  A^{k-1} B \bnu_{t} + \dots +  B \bnu_{t+k-1})  \right\|^2  \right] \leq \epsw^2(\delta),\nn \\
	&\text{where} \quad \epsw^2(\delta) = c_{\epsw} \left(\frac{\cwphi  (\ln|\Fclass| + \dimx^2)\ln^2(n/\delta)}{n} +  \Lonpo^2 \epsys^2 \left(\dimx \PsiSig +   \devx  \bclip^2\right)\right) \label{eq:tildeeps},
	\end{align}
	and where $c_{\epsw}$ is a sufficiently large constant, chosen to be at least $100$ without loss of generality, and $\cwphi$ is defined in \Cref{eq:cwphi_def}. 
      \end{lemma}
      We denote the event of \pref{lem:firstregression} by $\evphi(\delta)$.
	\begin{proof}
	We will apply
        \pref{cor:techtools_function_dependend_error}. We verify that
        the conditions of the corollary hold one by one.
	\begin{enumerate}
		\item \textbf{Substitutions.}  We apply
                  \pref{cor:techtools_function_dependend_error} with $\be=0$, $\bz \coloneqq (\by_{0:t}, \by_{t+k})$, $\bv = [\bnu_t^\top, \dots, \bnu_{t+k}^\top]^\top$, $\phi = \phi_{t,k}$, $\psi=\psi_{t,k}, \delphi = \updelta_{t,k}$, and $c = \cwphi $, where $\phi_{t,k}$, $\psi_{t,k},\updelta_{t,k},$ and $\cwphi $ are as in \pref{lem:condsatisf1}. Moreover, we let $c_{\psi}$ be the constant implicit in \Cref{eq:psi_tk_ub_pihat}. The dimension parameters are $\dimx,d_1 \gets \dimx$.
		\item \textbf{Realizability.}  By our assumption on
                  the function class $\scrH_{\onpo}$, there exists
                  $\fst \in \scrH_{\onpo}$ such that $\fst(y) = x$,
                  for all $y \in \supp q(\cdot \mid x)$. Therefore, by
                  the system's dynamics and the definition of the
                  policy $\wtilde{\pi}_t$, we have almost surely
	\begin{align}
	\phi(h_\star, \bz) &= M_k (\fstar(\maty_{t+k})-
                             A^{k}\fstar(\maty_t) - A^{k-1}B\Khat\fhat_t(\maty_{0:t})),\nn \\
          &= M_k ( A^{k-1} \bw_t + \dots + \bw_{t+k-1}+  A^{k-1} B \bnu_{t} + \dots +  B \bnu_{t+k-1}),\nn \\
	& = \E_{\wtilde \pi_t}[\bv \mid  A^{k-1} \bw_t + \dots + \bw_{t+k-1}+  A^{k-1} B \bnu_{t} + \dots +  B \bnu_{t+k-1}] ,\quad \text{(by \pref{fact:gaussian_expectation}) }\nn \\ 
	& = \E_{\wtilde \pi_t}\left[  \bv  \left |\;   \begin{matrix} \sum_{j=1}^k (A^{j-1} \bw_{t+k-j} + A^{j-1} B \bnu_{t+k-j}) \\ \by_{0:t} \end{matrix}\right.  \right], \label{eq:middlepoint}\\
	& = \E_{\wtilde \pi_t}\left[  \bv  \left |\;   \begin{matrix} A^k f_\star(\by_t) + A^{k-1} B \what K \hat{f}_t(\by_{0:t})+\sum_{j=1}^k (A^{j-1} \bw_{t+k-j} + A^{j-1} B \bnu_{t+k-j}) \\ \by_{0:t} \end{matrix}\right.  \right],\label{eq:conditioning}\\
	& =  \E_{\wtilde \pi_t}[\bv \mid \by_{0:t}, \fst(\by_{t+k})], \quad \quad  \label{eq:firstregression_dynamics}  \\
	& =  \E_{\wtilde \pi_t}[\bv \mid \by_{0:t}, \by_{t+k}], \label{eq:realize}
	\end{align}
	where \eqref{eq:middlepoint} follows by the fact that
        $(\bnu_\tau)_{\tau\geq t}$ and $(\bw_\tau)_{\tau \geq t}$ are
        independent of $\by_{0:t}$, \eqref{eq:conditioning} follows by
        the conditioning on $\by_{0:t}$ (which determines the term $A^k f_\star(\by_t) + A^{k-1} B \what K \hat{f}_t(\by_{0:t})$), and
        \eqref{eq:firstregression_dynamics} uses the system's
        dynamics. Finally, \eqref{eq:realize} uses the realizability assumption. Thus, \eqref{eq:realize} ensures the realizability
        assumption in \pref{cor:techtools_function_dependend_error} is
        satisfied.

    \item  \textbf{Conditions 1 \& 2.} \pref{lem:condsatisf1} ensures that conditions 1 and 2 of \pref{cor:techtools_function_dependend_error} are satisfied.
    \item  \textbf{Condition 3.} By the structure of $\Honpo$, condition 3 is satisfied with $L$ as in \Cref{asm:f_growth} and $b L = \Lonpo$. Examining $\phihat_{t,k}$, we can  take $X_1 = \Mhat_k$, and $X_2 = \Ahat^k$.
    \item  \textbf{Condition 4.} By \Cref{eq:psi_tk_ub_pihat}, this holds for some $c_{\psi} \lesssim 1$. 
	\end{enumerate}

	 Recall the notation $\somelogs(n,\delta)  \lesssim
         \ln^2(n/\delta)$ defined in
         \pref{cor:techtools_function_dependend_error}. With the
         substitutions above, \pref{cor:techtools_function_dependend_error}
         implies that with probability at least $1 - \frac{3\delta}{2}$:
		\begin{align*}
		\E\|\phi_{t,k}(\hhat,\matz) - \phi_{t,k}(\fst,\matz)\|^2 &\le \frac{12\cwphi  (\ln|\Fclass| + \dimx \cdot \dimx)\somelogs( c_{\psi}n,\delta)}{n} +  16 \E\|\mate\|^2 + 8\max_{h \in \scrH}\E\|\updelta_{t,k}(h,\bz)\|^2\\
		&= \frac{12\cwphi  (\ln|\Fclass| + \dimx^2)\somelogs( c_{\psi}n,\delta)}{n} +  8\max_{h \in \scrH}\E\|\updelta_{t,k}(h,\bz)\|^2,\\
		&\lesssim \frac{\cwphi  (\ln|\Fclass| + \dimx^2)\ln^2(n/\delta)}{n} +  \Lonpo^2 \epsys^2 \left(\dimx \PsiSig +   \devx  \bclip^2\right)  \tag*{(by \Cref{eq:pitil_zeta_bound})}.
		\end{align*}
	\end{proof}

        \subsubsection{Regression bound for $\upphi$ \label{sec:upphi_reg}}
        Let $t\geq{}0$ be fixed. Recall the various functions defined at the start of \Cref{sec:phitk_reg}. In addition, consider the following functions for $k \in [\dimk]$, $h\in \scrH_{\onpo}$:
	\begin{align}\what\phi_{t}(h, \by_{0:t+1}) &\coloneqq\what \calM  \left(h(\by_{t+1})- \what A h(\by_t) -  \what B \what K \hat{f}_t(\by_{0:t}) \right) , \nn \\
	\phi_{t}(h, \by_{0:t+1}) &\coloneqq  \calM  \left(h(\by_{t+1})-  A h(\by_t) -   B \what K \hat{f}_t(\by_{0:t}) \right) , \nn\\ 
		\updelta_{t}(h, \by_{0:t+1})& \coloneqq \what\phi_{t}(h,\by_{0:t+1})-  \phi_{t}(h,\by_{0:t+1}). \nn
	\end{align}
	Further, for $(\hat{h}_{t,k})_{k\in\kappa}$ as in \pref{eq:gtk}, define
	\begin{align*}
	\what\upphi_t(\by_{0:t+\dimk})&\coloneqq [ \what\phi_{t,1}(\hat{h}_{t,1}, \by_{0:t}, \by_{t+1})^\top, \dots,  \what\phi_{t,\dimk}(\hat{h}_{t,\dimk},\by_{0:t}, \by_{t+\dimk})^\top]^\top, \nn  \\
	\upphi_t(\by_{0:t+\dimk})&\coloneqq [ \phi_{t,1}(\hat{h}_{t,1},\by_{0:t}, \by_{t+1})^\top, \dots,  \phi_{t,\dimk}(\hat{h}_{t,\dimk}, \by_{0:t}, \by_{t+\dimk})^\top]^\top,	\nn\\
	 \upphi^\star_t(\by_{0:t+\dimk})&\coloneqq [ \phi_{t,1}(\fst,\by_{0:t}, \by_{t+1})^\top, \dots,  \phi_{t,\dimk}(\fst, \by_{0:t}, \by_{t+\dimk})^\top]^\top.
	 \end{align*}
	 Here, the first term uses estimated dynamics and estimates $\hhat$ of $\fst$; the second term uses true dynamics and estimates $\hhat$; the third term uses true dynamics \emph{and} the ground truth $\fst$. 
	\begin{lemma}
		\label{lem:condsatisf2} 
		Let $\bv \coloneqq \upphi^\star_t(\by_{0:t+\dimk})$,
                $\be \coloneqq 	\upphi^\star_t(\by_{0:t+\dimk}) -
                \what\upphi_t(\by_{0:t+\dimk})$. Recall the
                function $\psi_{t,1}$ defined in
                \Cref{eq:psi_tk_def}. Then the following properties hold.
		\begin{enumerate}
			\item We have the bound \begin{align}
		&\sup_{h \in \Honpo} \| \phi_{t}(h, \by_{0:t+1})\|^2 + \sup_{h \in \Honpo} \|\updelta_{t}(h,\by_{0:t+1})\|^2 \leq \psi_{t,1}(\by_{0:t+1})^2.\label{eq:0fineq}
		\end{align}
		\item We have the bound
		\begin{align}
		\Lonpo (\|\widehat{\cM}\|_{\op} + \|\what \cM \what A \|_{\op})(2 + \|\fst(\by_t)\|_{2} +\|\fst(\by_{t+1})\|_{2}) \lesssim \psi_{t,1}(\by_{0:t+1}). \label{eq:psi_tk_ub_pihat_0}
		\end{align}
		\item For any $\delta \in (0,1/e]$, we have 
		\begin{align}
		\Pr_{\pitil_t}[\psi_{t,1}(\by_{0:t+1})^2 \vee \|\matv - \mate\|^2 \le 2\dimk\cwphi(1+\ln(\dimk))\ln(1/\delta)] \leq \delta. \label{eq:e_min_v_bound}
		\end{align}
		\item For any $\delta \in (0,1/e]$, on the event
                  $\bigcap_{k=1}^\dimk\evphi[k](\delta)$
                  (cf. \pref{lem:firstregression}), we have that
		\begin{align}
		\E_{\pitil_t}\|\mate\|^2 \le 3\dimk \epsw(\delta)^2.  \label{eq:exp_e_bound}
		\end{align}
		\item  For any $\delta > 0$, we have the following bound (independent of $\delta$)
		\begin{align}
		\sup_{h \in \Honpo}\E\|\updelta_t(h,\by_{0:t+1})\|^2 \le \epsw(\delta)^2. \label{eq:zeta_bound_two}
		\end{align}
	\end{enumerate}
	\end{lemma}

	\begin{proof}[\pfref{lem:condsatisf2}]
		In what follows, let us suppress dependence on $\maty$ and $\matz_{t,k}$ when clear from context, where $\matz_{t,k}$ is as in \eqref{eq:thez}.
	\begin{enumerate}
		\item \textbf{Bounding $\sup_{h \in \Honpo} \| \phi_{t}(h, \by_{0:t+1})\| + \sup_{h \in \Honpo} \|\updelta_{t}(h,\by_{0:t+1})\| \le \psi_{t,1}(\by_{0:t+1})^2$.}

		The bound in \eqref{eq:0fineq} actually follows from
                the same argument as in the proof of \pref{lem:condsatisf1} with $k=1$ and $(\what M_k, M_k)$ replaced by $(\what\calM ,\calM)$ (using that $\PsiM$ also bounds $\|\calM\|_\op$, and $\epsys$ upper bounds $\|\calM - \what \calM\|_\op$) under $\cE_{\sys}$. 
	\item \textbf{Establishing \Cref{eq:psi_tk_ub_pihat_0}}. This is also analogous to the proof of \Cref{eq:psi_tk_ub_pihat} in \pref{lem:condsatisf1}.
	\item[3a.] \textbf{Bounding $\|\matv - \mate\|^2 $.} 
		We bound
		\begin{align*}
		\|\matv - \mate\|^2 = \left\|\upphihat(\by_{0:t+k})\right\|^2 = \sum_{k=1}^\dimk \left\|\phihat_{t,k}(\hhat_{t,k},\by_{0:t},\by_{t+k})\right\|^{2} \le \sum_{k=1}^\dimk \psi_{t,k}(\matz_{t,k})^2, 
		\end{align*}
		where we use \Cref{eq:fineq}.
	\item[3b.]\textbf{Establishing \Cref{eq:e_min_v_bound}.}
			We have
			\begin{align*}
			\psi_{t,1}(\by_{0:t+1})^2 \vee \|\matv - \mate\|^2 
			&\le \psi_{t,1}(\by_{0:t+1})^2 \vee \sum_{k=1}^\dimk \psi_{t,k}(\matz_{t,k})^2 ,\nn \\ & = \sum_{k=1}^\dimk \psi_{t,k}(\matz_{t,k})^2,\\
			&=  \sum_{k=1}^\dimk 6 \alphA^2 \Lonpo^2 \PsiM^2 \left(2 + \|\bx_t\|^2_{2} +\|\bx_{t+k}\|^2_{2} + \Psist^2 \| \what K\|^2_{\op} \bclip^2\right).
			\end{align*}
			Now, with probability $1 - (\dimk+1)\delta$,
                        we have that all $\|\bx_{t+i}\|^2$ for $i
                        \in\{0,1,\dots,\dimk\}$ simultaneously satisfy
			\begin{align*}
			\|\bx_{t+i}\|^2 \le 30\dimx\PsiSig  +  2\bclip^2   \cdot \devx  \ln \frac{2}{\delta} 
			\end{align*}
			by \Cref{lem:xconcentration}. Hence, with probability $1 - (\dimk+1)\delta$, we have 
			\begin{align*}
			&\psi_{t,1}(\by_{0:t+1})^2\vee \|\matv - \mate\|^2 \\
			&\le \sum_{k=1}^\dimk \psi_{t,k}(\matz_{t,k})^2\\
			&=  \dimk 6 \alphA^2 \Lonpo^2 \PsiM^2 \left(2 + 60 \dimx \PsiSig + 4\bclip^2   \cdot \devx  \ln \frac{2}{\delta}  +  \Psist^2 \| \what K\|^2_{\op} \bclip^2\right)\\
			&\le  \dimk 6 \alphA^2 \Lonpo^2 \PsiM^2 \left(62 \dimx \PsiSig + 5\bclip^2   \cdot \devx   \right) \ln \frac{2}{\delta} \le  \frac{1}{2}\dimk \cwphi \ln \frac{2}{\delta}.
			\end{align*}
			where in the last line, we absorb various parameters into larger ones. Finally, replacing $\delta$ by $\delta/(\kappa+1)$ gives $(1/2) \cdot \ln(2(\kappa+1)/\delta) = (1/2) \cdot \ln(2/\delta) + \ln(2(\kappa+1)) \le 2(1 + \ln(\kappa))\ln(1/\delta)$ for $\delta \in(0, 1/e]$. This gives that,  
			\begin{align*}
			\Pr\left[\psi_{t,1}(\by_{0:t+1})^2 \vee \|\matv - \mate\|^2 \geq 2\dimk\cwphi(1 + \ln(\dimk))\ln(1/\delta)\right] \le \delta.
			\end{align*}
		\item[4.]\textbf{Establishing \Cref{eq:exp_e_bound}.} 
			First, we bound
			\begin{align*}
			 \|\upphi_t (\by_{0:t+\dimk}) -  \upphist_t (\by_{0:t+\dimk})\|_2^2 &= \sum_{k=1}^{\dimk}\|\phi_{t,k}(\hhat_{t,k}) - \phi_{t,k}(\hst)\|^2.
			\end{align*}
			From \Cref{lem:firstregression}, we have on the event $\bigcap_{k=1}^{\dimk}\evphi(\delta)$ (recall the definition of the event $\cE_{\phi;t,k}(\delta)$ from \pref{lem:firstregression}) that  
			\begin{align}
			\E_{\pitil_k}\| \upphi_t(\by_{0:t+\dimk}) -  \what\upphi_t (\by_{0:t+\dimk})  \|^2 \le \sum_{k=1}^{\dimk}\epsw(\delta)^2= \dimk \epsw(\delta)^2. \label{eq:e_tail}
			\end{align}
			Second, we note that 
			\begin{align*}
			\|\upphi_t (\by_{0:t+\dimk}) -  \upphist_t (\by_{0:t+\dimk})\|_2^2 &= \sum_{k=1}^{\dimk}\|\phi_{t,k}(\hhat_{t,k}) - \phi_{t,k}(\hhat_{t,k})\|^2,\\
			&= \sum_{k=1}^{\dimk}\|\updelta_{t,1}(\hhat_{t,k})\|^2,
			\end{align*}
			so that by \Cref{eq:pitil_zeta_bound}, 
			\begin{align*}
			\E_{\pitil_k}\|\upphi_t (\by_{0:t+\dimk}) -  \upphist_t (\by_{0:t+\dimk})\|_2^2 \le   24 \dimk\Lonpo^2 \epsys^2 \left(\dimx \PsiSig +   \devx  \bclip^2\right).
			\end{align*}
			Hence, on $\cE_{\sys} \cap \bigcap_{k=1}^{\dimk}\evphi(\delta)$, it holds that 
			\begin{align*}
			\E\|\mate\|^2 &\le 2 \E_{\pitil_k}\|\upphi_t (\by_{0:t+\dimk}) -  \upphist_t (\by_{0:t+\dimk})\|_2 + 2 \E_{\pitil_k}\| \upphi_t(\by_{0:t+\dimk}) -  \what\upphi_t (\by_{0:t+\dimk})  \|^2 \\
			&\le 2\dimk \epsw(\delta)^2 +  48 \dimk\Lonpo^2 \epsys^2 \left(\dimx \PsiSig +   \devx  \bclip^2\right) \le 3\dimk \epsw(\delta)^2,
			\end{align*}
			where the last line uses the definition of $\epsw(\delta)^2$ from \Cref{lem:firstregression}.
		\item[5.] \textbf{Establishing
                    \Cref{eq:zeta_bound_two}.} By using an analogous
                  proof to that of \Cref{eq:pitil_zeta_bound} (in
                  particular, exploiting that $\vepssys$ bounds the
                  error of both $\wh{\cM}$ and $\wh{M}_k$), we can
                  show that
		\begin{align*}
		\E_{\wtilde \pi_t} \left[ \sup_{h\in \scrH_{\onpo}} \|\updelta_{t}(h,\by_{0:t+1}) \|^2 \right] \le 24 \Lonpo^2 \epsys^2 \left(\dimx \PsiSig +   \devx  \bclip^2\right).
		\end{align*}
		The right-hand-side is crudely bounded by $\epsw(\delta)^2$ for any $\delta \in (0,1)$. 
		\end{enumerate}
	\end{proof}
\subsubsection{Proof of \pref{thm:monster}}
Again, we appeal to \pref{cor:techtools_function_dependend_error}. We
verify one by one that the conditions require to apply the corollary hold.
	\begin{enumerate}
		\item \textbf{Substitutions.}  We appeal to the
                  corollary with $\be=\upphi^\star_t(\by_{0:t+\dimk})-\what\upphi_t(\by_{0:t+\dimk})$, $\matz \coloneqq \by_{0:t+1}$, $\bv = 	\upphi^\star_t(\by_{0:t+\dimk})$, $\phi = \phi_{t}$, $\psi = \psi_{t}, \updelta_{\phi} = \updelta_{t}$, and $c = \kappa (\ln \dimk +1) \cwphi$, where $\phi_{t}$, $\psi_{t},\updelta_{t},$ and $\cwphi$ are as in \pref{lem:condsatisf2}. We also take $d_1,\dimx \gets \dimx$.
		\item \textbf{Realizability.}

	By our assumption on the function class $\scrH_{\onpo}$, there exists $f_\star \in \scrH_{\onpo}$ such that $f_\star (y) = x$ for all $y \in \supp q(\cdot \mid x)$. Therefore, by the system's dynamics and the definition of the policy $\wtilde{\pi}_t$, we have
		\begin{align} 
	\phi(f_\star, \matz) &=\calM (\bw_t +B \bnu_t) ,\nn \\
	&= \begin{bmatrix}M_1  (\bw_t +B \bnu_t)   \\ \vdots  \\ M_{\dimk} A^{\dimk-1} (\bw_t +B \bnu_t)   \end{bmatrix},\nn \\
	&= \E_{\wtilde \pi_t}\left[ \left.  \begin{bmatrix}M_1  (  \bw_{t} + B \bnu_{t} )   \\ \vdots  \\ M_{\dimk} (\sum_{j=1}^{\dimk} A^{j-1} \bw_{t+ \dimk-j} + A^{j-1 } B \bnu_{t+\dimk -j} )   \end{bmatrix}\right| \bw_{t} + B \bnu_t  \right],\nn \\
	&= \E_{\wtilde \pi_t}\left[ \left.  \begin{bmatrix}M_1  (  \bw_{t} + B \bnu_{t} )   \\ \vdots  \\ M_{\dimk} (\sum_{j=1}^{\dimk} A^{j-1} \bw_{t+ \dimk-j} + A^{j-1 } B \bnu_{t+\dimk -j} )   \end{bmatrix}\right| \begin{matrix} \bw_{t} + B \bnu_t, \\ \by_{0:t} \end{matrix}\right],\label{eq:cond} \\
	&= \E_{\wtilde \pi_t}\left[ \left.  \begin{bmatrix}M_1  (  \bw_{t} + B \bnu_{t} )   \\ \vdots  \\ M_{\dimk} (\sum_{j=1}^{\dimk} A^{j-1} \bw_{t+ \dimk-j} + A^{j-1 } B \bnu_{t+\dimk -1} )   \end{bmatrix}\right|\begin{matrix}\by_{0:t}, \bx_{t+1} \\  \end{matrix}  \right],\label{eq:recover} \\
		& =  \E_{\wtilde \pi_t}\left[ \left.  \begin{bmatrix}  \phi^\star_{t,1}(\by_{0:t}, \by_{t+1})    \\ \vdots  \\   \phi^\star_{t,\dimk}(\by_{0:t}, \by_{t+\dimk})   \end{bmatrix}\right| \begin{matrix}\by_{0:t}, \by_{t+1} \\  \end{matrix}  \right], \label{eq:0realize}
		\end{align}
		where: \eqref{eq:cond} follows by the fact that $(\bnu_\tau)_{\tau\geq t}$ and $(\bw_\tau)_{\tau \geq t}$ are independent of $\by_{0:t}$; \eqref{eq:recover} follows by the fact that $\bw_t + B\bnu_t$ can recovered from $\bx_{t+1}$ given $\by_{0:t}$ and vice-versa; and finally, \eqref{eq:0realize} follows from the system's dynamics and the definition of $\wtilde\pi_t$. Thus, \eqref{eq:0realize} ensures that the realizability assumption in \pref{cor:techtools_function_dependend_error} is satisfied. 
		\item \textbf{Conditions 1\& 2.} \pref{lem:condsatisf2} ensures that conditions 1 and 2 of \pref{cor:techtools_function_dependend_error} are satisfied. 
		 \item  \textbf{Condition 3.} By the structure of
                   $\Honpo$, condition 3 is satisfied with $L$ as in
                   \Cref{asm:f_growth} and $b L = \Lonpo$. Examining
                   $\what\phi_{t}(h, \by_{0:t+1}) \coloneqq\what \calM  \left(h(\by_{t+1})- \what A h(\by_t) -  \what B \what K \hat{f}_t(\by_{0:t}) \right)$, we can  take $X_1 = \what \cM$, and
                   $X_2 = \Ahat$. The term $\calM \what B \what K \hat{f}_t(\by_{0:t})$ does not depend on $h$, and thus corresponds to $\updelta_0$.
    \item  \textbf{Condition 4.} By \Cref{eq:psi_tk_ub_pihat_0}, this holds for some $c_{\psi} \lesssim 1$.
	\end{enumerate}

	 Recall $\somelogs(n,\delta)  \lesssim \ln^2(n/\delta)$, defined in \pref{cor:techtools_function_dependend_error}. \pref{cor:techtools_function_dependend_error} implies that with probability at least $1 - \frac{3}{2}\delta$,
		\begin{align*}
		\E\|\phi_{t}(\hhat,\matz) - \phi_{t}(\fst,\matz)\|^2 &\le \frac{12 \kappa (\ln \dimk +1) \cwphi  (\ln|\Fclass| + \dimx \cdot \dimx)\somelogs( c_{\psi}n,\delta)}{n} +  16 \E\|\mate\|^2 + 8\max_{h \in \Honpo}\E\|\updelta_{t}(h,\bz)\|^2\\
		&\lesssim \frac{\cwphi  \kappa (\ln \dimk +1) (\ln|\Fclass| + \dimx^2)\ln^2(n/\delta)}{n} + \E\|\mate\|^2 + \max_{h \in \Honpo}\E\|\updelta_{t}(h,\bz)\|^2.
		\end{align*}
		Substituting in the bounds in
                \Cref{eq:exp_e_bound,eq:zeta_bound_two}, which hold on
                the events $\bigcap_{k=1}^\dimk\evphi[k](\delta)$
                (i.e., the intersection of the events from
                \Cref{lem:firstregression}), followed by the
                definition of $\epsw$ given in \Cref{eq:tildeeps}, the
                expression above is bounded as 
		\begin{align*}
		&\E\|\phi_{t}(\hhat,\matz) - \phi_{t}(\fst,\matz)\|^2 \\
		&\lesssim \frac{\cwphi  \kappa (\ln \dimk +1) (\ln|\Fclass| + \dimx^2)\ln^2(n/\delta)}{n} + \kappa \epsw(\delta)^2\\
		&\lesssim \kappa (1+\ln(\kappa))\epsw(\delta)^2\\
		&\lesssim \kappa(1 + \ln(\kappa)) \left(\frac{\cwphi  (\ln|\Fclass| + \dimx^2)\ln^2(n/\delta)}{n} +  \Lonpo^2 \epsys^2 \left(\dimx \PsiSig +   \devx  \bclip^2\right)\right).
		\end{align*}
		Finally, let us account for the total failure probability. By \Cref{lem:firstregression}, we have $\Pr[\bigcap_{k=1}^\dimk\evphi[k](\delta)] \ge 1 - \frac{3\kappa}{2}\delta$, and the above display holds with another probability $1 - \frac{3}{2}\delta$. Hence, our failure probability is at most $\frac{3(\kappa+1)\delta}{2}$. Rescaling $\delta \gets \frac{2\delta}{3(\kappa+1)}$, and noting that $\ln(c_1/\delta) \lesssim c_1\ln(1/\delta)$ for constants $c_1$, we find that with probability $1 - \delta$,
		\begin{align*}
                  &\E\|\phi_{t}(\hhat,\matz) - \phi_{t}(\fst,\matz)\|^2
		&\lesssim \kappa(1 + \ln(\kappa)) \left(\frac{\cwphi
           (\ln|\Fclass| + \dimx^2)\ln^2(n\kappa/\delta)}{n} +
           \Lonpo^2 \epsys^2 \left(\dimx \PsiSig +   \devx
           \bclip^2\right)\right).\\
                  &\lesssim\veps_{\noise}^2(\delta).
		\end{align*}

	\qed

\subsection{Proof of \pref{lem:decoder}}
\begin{proof}[Proof of \pref{lem:decoder}]
  Let $t\geq{}0$ be fixed. To begin, consider a fixed
  $0\leq\tau\leq{}t$, and let $(\hat{h}_\tau)$ and $(\phi_\tau)$ be as in \pref{lem:condsatisf2}. For notational convenience, we define $\wtilde \phi_\tau \coloneqq \phi_\tau - \calM B \bnu_\tau$. From the definitions of $\tilde{f}_{\tau+1}$ and $\phi_\tau$, we have
	\begin{align}
	\tilde{f}_{\tau+1}(\by_{0:\tau+1})&\coloneqq \what A \hat{f}_\tau(\by_{0:\tau}) + \hat{h}_\tau (\by_{\tau+1})- \what A  \hat{h}_\tau(\by_\tau), \nn \\& = \what A \hat{f}_\tau(\by_{0:\tau})  +  B \what K\hat{f}_\tau(\by_{0:\tau}) + B \bnu_\tau   + ((\hat{h}_\tau (\by_{\tau+1})-  A  \hat{h}_\tau(\by_\tau) -  B \what K \hat{f}_\tau(\by_{0:\tau}))-  B \bnu_\tau), \nn \\
	& =  \what A \hat{f}_\tau(\by_{0:\tau})  +  B \what K\hat{f}_\tau(\by_{0:\tau}) +  B \bnu_\tau   +\calM^{\dagger} \wtilde\phi_\tau(\hat{h}_\tau,\by_{0:\tau+1}),\nn 
	\end{align}
where we have used that $\cM$ has full row rank by
\pref{ass:sigma_small}. This implies that
	\begin{align}
	\tilde{f}_{\tau+1}(\by_{0:\tau+1}) -  f_\star(\by_{\tau+1}) =   (\what A - A) \hat{f}_\tau(\by_{0:\tau})+ A (\hat{f}_{\tau}(\by_{0:\tau}) - f_\star(\by_{0:\tau}))      + (\calM^{\dagger} \wtilde\phi_\tau(\hat{h}_\tau,\by_{0:\tau+1})- \bw_\tau).\nn
	\end{align}
	Under the event $\cEroll_{0:t}$, we have in particular that $\hat{f}_s = \tilde{f}_{s}$, for all $0\leq s \leq \tau$. Thus, by induction we have
	\begin{align}
	\tilde{f}_{\tau+1}(\by_{0:\tau+1}) -  f_\star(\by_{\tau+1}) &= \sum_{s=0}^\tau A^{\tau-s} \left((\what A- A) \tilde{f}_s(\by_{0:s}) +   (\calM^{\dagger}  \wtilde\phi_s(\hat{h}_s,\by_{0:s+1})- \bw_s)  \right) \nn \\
	& \quad + A^{\tau-1} (\hat{f}_{A,0}(\by_{0}) - A
   f_\star(\by_{0})),\nn
        \end{align}
        with $\fhat_0\equiv\tilde{f}_0\equiv0$. As a result, we have, for $\eps_0 \coloneqq \| \hat{f}_{A,0}(\by_0) -
A f_\star(\by_0)\|$ and $\veps_{\sys}$ as in \eqref{eq:event0},
\begin{align}
& \| 	\tilde{f}_{\tau+1}(\by_{0:\tau+1}) -  f_\star(\by_{\tau}) \|\\ &\leq   \left\|\sum_{s=0}^\tau  A^{\tau-s} (\what A - A) \tilde{f}_s(\by_{0:s}) \right\| +   \left\|  \sum_{s=0}^{\tau} A^{\tau- s} (\calM^{\dagger} \wtilde\phi_s(\hat{h}_s,\by_{0:s+1}) - \bw_s )  \right\|
 +\alpha_A \gamma_A^{\tau-1} \epsilon_0, \nn \\
 & \leq  \alpha_A \veps_{\sys} \bclip (1-\gamma_A)^{-1} + \alpha_A \gamma_A^{\tau-1} \epsilon_0
+  \sum_{s=0}^{\tau} \nrm*{ A^{\tau- s}}_{\op}\nrm*{ (\calM^{\dagger} \wtilde\phi_s(\hat{h}_s,\by_{0:s+1}) - \bw_s )} ,\nn \\
 & \leq  \alpha_A \veps_{\sys} \bclip (1-\gamma_A)^{-1} + \alpha_A \gamma_A^{\tau-1} \epsilon_0 
   +  \alphaa\sum_{s=0}^{\tau} \gammaa^{\tau-s}\nrm*{ (\calM^{\dagger}
   \wtilde\phi_s(\hat{h}_s,\by_{0:s+1}) - \bw_s )},\nn \\
   & \leq  \alpha_A \veps_{\sys} \bclip (1-\gamma_A)^{-1} + \alpha_A \gamma_A^{\tau-1} \epsilon_0 
   +  \alphaa\nrm[\big]{\cM^{\dagger}}_{\op}\sum_{s=0}^{\tau}
     \gammaa^{\tau-s}\nrm*{ \wtilde\phi_s(\hat{h}_s,\by_{0:s+1}) -
     \cM\bw_s },\nn \\
     & \leq  \alpha_A \veps_{\sys} \bclip (1-\gamma_A)^{-1} + \alpha_A \gamma_A^{\tau-1} \epsilon_0 
       +  \alphaa\nrm[\big]{\cM^{\dagger}}_{\op}\sqrt{\sum_{s=0}^{\tau} \gammaa^{2(\tau-s)}\sum_{s=0}^{\tau}\nrm*{ \wtilde\phi_s(\hat{h}_s,\by_{0:s+1}) - \cM\bw_s }^2},\nn \\
  \label{eq:Llst} 
\end{align}
	Taking the square on both sides of \eqref{eq:Llst}, then applying the expectation $\mathbb{E}_{\what {\pi}}$, we get
	\begin{align}
&\E_{\what \pi}\left[ \max_{ 0\leq \tau \leq t}	\|
                        \tilde{f}_{\tau+1}(\by_{0:\tau+1}) -
                        f_\star(\by_{\tau}) \|^2\cdot
                        \Ind\{\cEroll_{0:t}\}\right] \\
          &\leq  3 \alpha^2_A \veps_{\sys}^2 \bclip^2 (1-\gamma_A)^{-2}  +3 \alpha^2_A \epszero^2+ 3 \alpha_A^2(1-\gamma^2_A)^{-1} \sigma_{\min}(\calM)^{-2} \sum_{s=0}^t \E_{\what \pi} \left[\| \wtilde\phi_s(\hat{h}_s,\by_{0:s+1}) -\calM \bw_s\|^2\right] ,\nn \\
&\leq  3 \alpha^2_A \veps^2_{\sys} \bclip^2 (1-\gamma_A)^{-2}  +3 \alpha^2_A  \epszero^2   +3 \alpha_A^2(1-\gamma^2_A)^{-1} \sigma_{\min}(\calM)^{-2}  \veps^2_{\noise} t,\label{eq:pineq}
	\end{align}	
where the last inequality follows by the fact that under the event $\cE_{\noise}$, we have 
\begin{align*}
\E_{\what \pi} \left[\| \wtilde\phi_s(\hat{h}_s,\by_{0:s+1}) -\calM \bw_s\|^2 \right]\leq \veps^2_{\noise}, \quad \text{for all } 0\leq s \leq t.
\end{align*}
Finally, we simplify \Cref{eq:pineq} to
\begin{align*}
 3 \alpha^2_A  (1-\gamma_A)^{-2} \left(\veps^2_{\sys} \bclip^2  +  \epszero^2 + \sigma_{\min}(\calM)^{-2}  \veps^2_{\noise}t\right ). 
\end{align*}

\end{proof}

\subsection{Proof of \pref{thm:clippingprob}}
\begin{proof}[Proof of \pref{thm:clippingprob}.]
	Define $\cEroll_{0:t}''\coloneqq {\cEroll}_{0:t} \wedge
        \cEroll_{0:t}'$ and let $p_t \coloneqq \P_{\what
          \pi}[\cEroll_{0:t}'' ]$. We will recursively prove a lower bound on
        $p_{t+1}$ in terms of $p_t$. From \pref{lem:decoder} and Markov's inequality, for all $\tau \in [t+1]$, 
	\begin{align}
          &\P_{\what \pi} \left[\left.  \max_{\tau \in [t+1]} \| \tilde{f}_{\tau} (\by_{0:\tau}) - f_\star(\by_{\tau})\| \geq \veps_{\dec,t}\sqrt{\eta  } \right|  \cEroll_{0:t}'' \right]\\ &=  \P_{\what \pi} \left[ \left. \max_{\tau \in [t+1]} \| \tilde{f}_{\tau} (\by_{0:\tau}) - f_\star(\by_{\tau})\|^2 \geq \eta \veps_{\dec,t}^2  \right|  \cEroll_{0:t}'' \right] \nn \\
	&  \leq \frac{1}{\eta \veps_{\dec,t}^2} \E_{\what \pi} \left[ \left. \max_{\tau \in [t+1]}  \| \tilde{f}_{\tau} (\by_{0:\tau}) - f_\star(\by_{\tau})\|^2 \right|  \cEroll_{0:t}''\right]\nn \\
	& = \frac{1}{\eta p_t  \veps_{\dec,t}^2} \E_{\what \pi} \left[ \max_{\tau \in [t+1]}  \| \tilde{f}_{\tau} (\by_{0:\tau}) - f_\star(\by_{\tau})\|^2 \cdot  \Ind\{\cEroll_{0:t}''\}\right]\nn \\
	& \leq \frac{1}{\eta p_t  \veps_{\dec,t}^2} \E_{\what \pi} \left[  \max_{\tau \in [t+1]}  \| \tilde{f}_{\tau} (\by_{0:\tau}) - f_\star(\by_{\tau})\|^2 \cdot  \Ind\{\cEroll_{0:t}\}\right],\nn \\
	& \leq p_t^{-1} \eta^{-1}. \label{eq:fconc}
	\end{align}
	On the other hand, we also have that under the event
        $\cEroll_{0:t}$, since no clipping occurs, the dynamics satisfy  
	\begin{align}
	\bx_{t+1} = (A + B \what K) \bx_t  + B \bnu_t + \bdelta_t +   \bw_t, \quad \text{where} \ \ \bdelta_t\coloneqq  \what B  \what K \tilde{f}_{t}(\by_{0:t}) - B \what K f_\star(\by_{t}). \nn
	\end{align} 
	Thus, by induction we obtain,
	\begin{align}
	\bx_{t+1} = (A+ B \what K)^{t}\bx_0  +\sum_{\tau=0}^t (A + B \what K)^{t-\tau} (B\bnu_t + \bdelta_t + \bw_t) ,\nn
	\end{align}
	By Jensen's inequality, we have for $\bar \gamma_\infty$ as in \eqref{eq:stab},
	\begin{align}
	\|\bx_{t+1}\| &=  \alpha_\infty\bar\gamma^{t}_{\infty} \| \bx_0\|  + \alpha_\infty\sum_{\tau=0}^t \bar\gamma_{\infty}^{t-\tau}  \|\bdelta_\tau\|_{2} + \left \| \sum_{\tau=0}^t (A +B \what K)^{t-\tau}  ( B\bnu_\tau + \bw_\tau) \right\|, \nn \\
	& =   \alpha_\infty\bar\gamma^{t}_{\infty} \|\bx_0\|  + \alpha_\infty\sum_{\tau=0}^t \bar\gamma^{t-\tau}_{\infty}  \|\bdelta_\tau\|_{2} +  \|\bz_t\|, \label{eq:nsum}
	\end{align}
	where $\bz_t \coloneqq \sum_{\tau=0}^t (A +B \what K)^{t-\tau}  (B\bnu_\tau + \bw_\tau)$. In this case, we have $\bz_t \sim \cN(0, \Sigma_z)$, where 
	\begin{align}
	\Sigma_z & \coloneqq \sum_{\tau=0}^t (A +B \what K)^{t-\tau} (\sigma^2 B B^\top + \Sigma_w)((A +B \what K)^{t-\tau})^\top \preceq  \Sigma_{z,\infty}, \label{eq:partorder}
	\end{align}
	with $\Sigma_{z,\infty}$ is as in \eqref{eq:thecovar}. Under the event $\cEroll'_{0:t}$, (and since $\gammab <1$) we have 
	\begin{align}
	b_\eta \coloneqq \sqrt{2(\dimu \|\Sigma_0\|_{\op} \alpha_\infty^2+ \dimx \|\Sigma_{z,\infty}\|_{\op}) \ln (2\eta)}& \geq \sqrt{2\hat \alpha_{\infty}^2\|\bx_0\|^2 + 2\|\bz_t\|^2}, \nn \\ & \geq \alpha_\infty\bar\gamma_{\infty}^{t}\|\bx_0\|+\|\bz_t\|. \label{eq:zxconc}  
	\end{align}
	On the other hand, by H\"older's inequality, we have
	\begin{align}
\alpha_\infty	\sum_{\tau=0}^{t} \bar\gamma_{\infty}^{t - \tau} \|\bdelta_\tau\| & \leq \frac{\alpha_\infty}{1 - \bar\gamma_\infty} \max_{0\leq \tau \leq t} \|\bdelta_\tau\|,\nn \\
	& \leq \frac{\alpha_\infty}{1 - \bar\gamma_\infty} \left(
   \veps_{\sys}\bclip + \max_{0\leq \tau\leq t}
   \|B\Khat(\tilde{f}_{\tau}(\by_{0:\tau})  - f_\star(\by_\tau))\|
   \right),\nn\\
          & \leq \frac{\alpha_\infty}{1 - \bar\gamma_\infty} \left( \veps_{\sys}\bclip + 2\Psist\max_{0\leq \tau\leq t} \|\tilde{f}_{\tau}(\by_{0:\tau})  - f_\star(\by_\tau)\| \right),\nn
	\end{align}
        where we have used that under $\cE_{\sys}$,
$\nrm{B\Khat}_{\op}\leq{}\nrm*{A}_{\op}+\nrm{A+B\Khat}_{\op}\leq\Psist+\alphainf\gambarinf\leq2\Psist$. From
        \eqref{eq:fconc}, it follows that
	\begin{align}
	\P_{\what \pi}\left[ \left.	 \alpha_\infty \sum_{\tau=0}^{t} \bar\gamma_{\infty}^{t - \tau} \|\bdelta_\tau\|   \geq  \frac{\alpha_\infty(\veps_{\sys} \bclip+ 2\Psist\veps_{\dec,t}\sqrt{\eta})}{1 - \bar\gamma_\infty}  \ \right| \ \cEroll_{0:t}''   \right] \leq  p_t^{-1} \eta^{-1}. \label{eq:dconc}
	\end{align}
	Thus, by \eqref{eq:nsum}, \eqref{eq:fconc}, \eqref{eq:zxconc}, and \eqref{eq:dconc}, we have 
	\begin{align*}
	\P_{\what \pi }[\cEroll \mid \cEroll_{0:t}'' ] \leq  \eta^{-1}
          p_t^{-1}, \nn 
        \end{align*}
        where
        \begin{align*}
\cEroll \coloneqq \left\{ \begin{matrix}    \|\bx_{t+1}\| \geq   (1 - \bar\gamma_\infty)^{-1} \alpha_\infty(\veps_{\sys} \bclip+ 2\Psist\veps_{\dec,t}\sqrt{\eta}) + b_\eta,\\ \text{or} \  \ \| \tilde{f}_{t+1} (\by_{0:t+1})\| \geq \|\bx_{t+1} \| +  \veps_{\dec,t}\sqrt{\eta  }\end{matrix}\right\}.\nn
	\end{align*}
	This implies that 
	\begin{align}
	\P_{\what \pi}\left [  \|\tilde{f}_{t+1}(\by_{t+1})\| \geq  \left.   \frac{\alpha_\infty \veps_{\sys} \bclip+(2\Psist\alpha_\infty +1 - \bar\gamma_\infty) \veps_{\dec,t}\sqrt{\eta}}{1 - \bar\gamma_\infty}  + b_\eta \  \right|\  \cEroll_{0:t}'' \right] \leq \eta^{-1} p_t^{-1}.
	\end{align}
        which we simplify to
        \begin{align}
	\P_{\what \pi}\left [  \|\tilde{f}_{t+1}(\by_{t+1})\| \geq  \left.   \frac{\alpha_\infty \veps_{\sys} \bclip+3\Psist\alpha_\infty\veps_{\dec,t}\sqrt{\eta}}{1 - \bar\gamma_\infty}  + b_\eta \  \right|\  \cEroll_{0:t}'' \right] \leq \eta^{-1} p_t^{-1}. \label{eq:f2conc}
	\end{align}
	On the other hand, by \pref{lem:thechis}, we have 
	\begin{align}
	\P_{\what\pi}\left[ \alpha_\infty^2 \|\bx_0\|^2 + \|\bz_{t+1}\|^2 \geq b_\eta 2^{-1/2}  \right] \leq \eta^{-1}.\label{eq:zx2conc}
	\end{align}
	Thus, for \begin{equation}\bclip \geq (1 - \bar\gamma_\infty)^{-1} ({\alpha_\infty \veps_{\sys} \bclip+3\Psist\alpha_\infty\veps_{\dec,t}\sqrt{\eta}}) + b_\eta,\label{eq:bclip_proof}\end{equation} we have with \eqref{eq:zx2conc}, \eqref{eq:f2conc}, and a union bound, 
	\begin{align}
	\P_{\what \pi}\left [  \left\{ \|\tilde{f}_{t+1}(\by_{t+1})\| \geq    \bclip\right\}\left. \vee \left\{  \alpha_\infty^2\|\bx_0\|^2 + \|\bz_{t+1}\|^2 \geq b_\eta 2^{-1/2} \right\} \  \right|\  \cEroll_{0:t}'' \right] \leq 2 \eta^{-1} p_t^{-1}.\nn\
	\end{align}
	This implies that
	\begin{align}
	\P_{\what \pi} [\cEroll_{t+1}\wedge \cEroll'_{t+1}  \mid \cEroll''_{0:t}] \geq 1 - 2 \eta^{-1}p^{-1}_t.
	\end{align}
	Therefore, we have \begin{align}\P_{\what \pi} [ \cEroll''_{0:t+1}] = p_t \cdot \P_{\what \pi} [\cEroll_{t+1}\wedge \cEroll'_{t+1}  \mid \cEroll''_{0:t}] \geq p_t - 2 \eta^{-1}= \P_{\what \pi} [ \cEroll''_{0:t}] - 2 \eta^{-1}.\nn \end{align}
	Now by induction on $t$ we get, for all $\tau \geq 1$, 
	\begin{align}
	\P_{\what\pi}[\cEroll_{0:\tau}''] \geq \P_{\what \pi
          }[\cEroll_{0:1} \wedge\cEroll_{0:1'}] - 2\tau/\eta . \label{eq:ineq}
	\end{align}
	For the base case, by \eqref{eq:zx2conc} and a union bound, it follows that 
	\begin{gather}
	\P_{\what \pi}[\neg\cEroll_{0:1}'] \leq 2 \eta^{-1},
	\shortintertext{and therefore, by \eqref{eq:ineq}, we get}
		\P_{\what\pi}[\cEroll_{0:\tau}''] \geq \P_{\what \pi }[\cEroll_{0:1}] - 2 (\tau+1)/\eta.\nn
	\end{gather}
	Finally, as $\tilde{f}_0\equiv 0$, we have $\P_{\what\pi}[\cEroll_{0:1}] =\P_{\what\pi }[\cEroll_{1}]$, which completes the proof. To get the stated value for $\bclip_{\infty}$, we rearrange \pref{eq:bclip_proof} and recall that $\gambarinf=\frac{1}{2}(1+\gammainf)$.
\end{proof}

\subsection{Proof of \pref{thm:minusinit}}
\begin{proof}
	Let $\cEroll_{0:T}'' \coloneqq \cEroll_{0:T} \wedge
        \cEroll_{0:T}'$. By \pref{lem:decoder}, under the events
        $\cE_{\sys}$ and $\cE_{\noise}$, we have
	\begin{align}
	\E_{\what \pi} \left[ \max_{ 0\leq  t \leq T} \|  \hat{f}_t(\by_{0:t}) - f_\star(\by_{t}) \|^2 \cdot \Ind\{\cEroll_{0:T}\}  \right] \leq    \veps^2_{\dec,T}. \label{eq:thefbound}
	\end{align}	
	It follows that for all $0 \leq t\leq T$, we have 
	\begin{align}
	\E_{\what \pi} \left[ \max_{ 0\leq  t \leq T} \|  \hat{f}_t(\by_{0:t}) - f_\star(\by_{t}) \|^2   \right] & \leq \E_{\what \pi} \left[ \max_{ 0\leq  t \leq T} \|  \hat{f}_t(\by_{0:t}) - f_\star(\by_{t}) \|^2 \cdot \Ind\{\cEroll''_{0:T}\}  \right] \nn \\
	& \quad +  \E_{\what \pi} \left[ \max_{ 0\leq  t \leq T} \|  \hat{f}_t(\by_{0:t}) - f_\star(\by_{t}) \|^2 \cdot \Ind\{\neg \cEroll''_{0:T}\}  \right] ,\nn \\ 
	& \leq \E_{\what \pi} \left[ \max_{ 0\leq  t \leq T} \|  \hat{f}_t(\by_{0:t}) - f_\star(\by_{t}) \|^2 \cdot \Ind\{\cEroll_{0:T}\}  \right] \nn \\
	& \quad +  \E_{\what \pi} \left[(2 \bclip +2 \max_{ 0\leq  t \leq T} \|  \bx_t \|^2 )\cdot \Ind\{\neg \cEroll''_{0:T}\}  \right] ,\nn \\ 
	& \leq \veps^2_{\dec,T}+  2(1 -\P_{\what\pi}[ \cEroll''_{0:T}]) \left(\bclip +\sqrt{\E_{\what \pi} \left[ \max_{ 0\leq  t \leq T} \| \bx_t \|^4 \right]}\right),\label{eq:precconc}
	\end{align}where the last inequality follows by Cauchy Schwarz and \pref{eq:thefbound}. Now by \pref{lem:xconcentration}, we have that the random variable $\|\bx_t\|^2$ is $\cx$-concentrated for all $0\leq t \leq T$ with 
	\begin{align}
	\cx\coloneqq 30\dimx\PsiSig  +  2\bclip^2   \cdot \devx\nn.
	\end{align}
	Therefore, by \pref{lem:techtools_truncated_conc}, we have $\E_{\what \pi} \left[ \| \bx_t \|^4\right] \leq 4 \cx^2$. Using this, we get that
	\begin{align}
	\E_{\what \pi} \left[ \max_{ 0\leq  t \leq T} \| \bx_t \|^4\right]& \leq (T+1)\max_{0\leq{}t\leq{}T}\E_{\what \pi} \left[ \| \bx_t \|^4\right]\leq 8T \cx^2.\nn
	\end{align}
	Combining this with \eqref{eq:precconc}, and \pref{thm:clippingprob}, we have 
	\begin{align}
	\E_{\what \pi} \left[ \max_{ 0\leq  t \leq T} \|
          \hat{f}_t(\by_{0:t}) - f_\star(\by_{t}) \|^2   \right]
&  \leq \veps^2_{\dec,t} + (4T^{1/2} \cx+2 \bclip^2)(1-\bbP_{\pihat}\brk{\cEroll_{0:T}''}).\nn\\
          &  \leq \veps^2_{\dec,t} + (4T^{1/2} \cx+2 \bclip^2)(\tfrac{2(T+1)}{\eta}+1-\P_{\what \pi}[\{\tilde{f}_0(\by_0) = \hat{f}_0(\by_0)\} \wedge \cEroll_0' ] ).\nn
	\end{align}
\end{proof}

\subsection{Proof of \pref{lem:stateestimate}}
For the proof of \pref{lem:stateestimate}, we introduce the following functions and random vectors:
\begin{align*}
\varphi(\by_{1})  &\coloneqq   \Lonpo^2 ( 1\vee \|\bx_1\|^2), \nn \\
\bu  &\coloneqq  \bw_0, \quad \\
&\be  \coloneqq      \bw_0 - ( \hat{h}_{\ol,0} (\by_1^{(i)} )- \what A \hat{h}_{\ol,0}(\by_0^{(i)})    - \what B \bnu_0). 
\end{align*} 
Let us abbreviate $n\equiv\ninit$. Recall that we are analyzing the regression
\begin{align*}
\hat{h}_{\ol,1} \in \argmin_{h \in \scrH_{\onpo}} \sum_{i=1}^{n} \left\| h(\by^{(i)}_1)  - \left(\hat{h}_{\ol,0}(\by^{(i)}_1) - \what A \hat{h}_{\ol,0}(\by^{(i)}_0) -  \what B \bnu_0^{(i)}\right)  \right\|^2,
\end{align*}
where $\{(\by_{\tau}^{(i)},\bx_{\tau}^{(i)},\bnu_{\tau}^{(i)})\}_{1
  \leq i\leq n}$ are fresh i.i.d.~trajectories generated by the policy $\pi_{\ol}$. 

\begin{proof}[Proof of \pref{lem:stateestimate}]
	Our strategy will be to invoke \Cref{cor_techtools:reg_cor_simple} with $\varphi, \bu$, and $\be$ as above. We start by verifying the technical conditions of the corollary.
	\begin{enumerate}
		\item We directly verify from the structure of $\Honpo$  we may take $b = \Psi_{\star}^{3}$ and $L$ as in  \Cref{asm:f_growth}. Hence, $bL = \Lonpo$, and thus $\varphi(\by_1)$ satisifes the requisite conditions of the $\varphi$ function. In addition, we may take $\dimu,\dimx \gets \dimx$.
		\item \textbf{Concentration Property.} Next, we bound
                  the concentration parameter $c$. Recall that all
                  $h\in\Honpo$ have
                  $\nrm*{h(\maty_t)}\leq{}\Lonpo\max\crl{1,\nrm*{\matx_t}}$. Hence,
                  under the event $\cE_{\sys}$, we have by Jensen's inequality, Cauchy-Schwarz, and the fact that $\hat{h}_{\ol,0} \in \scrH_{\onpo}$, 
	\begin{align}
	\varphi(\by_1) \vee \|\be -  \bu \|^2 &\leq  \varphi(\by_1)  +5\Lonpo^2 (1 + \|\bx_1\|^2) + 5 \Lonpo^2(\|A\|^2_{\op} + \veps_\sys^2)(1+ \|\bx_0\|^2)\nn \\ & \quad + 5(\|B\|^2_{\op}+ \veps_{\sys}^2) \|\bnu_0\|^2, \nn \\
	& \leq 6\Lonpo^2 (1 + \|\bx_1\|^2) + 5 \Lonpo^2(\|A\|^2_{\op} + \veps^2_\sys)(1+ \|\bx_0\|^2)\nn \\ & \quad + 5(\|B\|^2_{\op}+ \veps_{\sys}^2) \|\bnu_0\|^2, \nn \\
	& = 6 \Lonpo^2 + 5 \Lonpo^2 (\|A\|^2_{\op} + \veps_\sys^2)+ 5(\|B\|^2_{\op}+ \veps_{\sys}^2) \|\bnu_0\|^2 \\ & \quad+ 6 \Lonpo^2 \|\bx_1\|^2 + 5 \Lonpo^2 (\|A\|^2_{\op} + \veps_\sys^2) \|\bx_{0}\|^2.\nn
	\end{align}
	Let us simplify the above. Assume $\epsys^2 \le \Psist^2$ (where $\Psist \ge \max\{1,\|A\|_{\op},\|B\|_{\op}\}$). This lets us simplify the above by 
	\begin{align}
	\varphi(\by_1) \vee \|\be -  \bu \|^2 &\le \Lonpo^2\Psist^{2}( 6 + 10 + 10 \|\bnu_0\|^2 + 6\|\bx_1\|^2 + 10 \|\bx_0\|^2 ) \nn\\
	&\le 16\Lonpo^2\Psist^{2}( 1 + \|\bnu_0\|^2  +\|\bx_1\|^2 + \|\bx_0\|^2 ). 	\label{eq:0pain}
	\end{align}
	Since $\bx_0 \sim \cN(0, \Sigma_0)$, $\bnu_0 \sim \cN(0,
        \sigma^2 I_{\dimu})$, and $\bx_1 \sim \cN(0, \sigma^2 B B^\top
        + A \Sigma_0 A^\top + \Sigma_w)$, we have by
        \pref{lem:thechis} and the fact that $\ln 3 +1\leq 3$, the
        following holds: For all $\delta \in(0,1/e]$,
	\begin{align}
	&	\P_{\pi_{\ol}} \left[ \varphi(\by_1) \vee \|\be -  \bu
   \|^2\geq c_1 \ln \delta^{-1} \right] \leq \delta, \label{eq:cond1}
        \end{align}
        where
        \begin{align}
           c_1 &\coloneqq  48\Lonpo^2\Psist^{2}( 1 + 5\dimu \sigma^2 + 5\dimx (\|\Sigma_1\|_{\op} + \|\Sigma_0\|_{\op}) ), \label{eq:the_c1}\\
	\Sigma_1 &:= \sigma^2 B B^\top  + A \Sigma_0 A^\top + \Sigma_w.
	\end{align}
	\item \textbf{Bounding the error $\mate$.} On the other hand, by Jensen's inequality and Cauchy-Schwarz, we can bound the error $\be$ by 
	\begin{align}
	\|\be\|^2 \leq 2 \|\bw_0 -  \hat{h}_{\ol,0} (\by_1^{(i)} )-  A \hat{h}_{\ol,0}(\by_0^{(i)}) - B \bnu_0\|^2 + 4 \veps_\sys^2 \Lonpo^2 \left( 1 + \|\bx_0\|^2 + \|\bnu_0\|^2 \right) .\nn
	\end{align}
	Therefore, by \pref{lem:theexpchi} and
        \eqref{eq:baseguarantee}, we have with probablity at least
        $1-\delta$ over the trajectories used to form $\hhat_{\ol,0}$,
	\begin{align}
	\E_{\pi_{\ol}}[ \|\be \|^2 ] \leq 2\sigma_{\min}(\calM)^{-2} \veps^2_{\noise}(\delta) + 4 \veps_{\sys}^2\Lonpo^2 (1 + \dimx \|\Sigma_0\|_{\op} + \sigma^2 \dimu).  \label{eq:thee}
	\end{align}
	\item \textbf{Realizability.} We have 
	\begin{align} 
	\E_{\pi_{\ol}}[\bu \mid \by_1 ]& = \E_{\pi_{\ol}}[\bw_0 \mid \bx_1], \nn  \\  &= \E_{\pi_{\ol}}[\bw_0 \mid A \bx_0 + B \bnu_0 + \bw_0], \nn \\
	& = \Sigma_w (\Sigma_w + \sigma^2 B B^\top + A \Sigma_0
   A^\top)^{-1} \bx_1 \rdef \hst(\by_1),\label{eq:cond3}
	\end{align}
	where the last inequality follows by \pref{fact:gaussian_expectation}. Therefore, by the definition of $\scrH_{\onpo}$ in \eqref{eq:funclass} and the fact that $\|\Sigma_w (\Sigma_w + \sigma^2 B B^\top + A \Sigma_0 A^\top)^{-1} \|_{\op}\leq 1 \le \Psi_\star^3$, we are guaranteed the existence of $h \in  \scrH_{\onpo}$ such that $h(\bx_1)= \E[\bu \mid \by_1]$. 
	\end{enumerate}
	Applying \Cref{cor_techtools:reg_cor_simple} with $c \gets
        c_1$, $\dimu,\dimx \gets \dimx$, and the above bound on
        $\E_{\pi_{\ol}}[ \|\be \|^2 ] $, we obtain for $n \gets
        n_{\init}$ that with probability $1 - \frac{3\delta}{2} -
        \delta = 1 - 5\delta/2$ (the second $\delta$ factor comes from the event used to bound $\E_{\pi_{\ol}}[ \|\be \|^2 ]$),
	\begin{align*}
		&\E\| \hat{h}_{\ol,1} (\by) - \hst(\by)\|^2  \lesssim \\
		 &\frac{ c_1(\dimx^2 +\ln|\Fclass|) \ln^2\tfrac{n_{\init}}{\delta}}{n_{\init}} + \sigma_{\min}(\calM)^{-2} \veps^2_{\noise}(\delta) +  \veps_{\sys}^2\Lonpo^2 (1 + \dimx \|\Sigma_0\|_{\op} + \sigma^2 \dimu), 
		\end{align*}
		as needed. Moreover, because the above bound
                suppresses constants, we can replace $c_1$ in
                \Cref{eq:the_c1} by $c_1 \gets  \Lonpo^2\Psist^{2}( 1
                + \dimu \sigma^2 + \dimx (\|\Sigma_1\|_{\op} +
                \|\Sigma_0\|_{\op}) )$ . Substituting in the
                definition of $\hst$ concludes the proof.
\end{proof}

\subsection{Proof of \pref{lem:covar}}
\begin{proof}
	Since $\hat{h}_{\ol,1}\in \scrH_{\onpo}$, we have $\|\hat{h}_{\ol,1}(\by_1)\| \leq L_\op (1 \vee \|\bx_1\|)$, where $\bx_1 \sim \cN(0, \Sigma_1)$ and $\Sigma_1 \coloneqq \sigma^2 B B^\top + A \Sigma_0 A^\top + \Sigma_w$. Therefore, by \pref{lem:thechis}, we have, for all $\delta \in (0,1/e]$,
	\begin{align}
	\P_{\pi_\ol} \left [  \|\hat{h}_{\ol, 1}   \|^2 \geq \Lonpo^2   (1 + (3 \dimx +2)\|\sigma^2 B B^\top + A \Sigma_0 A^\top + \Sigma_w\|_{\op} ) \ln \delta^{-1}  \right]\leq \delta.
	\end{align} 
	Combining this with the fact that $\|\hat{h}_{\ol,1}(\by_1)\| \leq L_\op (1 \vee \|\bx_1\|)$ implies that $\varphi(\by_1)\coloneqq  L_\op^2(1\vee \|\bx_1\|^2)$ is $c$-concentrated with $c= \Lonpo^2   (1 + (3 \dimx +2)\|\sigma^2 B B^\top + A \Sigma_0 A^\top + \Sigma_w\|_{\op})$. Thus, applying \pref{prop:techtools_general_pca} with $$(\varphi(\by_1), \hat h(\by_1), h_\star(\by_1))=  (L_\op^2(1\vee \|\bx_1\|^2),  \hat h_{\ol,1}(\by_1),\Sigma_w \Sigma_1^{-1} \bx_1)$$ and invoking \pref{lem:stateestimate}, we get for all $\delta \in (0,1/e]$, with probability at least $1 - (3\dimk + 4) \delta$,  
	\begin{gather} 
	\|  \what \Sigma_\cv - \Sigma_w \Sigma_1^{-1} \Sigma_w \|_{\op} \leq \veps_\cv', \label{eq:thekey}
      \end{gather}
      for $\veps_{\cv}'$ as in the lemma statement. By the triangle
      inequality, whenever the condition \eqref{eq:covevent} that
      $\veps_\cv' < \sigma_{\min}(\Sigma_\cv)/2\leq\nrm*{\Sigma_\cv}_{\op}$ holds, this implies that
	\begin{align}
	\|\what \Sigma_\cv\|_{\op} \leq 2 \|\Sigma_w \Sigma_1^{-1} \Sigma_w\|_{\op}, \mathand \sigmamin(\wh{\Sigma}_{\cv})\geq{}\sigmamin(\Sigma_{\cv})/2
	\end{align}
	which shows the second inequality in
        \eqref{eq:thebound}. Furthermore, whenever \eqref{eq:thekey}
        holds, \pref{lem:theeig} and the condition \pref{eq:covevent} imply that
	\begin{align}
	\|  \Sigma_w^{-1}- \what \Sigma^{-1}_\cv \Sigma_w \Sigma_1^{-1} \|_{\op} \leq \frac{2\|\Sigma_w^{-1}\|_{\op} \veps_\cv' }{\sigma_{\min}(\Sigma_w \Sigma_1^{-1} \Sigma_w) }. \label{eq:thepre}
	\end{align}
        By the triangle inequality, this implies that under
        $\cE_{\sys}$,
        \[
	\|  \wh{\Sigma}_w^{-1}- \what \Sigma^{-1}_\cv \Sigma_w
        \Sigma_1^{-1} \|_{\op} \leq \veps_{\sys} + \frac{2\|\Sigma_w^{-1}\|_{\op} \veps_\cv' }{\sigma_{\min}(\Sigma_w \Sigma_1^{-1} \Sigma_w) }.
        \]
This further implies that 
	\begin{align}
	\|  I_{\dimx}- \what \Sigma_w \what \Sigma^{-1}_\cv \Sigma_w \Sigma_1^{-1} \|_{\op} &\leq \nrm{\wh{\Sigma}_w}_{\op}\prn*{\veps_\sys  +\frac{2\|\Sigma_w^{-1}\|_{\op} \veps_\cv' }{\sigma_{\min}(\Sigma_w\Sigma_1^{-1} \Sigma_w) }}.
	\end{align}
        Since $\veps_{\sys}\leq1$ by the
        definition of $\cE_{\sys}$, we have that
        $\nrm{\wh{\Sigma}_w}_{\op}\leq2\Psist$,
        leading to the result.
	This establishes the main inequality in \eqref{eq:thebound}.
\end{proof}

\subsection{Proof of \pref{thm:stateestimate2}}
For the proof of \pref{thm:stateestimate2}, we introduce the following functions and random vectors:
\begin{align*}
\varphi(\by_{0})  &\coloneqq   \Lonpo^2 ( 1\vee\|\bx_0\|^2)\\
\bv  &\coloneqq  \Sigma_w (\Sigma_w + \sigma^2 B B^\top+  A \Sigma_0 A^\top)^{-1}  (A \bx_0 + \bnu_0 + \bw_0) \\
\be  &\coloneqq     \bv  -  \hat{h}_{\ol,1}(\by_1). 
\end{align*} 
Recall that we are analyzing the following regression problem, where for $n \equiv n_\init$:
\begin{align*}
\tilde{h}_{\ol,0} \in \argmin_{h \in \scrH_{\onpo}} \sum_{i=n+1}^{2 n} \left\|h(\by_0^{(i)})  -    \hat{h}_{\ol,1}(\by^{(i)}_1) \right\|^2.
\end{align*}

\begin{proof}[Proof of \pref{thm:stateestimate2}]
	Our strategy will be to invoke \Cref{cor_techtools:reg_cor_simple} with $\varphi$, $\bv$, and $\be$ as above. We verify the technical conditions of the corollary. 
	\begin{enumerate}
		\item We directly verify from the structure of $\Honpo$  we may take $b = \Psi_{\star}^3$ and $L$ as in  \Cref{asm:f_growth}. Hence, $bL = \Lonpo$, and thus $\varphi(\by_1)$ satisifes the requisite conditions of the $\varphi$ function. 
		\item \textbf{Concentration property.}
		Now, under event $\cE_{\sys}$, by Jensen's
                inequality, Cauchy-Schwarz, and the fact that
                $\hat{h}_{\ol,1} \in \scrH_{\onpo}$ (and hence
                satisfies
                $\nrm{\hat{h}_{\ol,1}(\maty_1)}\leq{}\Lonpo\max\crl{1,\nrm{\matx_1}}$),
                we have
		\begin{align}
		\varphi(\by_0) \vee \|\be -  \bv \|^2 &\leq  \varphi(\by_0) \vee (\Lonpo^2 (1 + \|\bx_1\|^2)),\nn\\
		& \leq \Lonpo^2 (1 + \|\bx_0\|^2 + \|\bx_1\|^2).\nn
		\end{align}
		Since $\bx_0 \sim \cN(0, \Sigma_0)$ and
                $\matx_1\sim\cN(0,\Sigma_1)$, where $\Sigma_1=\Sigma_w
                + \sigma^2 B B^\top+  A \Sigma_0 A^\top$, 
                we have by \pref{lem:thechis} that for all $\delta \in(0,1/e]$,
		\begin{align}
	          	\P_{\pi_{\ol}} \left[ \varphi(\by_0) \vee \|\be -  \bv \|^2\geq c'_0 \ln \delta^{-1} \right] \leq \delta, 
	   \quad\text{where} \quad c'_0 \coloneqq  2\Lonpo^2( 1 +
	                                 5\dimx (\|\Sigma_0\|_{\op}+\nrm{\Sigma_1}_{\op})
	                                 ). 
	                                 \label{eq:0cond1_prelim}
		\end{align}
                Moreover, since $\sigma\leq{}1$ under
                \pref{ass:sigma_small}, we have
                $\nrm{\Sigma_1}\leq3\Psist^{3}$, so that
                \[
                  c_0'\leq{}32\Lonpo^{2}\Psist^{3}\dimx\rdef{}c_0,
                \]
                and hence
                \begin{align}
                	          	\P_{\pi_{\ol}} \left[ \varphi(\by_0) \vee \|\be -  \bv \|^2\geq c_0 \ln \delta^{-1} \right] \leq \delta.
	                                 \label{eq:0cond1}
                \end{align}

	\item \textbf{Bounding the error $\mate$.} By \pref{lem:stateestimate}, we have with probablity at
          least $1-5\delta/2$ over $\hhat_{\ol,1}$,
	\begin{align}
	\E_{\pi_\ol}[ \|\be \|^2 ] \leq  \veps^2_{\ol,1}.  \label{eq:2thee}
	\end{align}
	\item \textbf{Realizability.} We have 
	\begin{align} 
	\E_{\pi_{\ol}}[\bv \mid \by_0 ]& = \E_{\pi_{\ol}}[\bv \mid \bx_0], \nn  \\  &= \E_{\pi_{\ol}}[\Sigma_w (\Sigma_w + \sigma^2 B B^\top+  A \Sigma_0 A^\top)^{-1}  (A \bx_0 + \bnu_0 + \bw_0)\mid \bx_0 ], \nn \\
	& = \Sigma_w (\Sigma_w + \sigma^2 B B^\top+  A \Sigma_0 A^\top)^{-1}  A \bx_0,\nn \\
	& = \Sigma_w \Sigma^{-1}_1 A \bx_0 := \hst(\bx_0).\label{eq:2cond3}
	\end{align}
Therefore, by the definition of $\scrH_{\onpo}$ in \eqref{eq:funclass} and the fact that $\|\Sigma_w \Sigma_1^{-1}A \|_{\op}\leq \Psi_\star^3$, we are guaranteed the existence of $h \in  \scrH_{\onpo}$ such that $h(\bx_0)= \E[\bv \mid \by_0]$.
	\end{enumerate}

	Applying \Cref{cor_techtools:reg_cor_simple} with $c \gets c_0$, $\dimu,\dimx \gets \dimx$, and the above bound on $\E_{\pi_{\ol}}[ \|\be \|^2 ] $, we obtain for $n \gets n_{\init}$ that with probability $1 - \frac{3\delta}{2} - \frac{5\delta}{2} = 1 - 4\delta$ (the second term comes from the event used to bound $\E_{\pi_{\ol}}[ \|\be \|^2 ]$),
	\begin{align*}
		&\E\| \tilde{h}_{\ol,0}(\by_0) - \hst(\by)\|^2 \lesssim \frac{ c_0(\dimx^2 +\ln|\Fclass|) \ln\prn{\tfrac{n_{\init}}{\delta}}^{2}}{n_{\init}} + \veps^2_{\ol,1}
		\end{align*}
		as needed. In particular, recalling that $\hst(\by_0) = \Sigma_w \Sigma_1^{-1} A \bx_0$ in the above realizability discussion, we find that for an appropriate upper bound $\tilde{\veps}^2_{\ol,1}$,
		\begin{align*}
		&\E\| \tilde{h}_{\ol,0}(\by_0) -  \Sigma_w \Sigma_1^{-1} A \bx_0\|^2 \le \tilde{\veps}^2_{\ol,1} \lesssim \frac{ c_0(\dimx^2 +\ln|\Fclass|) \ln\prn{\tfrac{n_{\init}}{\delta}}^{2}}{n_{\init}} + \veps^2_{\ol,1}
		\end{align*}
This further implies that 
	\begin{align}
	\E_{\pi_{\ol}} \left[ \| \what \Sigma_w \what \Sigma^{-1}_\cv \tilde{h}_{\ol,0} (\by_0) -  A \bx_0 \|^2  \right] &\leq  2 \|\what \Sigma^{-1}_w\|^2_{\op} \|\what \Sigma_\cv\|^2_{\op} \tilde{\veps}^2_{\ol,1} \nn \\
	& ~~~~ + 2\E_{\pi_{\ol}} \left[ \|(I_{\dimx} -  \what \Sigma_w \what \Sigma^{-1}_\cv  \Sigma_w \Sigma_1^{-1})A \bx_0\|^2 \right], \nn \\
	& \leq 2 \|\what \Sigma^{-1}_w\|^2_{\op} \|\what \Sigma_\cv\|^2_{\op} \tilde{\veps}^2_{\ol,1}  \nn \\
	& ~~~~ + 2 \dimx \|(I_{\dimx} -  \what \Sigma_w \what \Sigma^{-1}_\cv  \Sigma_w \Sigma_1^{-1})A\|_\op^2 \|\Sigma_0\|, \label{eq:pinek}
	\end{align}
	where the last inequality follows by \pref{lem:theexpchi} since $\bx_0 \sim \cN(0, \Sigma_0)$. 
	Thus, under the event $\cE_{\sys}$, we have by
        \pref{lem:covar} and a union bound, with probability at least $1-(3\kappa+9)\delta$,
	\begin{align}
          \E_{\pi_{\ol}} \left[ \| \what \Sigma_w \what \Sigma^{-1}_\cv \tilde{h}_{\ol,0} (\by_0) -  A \bx_0 \|^2  \right] &\approxleq  \| \Sigma^{-1}_w\|^2_{\op} \|\Sigma_\cv\|^2_{\op} \left(\frac{ c_0(\dimx^2 +\ln|\Fclass|) \ln\prn{\tfrac{n_{\init}}{\delta}}^{2}}{n_{\init}} + \veps^2_{\ol,1}\right) \nn \\
	& \quad  + \dimx \veps^2_\cv \|A\|_\op^2 \|\Sigma_0\|,
	\end{align}
	where we recall that $\Sigma_\cv = \Sigma_w
        \Sigma_1^{-1}\Sigma_w$ by definition. The desired bound \eqref{eq:firstb} follows by the fact $\pi_\ol$ and $\what \pi$ match at round zero. 
	
	We now prove that \pref{eq:secondb} holds. By definition of $\tilde{f}_1$, we have 
	\begin{align}
	\tilde{f}_1(\by_{0:1}) = \what \Sigma_w \what \Sigma^{-1}_\cv \tilde{h}_{\ol,0} (\by_0) + \hat{h}_0 (\by_{1}) - \what A \hat{h}_0(\by_0),
	\end{align}
	and so since $\tilde{h}_{\ol,0}$ and $\hat{h}_0$ are in $\scrH_{\onpo}$, we have by Jensen's inequality and Cauchy-Schwarz, 
	\begin{align}
	\|\tilde{f}_1(\by_{0:1})\|^2& \leq 4\|\what \Sigma_w\|_{\op}^2 \|\what \Sigma^{-1}_\cv\|_{\op}^2 \Lonpo^2 ( 1 +\|\bx_0\|^2) \nn \\ & \quad + 4\Lonpo^2 (1 + \|\bx_1\|^2) + 4\|A\|^2_{\op} (1 + \|\bx_0\|^2) + 4\|A - \what A\|_\op^2 (1 + \|\bx_0\|^2).
	\end{align}
        Under $\cE_{\sys}$ we have
        $\nrm{\wh{\Sigma}_{w}}_{\op}\leq{}2\nrm{\Sigma_w}$ and $\nrm{\Ahat-A}_{\op}\leq{}1$, and the
        event of \pref{lem:covar} implies that
        $\nrm{\wh{\Sigma}_{\cv}^{-1}}_{\op}\leq{}2
        \nrm{\Sigma_{\cv}^{-1}}_{\op}$. Hence, using that
        $\Lonpo\geq{}1$, we can further upper bound by
        \begin{align*}
          \|\tilde{f}_1(\by_{0:1})\|^2  \leq{}
          \Lonpo^{2}(64\nrm{\Sigma_w}_{\op}^2\nrm{\Sigma^{-1}_\cv}_{\op}^2
          + 4\nrm{A}_{\op}^{2}+4)(1+\nrm{\matx_0}^2) + 4\Lonpo^2(1+\nrm{\matx_1}^2).
        \end{align*}
Next, we note that $\nrm{\Sigma_w}_{\op}\leq{}\Psist$ and
$\nrm{\Sigma_\cv^{-1}}_{\op}\leq{}3\Psist^{5}$. Hence, we can further
simplify this bound to
\[
  584\Lonpo^{2}\Psist^{12}(2 + \nrm{\matx_0}^{2} + \nrm{\matx_1}^2).
\]
	Since $\bx_0 \sim \cN(0,\Sigma_0)$ and $\bx_1 \sim \cN(0, \Sigma_1)$, we have, by \pref{lem:thechis}, 
	\begin{gather}
	\P_{\what \pi} \left[ \|\tilde{f}_1(\by_{0:1})\|^2  \geq 584\Lonpo\Psist^{12}(2+(3\dimx+2)(\nrm{\Sigma_0}_{\op}+\nrm{\Sigma_1}_{\op}))\ln (2\eta) \right] \leq \eta^{-1}, \nn 
	\intertext{so that in particular, we may take} \bclip^2_0 \coloneqq  10^{4}\dimx\Lonpo^{2}\Psist^{12}(1+\nrm{\Sigma_0}_{\op}+\nrm{\Sigma_1}_{\op}).
	\end{gather}
	This establishes \pref{eq:secondb}.
\end{proof}

\subsection{Supporting Results}

\begin{lemma}
	\label{lem:theexpchi}
	Let $\bz \sim \cN(0, \Sigma)$, where $\Sigma\in \reals^{m\times m}$ is a positive definite matrix. Then $\E[\|\bz\|^2] \leq  m \|\Sigma\|_{\op}$.
\end{lemma}
\begin{proof}
	Let $\bz' \coloneqq \Sigma^{-1/2}\bz$ and note that $\bz' \sim \cN(0, I_m)$, and so $\|\bz'\|^2 \sim \chi^2(m)$. As a result, we have 
	\begin{align}
	\E[\|\bz\|^2 ] = \E[\|\Sigma^{1/2}\bz'\|^2]\stackrel{(*)}{\leq} \| \Sigma\|_{\op} \E[\|\bz'\|^2] = m \| \Sigma\|_{\op}, \nn
	\end{align} 
	where $(*)$ follows by Cauchy-Schwarz.
\end{proof}
\begin{lemma}
	\label{lem:theprechis}
	Let $a_0>0$ and $(a_1,c_1,z_1), \dots, (a_s,c_s,z_s)\subset \reals^3_{>0}$, where $(z_i)$ are (potentially dependent) non-negative random variables satisfying $\P\left[ z_i  \geq c_i \ln \delta^{-1} \right] \leq \delta$, for all $i\in[s]$ and $\delta \in(0,1/e]$. Then, for all $\delta \in(0,1/e]$ we have 
	\begin{align}
	\delta &\geq 	\P\left[  a_0 + a_1 z_1+ \dots +   a_s z_s \geq  \left(a_0 + \sum_{i=1}^s a_i c_i  \right)\ln (s/\delta)\right],  \nn \\ &\geq 	\P\left[  a_0 + a_1 z_1+ \dots +   a_s z_s \geq (\ln s +1) \left(a_0 + \sum_{i=1}^s a_i c_i  \right)\ln \delta^{-1}\right] . \label{eq:thesecond}
	\end{align}
\end{lemma}
\begin{proof}
	Let $c(\delta) \coloneqq  (a_0 + \sum_{i=1}^s a_i c_i)  \ln
        \delta^{-1}$. Define $z_0=c_0=1$. Since $\delta\in(0,e^{-1}]$,
        we have
	\begin{align}
	\P\left[a_0 + a_1z_1+ \dots +   a_s z_s  \geq c(\delta)
          \right] & =  \P\left[
                    \sum_{i=0}^{s}a_i(z_i-c_i\log\delta^{-1}) \geq{}0\right],
                    \nn \\
                  & \leq  \P\left[ \exists i\in\brk*{s} :z_i-c_i\log\delta^{-1}\geq{}0 \right], \nn \\
	& \leq \sum_{i=1}^s \P\left[ z_i-c_i\log\delta^{-1}\geq{}0 \right], \nn  \\
	& \leq s \delta. \label{eq:probbound} \quad (\text{by assumption})
	\end{align}
	For any given $\delta\leq1/e$, by applying this result with $\delta' \coloneqq \delta/s$, we have for all $\delta \in (0, 1/(se)]$, $\delta' \leq 1/e$, and so 
	\begin{align}
	c(\delta')& = \left(a_0 + \sum_{i=1}^s a_i c_i\right) \ln (s/\delta) \leq (\ln s +1)\left(a_0 + \sum_{i=1}^s a_i c_i \right) \ln (1/\delta). \nn
	\end{align}
	This together with \eqref{eq:probbound} implies \eqref{eq:thesecond}. 
\end{proof}

\begin{lemma}
	\label{lem:thechis}
	Let $a_0>0$ and $(a_1,c_1,\bz_1), \dots, (a_s,c_s,\bz_s)$ be such that $(a_i,c_i)\subset \reals^2_{>0}$ and $\bz_i\in\bbR^{d_i}$ are random vectors satisfying $\bz_i \sim \cN(0, \Sigma_i)$ for $i \in [s]$. Then, for all $\delta \in(0,1/e]$ we have 
	\begin{align}
	\delta &\geq \P\left[ a_0 + a_1 \| \bz_1  \|^2 + \dots +   a_s \| \bz_s  \|^2 \geq  \left(a_0 + \sum_{i=1}^s a_i \|\Sigma_i\|_{\op}  \cdot  (3 d_i + 2)\right) \ln (s/\delta) \right] , \nn \\ & \geq \P\left[ a_0 + a_1 \| \bz_1  \|^2 + \dots +   a_s \| \bz_s  \|^2 \geq  (\ln s +1)\left(a_0 + \sum_{i=1}^s a_i \|\Sigma_i\|_{\op}  \cdot  (3 d_i + 2)\right) \ln \delta^{-1} \right] . \label{eq:thefirst}
	\end{align} 
\end{lemma}
\begin{proof}For $i \in [s]$, let $\bz'_i\coloneqq \Sigma_{i}^{-1/2}
  \bz_i$; in this case, $\bz_i' \sim \cN(0, I_{d_i})$. Thus, by Lemma
  1 of \cite{laurent2000adaptive}, we have that 
	\begin{align}
	\delta & \geq \P\left[ \|\bz_i'\|^2 \geq  d_i + 2 \sqrt{d_i \ln \delta^{-1}}  + 2 \ln \delta^{-1} \right]\nn \\ & \geq \P\left[ \|\bz_i\|^2 \geq \|\Sigma_{i}\|_{\op}  \left(d_i + 2 \sqrt{d_i \ln \delta^{-1}}  + 2 \ln \delta^{-1}\right) \right]\nn \\ & \geq   \P\left[\|\bz_i\|^2 \geq \|\Sigma_i\|_{\op}  \left(d_i + 2 \sqrt{d_i \ln \delta^{-1}}  + 2 \ln \delta^{-1}\right) \right] \nn\\
	& \geq \P\left[\|\bz_i\|^2 \geq \|\Sigma_i\|_{\op}  \cdot  (3 d_i + 2) \ln \delta^{-1}\right],\label{eq:keykey}
	\end{align}
	where the last inequality follows by the fact that $\delta \in(0, 1/e]$. By \eqref{eq:keykey} and \pref{lem:theprechis}, we get \eqref{eq:thefirst}.
\end{proof}

\begin{lemma}
	\label{lem:theeig}
	Let $\veps>0$, and $M, N\in \reals^{m\times m}$ be
        given. Suppose $N$ is non-singular and $\|M - N \|_{\op}\leq
        \veps$. Then if $\veps < \sigma_{\min}(N)/2$, $M$ is non-singular and 
	\begin{align}
	\|I_m -  M^{-1}N   \|_{\op} \leq \frac{2\veps}{\sigma_{\min}(N)}. \label{eq:theassump}
	\end{align} 
\end{lemma}
\begin{proof}
	We first bound the minimum singular value of $M$. Let $x \in \reals^{m}$ be a unit-norm vector such that $\|M x\| = \sigma_{\min}(M)$. Then, from the fact that $\|M- N \|_{\op}\leq \veps$, we have, 
	\begin{align}
	\veps &  \geq \| M x - N x 	\|, \nn\\
	& \geq \|N x \|  - \| M  x \|,  \quad (\text{by the triangle inequality})\nn\\
	& =   \sigma_{\min}(N) - \sigma_{\min} (M).  \quad
   (\text{using that $\|x\| =1$})\nn 
	\end{align}
	In particular, the last inequality implies that 
	\begin{align}
	\label{eq:theineq}
	\sigma_{\min}(M)\geq \sigma_{\min}(N) - \veps. 
	\end{align}
	Thus, since $\veps < \sigma_{\min}(N)$, the matrix $M$ is invertible. On the other hand, we have 
	\begin{align}
	\veps & \geq  \|M - N  \|_{\op},\nn \\
	& \geq \sigma_{\min}(M) \cdot \|I_m- (M)^{-1}   N \|_{\op}, \nn  \\
	& \geq (\sigma_{\min}(N) - \veps)  \|I_m-     M^{-1} N \|_{\op}. \nn \quad (\text{by \eqref{eq:theineq}})	\end{align}
	The desired result follows by the fact that $\veps < \sigma_{\min}(N)/2$.
 \end{proof}

\section{Main Theorem and Proof}
\label{sec:main_proof}
We now state and prove the main guarantee for \richidce (\pref{alg:main}). To begin, we
state the values for the algorithm's parameters $\nid$ and $\nop$:
\begin{align}
  &    \nid =
\bigoms\prn*{
\lambdam^{-2}T^3\kappa^5(\dimx+\dimu)^{16}\log^{15}(1/\delta)\cdot\frac{\log\abs{\Fclass}}{\veps^{6}}
    },
    \label{eq:nid_value}\\
&      \nref = \bigoms\prn*{
\lambdam^{-2}T^3\kappa^3(\dimx+\dimu)^{12}\log^{11}(1/\delta)\cdot\frac{\log\abs{\Fclass}}{\veps^{6}}
                             }.
                             \label{eq:nop_value}
\end{align}
We also recall from \pref{sec:phase1} that for the burn-in time, we use the choice
\[
\kapnot \ldef \Ceil{(1-\gammastar)^{-1} \ln
          \left({84\Psistar^5 \alphastar^4 \dimx
          (1-\gammastar)^{-2}\ln(10^3\cdot\nid)} \right)}.
  \]
  Finally, we set $r_{\id}=\sqrt{\Psist}$, and set $r_{\op}=\bigoms(1)$ to
  be a sufficiently large problem-dependent constant. The values for
  $\sigma^{2}$ and $\bclip$ are given in the following theorem.
\begin{thmmod}{thm:main}{a}
  \label{thm:main_full}
  Let $\delta\in(0,1/e]$  and $\veps=\bigohch(1)$ be given. Suppose we set
$\bclip^2=\Theta_{\star}((\dimx+\dimu)\log(1/\delta))$,
$\sigma^{2}=\bigohch(\veps^{2}/\bclip^{2}\wedge\lambdam)$, and choose $\nid$ and $\nop$ as in
\eqref{eq:nid_value} and \eqref{eq:nop_value}. Then with probability at
least $1-\bigoh(\kappa{}T\cdot\delta)$, \pref{alg:main} produces a
policy $\pihat$ with
\begin{equation}
  \label{eq:richid_regret}
\cost(\pihat)  - \cost(\piinf) \leq{} \veps,
\end{equation}
and does so while using at most
\[
  \bigohs\prn*{
    \lambdam^{-2}T^4\kappa^5(\dimx+\dimu)^{16}\log^{15}(1/\delta)\cdot\frac{\log\abs{\Fclass}}{\veps^{6}}
    }
  \]
  trajectories of length $\bigohs(T)$.    
    \end{thmmod}

    \subsection{Proof of Theorem \ref*{thm:main_full}}
    \newcommand{\Kinfid}{K_{\infty,\mathrm{id}}}
    \newcommand{\fstarid}{f_{\star,\mathrm{id}}}
    \newcommand{\vepsop}{\veps_{\mathrm{op}}}
\begin{proof}[\pfref{thm:main_full}]
  We first restate \pref{thm:phase_two}, which bounds the estimation
  error for the system parameter estimates produced by Phase II of
  \pref{alg:main}.
  \theoremphasetwo*
  Going forward we condition on the event in \pref{thm:phase_two}, and
  define $\fstarid\ldef\Sid\fstarid$ and $\Kinfid\ldef{}\Kinf\Sid^{-1}$.  We recall that whenever this event holds, we have 
  \[
    \nrm*{\Sid}_{\op}\vee\nrm*{\Sid^{-1}}_{\op}\leq{}\Psistar^{1/2}\vee
(1-\gammastar)^{-1}(4\Psistar^2\alphastar^2)
\sigma^{-1}_{\min}(\contkap)=\bigohs(1),
\]
as per \pref{thm:phase_one}. As a consequence, we have the following
fact, which we will use heavily going forward: If we define $(\Psistar',\alphastar',\gammastar',\kappastar',L')$ to be the analogues of $(\Psistar,
\alphastar,\gammastar,\kappastar,L)$ for $(\Aid,\Bid,\Qid,R,\Sigwid,\fstarid)$, we have
$\Psistar'=\bigohs(\Psistar)$ and $L'=\bigohs(L)$, and we may take
$\alphastar'=\bigohs(\alphastar)$, $\gammastar'\leq\gammastar$,  and
$\kappastar'\leq\kappastar$.
  
  We first apply \pref{lem:id_to_sys}, which implies that once $\vepsid=\bigohch(1)$,
  we have
  \begin{equation}
    \label{eq:5}
    \nrm[\big]{\Khat-\Kinfid}_{\op}=\bigohs(\vepsid).
  \end{equation}

  Next, we invoke \pref{thm:phase_three}, associating $A\gets\Aid$,
  $B\gets\Bid$, $Q\gets\Qid$, $\Sigma_{w}\gets\Sigwid$, $\fstar\gets\fstarid$, and inflating the
  problem-dependent parameters by $\bigohs(1)$ accordingly. In
  particular, suppose that $\vepsid^{2} \leq{}
 \bigohch((\log\abs*{\Fclass}+\dimx^2)\nref^{-1})$
  for a problem-dependent constant $c'_{\mathrm{id}}=\bigohch(1)$, 
  and suppose we set $\bclip^2=\Theta_{\star}((\dimx+\dimu)\log(\nref))$
  and $\sigma^2=\bigohch(\lambdam')=\bigohch(\lambdam)$. Then conditioned on the event of
  \pref{thm:phase_two}, we are guaranteed that for any $\delta\in(0,1/e]$, with probability at least $1-\bigoh(\kappa{}T\delta)$,
	\begin{align*}
\vepsop^{2}\ldef{}	\E_{\what \pi} \left[ \max_{ 1\leq  t \leq T} \|
          \fref_t(\by_{0:t}) - \fstarid(\by_{t}) \|_2^2   \right]   
          \leq   \bigohs\prn*{
\frac{\lambdam'^{-2}}{\sigma^{4}}\cdot{}T^3\kappa^2(\dimx+\dimu)^4\cdot\frac{(\dimx^2+\log\abs{\Fclass}) \log^{5}(\nref/\delta)}{\nref}
            }.
        \end{align*}
        where $\lambdam'$ is the analogue of $\lambdam$ for the
        parameters $(\Aid,\Bid,\Sigwid)$; note that to apply
        the theorem, we must set the radius of the class
        $\scrH_{\onpo}$ based on $\Psistar'$ rather than $\Psistar$,
        which leads to the value for this parameter passed into
        \pref{alg:phase3} when it is invoked within
        \pref{alg:main}. Likewise, we must inflate $\bclip$ by
        $\bigoms(1)$. Lastly, we note that
        $\lambdam'\geq\bigoms(\lambdam)$; which can be
        quickly verified.

Taking a union bound and simplifying the upper bounds slightly, we are
guaranteed that with probability at least $1-\bigohs(\kappa{}T\delta)$,
\begin{align*}
  \vepsid^{2} &\leq{} \kappa^2(\dimx+\dimu)^{4}\log^{4}(\nid/\delta) \frac{\log\abs*{\Fclass}}{\nid}\\
  \vepsop^{2}                  &\leq \bigohs\prn*{
\frac{\lambdam^{-2}}{\sigma^{4}}T^{3}\kappa^3(\dimx+\dimu)^{6}\log^5(\nref/\delta)\cdot\frac{\log\abs{\Fclass}}{\nref}
                                 },
\end{align*}
so long as the conditions on $\nid$, $\vepsid$, $\bclip$, and
$\sigma^{2}$ described so far hold. We next invoke
\pref{thm:performance_difference} with $\tau=T$ (again, we use that
changing the basis by $\Sid$ inflates problem-dependent constants by $\bigohs(1)$),\footnote{It is
  possible to get better dependence on $T$ by choosing different
  values for $\tau$ based on $\veps$, but for the sake of simplicity
  we do not pursue this here.} which implies that
\[
  \cost(\pihat) - \cost(\piinf)
  \leq{} \bigohs(\bclip\cdot\cx\cdot(\cx\cdot\vepsid+\vepsop+\sigma)),
\]
where $\cx^2=\max_{1\leq{}t\leq{}T}\En_{\pihat}\nrm*{\matx_t}^{2}$, so long as $\sigma^{2}=\bigohs(1)$. From \pref{lem:xconcentration}, we
have
$\cx^{2}\leq{}\bigohs(\bclip^{2}+\dimx)=\bigohs(\bclip^2)$,
so we may further simplify to
\[
  \cost(\pihat) - \cost(\piinf)
  \leq{} \bigohs(\bclip^2\cdot(\bclip\cdot\vepsid+\vepsop+\sigma)).
\]
Hence, to ensure the regret is at most $\bigohs(\veps)$, as a first
step we choose $\sigma=\bigohch(\veps/\bclip^2)$. This leads to
\[
  \vepsop^{2} \leq{} \bigohs\prn*{
\lambdam^{-2}T^3\kappa^3(\dimx+\dimu)^{10}\log^{9}(1/\delta)\cdot\frac{\log\abs{\Fclass}}{\nref}\cdot\frac{1}{\veps^{4}}
    }.
  \]
  We next choose $\vepsop^2=\bigohch(\veps^2/\bclip^4)$ which, per the
  inequality above, entails setting
  \[
    \nref = \bigoms\prn*{
\lambdam^{-2}\kappa^3T^3(\dimx+\dimu)^{12}\log^{11}(1/\delta)\cdot\frac{\log\abs{\Fclass}}{\veps^{6}}
    }.
  \]
  Finally, we require that $\vepsid\leq\bigohch(\veps/\bclip^3)$, and we
  also require $\vepsid$ to satisfy the earlier constraint that $\vepsid^{2}\leq{}\bigohch((\log\abs*{\Fclass}+\dimx^2)\nref^{-1})$. To
  satisfy the first constraint, it suffices to take
\[
\nid = \bigoms\prn*{\kappa^2(\dimx+\dimu)^{6}\log^{6}(1/\delta)
  \frac{\log\abs*{\Fclass}}{\veps^{2}}
}.
\]
For the second constraint, it suffices to take
\begin{align*}
\nid &= \bigoms\prn*{
  \nop\cdot{}\kappa^2(\dimx+\dimu)^4\log^4(/\delta)
  } \\
&= \bigoms\prn*{
\lambdam^{-2}T^3\kappa^5(\dimx+\dimu)^{16}\log^{15}(1/\delta)\cdot\frac{\log\abs{\Fclass}}{\veps^{6}}
    }.
\end{align*}
Lastly, we observe that the algorithm uses $\bigoh(\nop\cdot{}T +
\nid)$ trajectories in total, leading to the final calculation in the
theorem statement.  
        \end{proof}

\end{document}